\def\ARXIVVERSION{1}
\documentclass[twoside,11pt]{article}
\ifdefined\ARXIVVERSION
  \usepackage[preprint]{jmlr2e}
\else
  \usepackage{jmlr2e}
\fi
\ifdefined\ONLINEAPPENDIX
  \usepackage{xr-hyper}
\else
  \ifdefined\ARXIVVERSION\else
    \usepackage{xr-hyper}
  \fi
\fi
\usepackage{amsmath,amsfonts}
\usepackage{booktabs,multirow,calc,accents,enumerate} 
\usepackage{thm-restate}
\usepackage[caption=false]{subfig}
\usepackage{threeparttable}
\PassOptionsToPackage{table,xcdraw}{xcolor}
\usepackage{tikz}
\usetikzlibrary{arrows.meta,positioning}

\usepackage{longtable}
\usepackage{placeins}
\usepackage{pdflscape}
\usepackage{array}
\usepackage{siunitx}
\graphicspath{{images/}}

\newcommand{\E}{\mathbb{E}}
\newcommand{\C}{\mathbb{C}}

\newcommand{\R}{\mathbb{R}}

\newcommand{\supp}{\text{supp}}
\newcommand{\Bchi}{\raisebox{.15\baselineskip}{\large\ensuremath{\chi}}}

\newcommand{\dbtilde}[1]{\accentset{\approx}{#1}}
\newif\ifincludeproofs
\includeproofsfalse
\ifdefined\ONLINEAPPENDIX
  \includeproofstrue
\fi
\ifdefined\ARXIVVERSION
  \includeproofstrue
\fi
\newcommand{\ProofReference}[1]{%
  \ifincludeproofs Appendix~\ref{#1}\else
    Section~\ref{OA-#1} of the online appendix\fi}
\hypersetup{hidelinks}

\ifdefined\ONLINEAPPENDIX
  \jmlrheading{1}{2026}{Online Appendix}{8/26}{}{}{Mohammadi}
\else
  \usepackage{lastpage}
  \jmlrheading{1}{2026}{1-\pageref{LastPage}}{8/26}{}{}{Mohammadi}
\fi

\ShortHeadings{Toscani-Fourier Distance on Probability Measures}{Mehrdad Mohammadi}
\firstpageno{1}
\begin{document}

\ifdefined\ONLINEAPPENDIX
\title{Online Appendix for ``Toscani--Fourier Distance on Probability Measures: Wasserstein Control, Topological Equivalence on Model Classes, and Duality''}
\else
\title{Toscani--Fourier Distance on Probability Measures: Wasserstein Control, Topological Equivalence on Model Classes, and Duality}
\fi

\author{\name Mehrdad Mohammadi \email mehrdad3@illinois.edu \\
       \addr Department of Statistics\\
       University of Illinois Urbana-Champaign\\
       Champaign, IL 61820, USA}

\editor{Under Review}

\maketitle

\ifdefined\ONLINEAPPENDIX
This online appendix contains the complete proofs and additional numerical details for the accompanying manuscript. The theorem numbering and notation agree with the main paper. Displayed equations local to this appendix are numbered (A.1), (B.1), \dots. Parenthetical equation numbers without a letter prefix refer to displays in the main manuscript, whose numbering is imported here.
\else
\begin{abstract}%
Comparing probability measures in machine learning trades transport geometry against computational cost. Wasserstein distances encode the geometry of $\mathbb R^d$ but require solving a transport problem, while kernel discrepancies are cheap to evaluate yet depend delicately on their test class. We study the Toscani--Fourier family $\mathrm T_{s,p}$, the weighted $L^p$ norm of the difference of two characteristic functions, as a Fourier-side discrepancy on $\mathbb R^d$. For $1\le p<\infty$ we show that $d/p<s<1+d/p$ is exactly the window in which $\mathrm T_{s,p}$ is finite on $\mathcal P_p(\mathbb R^d)$, both endpoints failing already for Dirac pairs, and we establish the metric, embedding, and compactness structure of the resulting space, which is complete. Duality identifies $\mathrm T_{s,p}$ as an integral probability metric over a homogeneous Fourier--Lebesgue ball, with an explicit extremizer when $1<p<\infty$. We prove the global bound $\mathrm T_{s,p}\lesssim W_p^{\,s-d/p}$, whose exponent is sharp, show that no global converse can hold, and recover topological equivalence with $W_p$ on bounded-support and uniform-tail classes, together with explicit reverse moduli on bounded support that improve the imported energy-kernel exponent at $p=2$. There, $\mathrm T_{s,2}$ is a constant multiple of the classical energy distance, which yields an exact finite-sample identity for the mean of the empirical discrepancy. The numerical experiments are otherwise diagnostic.
\end{abstract}

\section{Introduction}

Many problems in machine learning require comparing probability distributions or sample distributions, including two-sample testing, goodness-of-fit, generative modeling, distribution regression, and clustering. Among the most widely used discrepancies are Wasserstein distances from optimal transport and kernel discrepancies such as maximum mean discrepancy (MMD). Wasserstein distances capture the geometry of the ambient space and often provide a meaningful notion of distributional similarity, but their computational cost can become prohibitive in large-scale or high-dimensional settings. This motivates the search for alternative discrepancies that retain useful geometric information while being easier to estimate and compute.

A natural alternative is to compare distributions through their characteristic functions, which
uniquely determine probability measures on $\mathbb{R}^d$ and can be estimated directly from
samples by empirical averages. This leads to Fourier-domain discrepancies that avoid optimization
over transport couplings. However, Fourier-based and transport-based discrepancies capture
different aspects of distributional mismatch. Wasserstein distances quantify the cost of moving
mass in the ambient space, whereas Fourier-based metrics aggregate differences across frequency
scales. As a result, comparison and equivalence results between the two viewpoints are subtle and
typically require structural assumptions or suitable normalization.

Fourier-based probability metrics have their roots in statistical physics and kinetic theory.
Starting from \citet{gabetta1995metrics}, they were introduced to study convergence to equilibrium
for the spatially homogeneous Boltzmann equation and were further developed in
\citet{toscani1999probability,carlen2000central,bisi2006dilute,bisi2006decay}. These works
showed that weighted discrepancies between characteristic functions can provide sharp control of
distributional evolution in settings where Fourier methods are natural.

More recently, several works revisited Fourier-type metrics from the perspectives of probability
on metric spaces and optimal transport. For closely related supremum-type Fourier metrics,
\citet{cho2015absolute} connected finiteness to absolute moments,
\citet{stawiska2020completeness} analyzed completeness of fixed-moment subspaces and the
distinction between even and odd orders, and \citet{randles2024note} gave further completeness
results. Comparison results with Wasserstein distances were developed in
\citet{carrillo2007contractive,cao2024equivalence}. The continuous integral family studied here is itself classical.
\citet{auricchio2020equivalence} write it as
$D_{\sigma,p}(\mu,\nu)=(\int_{\R^d}|\widehat\mu(k)-\widehat\nu(k)|^p\,|k|^{-(p\sigma+d)}\,dk)^{1/p}$,
credit the one-dimensional case to \citet{baringhaus1997class}, and identify these functionals as
ideal metrics in the sense of Zolotarev \citep{zolotarev1984probability}. In our normalization
$\mathrm T_{s,p}=D_{\,s-d/p,\,p}$, and the dilation identity of Lemma~\ref{lem:scaling} expresses
precisely ideality of order $s-d/p$. The same authors also introduced modified Fourier metrics for
gridded probability measures, proved equivalence results with \(W_1\) and \(W_2\), and demonstrated
their computational use for image data. Very recently, the quadratic member of the family has been
developed expressly for machine learning. \citet{auricchio2025kinetic} identify the $p=2$ Fourier
metric with the energy distance up to an explicit constant and survey its properties for
high-dimensional inference, and \citet{auricchio2026energy} apply it to evolution equations. Further harmonic-analytic bounds for Wasserstein distances were established
by \citet{loeper2006uniqueness,peyre2018comparison,steinerberger2021wasserstein,nilesweed2021minimax,hong2023fourier,gonzalez2024asymptotically}.
In particular, \citet{peyre2018comparison} compares $W_2$ with the homogeneous $\dot H^{-1}$
norm, the closest Sobolev-side relative of the dual representation developed below. In addition,
\citet{mariucci2018wasserstein} used Toscani--Fourier distances to control Wasserstein and total
variation distances between marginals of L\'evy processes, illustrating both the reach and the
limitations of Fourier representations of transport-type geometry.

Our perspective is also closely related to the machine learning literature on integral probability
metrics, kernel mean embeddings, and maximum mean discrepancy (MMD)
\citep{mueller1997integral,smola2007hilbert,sriperumbudur2010hilbert,sriperumbudur2011universality,sriperumbudur2012empirical,muandet2017kernel}.
For translation-invariant kernels, Fourier analysis characterizes characteristicness and the
induced discrepancy \citep{sriperumbudur2010hilbert,modeste2024translation}, and Fourier arguments
underlie the equivalence between energy distances and kernel discrepancies
\citep{sejdinovic2013equivalence}. The energy distance itself, together with its weighted-$L^2$
characteristic-function representation and the explicit constant used in
Section~\ref{sec:mmd}, is due to
\citet{szekely2005new,szekely2007measuring,szekely2013energy}, and \citet{lyons2013distance}
extended the theory to metric spaces of negative type. Recent work has clarified when MMD metrizes weak convergence
and when kernel discrepancies control Wasserstein distances on structured classes
\citep{simon2018kernel,simon2023metrizing,vayer2023controlling}. Consequently, the quadratic
identification below should be viewed as an exact normalization, within the Toscani--Fourier
family, of the classical energy-distance representation of Sz\'ekely and Rizzo and of the
established energy-distance/MMD correspondence --- not as a new connection.

Characteristic-function discrepancies also have a long history as statistical and machine
learning tools in their own right. Empirical characteristic function methods go back to
\citet{feuerverger1977empirical} and the two-sample test of \citet{epps1986omnibus}. Modern
kernel-based variants compare smoothed characteristic functions at a few random frequencies to
obtain linear-time two-sample tests \citep{chwialkowski2015fast,jitkrittum2016interpretable}. Energy/Cram\'er-type characteristic-function losses are used to train generative models
\citep{bellemare2017cramer}. Relative to this literature, the distinguishing feature of
$\mathrm T_{s,p}$ is its non-integrable homogeneous weight $\|u\|^{-ps}$. The weight is exactly
what produces the transport comparisons of Sections~\ref{sec:Ts-vs-Wp}--\ref{sec:mmd}, at the
price of the finiteness window studied in Section~\ref{sec:prelim}.

Against this background, we study the continuous, weighted integral family
\(\mathrm{T}_{s,p}\) on \(\mathbb R^d\) for general \(1\le p<\infty\). The focus is narrower than
``Fourier probability metrics'' as a whole. We ask when this particular integral formula is finite
on natural moment classes, how it compares globally with \(W_p\), and which scale-controlled
families admit reverse or topological control. This scope distinguishes the paper from the
supremum-type completeness theory above and from modified discrete-grid metrics, while the
\(p=2\) results deliberately connect to, and build on, existing MMD and energy-distance theory.

Figure~\ref{fig:metric-map} summarizes the picture developed in this paper. Proposition~\ref{prop:integrability} and Theorem~\ref{thm:completeness} identify the fundamental metric and embedding regime \(d/p < s < 1 + d/p\). Theorem~\ref{thm:upper-holder} shows that \(\mathrm{T}_{s,p}\) is globally controlled by \(W_p\), while Lemma~\ref{lem:scaling} proves that no global reverse inequality can hold in general. Reverse control and topological equivalence are then recovered on scale-controlled families, including bounded-support and uniform-tail classes, in Theorems~\ref{thm:Wp-Tsp-equivalence}--\ref{thm:Wp-Tsp-equivalence-moment} and Theorem~\ref{thm:explicit-reverse-bounded}. In the Hilbertian case \(p=2\), Corollary~\ref{cor:hilbert-dual} and Lemma~\ref{lem:energy-constant} further identify \(\mathrm{T}_{s,2}\) with a translation-invariant MMD, equivalently a negative-Sobolev IPM, thereby linking the Toscani--Fourier framework to the kernel literature. These issues are particularly relevant in machine learning, where one seeks discrepancies that remain computationally tractable while preserving meaningful distributional geometry in high dimension.

\begin{figure}[tbp]
\centering
\begin{tikzpicture}[
  >=Stealth, thick, every node/.style={font=\small},
  box/.style={draw, rounded corners,align=center,inner sep=6pt,text width=4cm,fill=gray!10},
  tfbox/.style={font=\footnotesize, draw,rounded corners,align=center,inner sep=7pt,text width=4.4cm,fill=gray!8},
  lab/.style={draw, rounded corners,font=\footnotesize,align=center,fill=white,inner sep=4pt,text opacity=1}
]

\node[box] (W) at (-5.2,7.1) {\textbf{Wasserstein}\\ $W_p$};

\node[tfbox] (T) at (0,1) {{\small\textbf{Toscani-Fourier}}\\ $\dfrac{d}{p}<s<1+\dfrac{d}{p}$ \quad (Prop.~\ref{prop:integrability})\\[1mm]
embedding / completeness \quad (Thm.~\ref{thm:completeness})};

\node[box] (M) at (4.4,7) {\textbf{MMD / negative-Sobolev IPM}\\ $p=2$};

\draw[->](W) to[bend left=44]node[lab, above, pos=.6]{global upper control\\$\mathrm{T}_{s,p}(\mu,\nu)
\le
C\ W_p(\mu,\nu)^{\ s-\frac{d}{p}}$\\ (sharp exponent, Thm.~\ref{thm:upper-holder})}(T);

\draw[->, dashed](T) to[bend left=1] node[lab, below, pos=.35] {no global control\\ (Lemma~\ref{lem:scaling})}(W);

\draw[<->](W) to[bend right=55]node[lab, below, pos=.22]{bounded support / uniform tails\\ $W_p \Longleftrightarrow \mathrm{T}_{s,p}$\\(Thms.~\ref{thm:Wp-Tsp-equivalence}-\ref{thm:Wp-Tsp-equivalence-moment}, \ref{thm:explicit-reverse-bounded})}(T);

\draw[->] (T) to[bend right=44] node[lab, above, pos=.43]{Identification\\$\mathrm{T}_{s,2}(\mu,\nu)^2=
c(d,H)\ \mathrm{MMD}_{k_H}(\mu,\nu)^2$\\ (Cor.~\ref{cor:hilbert-dual}, Lemma~\ref{lem:energy-constant})}(M);
\end{tikzpicture}
\caption{Main metric relations in the paper.}
\label{fig:metric-map}
\end{figure}
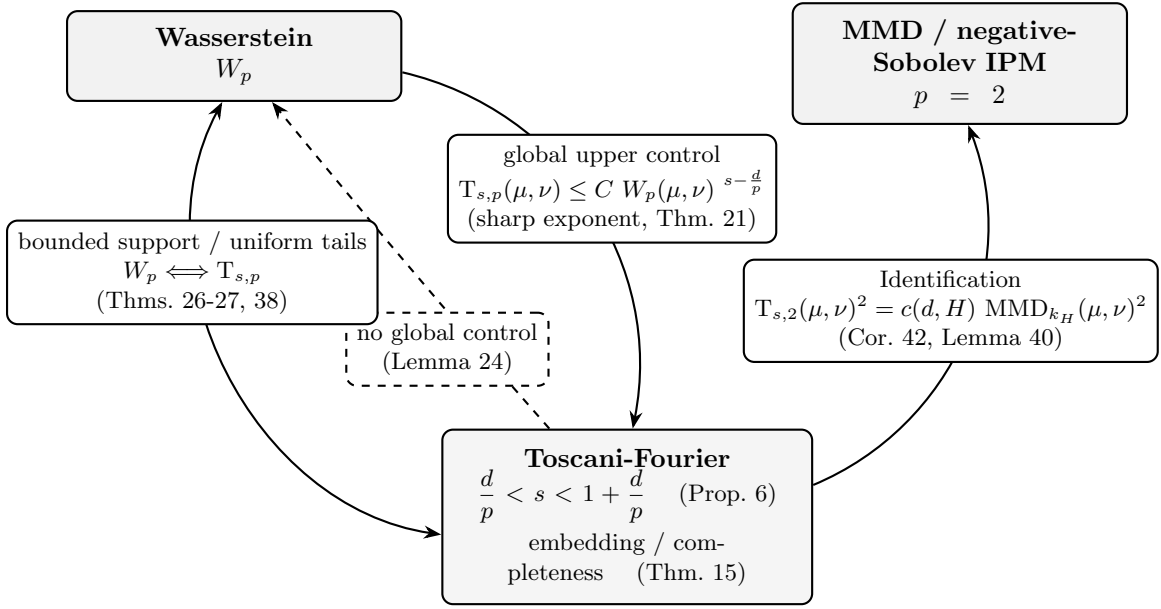

\subsection{Contributions and Scope}
Our contributions are of three kinds.

\emph{(i) A sharp finiteness window and the function-space structure it induces.} We show that
\(\frac dp<s<1+\frac dp\) is exactly the range in which \(\mathrm{T}_{s,p}\) is finite on
\(\mathcal P_p(\R^d)\) and defines a metric there, and that \emph{both} endpoint restrictions are
necessary. Outside the window the defining integral already diverges for a pair of Dirac measures
(Proposition~\ref{prop:integrability} and Proposition~\ref{prop:window-sharp}). Inside the window we
formulate the geometry through a weighted Fourier--Lebesgue embedding and obtain the extended
seminorm, monotonicity, and compactness statements, together with the completeness of the
resulting metric space. The embedded image is closed, so no escape of mass is possible in the
Fourier geometry
(Theorem~\ref{thm:seminorm}, Proposition~\ref{prop:embeddings},
Proposition~\ref{prop:Compactness}, Theorem~\ref{thm:completeness}).

\emph{(ii) A quantitative comparison with transport geometry.} We prove the global bound
\(\mathrm{T}_{s,p}\le C\,W_p^{\,s-d/p}\) with an explicit constant and a H\"older exponent that
cannot be improved (Theorem~\ref{thm:upper-holder}), show that no global converse of any form --- H\"older or otherwise --- can hold, by combining
the homogeneity obstruction with a mass-splitting construction (Lemma~\ref{lem:scaling}), and
recover reverse control or equivalent convergence once dilations and mass splitting are excluded.
Topological equivalence on bounded-support and uniform-tail classes
(Theorems~\ref{thm:Wp-Tsp-equivalence}--\ref{thm:Wp-Tsp-equivalence-moment}), an explicit
bounded-support reverse modulus valid for every \(1\le p<\infty\)
(Theorem~\ref{thm:explicit-reverse-bounded}), and a sharper Hausdorff--Young refinement for
\(1\le p\le2\) (Proposition~\ref{prop:reverse-bounded-HY}).

\emph{(iii) Duality, and the exact place of the quadratic case in the kernel literature.} We derive
a dual integral probability metric representation over a homogeneous Fourier--Lebesgue ball for
general \(p\), together with an explicit extremizer when \(1<p<\infty\)
(Theorem~\ref{thm:dual-ipm}, Proposition~\ref{prop:dual-optimizer}). At \(p=2\) we record, in our
normalization and with a self-contained verification, the classical Sz\'ekely--Rizzo constant
identifying \(\mathrm{T}_{s,2}\) with an energy-kernel MMD and with a negative-Sobolev
integral probability metric (Lemma~\ref{lem:energy-constant},
Corollary~\ref{cor:hilbert-dual}). This fixes the dictionary needed to import known
kernel-to-Wasserstein bounds, shows that our general-\(p\) bounded-support moduli strictly improve
the imported exponent at \(p=2\) (Remark~\ref{rem:compare-bounded-moduli}), and yields an exact
expression for the mean of the empirical discrepancy
(Proposition~\ref{prop:empirical-mean}).

Section~\ref{sec:num_exp} provides numerical diagnostics of truncation and of the quadratic
identity, together with a small DOTmark benchmark. These computations illustrate behavior and
measured cost. They are not presented as finite-sample statistical validation, and the paper proves
no estimator-level guarantee for \(\mathrm{T}_{s,p}(\hat\mu_n,\hat\nu_n)\).

\paragraph{Organization.}
Section~\ref{sec:prelim} fixes notation and establishes the finiteness window and the basic metric
properties. Section~\ref{sec:structure} develops the functional-analytic structure. Section~\ref{sec:Ts-vs-Wp} contains the comparison with Wasserstein distances.
Section~\ref{sec:dual-ipm} gives the dual/IPM representation and the explicit bounded-support
reverse modulus. Section~\ref{sec:mmd} treats the Hilbertian case \(p=2\).
Section~\ref{sec:num_exp} reports the numerical diagnostics.
\ifincludeproofs
Complete proofs of every statement are collected in Appendix~\ref{app:proofs}.
\else
Complete proofs of every numbered statement are collected in Appendix~A of the accompanying
online appendix, whose subsections are titled by the result they prove and appear in the order in
which the results are stated below. No proof has been abbreviated or omitted.
\fi

\section{Definition and Basic Properties of the Toscani--Fourier Distance}\label{sec:prelim}

We work on $\R^d$ ($d\ge 1$) with its Borel $\sigma$--algebra.
Let $\mathcal P(\R^d)$ be the set of Borel probability measures and $\mathcal M(\R^d)$ the vector space of finite signed Borel measures.
For $p\in[1,\infty)$, the $p$--moment class is
\[
\mathcal P_p(\R^d)
:=
\Big\{\mu\in\mathcal P(\R^d): \int_{\R^d}\|x\|^p\ d\mu(x)<\infty\Big\}
\]
We use the same notation $\mathcal P_r(\R^d)$ with $r\in(0,1)$ for the fractional moment class
$\{\mu\in\mathcal P(\R^d):\int\|x\|^{r}\,d\mu(x)<\infty\}$. The Wasserstein distance below is
used only for $p\ge1$.
We write $\mathcal M_0(\R^d):=\{\eta\in\mathcal M(\R^d):\eta(\R^d)=0\}$ for the zero-mass subspace.
For $\mu\in\mathcal M(\R^d)$, define the Fourier transform
\[
\widehat\mu(u):=\int_{\R^d}e^{i\langle u,x\rangle}\ d\mu(x),\qquad u\in\R^d
\]
Then $\widehat\mu(0)=\mu(\R^d)$. For $\mu\in\mathcal P(\R^d)$ we write $\widehat\mu=\Bchi_\mu$. Thus $\Bchi_\mu(0)=1$.
If $\mu$ admits a density $g$, then
\[
\Bchi_\mu(u)=\int_{\R^d}e^{i\langle u,x\rangle}g(x)\ dx
=\overline{\int_{\R^d}e^{-i\langle u,x\rangle}g(x)\ dx}
=\overline{\mathcal F_g(u)}
\]
where $\mathcal F_g$ is the (continuous) Fourier transform of $g$.

\begin{remark}[Moment Regularity at the Origin]\label{rem:chi-regularity}
If $\mu$ has finite moments through order \(k\), then $\Bchi_\mu$ is \(k\)-times continuously differentiable. Its derivatives satisfy
\[
\frac{\partial^k\Bchi_\mu}{\partial u_{i_1}\cdots\partial u_{i_k}}(u)
=
i^k\int_{\R^d}x_{i_1}\cdots x_{i_k}e^{i\langle u,x\rangle}\ d\mu(x)
\]
In particular, if $\mu\in\mathcal P_p(\R^d)$ then $\Bchi_\mu\in C^{\lfloor p\rfloor}(\R^d)$, and its $\lfloor p\rfloor$th derivatives are H\"older of order $p-\lfloor p\rfloor$.
What moments \emph{do} control is the behavior of $\Bchi_\mu(u)$ near $u=0$, not decay as $\|u\|\to\infty$. See Remark~\ref{rem:weak-not-T}.
\end{remark}

For $\mu,\nu\in\mathcal P(\R^d)$ and $p\in[1,\infty)$, we denote the Wasserstein distance \citep{vaserstein1969markov,villani2008optimal} as
\[
W_p(\mu,\nu)
:=
\Bigg(\inf_{\pi\in\Pi(\mu,\nu)}\int_{\R^d\times\R^d}\|x-y\|^p\ d\pi(x,y)\Bigg)^{1/p}
\]
where $\Pi(\mu,\nu)$ is the set of couplings of $(\mu,\nu)$.

The Fourier transform captures the shape of a measure through its action on oscillatory test functions. At the origin, however, it does not describe that structure. It records only the total mass, since \(\widehat{\mu}(0)=\mu(\mathbb{R}^d)\). In particular, for probability measures this value is always \(1\). Thus the term at \(u=0\) contains no information about how two probability measures differ. For this reason, we subtract the value at the origin and measure only the remaining variation of the characteristic function. Equivalently, the construction acts on zero-mass differences such as \(\mu-\nu\), which are the natural objects when comparing probability measures. This normalization removes the trivial constant part and underlies the metric, embedding, and dual formulations developed later.

\begin{definition}[Extended Toscani seminorm on finite measures]\label{def:Toscani_seminorm}
Fix $s>0$ and $p\in[1,\infty]$. For $\mu\in\mathcal M(\R^d)$ define the possibly infinite quantity
\[
|\mu|_{\mathrm{T}_{s,p}}
:=
\Bigg(\int_{\R^d\setminus\{0\}}
\Big|\frac{\widehat\mu(u)-\widehat\mu(0)}{\|u\|^{s}}\Big|^p\ du\Bigg)^{1/p}
\qquad (1\le p<\infty)
\]
and for $p=\infty$,
\[
|\mu|_{\mathrm{T}_{s,\infty}}
:=
\sup_{u\ne 0}\Big|\frac{\widehat\mu(u)-\widehat\mu(0)}{\|u\|^{s}}\Big|
\]
\end{definition}

\begin{theorem}[Extended Toscani seminorm]\label{thm:seminorm}
For each $s>0$ and $p\in[1,\infty]$, the mapping $\mu\mapsto |\mu|_{\mathrm{T}_{s,p}}$ is an extended seminorm on $\mathcal M(\R^d)$. Its finite domain
\[
\mathcal D_{s,p}:=\{\mu\in\mathcal M(\R^d):|\mu|_{\mathrm{T}_{s,p}}<\infty\}
\]
is a linear subspace of \(\mathcal M(\R^d)\), and the restriction of \( |\cdot|_{\mathrm{T}_{s,p}}\) to \(\mathcal D_{s,p}\) is an ordinary seminorm.
\end{theorem}

\begin{remark}[Kernel and the zero-mass norm]\label{rem:kernel}
The seminorm $|\cdot|_{\mathrm{T}_{s,p}}$ is homogeneous and subadditive on $\mathcal M(\R^d)$ but not definite.
Indeed, for the Dirac measure at the origin $\delta_0$ we have $\widehat{(c\delta_0)}(u)\equiv c$, so $\widehat{(c\delta_0)}(u)-\widehat{(c\delta_0)}(0)\equiv 0$ and $|c\delta_0|_{\mathrm{T}_{s,p}}=0$.

In fact, $\ker(|\cdot|_{\mathrm{T}_{s,p}})=\mathrm{span}\{\delta_0\}$. To see the nontrivial inclusion, suppose $\mu\in\mathcal M(\R^d)$ satisfies $|\mu|_{\mathrm{T}_{s,p}}=0$.
For $1\le p<\infty$ this forces $\widehat\mu(u)=\widehat\mu(0)$ for Lebesgue-almost every $u$, and for $p=\infty$ for every $u\neq 0$.
Since the Fourier transform of a finite signed measure is uniformly continuous on $\R^d$, and a continuous function that agrees almost everywhere with the constant $\widehat\mu(0)$ agrees with it everywhere (the set where two continuous functions coincide is closed, and here it has full measure, hence is dense), we conclude $\widehat\mu\equiv\widehat\mu(0)$ on all of $\R^d$ in either case.
Write $c:=\widehat\mu(0)=\mu(\R^d)$. Then $\mu$ and $c\delta_0$ are finite signed measures with identical Fourier transforms, and the uniqueness theorem for Fourier--Stieltjes transforms yields $\mu=c\delta_0$.
For completeness we recall why uniqueness holds for \emph{signed} measures. If $\eta:=\mu-c\delta_0$ has $\widehat\eta\equiv0$, decompose $\eta=\eta^+-\eta^-$ by Jordan decomposition. Then $\widehat{\eta^+}\equiv\widehat{\eta^-}$, and evaluating at $u=0$ gives $\eta^+(\R^d)=\eta^-(\R^d)=:m$. If $m=0$ then $\eta^\pm=0$. If $m>0$ then $\eta^+/m$ and $\eta^-/m$ are probability measures with equal characteristic functions, hence equal by the uniqueness theorem for characteristic functions \citep[Theorem~26.2]{billingsley2017probability}, so $\eta=0$ in either case.

On the zero-mass subspace $\mathcal M_0(\R^d)=\{\eta\in\mathcal M(\R^d):\widehat\eta(0)=0\}$ the seminorm therefore becomes a genuine norm.
\[
\|\eta\|_{\mathrm{T}_{s,p}}
:=
\Bigg(\int_{\R^d\setminus\{0\}}
\frac{|\widehat\eta(u)|^p}{\|u\|^{ps}}\ du\Bigg)^{1/p},
\qquad
\|\eta\|_{\mathrm{T}_{s,\infty}}
:=
\sup_{u\ne 0}\frac{|\widehat\eta(u)|}{\|u\|^{s}}
\]
since $\mathcal M_0(\R^d)\cap\mathrm{span}\{\delta_0\}=\{0\}$.
If $\mu,\nu\in\mathcal M(\R^d)$ have the same total mass, then $\mu-\nu\in\mathcal M_0(\R^d)$ and
$|\mu-\nu|_{\mathrm{T}_{s,p}}=\|\mu-\nu\|_{\mathrm{T}_{s,p}}$.
\end{remark}

\begin{definition}[Toscani--Fourier distance]\label{def:Toscani_Dist}
Fix $s>0$ and $p\in[1,\infty]$. For $\mu,\nu\in\mathcal P(\R^d)$ define
\[
\mathrm{T}_{s,p}(\mu,\nu)
:=
\|\mu-\nu\|_{\mathrm{T}_{s,p}}
=
\Bigg(\int_{\R^d\setminus\{0\}}
\frac{|\Bchi_{\mu}(u)-\Bchi_{\nu}(u)|^{p}}{\|u\|^{ps}}\ du\Bigg)^{1/p}
\quad (1\le p<\infty)
\]
and for $p=\infty$,
\[
\mathrm{T}_{s,\infty}(\mu,\nu)
:=
\sup_{u\ne 0}\frac{|\Bchi_\mu(u)-\Bchi_\nu(u)|}{\|u\|^{s}}
\]
\end{definition}

\begin{proposition}[Finiteness and metric property for $1\le p<\infty$]\label{prop:integrability}
Let $d\ge 1$, $1\le p<\infty$, and $\frac{d}{p}<s<1+\frac{d}{p}$.
Then $\mathrm{T}_{s,p}(\mu,\nu)<\infty$ for all $\mu,\nu\in\mathcal P_p(\R^d)$, and $\mathrm{T}_{s,p}$ defines a metric on $\mathcal P_p(\R^d)$.
\end{proposition}
The window is the structural regime for the comparison theory of
Sections~\ref{sec:Ts-vs-Wp}--\ref{sec:mmd}. Its two constraints come from the two singular regions of the defining integral. The lower constraint $s>\frac dp$ controls the high-frequency tail, where only $|\Bchi_\mu-\Bchi_\nu|\le2$ is available, while the upper constraint $s<1+\frac dp$ is forced by the first-order cancellation at the origin coming from $\Bchi_\mu(0)=\Bchi_\nu(0)=1$, which gives $|\Bchi_\mu(u)-\Bchi_\nu(u)|\lesssim\|u\|$ there. Proposition~\ref{prop:window-sharp} below shows that neither constraint can be dropped.
\begin{remark}[The case $p=\infty$]\label{rem:pinf}
For $p=\infty$, finiteness at the origin requires a first-moment bound and $0<s\le 1$, since
$|\Bchi_\mu(u)-\Bchi_\nu(u)|\lesssim \|u\|$ as $u\to 0$ when $\mu,\nu\in\mathcal P_1(\R^d)$.
Accordingly, whenever we compare $\mathrm{T}_{s,\infty}$ on moment classes we assume
$\mu,\nu\in\mathcal P_1(\R^d)$ and $0<s\le 1$. Statements phrased on the Toscani class itself,
such as Theorem~\ref{thm:completeness}, need no such assumption (though for large $s$ that class
may be small).
\end{remark}

The next proposition shows that both endpoint restrictions are necessary, not merely convenient. Outside the window the defining integral diverges already for a pair of Dirac measures, which lies in every moment class.

\begin{proposition}[Sharpness of the finiteness window]\label{prop:window-sharp}
Let $d\ge 1$, $1\le p<\infty$, $s>0$, and fix any $x\neq 0$. Then
\[
\mathrm{T}_{s,p}(\delta_0,\delta_x)<\infty
\quad\Longleftrightarrow\quad
\frac{d}{p}<s<1+\frac{d}{p}
\]
More precisely, the defining integral diverges at the origin whenever $s\ge 1+\frac dp$ and diverges at infinity whenever $s\le \frac dp$.
Since $\delta_0,\delta_x\in\mathcal P_p(\R^d)$ for every $p$, no choice of $s$ outside the window makes $\mathrm{T}_{s,p}$ finite on all of $\mathcal P_p(\R^d)\times\mathcal P_p(\R^d)$.
\end{proposition}

\begin{definition}[Toscani class and Toscani space]\label{def:Toscani_space_metric}
Fix $s>0$ and $p\in[1,\infty]$. Define the \emph{Toscani class of probability measures}
\[
\mathcal P^{\mathrm T}_{s,p}(\R^d)
:=
\{\mu\in\mathcal P(\R^d):\ \mathrm{T}_{s,p}(\mu,\delta_0)<\infty\}
\]
Equipped with the metric $\mathrm{T}_{s,p}$, we call
$\big(\mathcal P^{\mathrm T}_{s,p}(\R^d),\mathrm{T}_{s,p}\big)$ the \emph{Toscani space}.
\end{definition}
For \emph{every} $s>0$ and $p\in[1,\infty]$, $\mathrm{T}_{s,p}$ is a genuine, finite-valued
metric on $\mathcal P^{\mathrm T}_{s,p}(\R^d)$. Finiteness follows from the triangle inequality
through $\delta_0$, and definiteness follows from the argument in part~(i) of the proof of
Proposition~\ref{prop:integrability}, which uses only the continuity of characteristic functions
and no moment assumption (for $p=\infty$ the vanishing supremum forces
$\Bchi_\mu=\Bchi_\nu$ everywhere directly, after which the same uniqueness theorem applies).

\begin{remark}[How $\mathcal P^{\mathrm T}_{s,p}$ compares with the moment classes]\label{rem:class-in-PT}
Under the assumptions of Proposition~\ref{prop:integrability} (i.e.\ $\frac{d}{p}<s<1+\frac{d}{p}$ and $1\le p<\infty$),
we have the chain of inclusions
\[
\mathcal P_p(\R^d)\subseteq\mathcal P_1(\R^d)\subseteq\mathcal P^{\mathrm T}_{s,p}(\R^d)
\]
The first inclusion is Jensen's (or H\"older's) inequality for $p\ge 1$. The second holds because the finiteness argument in \ProofReference{proof:integrability} only uses \emph{first} moments. For $\mu\in\mathcal P_1(\R^d)$, the bound $|\Bchi_\mu(u)-1|\le\|u\|\,\E\|X\|$ near the origin and $|\Bchi_\mu(u)-1|\le 2$ at infinity give $\mathrm{T}_{s,p}(\mu,\delta_0)<\infty$ throughout the window. In particular, $\delta_0\in\mathcal P^{\mathrm T}_{s,p}(\R^d)$ and $\mathrm{T}_{s,p}$ is finite on all of $\mathcal P_p(\R^d)\times\mathcal P_p(\R^d)$.

For $p>1$ the second inclusion makes $\mathcal P^{\mathrm T}_{s,p}(\R^d)$ \emph{strictly larger} than $\mathcal P_p(\R^d)$.
A concrete example with $d=1$, $p=2$, and $s\in(\tfrac12,\tfrac32)$. Let $\mu$ have density proportional to $|x|^{-5/2}\mathbf 1_{\{|x|>1\}}$.
Then $\int|x|\,d\mu<\infty$ while $\int|x|^2\,d\mu=\infty$, so $\mu\in\mathcal P_1\setminus\mathcal P_2$. Yet by the first-moment bounds just quoted, $\mathrm{T}_{s,2}(\mu,\delta_0)<\infty$, i.e.\ $\mu\in\mathcal P^{\mathrm T}_{s,2}(\R^d)\setminus\mathcal P_2(\R^d)$.
Thus, inside the window, the correct statement is that the Toscani class \emph{contains} the Wasserstein class $\mathcal P_p$ on which our comparison theory is developed. The two classes do not coincide in general.
Outside the window, the first inclusion fails in the opposite direction. There we have
$\mathcal P_p(\R^d)\not\subseteq\mathcal P^{\mathrm T}_{s,p}(\R^d)$, since the defining integral
already diverges for a pair of Dirac measures (Proposition~\ref{prop:window-sharp}).
\end{remark}

\begin{definition}[Linear Toscani Space]\label{def:Toscani_space_linear}
Fix $s>0$ and $p\in[1,\infty]$. Define the normed linear space
\[
\mathcal T_{s,p}(\R^d)
:=
\{\eta\in\mathcal M_0(\R^d):\ \|\eta\|_{\mathrm{T}_{s,p}}<\infty\}
\]
where $\|\cdot\|_{\mathrm{T}_{s,p}}$ is the norm from Remark~\ref{rem:kernel}.
(We reserve the roman symbol $\mathrm{T}_{s,p}$ for the distance and the calligraphic symbol $\mathcal T_{s,p}(\R^d)$ for this linear space, so that the two objects are never denoted alike.)
For $\mu,\nu\in\mathcal P^{\mathrm T}_{s,p}(\R^d)$, the triangle inequality applied to $\mu-\delta_0$ and $\nu-\delta_0$ gives $\mu-\nu\in\mathcal T_{s,p}(\R^d)$ and
\[
\mathrm{T}_{s,p}(\mu,\nu)=\|\mu-\nu\|_{\mathrm{T}_{s,p}}
\]
We also use the shorthand $\|\mu\|_{\mathrm{T}_{s,p}}:=\mathrm{T}_{s,p}(\mu,\delta_0)$ for $\mu\in\mathcal P^{\mathrm T}_{s,p}(\R^d)$.
\end{definition}

These constructions separate the nonlinear metric from its linear ambient geometry. The distance $\mathrm{T}_{s,p}$ compares probability measures, while the associated zero-mass norm embeds that metric space into a concrete $L^p$-type function space. The Toscani class records the maximal domain of finiteness, which contains $\mathcal P_p(\R^d)$ in the regime of Proposition~\ref{prop:integrability} and may contain it strictly (Remark~\ref{rem:class-in-PT}).
\section{Functional-Analytic Structure of \texorpdfstring{$\mathrm{T}_{s,p}$}{Toscani--Fourier Distance}}\label{sec:structure}
\subsection{Monotonicity, Embedding, and Completeness}\label{sec:embeddings}
This section develops the internal geometry of the Toscani--Fourier scale. Monotonicity in $(s,p)$, the endpoint behavior as $p\to\infty$, and the linearized embedding that yields relative compactness and completeness. Increasing either $s$ or $p$ imposes a stronger Fourier-side integrability requirement, so membership in a stronger Toscani class should imply membership in weaker ones. The next proposition makes this precise with an explicit continuous inclusion estimate.

\begin{proposition}[Continuous Inclusion]\label{prop:embeddings}
Let $1\le q\le p<\infty$ and $0<r\le s$. Assume $qr>d$.
Then $\mathcal P^{\mathrm T}_{s,p}(\R^d)\subseteq \mathcal P^{\mathrm T}_{r,q}(\R^d)$ and the inclusion map is continuous.
Moreover, for all $\mu\in\mathcal P^{\mathrm T}_{s,p}(\R^d)$,
\begin{equation}\label{eq:embedding_estimate}
\|\mu\|_{\mathrm{T}_{r,q}}
\le
|B_1|^{\frac1q-\frac1p}\ \|\mu\|_{\mathrm{T}_{s,p}}
+
2\Big(\int_{\|u\|\ge 1}\|u\|^{-qr}\ du\Big)^{1/q}
\end{equation}
where $|B_1|$ is the Lebesgue volume of the unit ball in $\R^d$.
\end{proposition}
The Toscani classes therefore form a nested scale. At the endpoint of that scale, for fixed $s\le1$, the metric converges to its $L^\infty$ analogue.
\begin{proposition}[Limit as $p\to\infty$]\label{prop:asymptotics}
Fix $0<s\le 1$, and let $\mu,\nu\in\mathcal P_1(\R^d)$. Then
$\mathrm{T}_{s,p}(\mu,\nu)$ is finite for every $p>d/s$, and
\[
\lim_{p\to\infty}\mathrm{T}_{s,p}(\mu,\nu)=\mathrm{T}_{s,\infty}(\mu,\nu)
\]
In particular, taking $\nu=\delta_0$,
\[
\lim_{p\to\infty}\|\mu\|_{\mathrm{T}_{s,p}}=\|\mu\|_{\mathrm{T}_{s,\infty}}
\]
\end{proposition}

Beyond monotonicity, the key structural observation is that in the finiteness window the metric $\mathrm{T}_{s,p}$ is the pullback of an $L^p$ norm applied to the normalized characteristic-function perturbation
\[
u\mapsto \frac{\Bchi_\mu(u)-1}{\|u\|^s}
\]
In particular, compactness and completeness can be studied inside a concrete $L^p$ space rather than directly at the level of probability measures. It also anticipates the special role of the case $p=2$, where this embedding becomes Hilbertian and connects to kernel methods.

\begin{proposition}[Relative Compactness]\label{prop:Compactness}
Fix $1\le p<\infty$ and $\frac{d}{p}<s<1+\frac{d}{p}$.
Define the embedding
\[
J:\mathcal P^{\mathrm T}_{s,p}(\R^d)\to L^p(\R^d\setminus\{0\}),
\qquad
(J\mu)(u):=\frac{\Bchi_\mu(u)-1}{\|u\|^{s}}
\]
Let $\mathbf S\subset \mathcal P(\R^d)$ satisfy the uniform first-moment bound
\[
M_1:=\sup_{\mu\in\mathbf S}\int_{\R^d}\|x\|\ d\mu(x)<\infty
\]
Then $\mathbf S\subseteq\mathcal P^{\mathrm T}_{s,p}(\R^d)$, the norms $\|\mu\|_{\mathrm{T}_{s,p}}$
are uniformly bounded on $\mathbf S$, and $\mathbf S$ is relatively compact in
$\big(\mathcal P^{\mathrm T}_{s,p}(\R^d),\mathrm{T}_{s,p}\big)$.
Equivalently, $J(\mathbf S)$ is relatively compact in $L^p(\R^d\setminus\{0\})$.
\end{proposition}

\begin{theorem}[Isometric Embedding, Closed Image, and Completeness] \label{thm:completeness}
Fix $s>0$ and $1\le p\le\infty$, and write $L^p:=L^p(\R^d\setminus\{0\})$.
The map $J$ (defined above when $p<\infty$ and analogously for $p=\infty$) is an isometric embedding in the sense that
\[
\mathrm{T}_{s,p}(\mu,\nu)=\|J\mu-J\nu\|_{L^p}\qquad(\mu,\nu\in\mathcal P^{\mathrm T}_{s,p}(\R^d))
\]
Assume moreover that either $1\le p<\infty$ and $ps\ge d$ --- in particular, any parameters in the
window $\frac dp<s<1+\frac dp$ --- or $p=\infty$ and $s>0$. Then the image
$J\big(\mathcal P^{\mathrm T}_{s,p}(\R^d)\big)$ is \emph{closed} in $L^p$. Every $L^p$-limit of
normalized characteristic-function perturbations is again of the form $J\nu$ for a unique
probability measure $\nu\in\mathcal P^{\mathrm T}_{s,p}(\R^d)$. Consequently,
$\big(\mathcal P^{\mathrm T}_{s,p}(\R^d),\mathrm{T}_{s,p}\big)$ is a complete metric space.
\end{theorem}

\begin{remark}[No escape of mass in the Fourier geometry]\label{rem:completion-meaning}
The closedness statement rules out the two degenerations that usually obstruct completeness of
spaces of probability measures. Escape of mass to spatial infinity, and loss of total mass along a
Cauchy sequence. Either degeneration would force the limiting normalized perturbation to behave
like a nonzero constant multiple of $\|u\|^{-s}$ near the origin, whose $p$-th power is
non-integrable there precisely when $ps\ge d$. The weighted integral therefore detects and forbids
the escape. This contrasts with the supremum-type Fourier metrics of fixed moment order, whose
completeness theory is more delicate \citep{stawiska2020completeness,randles2024note}. For
$ps<d$ --- possible only outside the finiteness window --- the argument does not apply, and we make
no closedness claim in that regime.
\end{remark}

\subsection{Toscani Convergence and Weak Convergence}

We now return to the probabilistic content of the metric. Toscani convergence forces weak
convergence with \emph{no} auxiliary assumption --- and, unlike the closedness statement of
Theorem~\ref{thm:completeness}, with no restriction on $(s,p)$ at all. The Fourier-side
identification alone prevents both escape of mass and degeneration of total mass along a
convergent sequence.

\begin{proposition}[Toscani convergence implies weak convergence]\label{prop:convergence}
Let $s>0$ and $1\le p\le\infty$. Let $\mu_n,\mu$ be probability measures with
$\mathrm{T}_{s,p}(\mu_n,\mu)\to 0$. Then the following hold.
\begin{enumerate}
\item[(i)] $\mu_n \overset{w}{\rightharpoonup} \mu$
\item[(ii)] $\Bchi_{\mu_n}\to\Bchi_{\mu}$ uniformly on compact subsets of $\R^d$
\end{enumerate}
No moment assumption is imposed on $\mu_n$ or $\mu$.
\end{proposition}
The converse fails. Weak convergence does not imply Toscani convergence without additional scale
normalization, and the next remark quantifies the obstruction exactly.

\begin{remark}\label{rem:weak-not-T}
Moment conditions do \emph{not} imply high-frequency decay of $\Bchi_\mu$. They control its regularity near the origin.
Consequently, weak convergence $\mu_n\overset{w}{\rightharpoonup}\mu$, equivalently pointwise convergence of the characteristic functions, does \emph{not} imply $\mathrm{T}_{s,p}(\mu_n,\mu)\to 0$ without an additional normalization such as a uniform first-moment bound or common compact support.

A two-atom construction quantifies the obstruction \emph{exactly}.
\begin{example}\label{ex:two-atom}
Fix $1\le p<\infty$, $\frac{d}{p}<s<1+\frac{d}{p}$, and $\alpha>0$. Define
\[
\mu_n=\Big(1-\frac{1}{n}\Big)\delta_{0}+\frac{1}{n}\delta_{n^\alpha e_1}\quad ,
\qquad
\mu=\delta_0
\]
Since $\mu_n-\delta_0=\frac1n(\delta_{n^\alpha e_1}-\delta_0)$, absolute homogeneity of the
zero-mass norm and the change of variables $v=n^{\alpha}u$ give the exact identities
\[
\mathrm{T}_{s,p}(\mu_n,\delta_0)=n^{\,\alpha(s-\frac dp)-1}\ t_0,
\qquad
W_p(\mu_n,\delta_0)=n^{\,\alpha-\frac1p},
\qquad
t_0:=\mathrm{T}_{s,p}(\delta_0,\delta_{e_1})\in(0,\infty)
\]
where $t_0$ is finite and positive by Proposition~\ref{prop:window-sharp}. Two regimes are of
interest.
\begin{enumerate}
\item[(a)] If $\alpha>\frac{p}{ps-d}$, then $\mu_n\overset{w}{\rightharpoonup}\delta_0$ while
$\mathrm{T}_{s,p}(\mu_n,\delta_0)\to\infty$. Weak convergence does not imply Toscani convergence.
\item[(b)] If $1<\alpha<\frac{p}{ps-d}$ (such $\alpha$ exist since $s<1+\frac dp$ makes
$\frac{p}{ps-d}>1$), then $\mathrm{T}_{s,p}(\mu_n,\delta_0)\to 0$ while
$\int\|x\|\,d\mu_n=n^{\alpha-1}\to\infty$ and $W_p(\mu_n,\delta_0)\to\infty$. Toscani convergence
is compatible with unbounded first moments and with divergent Wasserstein distance, which is why
Proposition~\ref{prop:convergence} carries no moment hypothesis and why no global reverse
comparison can hold (Lemma~\ref{lem:scaling}).
\end{enumerate}
\end{example}
Thus, even though the extra atom has weight only $1/n$ and disappears in the weak limit, its location at scale $n^\alpha$ produces increasingly severe oscillations in frequency space. The Toscani metric weighs this long-range escape through its weighted Fourier integral in a manner incomparable to weak convergence or transport cost.
\end{remark}

\section{Toscani--Fourier Comparison with Wasserstein}\label{sec:Ts-vs-Wp}

This section develops the comparison theory between $\mathrm{T}_{s,p}$ and Wasserstein distances. We first prove a global upper bound showing that Wasserstein closeness always implies Toscani-Fourier closeness in the finiteness window. We then show that no comparable global reverse inequality can hold, because the two metrics scale differently under dilation. This obstruction leads naturally to scale-normalized model classes, such as bounded-support and uniform-moment families, on which reverse comparison and topological equivalence become possible.

\begin{definition}[Bounded-support Class]
For $K>0$, define the bounded-support class
\[
\mathcal P_{p,K}(\R^d)
:=
\big\{\mu\in\mathcal P_p(\R^d):\mathrm{supp}(\mu)\subseteq B(0,K)\big\}
\]
where here and throughout $B(0,K)$ denotes the \emph{closed} Euclidean ball, so that
$\mathcal P_{p,K}(\R^d)$ is weakly closed.
\end{definition}
On $\mathcal P_{p,K}(\R^d)$, weak convergence and $W_p$-convergence coincide. Moments of every
order are uniformly bounded on the common compact support, so weak convergence upgrades to
$W_p$-convergence by Lemma~\ref{lem:weak-moment-wp} below.
\begin{definition}[Uniform Moment Class]
For $\gamma>0$ and $S>0$, define the uniform $\gamma$-moment class
\[
\mathcal M_{\gamma,S}
:=
\Big\{\mu\in\mathcal P(\R^d):\int_{\R^d}\|x\|^\gamma\ d\mu(x)\le S\Big\}
\]
\end{definition}
Each statement below specifies the range of $\gamma$ it requires. When $\gamma\ge p$ we have
$\mathcal M_{\gamma,S}\subset\mathcal P_p(\R^d)$, and when $\gamma\ge1$, H\"older's inequality
gives \(\int\|x\|\ d\mu(x)\le S^{1/\gamma}\) for every \(\mu\in\mathcal M_{\gamma,S}\), so the
first moments are uniformly bounded.

\begin{theorem}[Global upper H\"older control of $\mathrm{T}_{s,p}$ by $W_p$]\label{thm:upper-holder}
Let $d\ge 1$, $1\le p<\infty$, and $\frac{d}{p}<s<1+\frac{d}{p}$. Write $v_d:=|\mathbb S^{d-1}|$ and
\[
C(d,p,s)
:=
\Bigg[v_d\bigg(\frac{1}{d+p-ps}+\frac{2^{p}}{ps-d}\bigg)\Bigg]^{1/p}<\infty
\]
Then, for all $\mu,\nu\in\mathcal P_p(\mathbb R^d)$,
\begin{equation}\label{eq:upper-holder}
\mathrm{T}_{s,p}(\mu,\nu)
\le
C(d,p,s)\ W_p(\mu,\nu)^{\ s-\frac{d}{p}}
\end{equation}
The H\"older exponent $s-\frac dp\in(0,1)$ in \eqref{eq:upper-holder} cannot be improved. By Lemma~\ref{lem:scaling}, $\mathrm{T}_{s,p}(\delta_0,\delta_{\lambda x})=\lambda^{\,s-d/p}\mathrm{T}_{s,p}(\delta_0,\delta_x)$ while $W_p(\delta_0,\delta_{\lambda x})=\lambda W_p(\delta_0,\delta_x)$, so \eqref{eq:upper-holder} is saturated, up to the constant, along every dilation family of Dirac pairs.
\end{theorem}

\begin{remark}[A $W_1$ strengthening]\label{rem:upper-holder-W1}
The same argument, run with $W_1$ in place of $W_p$ (choosing the splitting radius
$R=W_1(\mu,\nu)^{-1}$ and using that $W_1$ is definite on $\mathcal P_1$), establishes the
formally stronger bound
\[
\mathrm{T}_{s,p}(\mu,\nu)\le C(d,p,s)\ W_1(\mu,\nu)^{\ s-\frac{d}{p}}
\qquad(\mu,\nu\in\mathcal P_1(\R^d))
\]
with the same constant. The pointwise characteristic-function estimate in the proof involves
only $W_1$, and $W_1\le W_p$ recovers \eqref{eq:upper-holder}. The global upper control therefore
already holds on the larger class $\mathcal P_1(\R^d)$ and against the weaker distance $W_1$.
\end{remark}

\begin{corollary}[Two-regime form]\label{cor:upper-holder-max}
Under the assumptions of Theorem~\ref{thm:upper-holder},
\[
\mathrm{T}_{s,p}(\mu,\nu)
\le
C(d,p,s)\ \max\Big\{W_p(\mu,\nu)^{\ s-\frac{d}{p}},\  W_p(\mu,\nu)\Big\}
\qquad(\mu,\nu\in\mathcal P_p(\R^d))
\]
that is, $\mathrm{T}_{s,p}\le C\,W_p^{\,s-d/p}$ when $W_p\le 1$ and $\mathrm{T}_{s,p}\le C\,W_p$ when $W_p\ge 1$.
\end{corollary}
\begin{proof}
Immediate from \eqref{eq:upper-holder}, since $s-\frac dp\in(0,1)$ gives $W_p^{\,s-d/p}\le\max\{W_p^{\,s-d/p},W_p\}$ for every value of $W_p$, and $W_p^{\,s-d/p}\le W_p$ once $W_p\ge1$.
\end{proof}
Corollary~\ref{cor:upper-holder-max} is the form used as a calibration reference in the numerical experiments of Section~\ref{sec:num_exp}. Theorem~\ref{thm:upper-holder} shows that it is never sharp in the far regime $W_p>1$.

Theorem~\ref{thm:upper-holder} is the global one-sided comparison available under minimal assumptions. Whether any reverse control holds on all of $\mathcal P_p(\R^d)$ is settled negatively by the following scaling lemma, which also identifies the precise homogeneity obstruction.
\begin{lemma}[Scaling, and the failure of any global converse]\label{lem:scaling}
Let $1\le p<\infty$ and $\frac dp<s<1+\frac dp$.
For $\lambda>0$, let $D_\lambda(x):=\lambda x$ and pushforward measure $\mu^\lambda:=(D_\lambda)_{\#}\mu$.
Then for all $\mu,\nu\in\mathcal P_p(\R^d)$,
\[
W_p(\mu^\lambda,\nu^\lambda)=\lambda\ W_p(\mu,\nu),
\qquad
\mathrm{T}_{s,p}(\mu^\lambda,\nu^\lambda)=\lambda^{\ s-\frac{d}{p}}\ \mathrm{T}_{s,p}(\mu,\nu)
\]
Two consequences follow.
\begin{enumerate}
\item[(i)] If
\(
W_p(\mu,\nu)\le C\ \mathrm{T}_{s,p}^{\rho}(\mu,\nu)
\)
holds for all $\mu,\nu\in\mathcal P_p(\R^d)$ with some $\rho>0$ and $C<\infty$, then necessarily
$\rho=(s-\tfrac{d}{p})^{-1}>1$. In particular, no reverse inequality with exponent
$\rho\in(0,1]$ can hold globally.
\item[(ii)] Not even the exponent in (i) survives. For every $t>0$,
\[
\sup\Big\{W_p(\mu,\nu):\ \mu,\nu\in\mathcal P_p(\R^d),\ \mathrm{T}_{s,p}(\mu,\nu)\le t\Big\}=\infty
\]
Hence there is \emph{no} function $\Phi:[0,\infty)\to[0,\infty)$, finite everywhere, with
$W_p(\mu,\nu)\le\Phi\big(\mathrm{T}_{s,p}(\mu,\nu)\big)$ on all of $\mathcal P_p(\R^d)$. No
global converse of any form --- H\"older or otherwise --- exists.
\end{enumerate}
\end{lemma}
Part~(i) comes from dilation alone.
Part~(ii) combines dilation with the mass-splitting family of Example~\ref{ex:two-atom}.

Lemma~\ref{lem:scaling} shows that reverse controls are plausible only on model classes that exclude arbitrary dilations and arbitrary mass splitting. Bounded support, fixed scale, uniform tails, or density regularity. Such classes remain broad in applications, covering compactly supported distributions, uniformly sub-Gaussian or log-concave families, and learned pushforwards through bounded generators. On them the failure of a global reverse inequality no longer prevents $\mathrm{T}_{s,p}$ and $W_p$ from inducing the same convergence, which is the subject of the next subsection.
\subsection{Topological Equivalence on Bounded Support and Uniform Tail Classes}
On the two scale-normalized classes just defined, $W_p$ and $\mathrm{T}_{s,p}$ generate the same topology throughout the finiteness window. These are \emph{topological} statements. They identify the correct convergence structure even where no explicit reverse modulus is available. Their proofs combine the unconditional convergence result of Proposition~\ref{prop:convergence} with standard compactness and moment arguments. Section~\ref{sec:lip-pointwise} develops complementary \emph{quantitative} pointwise bounds under scale normalization.

\begin{lemma}[Weak Convergence and Uniform Moments]\label{lem:weak-moment-wp}
Let $\gamma>p$. Suppose $\mu_n\overset{w}{\rightharpoonup}\mu$. If the \(\gamma\)-moments
satisfy \(\sup_n\int \|x\|^\gamma\ d\mu_n(x)<\infty\), then $W_p(\mu_n,\mu)\to 0$.
\end{lemma}

Lemma~\ref{lem:weak-moment-wp} supplies the probabilistic input on uniform-moment classes. Combining it with pointwise convergence of characteristic functions and scale normalization gives the two equivalence theorems.

\begin{theorem}[Topological Equivalence on Bounded Support]\label{thm:Wp-Tsp-equivalence}
Let $d\ge 1$, $p\ge 1$, $\frac{d}{p}<s<1+\frac{d}{p}$, and $K>0$.
On $\mathcal P_{p,K}(\R^d)$, the metrics $W_p$ and $\mathrm{T}_{s,p}$ generate the same topology.
Equivalently, for any $\{\mu_n\}\subset\mathcal P_{p,K}(\R^d)$ and $\mu\in\mathcal P_{p,K}(\R^d)$,
\[
W_p(\mu_n,\mu)\to 0
\quad\Longleftrightarrow\quad
\mathrm{T}_{s,p}(\mu_n,\mu)\to 0
\]
\end{theorem}

\begin{theorem}[Topological Equivalence on Uniform Tail Classes]\label{thm:Wp-Tsp-equivalence-moment}
Let $d\ge 1$, $p\ge 1$, $\frac{d}{p}<s<1+\frac{d}{p}$, and choose $\gamma>p$, $S>0$.
On $\mathcal M_{\gamma,S}$, the metrics $W_p$ and $\mathrm{T}_{s,p}$ generate the same topology.
For any $\{\mu_n\}\subset\mathcal M_{\gamma,S}$ and $\mu\in\mathcal M_{\gamma,S}$,
\[
W_p(\mu_n,\mu)\to 0
\quad\Longleftrightarrow\quad
\mathrm{T}_{s,p}(\mu_n,\mu)\to 0
\]
\end{theorem}

\begin{remark}[Uniform Equivalence Moduli]\label{rem:uniform-moduli}
Because $\mathcal P_{p,K}(\R^d)$ is weakly compact and $W_p$ metrizes weak convergence on $\mathcal P_{p,K}$,
Theorem~\ref{thm:Wp-Tsp-equivalence} implies uniform equivalence.
There exist increasing functions $\omega_1,\omega_2:[0,\infty)\to[0,\infty)$ with $\omega_i(0+)=\omega_i(0)=0$ such that
\[
\mathrm{T}_{s,p}(\mu,\nu)\le \omega_1\big(W_p(\mu,\nu)\big)\quad ,
\qquad
W_p(\mu,\nu)\le \omega_2\big(\mathrm{T}_{s,p}(\mu,\nu)\big)\quad ,
\qquad
\forall \mu,\nu\in\mathcal P_{p,K}(\R^d)
\]
The existence of such moduli is a soft consequence of compactness. The set $\mathcal P_{p,K}(\R^d)$ is weakly closed and tight, hence weakly compact, and $W_p$ metrizes weak convergence on it, so $(\mathcal P_{p,K},W_p)$ is a compact metric space. By Theorem~\ref{thm:Wp-Tsp-equivalence} the metric $\mathrm{T}_{s,p}$ induces the same topology, so the identity map between the two compact metric spaces is a homeomorphism and therefore uniformly continuous in both directions. Setting
\[
\begin{aligned}
\omega_1(t)&:=\sup\big\{\mathrm{T}_{s,p}(\mu,\nu):\ \mu,\nu\in\mathcal P_{p,K}(\R^d),\ W_p(\mu,\nu)\le t\big\}\\
\omega_2(t)&:=\sup\big\{W_p(\mu,\nu):\ \mu,\nu\in\mathcal P_{p,K}(\R^d),\ \mathrm{T}_{s,p}(\mu,\nu)\le t\big\}
\end{aligned}
\]
gives nondecreasing functions that are finite (both metrics are bounded on the compact class) and satisfy $\omega_i(0+)=0$ precisely by that uniform continuity.
Theorem~\ref{thm:upper-holder} provides an explicit admissible $\omega_1$. An explicit admissible $\omega_2$ is provided in Theorem~\ref{thm:explicit-reverse-bounded}.
\end{remark}
\subsection{Lipschitz Control and Pointwise Bounds}\label{sec:lip-pointwise}

On scale-normalized classes, characteristic functions enjoy uniform Lipschitz control, and that regularity upgrades $\mathrm{T}_{s,p}$-smallness to \emph{quantitative} pointwise control, complementing the qualitative Proposition~\ref{prop:convergence}. The following lemmas make the principle precise. They are stated and proved independently of Theorems~\ref{thm:Wp-Tsp-equivalence}--\ref{thm:Wp-Tsp-equivalence-moment}, so no circularity is involved.

\begin{lemma}[Lipschitz Regularity on $\mathcal P_{p,K}$]\label{lem:chi-lip}
If $\mu\in\mathcal P_{p,K}(\R^d)$, then $\Bchi_\mu$ is globally Lipschitz. More precisely,
\[
|\Bchi_\mu(u)-\Bchi_\mu(v)|\le K\|u-v\|\quad,
\qquad \forall u,v\in\R^d
\]
Consequently, for $\mu,\nu\in\mathcal P_{p,K}(\R^d)$, the difference $f:=\Bchi_\mu-\Bchi_\nu$ satisfies $\mathrm{Lip}(f)\le 2K$.
\end{lemma}

\begin{lemma}[Uniform Lipschitz Control on $\mathcal M_{\gamma,S}$]\label{lem:chi-lip-moment}
Let $\gamma\ge1$, $S>0$, and $\mu\in\mathcal M_{\gamma,S}$. Then for all $u,v\in\R^d$,
\[
|\Bchi_\mu(u)-\Bchi_\mu(v)|
\le \|u-v\|\int \|x\|\ d\mu(x)
\le \|u-v\|\ S^{1/\gamma}
\]
Consequently, for $\mu,\nu\in\mathcal M_{\gamma,S}$, $f:=\Bchi_\mu-\Bchi_\nu$ satisfies $\mathrm{Lip}(f)\le 2S^{1/\gamma}$.
\end{lemma}

\begin{lemma}[Pointwise Control by $\mathrm{T}_{s,p}$]\label{lem:Tp-to-pointwise}
Fix $p\ge 1$ and $\frac{d}{p}<s<1+\frac{d}{p}$. Let $\mu,\nu$ be probability measures, and
suppose $\sigma:=\Bchi_\mu-\Bchi_\nu$ is globally $L$--Lipschitz.
Then, with $c_{d,p}:=|B(0,1)|^{-1/p}$, for any $u\neq 0$ and any $r\in(0,\|u\|/2)$,
\[
|\sigma(u)|:=|\Bchi_\mu(u)-\Bchi_\nu(u)|
\le
c_{d,p}\ r^{-d/p}\ \|\sigma\|_{L^p(B(u,r))}
+ Lr
\le
c_{d,p}\ r^{-d/p}\ (\|u\|+r)^{s}\ \mathrm{T}_{s,p}(\mu,\nu)
+ Lr
\]
In particular, for each fixed $u\in\R^d$, $\mathrm{T}_{s,p}(\mu_n,\mu)\to 0$ and a uniform Lipschitz bound on $\sigma_n:=\Bchi_{\mu_n}-\Bchi_\mu$ imply $\Bchi_{\mu_n}(u)\to \Bchi_\mu(u)$.
\end{lemma}

\begin{corollary}[Pointwise Bound]\label{cor:pointwise-bound}
Fix $p\ge 1$ and $\frac{d}{p}<s<1+\frac{d}{p}$.
Under the assumptions of Lemma~\ref{lem:Tp-to-pointwise}, for every $u\neq 0$,
\begin{equation}\label{eq:pointwise-bound}
|\Bchi_\mu(u)-\Bchi_\nu(u)|
\le
C_{d,p,s}\ \|u\|^{\ s-\frac{d}{p}}\ \mathrm{T}_{s,p}(\mu,\nu)
+\frac{L}{2}\ \|u\|
\end{equation}
where $C_{d,p,s}=c_{d,p}\ 2^{d/p}(3/2)^s$.
\end{corollary}

\begin{remark}[Sobolev-Type Density Regularity]\label{rem:sobolev-density}
A common route to quantitative reverse bounds assumes densities with uniform Sobolev control.
This is standard in the MMD--Wasserstein literature. Under bounded $r$--moments and Sobolev regularity, $W_p$ can be controlled by a translation-invariant MMD with an explicit H\"older exponent. See \citep[Theorem~15]{vayer2023controlling}.
For reference, Vayer--Gribonval consider
\[
\mathcal S_{B,M,r,t}
:=
\Big\{\pi\in\mathcal P(\R^d): \pi=f\ dx,\ \|f\|_{H^t(\R^d)}\le B,\ M_r[\pi]\le M\Big\}
\]
and obtain reverse H\"older-type controls of $W_p$ by an MMD norm on $\mathcal S_{B,M,r,t}$ (for $1\le p<r$) provided $t$ is large enough relative to the kernel smoothness.
In our setting, such assumptions offer a natural route toward explicit moduli on density-regular classes. Carrying this out for general $p$ is beyond our scope, while in the Hilbert case $p=2$ the MMD interpretation of Section~\ref{sec:mmd} applies directly.
\end{remark}

The quadratic case admits a sharper Hilbertian interpretation, developed in Section~\ref{sec:mmd}. There $\mathrm{T}_{s,2}$ is identified with an MMD, which makes it possible to import additional reverse comparison results from the kernel/Wasserstein literature.
\section{Dual Representation in Homogeneous Fourier--Lebesgue Space}\label{sec:dual-ipm}
This section strengthens the comparison theory from Section~\ref{sec:Ts-vs-Wp}. We identify $\mathrm{T}_{s,p}$ as an integral probability metric over a homogeneous Fourier-Lebesgue unit ball, which gives a dual representation valid for general $1\le p<\infty$.

\subsection{Integral Probability Metric (IPM) Representation}
Fix $1\le p<\infty$ and let $q$ be the conjugate exponent, $\frac1p+\frac1q=1$.
We use the Fourier convention
\[
\widehat{\phi}(u)=\int_{\R^d} e^{i\langle u,x\rangle}\phi(x)\ dx\quad,
\qquad
\phi(x)=(2\pi)^{-d}\int_{\R^d} e^{-i\langle u,x\rangle}\widehat{\phi}(u)\ du
\]
so that for $\phi\in\mathcal S(\R^d)$ and any $\sigma\in\mathcal M(\R^d)$,
\begin{equation}\label{eq:fourier-pairing}
\int_{\R^d}\phi\ d\sigma
=(2\pi)^{-d}\int_{\R^d}\widehat{\phi}(u)\ \widehat{\sigma}(-u)\ du
\end{equation}
For $\mu,\nu\in\mathcal P(\R^d)$ set $\sigma:=\mu-\nu\in\mathcal M_0(\R^d)$. Then $\widehat{\sigma}=\Bchi_\mu-\Bchi_\nu$ and $\widehat{\sigma}(0)=0$.
In the finiteness window $\frac{d}{p}<s<1+\frac{d}{p}$,
\[
\mathrm{T}_{s,p}(\mu,\nu)=\Big\|\frac{\widehat{\sigma}}{\|u\|^{s}}\Big\|_{L^p(\R^d)}
\]

Define the homogeneous Fourier--Lebesgue seminorm
\[
\|\phi\|_{\dot{\mathcal F}L^{q}_s}
:=
\big\|\|u\|^{s}\widehat{\phi}(u)\big\|_{L^q(\R^d)}
\]
We define $\dot{\mathcal F}L^{q}_s$ as the abstract completion of $\mathcal S(\R^d)$ under
$\|\cdot\|_{\dot{\mathcal F}L^{q}_s}$. This is already a norm on $\mathcal S(\R^d)$, which
contains no nonzero constant function. The completion is concrete on the Fourier side. For
$1<q<\infty$ the map $\phi\mapsto\|u\|^{s}\widehat\phi$ extends to a surjective linear isometry
from $\dot{\mathcal F}L^{q}_s$ onto $L^q(\R^d)$ (Step~2 of the proof of
Proposition~\ref{prop:dual-optimizer}). Only when \emph{realizing} elements of
$\dot{\mathcal F}L^{q}_s$ as distributions on the space side does the usual quotient by constants
of homogeneous spaces appear. Two realizations differing by an additive constant must be
identified, which is harmless here because $\sigma\in\mathcal M_0$ annihilates constants (compare
Corollary~\ref{cor:hilbert-dual} for the Hilbertian realization $\dot H^s(\R^d)/\R$).
Test functions are allowed to be complex-valued. Since $\sigma$ below is a real signed measure, we phrase the duality through real parts, and Remark~\ref{rem:dual-real} records that this entails no loss.

With this quotient space in place, the Toscani-Fourier metric admits a natural dual characterization. The next theorem shows that $\mathrm{T}_{s,p}$ is exactly the operator norm of the signed measure $\sigma=\mu-\nu$ acting on the unit ball of $\dot{\mathcal F}L_s^q$. In this sense, $\mathrm{T}_{s,p}$ is an integral probability metric whose test class is determined by Fourier-side regularity.

\begin{theorem}[Dual/IPM Representation]\label{thm:dual-ipm}
Let $1\le p<\infty$ and $\frac1p+\frac1q=1$.
If $\sigma\in\mathcal M_0(\R^d)$ satisfies $\widehat{\sigma}/\|u\|^{s}\in L^p(\R^d)$, then $\phi\mapsto\int\phi\ d\sigma$ extends uniquely to a continuous linear functional on $\dot{\mathcal F}L^{q}_s$, and
\begin{equation}\label{eq:dual-ipm}
\sup_{\substack{\phi\in\dot{\mathcal F}L^{q}_s\\ \|\phi\|_{\dot{\mathcal F}L^{q}_s}\le 1}}
\mathrm{Re}\int_{\R^d}\phi\ d\sigma
= (2\pi)^{-d}\Big\|\frac{\widehat{\sigma}}{\|u\|^{s}}\Big\|_{L^p(\R^d)}
= (2\pi)^{-d}\mathrm{T}_{s,p}(\mu,\nu)
\end{equation}
Equivalently,
\[
\mathrm{T}_{s,p}(\mu,\nu)
=
(2\pi)^{d}
\sup\Big\{\mathrm{Re}\int_{\R^d}\phi\ d(\mu-\nu):\ \phi\in\dot{\mathcal F}L^{q}_s,\ \|\phi\|_{\dot{\mathcal F}L^{q}_s}\le 1\Big\}
\]
For $1<p<\infty$ the supremum in \eqref{eq:dual-ipm} is attained in $\dot{\mathcal F}L^{q}_s$ (Proposition~\ref{prop:dual-optimizer}). For $p=1$ the identity \eqref{eq:dual-ipm} remains valid, but the supremum need not be attained.
\end{theorem}

\begin{remark}[Real, complex, and modulus formulations]\label{rem:dual-real}
Because $\sigma=\mu-\nu$ is a real measure, the three natural formulations of the dual functional coincide.
\[
\sup_{\|\phi\|_{\dot{\mathcal F}L^{q}_s}\le 1}\mathrm{Re}\int\phi\ d\sigma
=\sup_{\|\phi\|_{\dot{\mathcal F}L^{q}_s}\le 1}\Big|\int\phi\ d\sigma\Big|
=\sup_{\substack{\|\phi\|_{\dot{\mathcal F}L^{q}_s}\le 1\\ \phi\ \text{real-valued}}}\int\phi\ d\sigma
\]
Indeed, for any $\phi$ there is a unimodular $c\in\C$ with $|\int\phi\,d\sigma|=\mathrm{Re}\int(\bar c\phi)\,d\sigma$ and $\|\bar c\phi\|_{\dot{\mathcal F}L^q_s}=\|\phi\|_{\dot{\mathcal F}L^q_s}$, which proves the first equality. For the second, write $\phi=\phi_1+i\phi_2$ with $\phi_1,\phi_2$ real-valued, note $\mathrm{Re}\int\phi\,d\sigma=\int\phi_1\,d\sigma$, and observe that $\widehat{\phi_1}(u)=\tfrac12\big(\widehat\phi(u)+\overline{\widehat\phi(-u)}\big)$ together with Minkowski's inequality and the invariance of $\|\|u\|^s g(u)\|_{L^q}$ under $u\mapsto-u$ give $\|\phi_1\|_{\dot{\mathcal F}L^q_s}\le\|\phi\|_{\dot{\mathcal F}L^q_s}$.
\end{remark}

\begin{proposition}[Extremizer and Schwartz near-attainment]\label{prop:dual-optimizer}
Let $1<p<\infty$ with conjugate $q$.
Let $\sigma\in\mathcal M_0(\R^d)$ and set $f(u):=\widehat{\sigma}(-u)/\|u\|^{s}\in L^p(\R^d)$ with $f\not\equiv 0$.
The $L^p$--$L^q$ maximizer is
\[
g^\ast(u):=\frac{|f(u)|^{p-2}\overline{f(u)}}{\|f\|_{L^p}^{\ p-1}},
\qquad
\|g^\ast\|_{L^q}=1,
\qquad
\int_{\R^d} f(u)g^\ast(u)\ du=\|f\|_{L^p}
\]
Hence the supremum in \eqref{eq:dual-ipm} is attained in $\dot{\mathcal F}L^{q}_s$ by some $\phi^\ast$ satisfying $\|u\|^{s}\widehat{\phi^\ast}=g^\ast$ in $L^q$.
Moreover, there exists $\phi_n\in\mathcal S(\R^d)$ with $\|\phi_n\|_{\dot{\mathcal F}L^{q}_s}\le 1$ such that
\[
\mathrm{Re}\int_{\R^d}\phi_n\ d\sigma \longrightarrow (2\pi)^{-d}\ \mathrm{T}_{s,p}(\mu,\nu)
\]
\end{proposition}
The dual formula identifies the exact test-function class behind $\mathrm{T}_{s,p}$, and Proposition~\ref{prop:dual-optimizer} makes the H\"older extremizer explicit at the Fourier level. We now convert this into a quantitative statement on bounded-support classes.

\subsection{Explicit Bounded-Support Reverse Modulus}

To turn the dual representation into a reverse Wasserstein bound, we test against mollified Kantorovich potentials on bounded-support classes. The key analytic input is that a compactly supported Lipschitz test function acquires controlled $\dot{\mathcal F}L_s^q$ size after mollification, with an explicit dependence on the support radius and smoothing scale.

\begin{lemma}[A mollified Lipschitz test has controlled $\dot{\mathcal F}L_s^q$ size]\label{lem:mollified-FL}
Fix $1\le p<\infty$ with conjugate $q\in(1,\infty]$ and assume $\frac{d}{p}<s<1+\frac{d}{p}$.
Let $\rho\in C_c^\infty(\R^d)$ satisfy $\rho\ge 0$ and $\int\rho=1$, and set $\rho_\varepsilon(x)=\varepsilon^{-d}\rho(x/\varepsilon)$.
If $\psi:\R^d\to\R$ is $1$--Lipschitz, $\psi(0)=0$, and $\mathrm{supp}(\psi)\subseteq B(0,2K)$, then $\psi_\varepsilon:=\psi*\rho_\varepsilon$ satisfies $\psi_\varepsilon\in\mathcal S(\R^d)$ and, for all $\varepsilon>0$,
\begin{equation}\label{eq:mollified-FL-bound}
\|\psi_\varepsilon\|_{\dot{\mathcal F}L_s^q}
\le
C_{d,p,s,\rho}\ K^{d}\ \varepsilon^{\ 1-s-\frac{d}{q}}
\end{equation}
where for $p=1$ (so $q=\infty$) the norm is $\|\psi_\varepsilon\|_{\dot{\mathcal F}L_s^\infty}=\sup_{u\neq 0}\|u\|^s|\widehat{\psi_\varepsilon}(u)|$ and the exponent reads $\varepsilon^{\,1-s}$.
\end{lemma}
Lemma~\ref{lem:mollified-FL} supplies the normalized test functions needed to pair the dual formulation of $\mathrm{T}_{s,p}$ against the Kantorovich--Rubinstein representation of $W_1$, which yields the following explicit reverse modulus for every $1\le p<\infty$.

\begin{theorem}[Explicit Bounded-Support Reverse Modulus]\label{thm:explicit-reverse-bounded}
Fix $1\le p<\infty$ and let \(q\) be its conjugate. Assume $\frac{d}{p}<s<1+\frac{d}{p}$.
There exists $C=C(d,p,s)<\infty$ such that for all $K>0$ and all $\mu,\nu\in\mathcal P_{p,K}(\R^d)$,
\begin{equation}\label{eq:Wp-by-Tsp-bounded}
W_1(\mu,\nu)
\le
C\ K^{\frac{d}{\ s+\frac{d}{q}\ }}\ \mathrm{T}_{s,p}(\mu,\nu)^{\frac{1}{\ s+\frac{d}{q}\ }}
\end{equation}
Consequently, for any $r\ge 1$,
\begin{equation}\label{eq:Wrr-by-Tsp-bounded}
W_r(\mu,\nu)
\le
(2K)^{1-\frac1r}\ 
\Big(C\ K^{\frac{d}{\ s+\frac{d}{q}\ }}\Big)^{\ \frac1r}\ 
\mathrm{T}_{s,p}(\mu,\nu)^{\frac{1}{\ r\big(s+\frac{d}{q}\big)\ }}
\end{equation}
\end{theorem}
Theorem~\ref{thm:explicit-reverse-bounded} upgrades the qualitative equivalence from Section~\ref{sec:Ts-vs-Wp} to an explicit quantitative reverse bound on bounded-support classes. The exponent reflects the Fourier regularity index $s$ together with the dual integrability exponent $q$, and is not claimed to be sharp in general. Its main role is to show that once arbitrary dilations are excluded, $\mathrm{T}_{s,p}$ controls Wasserstein distance with an explicit H\"older-type modulus.

For $1\le p\le2$ the dual exponent satisfies $q\ge2$, and a Hausdorff--Young estimate replaces the
$L^1$-based bound of Lemma~\ref{lem:mollified-FL} by a sharper one, improving the modulus.

\begin{proposition}[Improved bounded-support modulus for $1\le p\le2$]\label{prop:reverse-bounded-HY}
Let $1\le p\le 2$ with conjugate $q$, let $d\ge1$, and assume $\frac dp<s<1+\frac dp$ together
with $s\ge1$ (the latter is automatic when $d\ge2$, since then $\frac dp\ge1$).
There exists $C=C(d,p,s)<\infty$ such that for all $K>0$ and all
$\mu,\nu\in\mathcal P_{p,K}(\R^d)$,
\[
W_1(\mu,\nu)
\le
C\,K^{\frac{d}{ps}}\ \mathrm T_{s,p}(\mu,\nu)^{1/s}
\qquad(s>1),
\]
while for $s=1$ (possible only when $d<p$, hence only for $d=1$, $1<p\le 2$),
\[
W_1(\mu,\nu)\le C\,K^{\frac dp}\ \mathrm T_{1,p}(\mu,\nu)
\]
For $1<p\le2$ we have $\frac1s>\frac{1}{s+d/q}$, so these moduli strictly improve
\eqref{eq:Wp-by-Tsp-bounded} in the small-$\mathrm T_{s,p}$ regime (at $p=1$, where $q=\infty$,
the two bounds coincide, as they must. The Hausdorff--Young step degenerates to the same
$L^1$--$L^\infty$ bound used in Lemma~\ref{lem:mollified-FL}). The $W_r$ upgrade of
\eqref{eq:Wrr-by-Tsp-bounded} applies verbatim with the improved $W_1$ bound.
\end{proposition}

In the quadratic case $p=2$, the dual space becomes the homogeneous Sobolev space $\dot{H}^s(\R^d)/\R$, and the resulting dual formulation admits a Hilbertian interpretation as an MMD. We develop this specialization in Section~\ref{sec:mmd}.

\section{MMD and Energy-Distance Representation in the Hilbertian Case \texorpdfstring{$p=2$}{p=2}}\label{sec:mmd}

When $p=2$, the Toscani--Fourier distance is Hilbertian.
In the finiteness window $\frac d2<s<1+\frac d2$,
\begin{equation}\label{eq:Ts2}
\mathrm{T}_{s,2}^2(\mu,\nu)
= \int_{\R^d\setminus\{0\}}
\frac{|\Bchi_\mu(u)-\Bchi_\nu(u)|^2}{\|u\|^{2s}}\ du
\end{equation}
This is the $L^2(\Lambda_s)$--norm of $\Bchi_\mu-\Bchi_\nu$ under the (typically $\sigma$--finite) spectral measure
\[
\Lambda_s(d\xi)=\|\xi\|^{-2s}\ d\xi
\]
\citet{modeste2024translation} identify translation-invariant MMDs through characteristic
functions. For a large class of possibly unbounded translation-invariant kernels,
\begin{equation}\label{eq:MMDspectral}
\mathrm{MMD}^2_{\Lambda}(\mu,\nu)
=\int_{\R^d}|\Bchi_\mu(\xi)-\Bchi_\nu(\xi)|^2\ \Lambda(d\xi)
\qquad
(\mu,\nu\ \text{in a suitable moment class})
\end{equation}

Comparing \eqref{eq:Ts2} and \eqref{eq:MMDspectral} shows that $\mathrm{T}_{s,2}$ is precisely the MMD induced by $\Lambda=\Lambda_s$ (up to Fourier normalization conventions). Thus the quadratic case is not merely an $L^p$-type Fourier discrepancy. It is a genuine Hilbert-space distance, placing $\mathrm{T}_{s,2}$ directly in the framework of kernel discrepancies and negative-Sobolev IPMs.

We now record the exact normalization relating $\mathrm{T}_{s,2}$ to the energy/fractional
Brownian kernel MMD. This identification is classical. The representation of the (generalized)
energy distance as a weighted $L^2$ distance between characteristic functions, with precisely the
constant below, is due to \citet{szekely2005new,szekely2013energy}. See
\citet[Lemma~1]{szekely2007measuring} for the key Fourier identity, \citet{sejdinovic2013equivalence} for the kernel formulation, and
\citet[Theorem~1]{auricchio2025kinetic} for a recent account in the Fourier-metric normalization.
The following lemma restates the identity in our normalization with a self-contained proof,
because the reverse bounds imported later in this section depend on the precise constant.

\begin{lemma}[Constant for the Energy Kernel Representation]\label{lem:energy-constant}
Fix $d\ge 1$ and $H\in(0,1)$. Let $\alpha:=2H\in(0,2)$ and set $s=\frac{d+\alpha}{2}$ (equivalently $H=s-\frac d2$).
Consider
\[
k_H(x,y)=\|x\|^{\alpha}+\|y\|^{\alpha}-\|x-y\|^{\alpha}
\]
Then for all $\mu,\nu\in\mathcal P_{\alpha}(\R^d)$,
\begin{equation}\label{eq:energy-const}
\mathrm{MMD}_{k_H}^2(\mu,\nu)
=
\frac{1}{c(d,H)}
\int_{\R^d}
\frac{|\Bchi_\mu(u)-\Bchi_\nu(u)|^2}{\|u\|^{d+\alpha}}\ du,
\qquad
c(d,H)=\frac{\pi^{d/2}\Gamma(1-H)}{H\ 2^{2H}\Gamma(\frac d2+H)}
\end{equation}
In particular, since $2s=d+\alpha$,
\begin{equation}\label{eq:Ts2-is-MMD}
\mathrm{T}_{s,2}^2(\mu,\nu)
=
\int_{\R^d}
\frac{|\Bchi_\mu(u)-\Bchi_\nu(u)|^2}{\|u\|^{2s}}\ du
=
c(d,H)\ \mathrm{MMD}^2_{k_H}(\mu,\nu)
\end{equation}
\end{lemma}

\begin{remark}[Normalization conventions, and a check of $c(d,H)$]\label{rem:energy-normalization}
Three normalizations must be kept apart when comparing \eqref{eq:energy-const} with the kernel
literature, and we fix them here once and for all.
\begin{enumerate}
\item[(a)] \emph{Fourier convention.} Throughout, $\widehat\mu(u)=\int e^{i\langle u,x\rangle}\,d\mu(x)$
and all frequency integrals are against Lebesgue measure $du$. Replacing this by a unitary
convention rescales the spectral density, hence $c(d,H)$, by a power of $2\pi$. Every constant
below refers to the convention just stated.
\item[(b)] \emph{Kernel convention.} Our $k_H$ has no factor $\tfrac12$ and coincides
exactly with the Energy Kernel of \citet[Section~3.3]{modeste2024translation}. Note that
Modeste--Dombry write $\alpha$ for the \emph{order} parameter that we call $H$ (their order equals
half our kernel exponent $\alpha=2H$), so their admissible order range $(0,1)$ is our
$H\in(0,1)$. The fractional Brownian covariance
$\tfrac12\big(\|x\|^{2H}+\|y\|^{2H}-\|x-y\|^{2H}\big)$ of their Example~4 is half of $k_H$, so the
two associated squared MMDs differ by a factor $2$. Moreover, the spectral constant denoted
$c(d,H)$ in their Example~4 is a \emph{different} normalization from our $c(d,H)$ and the two must
not be interchanged. The reverse inequalities we import (Theorem~\ref{thm:reverse-holder} and
Corollary~\ref{cor:bounded-support}) are stated for the Section~3.3 kernel, so they transfer to
$\mathrm{T}_{s,2}$ through \eqref{eq:Ts2-is-MMD} with no further conversion.
\item[(c)] \emph{Order versus exponent.} Their order parameter is $H$, our kernel exponent is
$\alpha=2H$. See the dictionary preceding Theorem~\ref{thm:reverse-holder}.
\end{enumerate}
The value of $c(d,H)$ can be checked independently in the Brownian case $d=1$, $H=\tfrac12$
(so $\alpha=1$, $s=1$). There the formula gives
$c(1,\tfrac12)=\pi^{1/2}\Gamma(\tfrac12)/(\tfrac12\cdot2\cdot\Gamma(1))=\pi$, which agrees with the
direct evaluation $c(1,\tfrac12)=\int_{\R}\frac{1-\cos t}{t^2}\,dt=2\int_0^\infty\frac{1-\cos t}{t^2}\,dt=\pi$.
Substituting into \eqref{eq:energy-const} and using
$\int_{\R}u^{-2}|\Bchi_\mu(u)-\Bchi_\nu(u)|^2\,du=2\pi\int_{\R}|F_\mu-F_\nu|^2\,dx$ (Plancherel applied
to $\widehat{\mu-\nu}(u)=-iu\,\widehat{(F_\mu-F_\nu)}(u)$) yields
$\mathrm{MMD}^2_{k_{1/2}}(\mu,\nu)=2\int_{\R}|F_\mu-F_\nu|^2\,dx$, which is exactly the classical
Cram\'er/energy-distance identity $2\E|X-Y|-\E|X-X'|-\E|Y-Y'|=2\int_\R(F_\mu-F_\nu)^2\,dx$
\citep{sejdinovic2013equivalence}. The constant in \eqref{eq:energy-const} is therefore pinned down
by an independent classical computation in the case where one is available.
\end{remark}

The kernel representation gives one Hilbertian realization of $\mathrm{T}_{s,2}$. The same quadratic structure also admits a dual interpretation in Sobolev terms, inherited from the general Fourier--Lebesgue duality of Section~\ref{sec:dual-ipm}. In the quadratic case, this dual space becomes the homogeneous Sobolev space.

\begin{corollary}[Hilbert case as negative Sobolev IPM]\label{cor:hilbert-dual}
If $p=q=2$, then $\dot{\mathcal F}L^{2}_s$ identifies with $\dot H^{s}(\R^d)/\R$ as a vector space, with
\[
\|\phi\|_{\dot{\mathcal F}L_s^2}=(2\pi)^{d/2}\|\phi\|_{\dot H^s}
\]
Accordingly,
\[
\sup_{\|\phi\|_{\dot H^{s}}\le 1}\int_{\R^d}\phi\ d(\mu-\nu)
=
(2\pi)^{-d/2}\ \mathrm{T}_{s,2}(\mu,\nu)
\]
Consequently, the Hilbert case is simultaneously a negative-Sobolev IPM and a translation-invariant MMD.
\end{corollary}

With the Hilbertian structure now identified, we can turn to quantitative comparison results. In the quadratic case, the MMD representation allows us to import reverse Wasserstein bounds from the energy-kernel literature on regular model classes. Before stating the theorem, we fix the dictionary between our normalization and that of \citet{modeste2024translation}, since the two conventions differ by a factor of two in the exponent and confusing them changes the imported rates. Modeste--Dombry define, for an order parameter (which they denote $\alpha$) in $(0,1)$, the \emph{Energy Kernel of order} $\alpha_{\mathrm{MD}}$ as
\[
k^{\mathrm{MD}}_{\alpha_{\mathrm{MD}}}(x,y)=\|x\|^{2\alpha_{\mathrm{MD}}}+\|y\|^{2\alpha_{\mathrm{MD}}}-\|x-y\|^{2\alpha_{\mathrm{MD}}}
\]
\citep[Section~3.3]{modeste2024translation}. That is, their order parameter is \emph{half} the kernel exponent. Our kernel $k_H$ from Lemma~\ref{lem:energy-constant}, whose exponent is $\alpha=2H=2s-d$, therefore coincides with their Energy Kernel of order $\alpha_{\mathrm{MD}}=H=s-\frac d2$, and their admissible range $\alpha_{\mathrm{MD}}\in(0,1)$ corresponds exactly to the \emph{full} Hilbert finiteness window $s\in(\frac d2,\frac d2+1)$. With this dictionary, their quantitative reverse inequalities \citep[Propositions~16 and~17]{modeste2024translation} read as follows in our notation.

\begin{theorem}[Reverse H\"older Control on Uniform Moment Classes]\label{thm:reverse-holder}
Let $d\ge 1$. Choose $s\in(\frac d2,\frac d2+1)$ and set
\(\alpha:=2s-d\in(0,2)\) and \(H:=\frac\alpha2=s-\frac d2\in(0,1)\).
Fix $\gamma>1$ and $S>0$, and consider the moment class $\mathcal M_{\gamma,S}$.
Then there exist constants $C=C(d,H,\gamma,S)<\infty$ and a H\"older exponent
\[
\rho
=
\frac{\gamma-1}{\gamma\big(d+H+1\big)-H}
=
\frac{\gamma-1}{\gamma\big(d+\frac\alpha2+1\big)-\frac\alpha2}\in(0,1)
\]
such that for all $\mu,\nu\in \mathcal M_{\gamma,S}$,
\[
W_1(\mu,\nu)
\le
C\ \mathrm{T}_{s,2}^{\rho}(\mu,\nu)
\]
\end{theorem}

\begin{remark}[The full Hilbert window is covered]\label{rem:range-mismatch}
The finiteness window for $\mathrm T_{s,2}$ on $\mathcal P_2$ is $\frac d2<s<1+\frac d2$, and Theorem~\ref{thm:reverse-holder} covers this \emph{entire} window. The hypothesis of \citet[Proposition~17]{modeste2024translation} constrains their \emph{order} parameter, which equals half our kernel exponent $\alpha=2s-d\in(0,2)$, so no part of the window is excluded. Conflating the order with the exponent would artificially restrict the comparison to $\alpha\in(0,1)$. The dictionary above shows that no such restriction is needed.
\end{remark}

\begin{remark}[Normalization of Constants in the $p=2$ Reverse Bound]\label{rem:reverse-const-normalization}
Under the assumptions of Theorem~\ref{thm:reverse-holder}, \citet[Proposition~17]{modeste2024translation} provides the reverse bound in the form
\[
W_1(\mu,\nu)\le C_{\mathrm{MD}}\ \mathrm{MMD}_{k_H}^{\rho}(\mu,\nu),
\qquad (\mu,\nu\in\mathcal M_{\gamma,S})
\]
with $\rho$ as in Theorem~\ref{thm:reverse-holder} and an explicit constant $C_{\mathrm{MD}}=C_{\mathrm{MD}}(d,H,\gamma,S)$ \citep[Equation~37]{modeste2024translation}. By Lemma~\ref{lem:energy-constant},
\[
W_1(\mu,\nu)\le C_{\mathrm{MD}}\ c(d,H)^{-\rho/2}\ \mathrm{T}_{s,2}^{\rho}(\mu,\nu)\quad ,
\qquad (\mu,\nu\in\mathcal M_{\gamma,S})
\]
where $c(d,H)$ is the constant from Lemma~\ref{lem:energy-constant}. In particular, all constants in Theorem~\ref{thm:reverse-holder} are explicit.
\end{remark}

\begin{corollary}[Bounded-Support Specialization for \(p=2\)]\label{cor:bounded-support}
Let $\alpha\in(0,2)$. Set $s=(d+\alpha)/2$ and $H=\alpha/2\in(0,1)$.
Assume $\mu,\nu$ are supported in a ball $B(0,K)$.
Then there exist constants $\widetilde C=\widetilde C(d,H,K) <\infty$ and an exponent
\[
\widetilde{\rho} = \frac{1}{d+1+H}=\frac{1}{d+1+\frac\alpha2}
\]
such that
\[
W_1(\mu,\nu)\le \widetilde{C}\ \mathrm{T}_{s,2}^{\widetilde{\rho}}(\mu,\nu)
\]
Moreover, if $\mu,\nu$ are supported in a set of diameter $\mathrm{diam}$, then for any $r\ge 1$,
\[
W_r(\mu,\nu)
\le \mathrm{diam}^{\ 1-\frac1r}\ W_1(\mu,\nu)^{1/r}
\le \dbtilde{C}\ \mathrm{T}_{s,2}^{\dbtilde{\rho}}(\mu,\nu)\quad ,
\qquad
\dbtilde{\rho}=\frac{1}{r\big(d+1+\frac\alpha2\big)}
\]
\end{corollary}

These reverse bounds come from the Hilbertian route through the energy-kernel MMD. It is therefore worth comparing them with the general-$p$ bounded-support modulus of Theorem~\ref{thm:explicit-reverse-bounded}.

\begin{remark}\label{rem:compare-bounded-moduli}
An explicit bounded-support reverse modulus holds for general \(p\)
(Theorem~\ref{thm:explicit-reverse-bounded}). At \(p=2\), its $W_1$-exponent is
\[
\frac{1}{s+\frac d2}=\frac{1}{d+\frac\alpha2},
\]
and the Hausdorff--Young refinement of Proposition~\ref{prop:reverse-bounded-HY} improves this
further to
\[
\frac1s=\frac{1}{\frac d2+\frac\alpha2}
\]
whenever $s\ge1$ (automatic for $d\ge2$). Corollary~\ref{cor:bounded-support} gives the exponent
$\frac{1}{d+1+\frac\alpha2}$ imported from the energy-kernel route. Since
\[
\frac d2+\frac\alpha2\ <\ d+\frac\alpha2\ <\ d+1+\frac\alpha2,
\]
the direct Fourier--Lebesgue duality argument is never worse than the imported bound throughout
the window $\alpha\in(0,2)$, and wherever $s\ge1$ holds --- for every $d\ge2$, and for $d=1$ with
$\alpha\in[1,2)$ --- the Hausdorff--Young refinement strictly improves both in the
small-$\mathrm T_{s,2}$ regime. The statements come from different mechanisms and constants, and
none of the exponents is asserted to be sharp.
\end{remark}

Both routes lead to the same qualitative conclusion. On regular classes, $\mathrm{T}_{s,2}$ and Wasserstein distance induce the same convergence.

\begin{corollary}[Equivalent Convergence on Regular Classes]\label{cor:top_equiv}
Fix $d\ge 1$, $s\in(\frac d2,\frac d2+1)$, $\gamma>1$, and $S>0$.
\begin{enumerate}
\item On $\mathcal M_{\gamma,S}$, the metrics $\mathrm{T}_{s,2}$ and $W_1$ generate the same notion of convergence.
For any sequence $\{\mu_n\}\subset\mathcal M_{\gamma,S}$ and $\mu\in\mathcal M_{\gamma,S}$,
\[
\mathrm{T}_{s,2}(\mu_n,\mu)\to 0 \quad \Longleftrightarrow \quad W_1(\mu_n,\mu)\to 0
\]
\item On any bounded-diameter class
\[
\{\mu:\operatorname{supp}(\mu)\subseteq A,\ \operatorname{diam}(A)\le D\},
\qquad D>0
\]
the metrics $\mathrm{T}_{s,2}$ and $W_r$ generate the same notion of convergence for every $r\ge 1$.
\end{enumerate}
\end{corollary}
These $p=2$ consequences fit the broader pattern in the kernel literature. Global one-sided control under weak assumptions, reverse bounds only on normalized or regular model classes.
\begin{remark}[Relation to the MMD Literature]\label{rem:mmd-literature}
Theorem~\ref{thm:upper-holder} and Lemma~\ref{lem:scaling} match the usual kernel/Wasserstein
pattern. Global upper control of an IPM or MMD by Wasserstein under minimal assumptions, and
reverse bounds only on regular model sets.
For translation-invariant PSD kernels, Vayer--Gribonval \citep[Theorem~15]{vayer2023controlling} provide reverse H\"older-type controls of $W_p$ by MMD on Sobolev--moment classes.
For the energy-kernel MMD, Modeste--Dombry \citep[Propositions~16 and~17]{modeste2024translation} provide quantitative reverse inequalities on bounded-support and moment classes, which yield Theorem~\ref{thm:reverse-holder} and Corollary~\ref{cor:bounded-support}. As they note, their Proposition~16 is similar in spirit to the Fourier-analytic bounded-support comparison of \citet{auricchio2020equivalence}.
\end{remark}

\subsection{An Exact Finite-Sample Identity}\label{sec:empirical-mean}
The kernel identification yields one statistical statement essentially for free. Let
$X_1,\dots,X_n$ be i.i.d.\ with law $\mu$ and let
$\hat\mu_n:=n^{-1}\sum_{i=1}^n\delta_{X_i}$ denote the empirical measure.

\begin{proposition}[Exact mean of the empirical discrepancy at $p=2$]\label{prop:empirical-mean}
Fix $d\ge1$ and $s\in(\frac d2,\frac d2+1)$, and set $H:=s-\frac d2\in(0,1)$. If
$\mu\in\mathcal P_{2H}(\R^d)$ and $X,X'\sim\mu$ are independent, then for every $n\ge1$,
\[
\E\,\mathrm T_{s,2}^2(\hat\mu_n,\mu)
= \frac1n\int_{\R^d}\frac{1-|\Bchi_\mu(u)|^2}{\|u\|^{2s}}\,du
= \frac{c(d,H)}{n}\,\E\|X-X'\|^{2H}
\]
with $c(d,H)$ the constant of Lemma~\ref{lem:energy-constant}. In particular
$\mathrm T_{s,2}(\hat\mu_n,\mu)=O_P(n^{-1/2})$, with an explicit constant. Moreover, the two
identities hold in $[0,\infty]$ for arbitrary $\mu\in\mathcal P(\R^d)$, and the common value is
finite if and only if $\mu\in\mathcal P_{2H}(\R^d)$.
\end{proposition}

Beyond the mean, the full two-sample apparatus of the energy distance --- permutation tests,
V-/U-statistic limit theory, consistency against fixed alternatives --- transfers to
$\mathrm T_{s,2}$ verbatim through \eqref{eq:Ts2-is-MMD}
\citep{szekely2013energy,gretton2012kernel,sejdinovic2013equivalence}. What is genuinely open is
the corresponding theory for $p\neq2$. See Section~\ref{sec:further}.

\section{Numerical Experiments}\label{sec:num_exp}
We provide two sets of numerical diagnostics. First, we examine the sensitivity of an importance-sampling approximation of $\mathrm{T}_{s,p}$ to frequency truncation and compare it with the special quadratic representation when $p=2$. Second, we use the DOTmark image benchmark as a small, illustrative comparison with sliced Wasserstein. After normalization, the induced measures are compactly supported on $[0,1]^2$, placing them in the scale-controlled regime considered by the theory. The computations probe numerical behavior. They neither establish the population results nor supply finite-sample guarantees for the empirical estimator.

\subsection{Synthetic Validation}\label{subsec:synth_Val}

Throughout, we approximate $\mathrm{T}_{s,p}$ by importance sampling in the Fourier domain with truncation. Frequencies are sampled on an annulus $\{r_{\min}\le \|u\|\le r_{\max}\}$ with density proportional to $\|u\|^{-ps}$, and we use
\[
\mathrm{T}_{s,p}(\mu,\nu)^p
\ \approx\ 
Z \cdot \frac{1}{M}\sum_{j=1}^M \big|\Bchi_\mu(u_j)-\Bchi_\nu(u_j)\big|^p
\]
where $Z$ is the analytic normalizing constant of the truncated weight. Since $|\Bchi_\mu(u)-\Bchi_\nu(u)|\le 2$ pointwise, each summand lies in $[0,2^p]$, so for a fixed pair $(\mu,\nu)$ the Monte Carlo estimator of the truncated functional has variance at most $Z^2\,4^{p}/M$. The frequency-sampling error is therefore controlled at the parametric rate $M^{-1/2}$ for fixed truncation. (The prefactor $Z^2\,4^p$ --- essentially the truncated weight mass --- can be large, especially in high dimension, so the parametric rate does not by itself guarantee small error at moderate $M$. The diagnostics below probe the actual accuracy.) When dilation families $(\mu^\lambda,\nu^\lambda)$ are considered, the annulus is co-dilated to $(r_{\min}/\lambda,\,r_{\max}/\lambda)$, so the \emph{truncated} functional obeys the homogeneity of Lemma~\ref{lem:scaling} exactly. Such runs validate the implementation rather than the untruncated identity. In one dimension, Wasserstein distances are computed exactly by sorting. For $d\ge 2$ we use sliced Wasserstein \citep{rabin2012wasserstein,peyre2019computational} as the primary transport baseline and Sinkhorn where feasible. All \(d\ge2\) comparisons in this section are therefore with sliced \(W_p\), whereas the theorems concern \(W_p\) itself. Dimension-dependent comparison results between sliced and ordinary Wasserstein distances do not turn the plots below into numerical verification of our bounds.

\paragraph{Configurations.}
For reproducibility we record the exact settings of the three synthetic studies. The same values
appear in the task metadata shipped with the code supplement.
\emph{Cutoff sensitivity} (Figure~\ref{fig:exp1}). Here $\mu$ and $\nu$ are unit-covariance
Gaussians with means $0$ and $0.8\,e_1$, each represented by $n=20{,}000$ samples drawn with seed
$0$, and the estimator uses $M=8{,}192$ frequencies. The fixed cutoffs are $r_{\min}=10^{-3}$ and
$r_{\max}=80$, while the varied cutoff runs over the grids shown on the horizontal axes.

\emph{Translation family} (Table~\ref{tab:exp3_translation}). The base measure is a fixed draw of
a $20$-component Gaussian mixture, with means at scale $1$ and covariances with floor $0.15$ and
spread $0.6$, sampled once with $n=30{,}000$ points in $d=2$ and $n=50{,}000$ in $d=16$. The
second sample is the same point set translated in the first coordinate by a shift drawn uniformly
from $[0.02,3]$, and we average over $K=300$ such shifts in $d=2$ and $K=250$ in $d=16$. The
estimator uses $M=8{,}192$ frequencies in $d=2$ and between $16{,}384$ and $32{,}768$ in $d=16$,
and sliced Wasserstein uses $8{,}192$ projections in $d=2$ and $16{,}384$ in $d=16$. Frequencies
are sampled on the annulus $[10^{-3},8]$ for $p=2$ and $[10^{-3},6]$ for $p=5$, with the lower
endpoint reduced to $5\cdot10^{-4}$ in $d=16$.

\emph{Quadratic identity} (Table~\ref{tab:exp5_identity}, Figure~\ref{fig:exp5}). We draw $K=300$
independent pairs of Gaussians in $d=2$ and $K=250$ in $d=16$, with a common covariance scale
drawn uniformly from $[0.4,1.3]$ and a mean shift drawn uniformly from $[0,2]$ in the first
coordinate. Each measure is represented by $n=30{,}000$ samples, the estimator uses $M=8{,}192$
frequencies in $d=2$ and $16{,}384$ in $d=16$, and the sampling annuli are as in the translation
study.

\subsubsection{Estimator Stability and a Quadratic Consistency Check}\label{subsec:stability}
We use cutoff sensitivity, a structured translation family, and the \(p=2\) identity as three separate implementation checks.

\subsubsection{Cutoff Sensitivity}
Because our Monte Carlo approximation evaluates the Toscani functional only over the annulus $\{u: r_{\min}\le \|u\|\le r_{\max}\}$, the experiment in Figure~\ref{fig:exp1} should be interpreted as a \emph{truncation-sensitivity diagnostic} for the two endpoint regions of the defining integral, rather than as a direct numerical verification of the full finiteness theorem. 

The theoretical threshold $\tfrac{d}{p}<s<1+\tfrac{d}{p}$ distinguishes which endpoint is potentially nonintegrable. When $s<d/p$, instability is expected to arise from the high-frequency tail and hence should be reflected primarily in sensitivity to the upper cutoff $r_{\max}$. When $s>1+d/p$, the critical contribution comes from low frequencies near the origin and should therefore appear primarily through sensitivity to the lower cutoff $r_{\min}$. When $\tfrac{d}{p}<s<1+\tfrac{d}{p}$, both endpoints are integrable for the population quantity. 

Accordingly, the relevant numerical signature outside the admissible window is not simultaneous instability in both cutoff directions, but rather instability at the endpoint predicted by the theorem. At the same time, since Figure~\ref{fig:exp1} is computed from a truncated empirical estimator on a fixed sampled pair, residual variation for interior values of $s$ should not be over-interpreted. Such variation may reflect the fixed opposite cutoff, truncation bias, and finite-sample Monte Carlo error, rather than genuine endpoint divergence. For this reason, we use these plots only to confirm that the observed sensitivity is concentrated at the theoretically relevant endpoint, and rely on the translation-family and DOTmark comparisons below for sharper quantitative diagnostics.

\begin{figure}[tbp]
\centering
\subfloat[Origin cutoff, $d=2$, $p=2$, $r_{\max}=80$]{
    \includegraphics[width=0.48\textwidth]{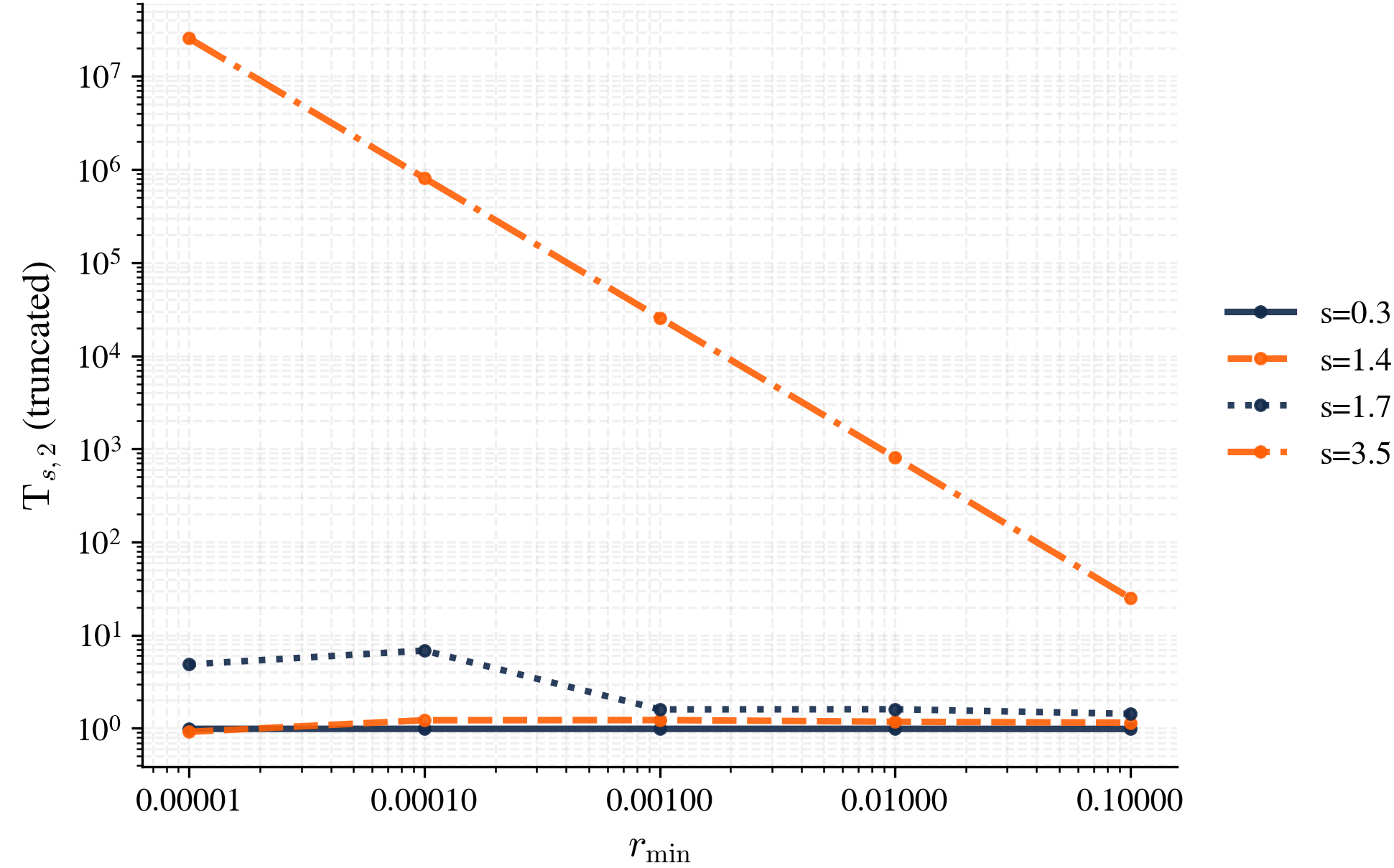}
}
\hfill
\subfloat[Infinity cutoff, $d=2$, $p=2$, $r_{\min}=0.001$]{
    \includegraphics[width=0.48\textwidth]{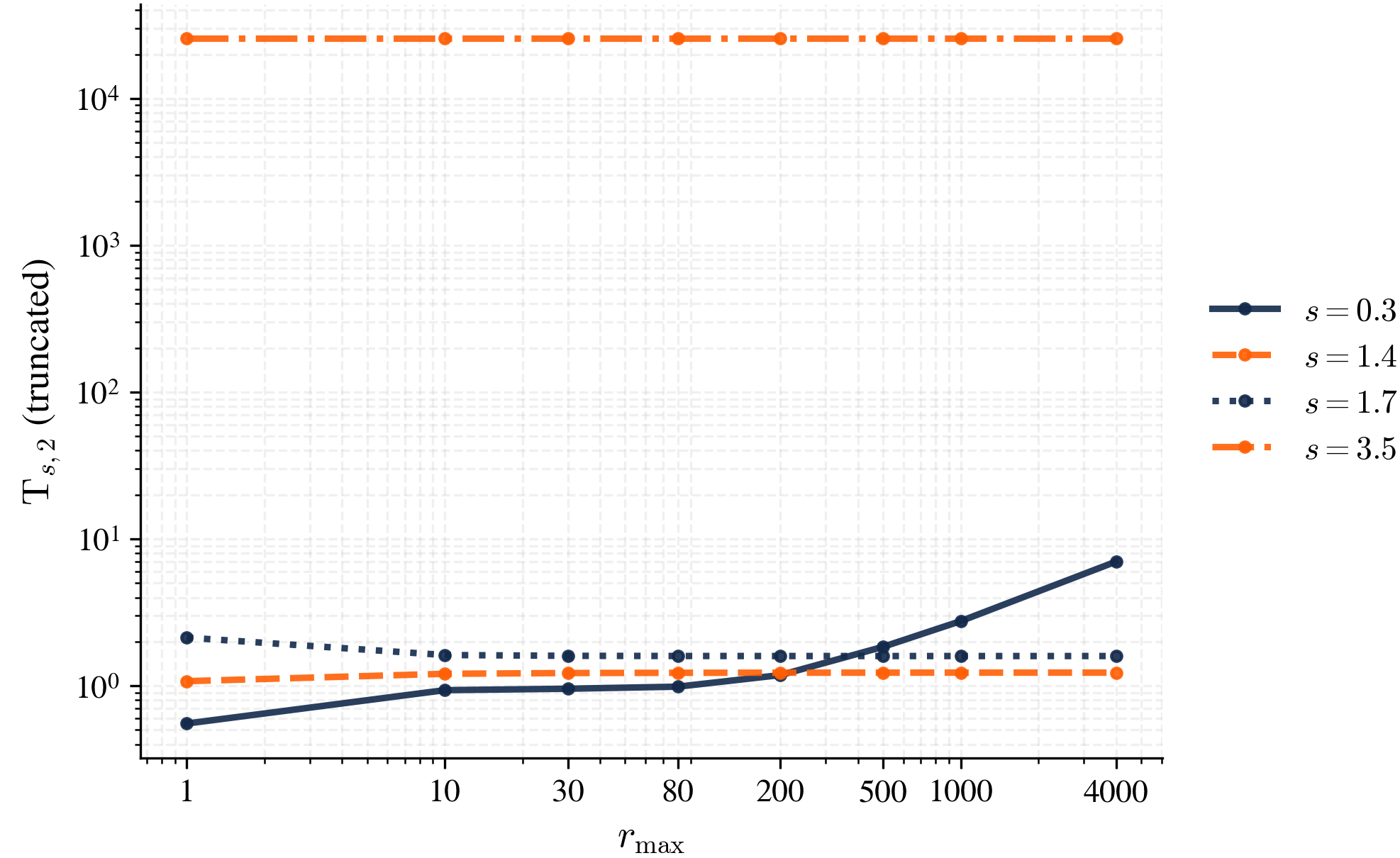}
}\\[0.5ex]

\subfloat[Origin cutoff, $d=2$, $p=5$, $r_{\max}=80$]{
    \includegraphics[width=0.48\textwidth]{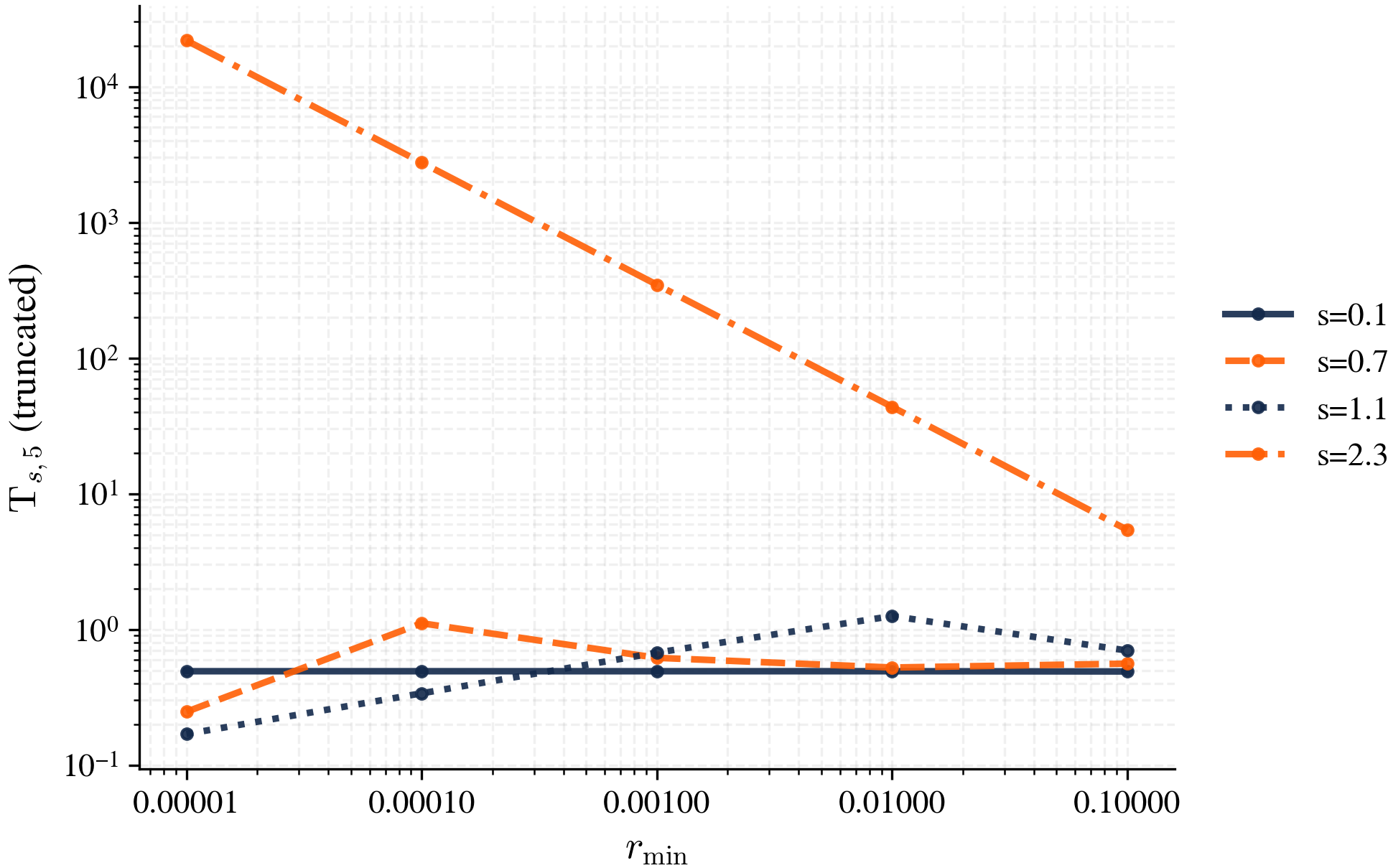}
}
\hfill
\subfloat[Infinity cutoff, $d=2$, $p=5$, $r_{\min}=0.001$]{
    \includegraphics[width=0.48\textwidth]{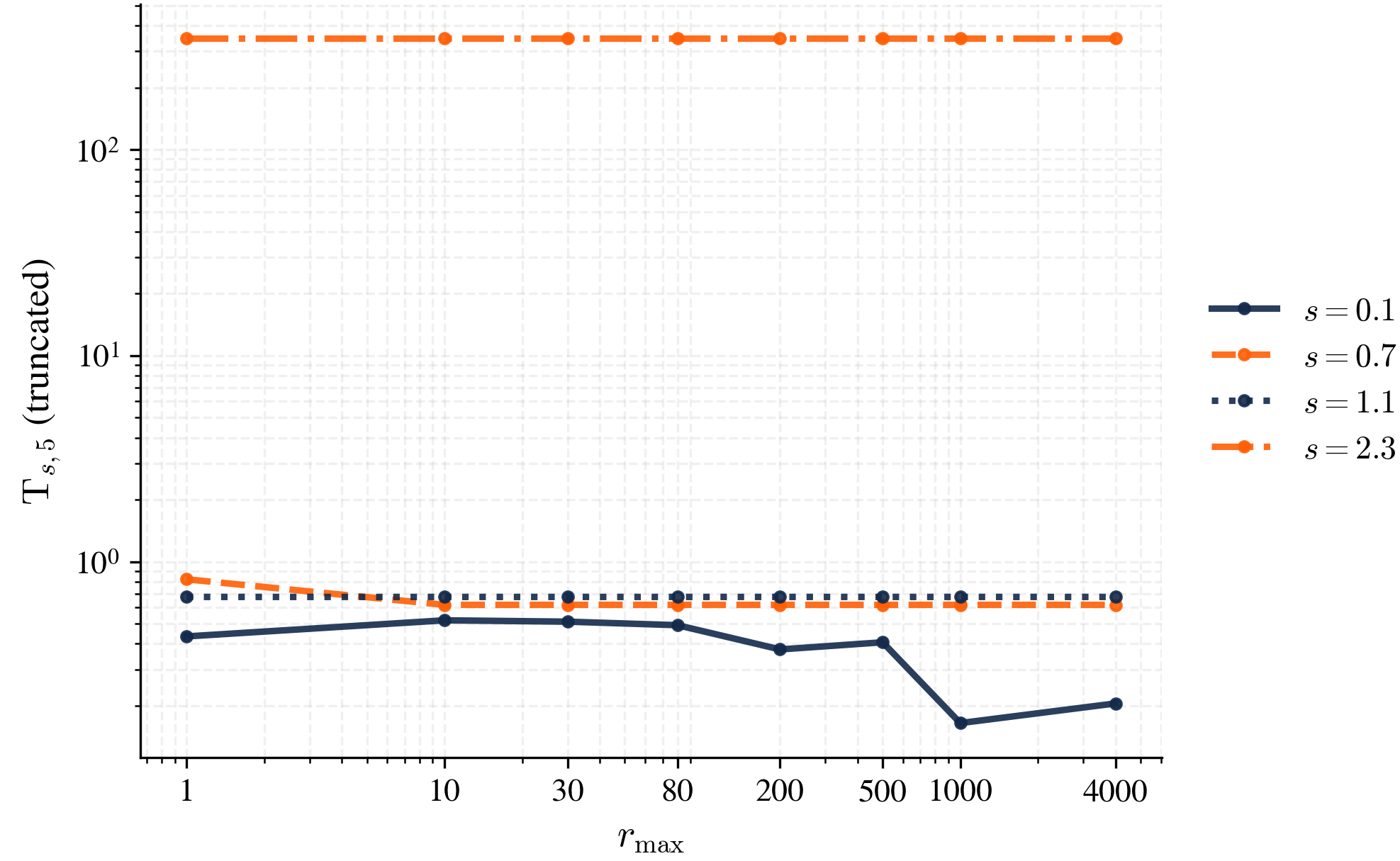}
}\\[0.5ex]

\subfloat[Origin cutoff, $d=16$, $p=2$, $r_{\max}=80$]{
    \includegraphics[width=0.48\textwidth]{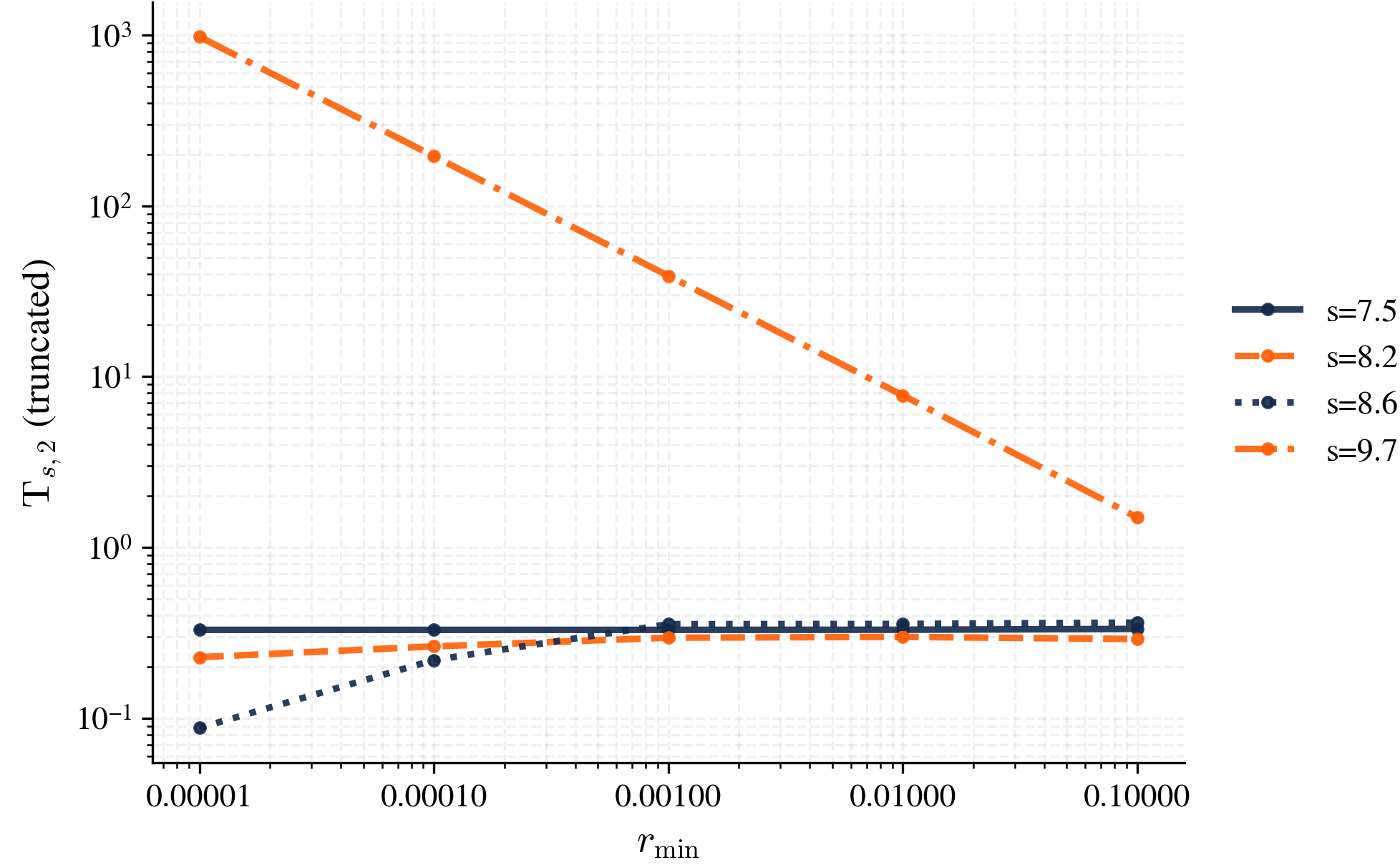}
}
\hfill
\subfloat[Infinity cutoff, $d=16$, $p=2$, $r_{\min}=0.001$]{
    \includegraphics[width=0.48\textwidth]{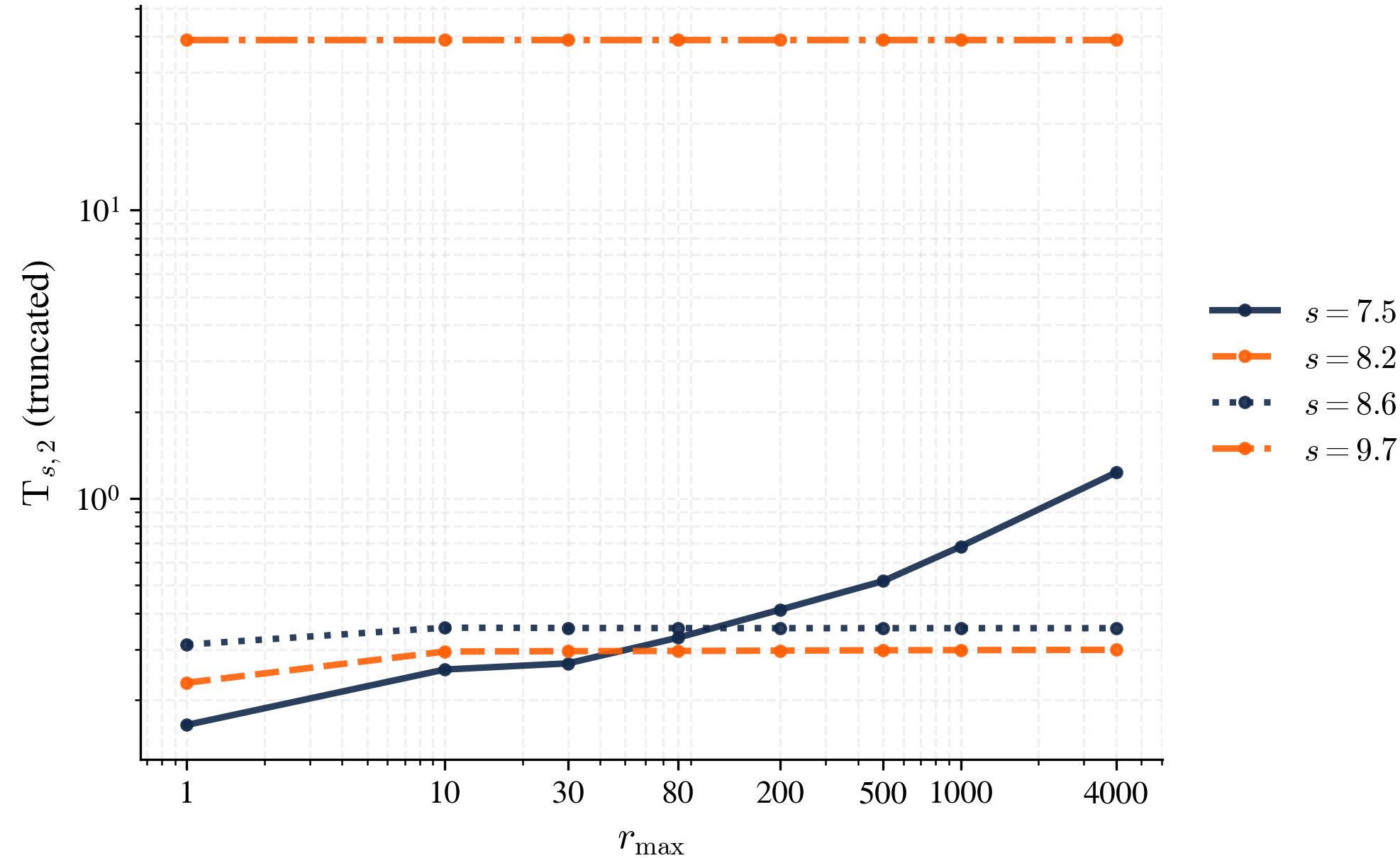}
}
\caption{Sensitivity of $\widehat{\mathrm{T}}_{s,p}$ to the truncation cutoffs. The origin
cutoff $r_{\min}$ (left column) and the infinity cutoff $r_{\max}$ (right column), for several
values of $s$ relative to the admissible window. Instability concentrates at the endpoint
predicted by Proposition~\ref{prop:window-sharp}. The lower cutoff matters for $s$ at or above
$1+d/p$, the upper cutoff for $s$ at or below $d/p$, while interior values of $s$ are stable
against both.}
\label{fig:exp1}
\end{figure}

\subsubsection{Translation-Family Comparison with Sliced Wasserstein}
We next study a structured one-parameter family of distribution pairs obtained by fixing a base sample and translating one coordinate of the second distribution. In this setting, $\widehat{\mathrm{T}}_{s,p}$ exhibits an almost linear log-log relation with sliced Wasserstein in both cases considered. For $(d,p,s)=(2,2,1.6)$, the fitted log-log slope of $\widehat{\mathrm{T}}_{s,p}$ versus ${}^{\text{(sliced)}}\mathrm{W}_2$ is $0.994$, while for $(2,5,1.0)$ the corresponding slope versus ${}^{\text{(sliced)}}\mathrm{W}_5$  is $0.993$. By contrast, when plotted against the quantity $\max\{{}^{\text{(sliced)}}\mathrm{W}_p^\alpha,{}^{\text{(sliced)}}\mathrm{W}_p\}$ with $\alpha=s-d/p=0.6$, the fitted slopes are $1.443$ and $1.390$, respectively, and the normalized ratio $\widehat{\mathrm{T}}_{s,p}/\max\{{}^{\text{(sliced)}}\mathrm{W}_p^\alpha,{}^{\text{(sliced)}}\mathrm{W}_p\}$ is not flat. Thus, for this translation family, the numerics are best interpreted as showing that $\widehat{\mathrm{T}}_{s,p}$ tracks sliced Wasserstein approximately linearly, rather than as providing near-constant calibration under the normalization of Corollary~\ref{cor:upper-holder-max}. The concentration is tighter for $p=2$ than for $p=5$.

The observed unit slope is expected rather than anomalous, and it is worth saying why, since it may otherwise appear to contradict the H\"older exponent $s-\frac dp<1$ of Theorem~\ref{thm:upper-holder}. Translations are the one family on which the exponent in \eqref{eq:upper-holder} is far from attained. If $\nu_t:=\mu*\delta_{te_1}$, then $\Bchi_{\nu_t}(u)=e^{itu_1}\Bchi_\mu(u)$, so
\[
\big|\Bchi_{\nu_t}(u)-\Bchi_\mu(u)\big|
=
\big|\Bchi_\mu(u)\big|\ \big|e^{itu_1}-1\big|
\le
t\,|u_1|\,\big|\Bchi_\mu(u)\big|
\]
and the extra factor $|\Bchi_\mu(u)|$ suppresses exactly the high frequencies whose contribution produces the H\"older loss in Theorem~\ref{thm:upper-holder}. Whenever $u\mapsto\|u\|^{1-s}|\Bchi_\mu(u)|$ lies in $L^p$, dominated convergence gives $\mathrm{T}_{s,p}(\nu_t,\mu)=t\,\kappa_\mu(1+o(1))$ as $t\downarrow0$, with $\kappa_\mu=\big\|\,|u_1|\,\|u\|^{-s}\Bchi_\mu\big\|_{L^p}$, while $W_p(\nu_t,\mu)=t$ exactly (the translation coupling gives ``$\le$'', and testing $W_1$ against $\varphi(x)=x_1$ gives ``$\ge$''. The sliced baseline actually plotted equals $t$ times a fixed constant depending only on $d$ and $p$, so it is linear in $t$ as well). The slope-one behavior in Table~\ref{tab:exp3_translation} is therefore the correct prediction for smooth base measures, and Theorem~\ref{thm:upper-holder} remains a genuine upper bound whose exponent is saturated by Dirac dilations rather than by translations of a fixed smooth measure. (Strictly speaking, the plotted quantities are truncated empirical estimators evaluated on a fixed base sample rather than the population functionals in this display. The observed slopes indicate that the estimator inherits the population behavior at the configurations recorded above.)

\begin{table}[tbp]
\centering
\caption{Translation-family comparison of $\widehat{\mathrm{T}}_{s,p}$ with sliced Wasserstein.}
\label{tab:exp3_translation}
\smallskip
\renewcommand{\arraystretch}{1.12}
\begin{tabular}{cccccccc}
\toprule\toprule
$d$ & $p$ & $s$ & $\alpha$ & slope vs.\ ${}^{\text{(sliced)}}\mathrm{W}_p$ & RMSE & slope vs.\ control & RMSE \\
\midrule
2  & 2 & 1.6 & 0.6 & 0.994 & 0.124 & 1.443 & 0.164 \\
2  & 5 & 1.0 & 0.6 & 0.993 & 0.347 & 1.390 & 0.374 \\
16 & 2 & 8.6 & 0.6 & 0.987 & 0.153 & 1.645 & 0.153 \\
16 & 5 & 3.8 & 0.6 & 1.013 & 0.327 & 1.687 & 0.327 \\
\bottomrule\bottomrule
\end{tabular}
\end{table}
\ifincludeproofs
Figure~\ref{fig:exp3_translation} shows the corresponding multi-panel plots.
\begin{figure}[!htb]
\centering
\setlength{\tabcolsep}{2pt}
\renewcommand{\arraystretch}{1.0}
\begin{tabular}{ccc}
\textbf{\(\widehat{\mathrm{T}}_{s,p}\) vs.\ control} &
\textbf{\(\widehat{\mathrm{T}}_{s,p}\) vs.\ sliced \(W_p\)} &
\textbf{Ratio check} \\
\includegraphics[width=0.285\textwidth]{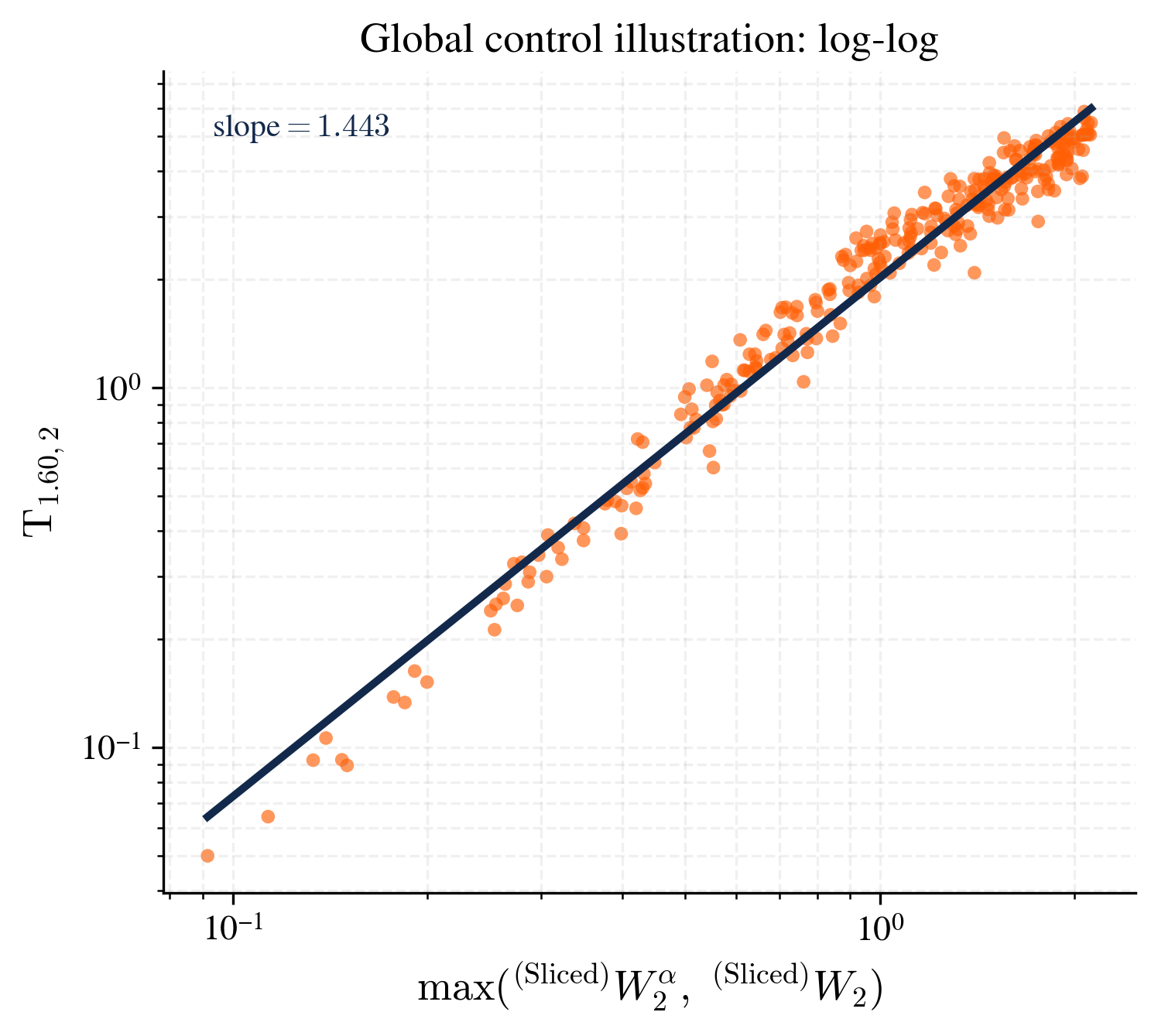} &
\includegraphics[width=0.285\textwidth]{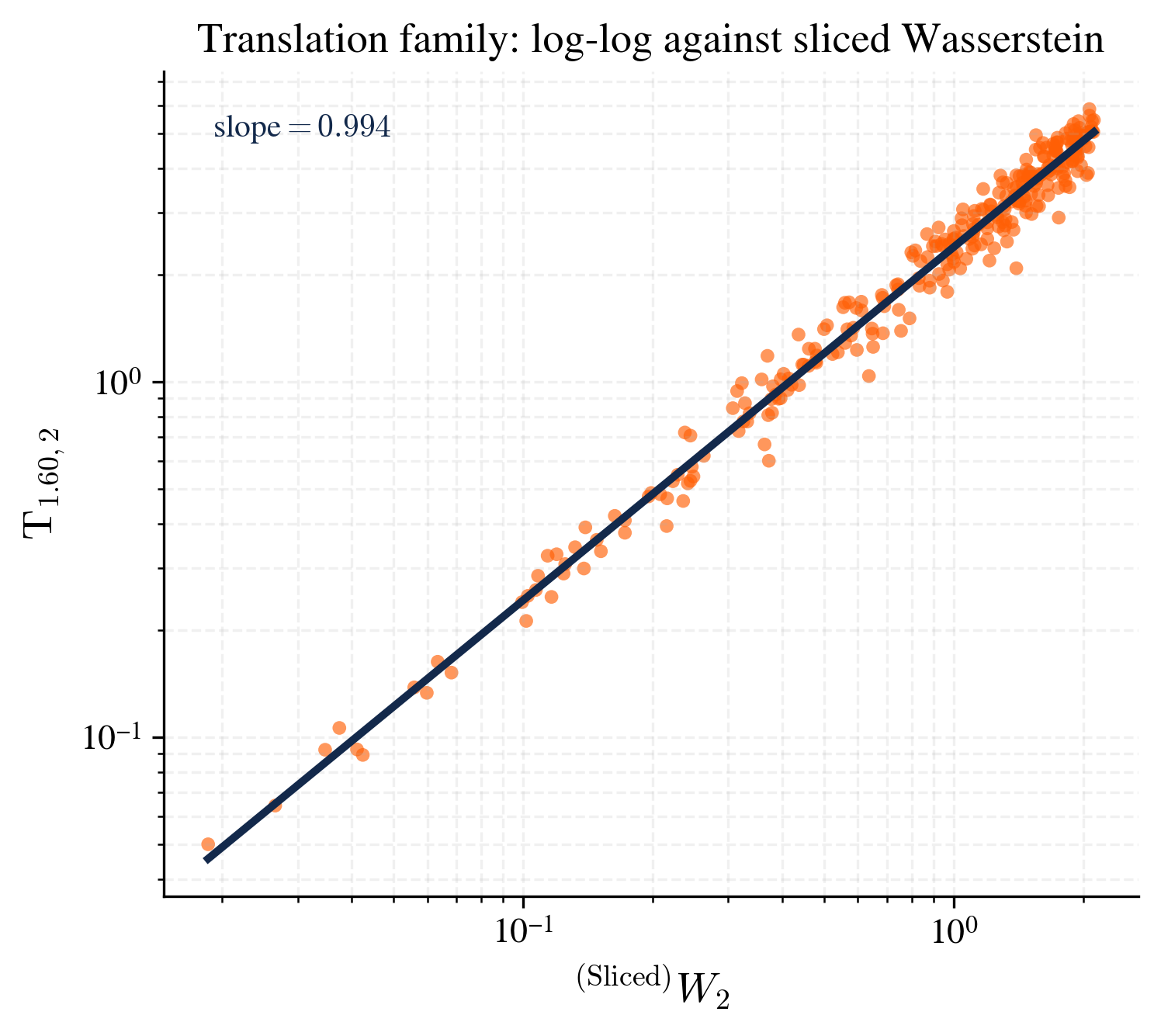} &
\includegraphics[width=0.285\textwidth]{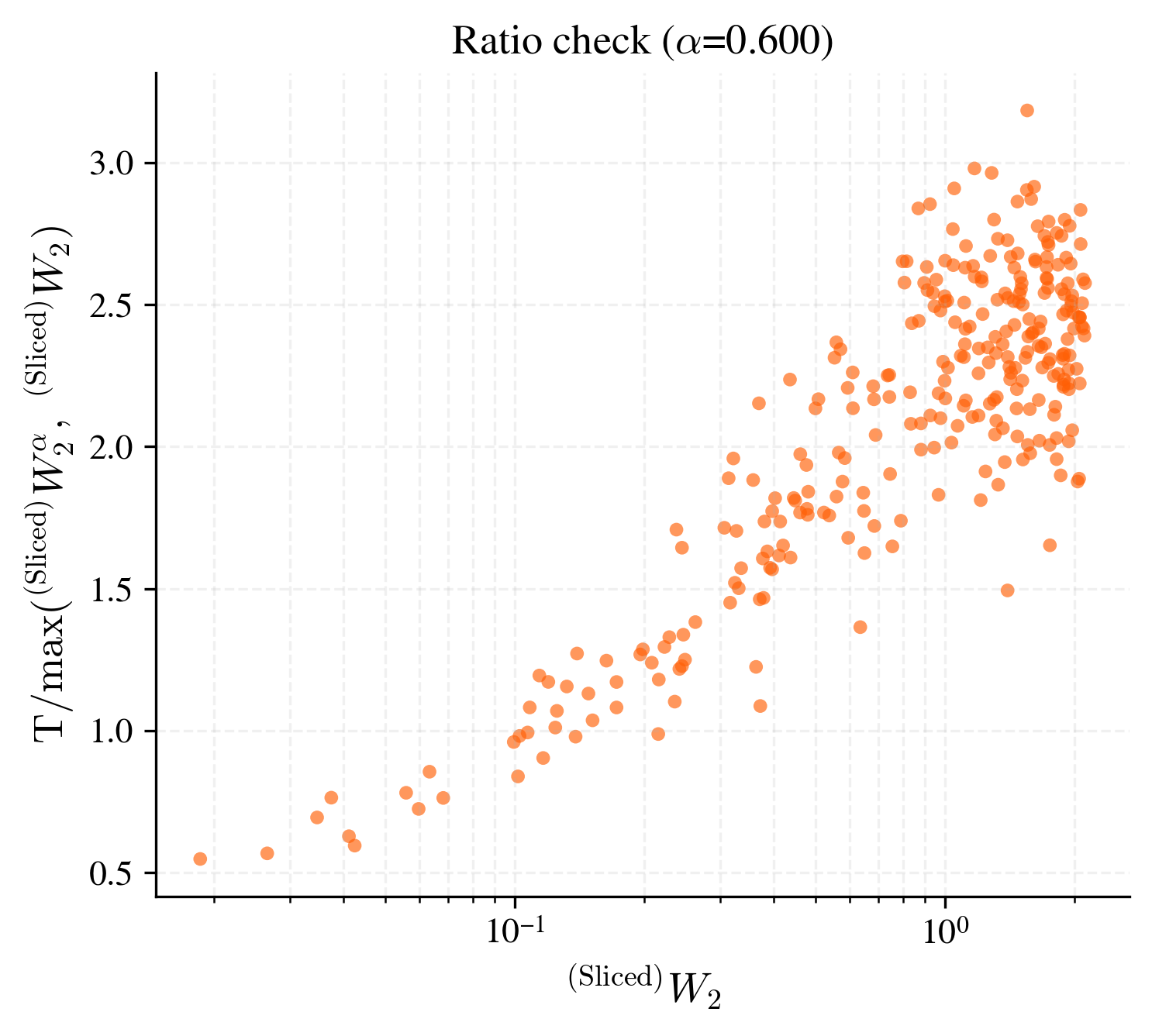} \\
\multicolumn{3}{c}{(a)\quad \(s=1.6,\ d=2,\ p=2\)}\\
\includegraphics[width=0.285\textwidth]{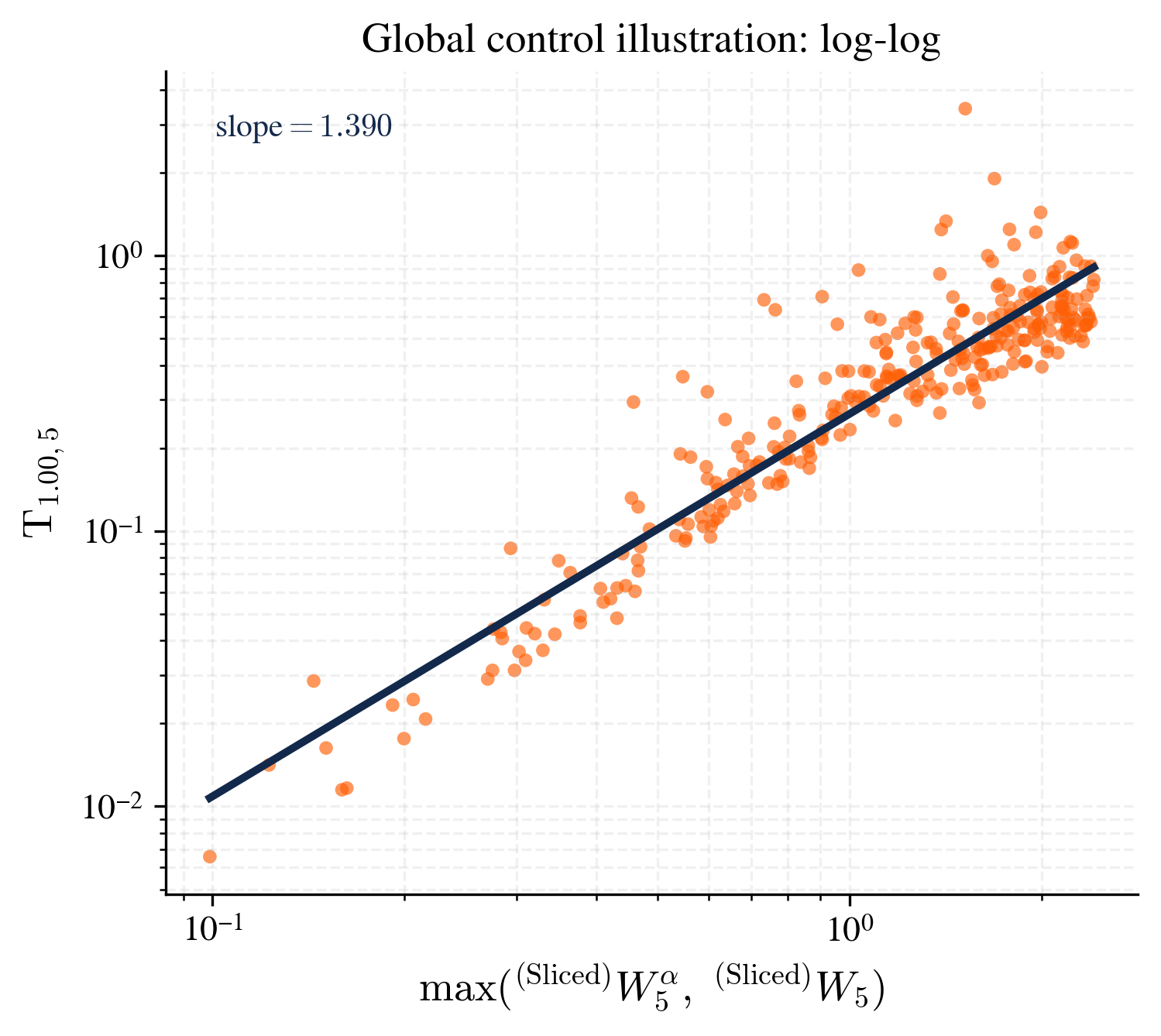} &
\includegraphics[width=0.285\textwidth]{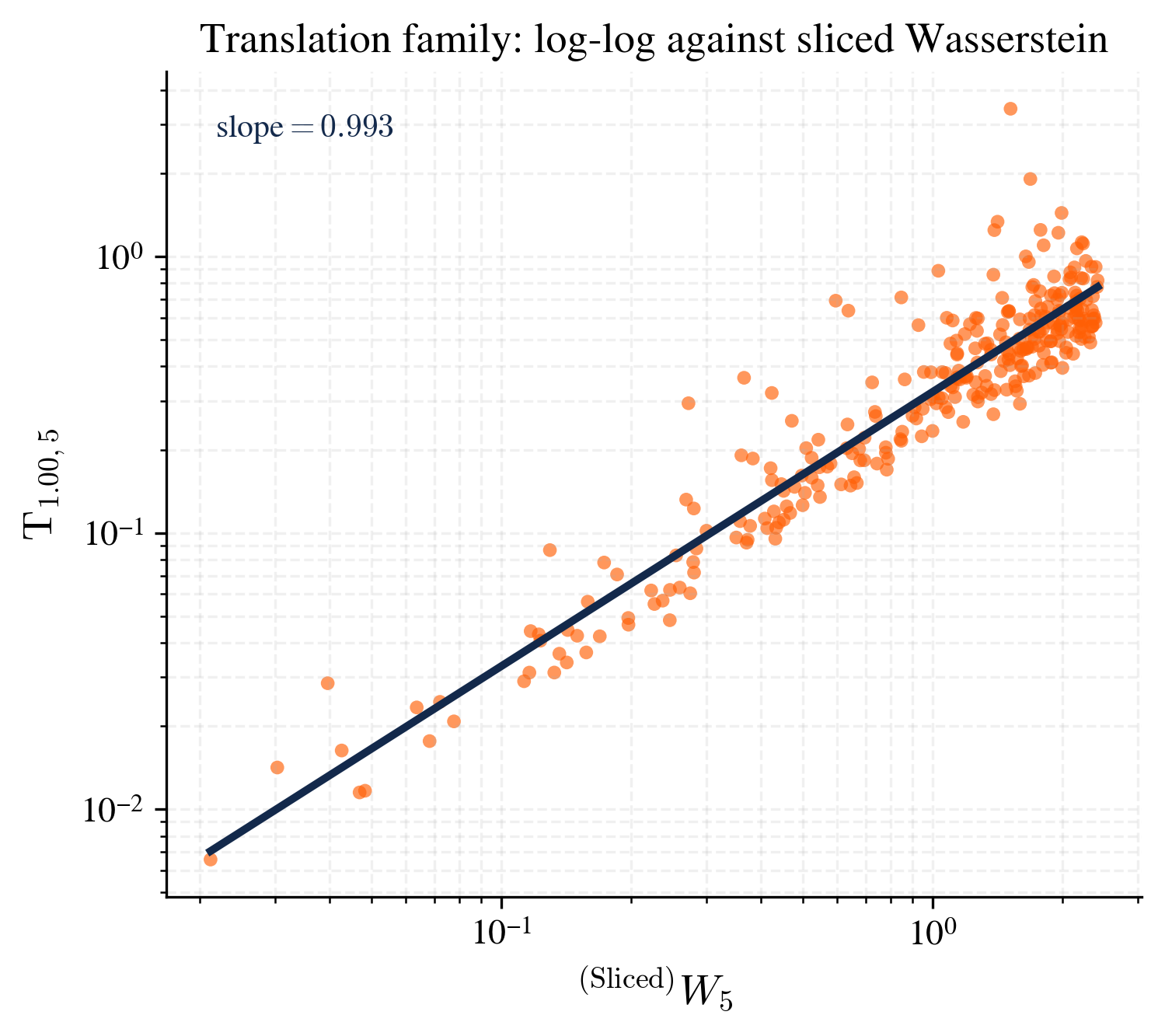} &
\includegraphics[width=0.285\textwidth]{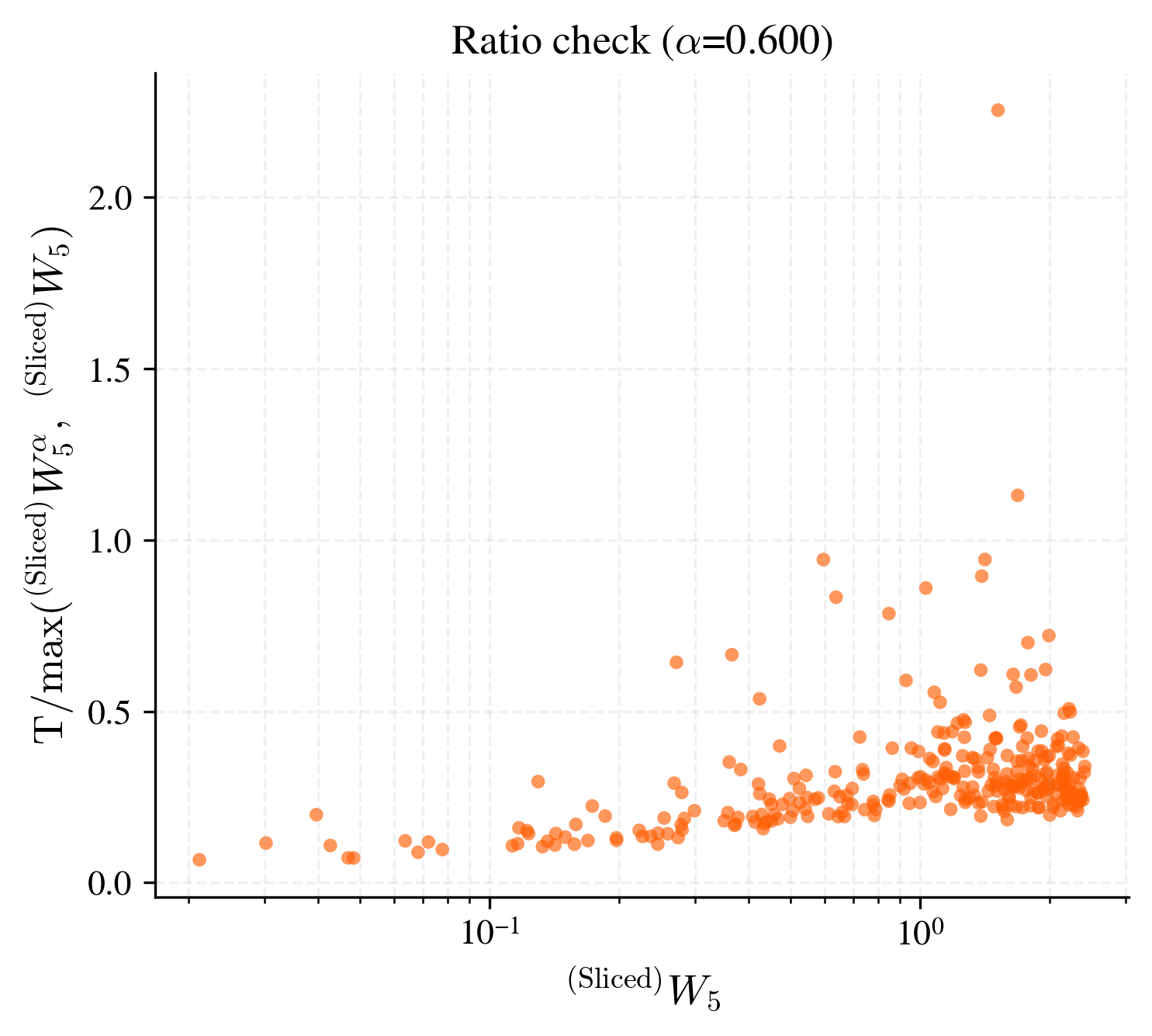} \\
\multicolumn{3}{c}{(b)\quad \(s=1,\ d=2,\ p=5\)}\\
\includegraphics[width=0.285\textwidth]{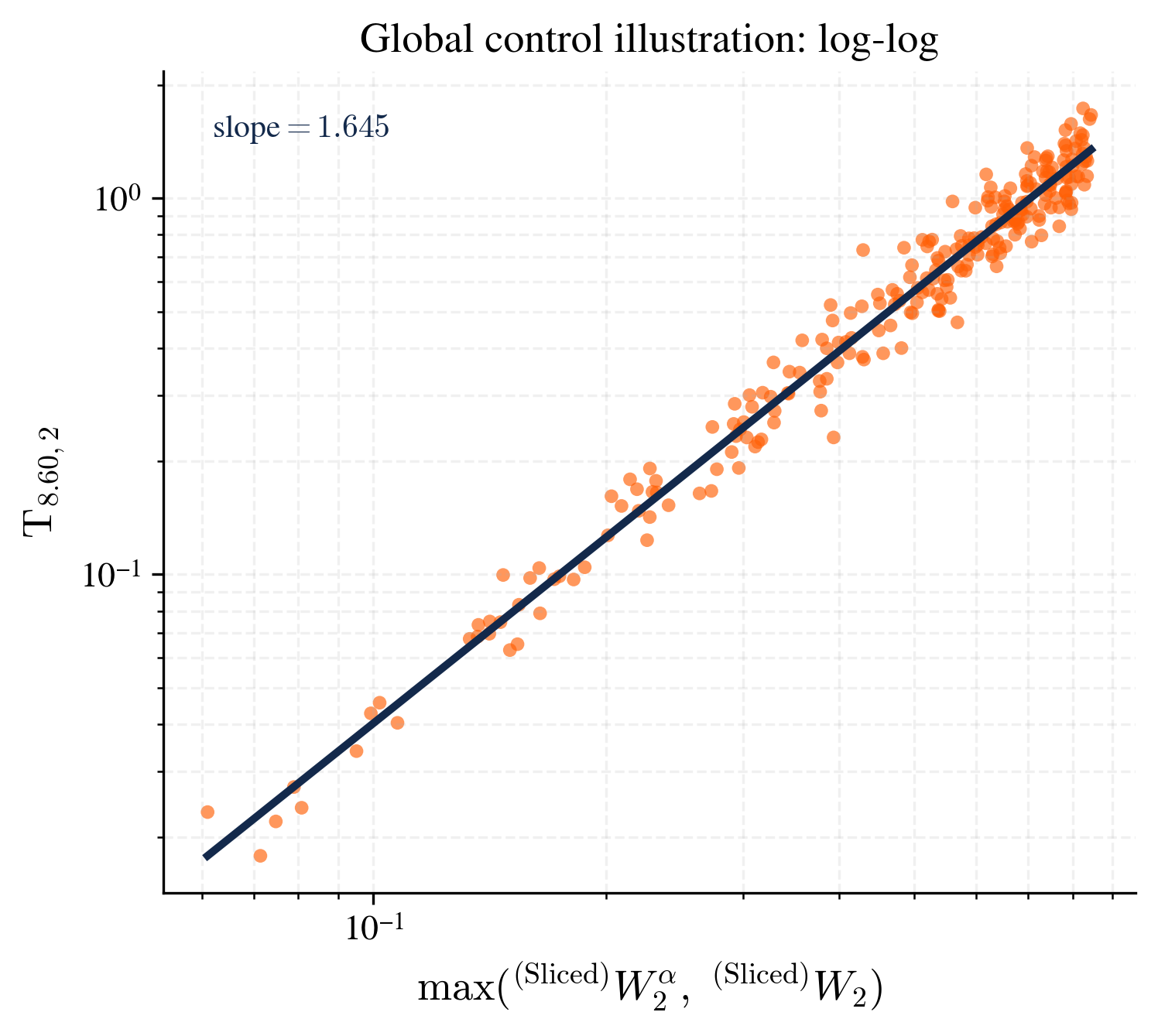} &
\includegraphics[width=0.285\textwidth]{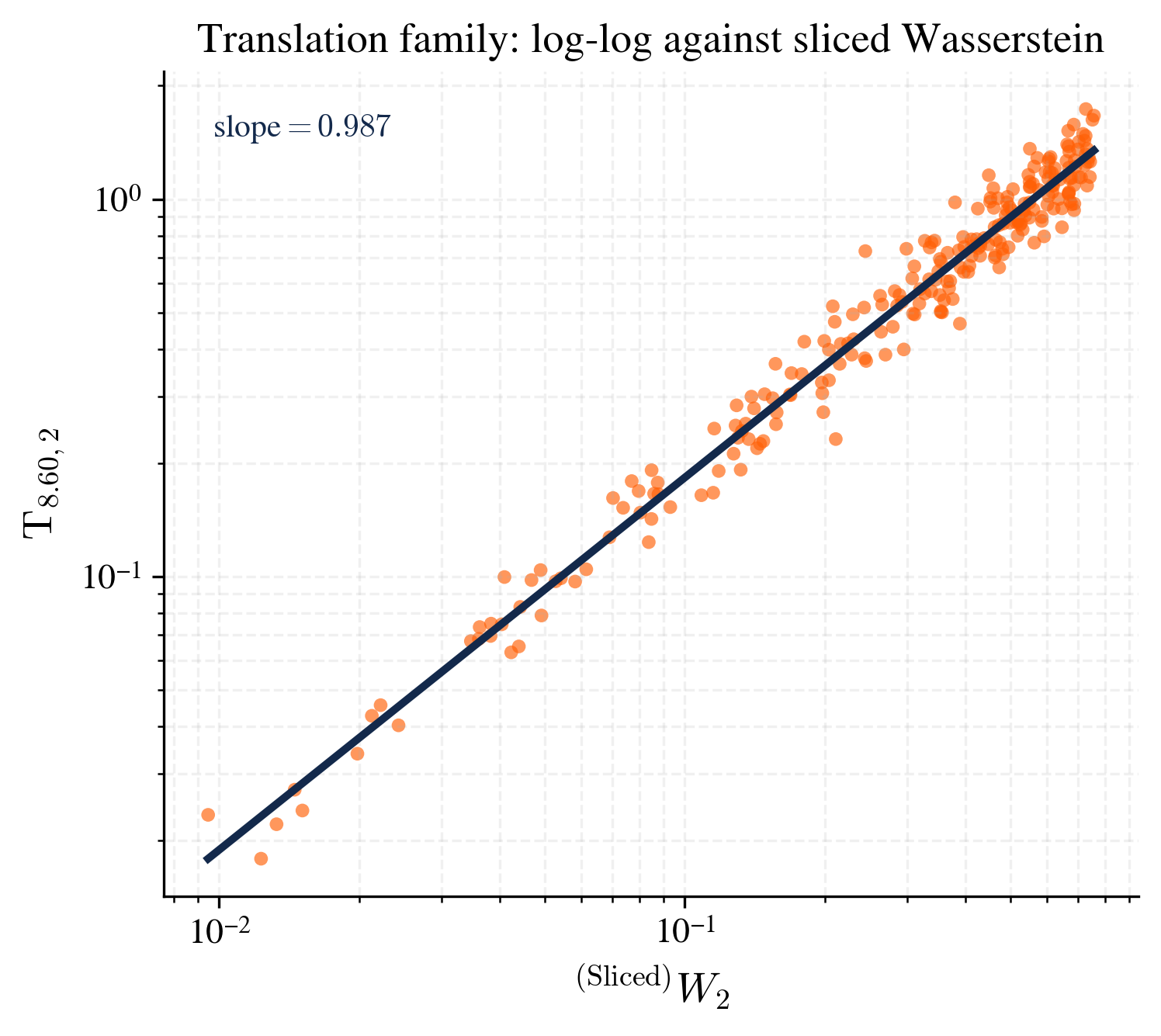} &
\includegraphics[width=0.285\textwidth]{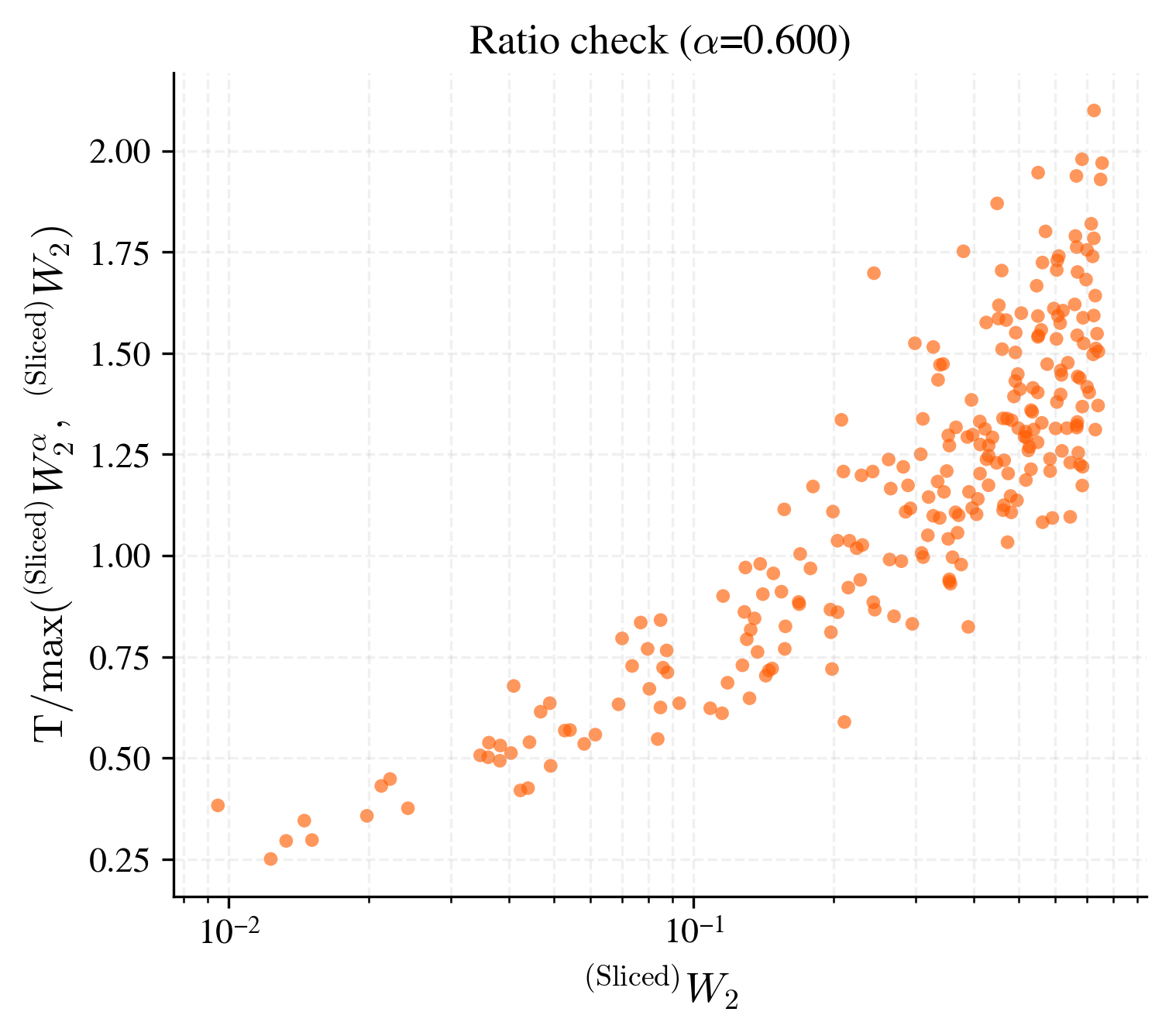} \\
\multicolumn{3}{c}{(c)\quad \(s=8.6,\ d=16,\ p=2\)}\\
\includegraphics[width=0.285\textwidth]{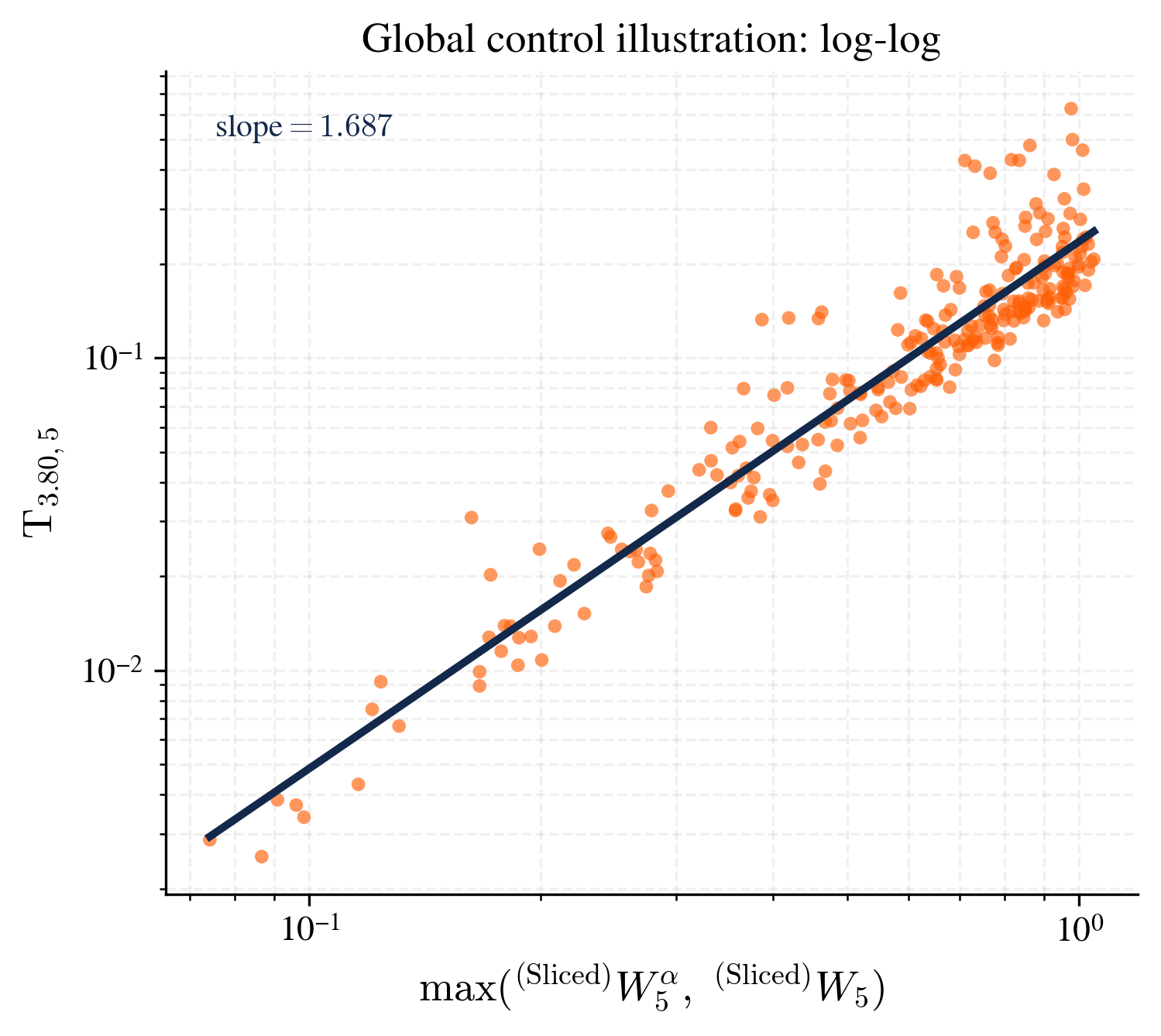} &
\includegraphics[width=0.285\textwidth]{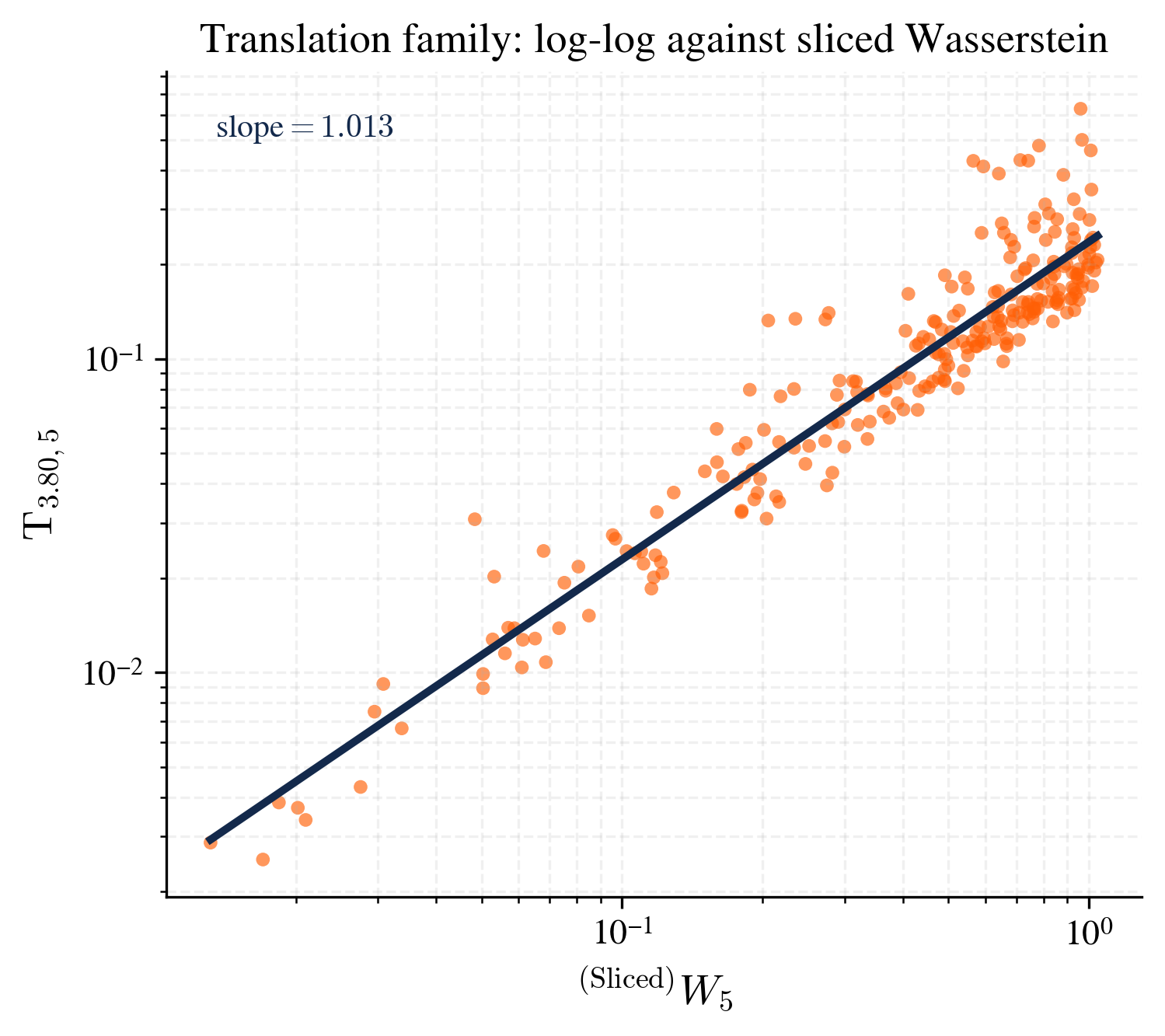} &
\includegraphics[width=0.285\textwidth]{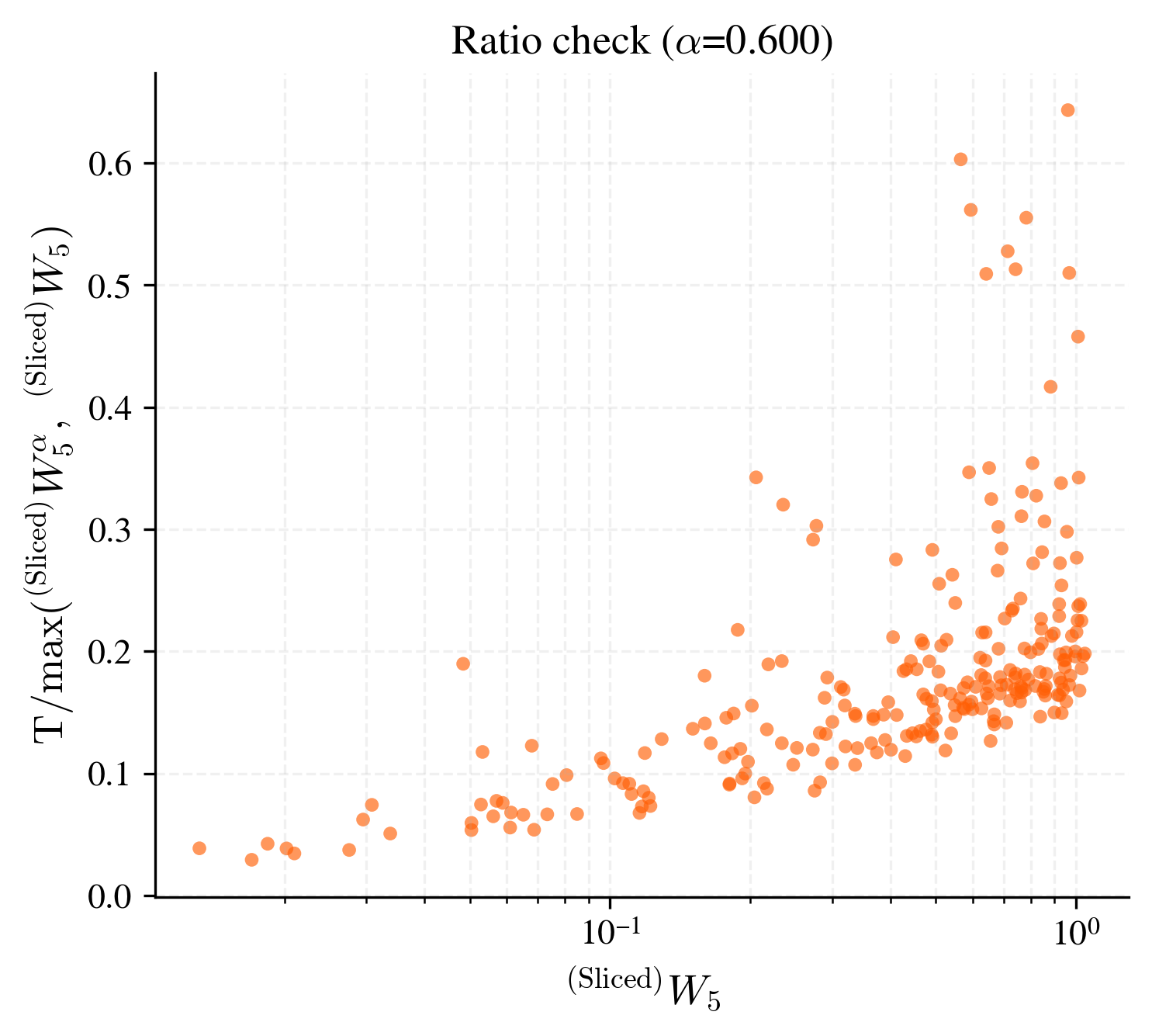} \\
\multicolumn{3}{c}{(d)\quad \(s=3.8,\ d=16,\ p=5\)}\\
\end{tabular}
\caption{Translation-family comparison of \(\widehat{\mathrm{T}}_{s,p}\) with sliced
Wasserstein. Each row corresponds to one \((d,p)\) setting. The first column plots
\(\widehat{\mathrm{T}}_{s,p}\) against
\(\max\{({}^{\mathrm{(sliced)}}W_p)^\alpha,{}^{\mathrm{(sliced)}}W_p\}\), the second
plots it directly against \({}^{\mathrm{(sliced)}}W_p\), and the third shows the
corresponding normalized ratio.}
\label{fig:exp3_translation}
\end{figure}

\else
The corresponding multi-panel plots are provided in the online appendix
(Figure~\ref{OA-fig:exp3_translation}).
\fi
\subsubsection{Quadratic Identity Check for \texorpdfstring{\(p=2\)}{p=2}}
For \(p=2\), the theory identifies \(\mathrm{T}_{s,2}^2\) with an energy/MMD discrepancy up to an explicit multiplicative constant. We use this identity as a numerical consistency check by comparing \(\widehat{\mathrm{T}}_{s,2}^2(\mu,\nu)\) with \(\mathrm{MMD}_{k_\alpha}^2(\mu,\nu)\) over random pairs \((\mu,\nu)\), where \(\alpha=2s-d\). For \((d,s)=(2,1.6)\), the fitted proportionality constant is \(5.6145\), compared with the theoretical value \(5.6579\), a relative gap of about \(0.77\%\). For \((d,s)=(16,8.6)\), the fitted constant is \(0.9291\), compared with the theoretical value \(0.8831\), a relative gap of about \(5.21\%\). The pointwise relative errors in Table~\ref{tab:exp5_identity} are appreciably larger than these aggregate gaps, especially in \(d=16\). This is expected. The pointwise errors are dominated by the Monte Carlo and truncation error of the importance-sampling estimator of \(\widehat{\mathrm T}_{s,2}^2\), which grows with dimension at fixed \(M\), while the fitted constant is a least-squares slope through the origin and is therefore weighted toward the pairs with the largest discrepancies. The scatter plots in Figure~\ref{fig:exp5} make both effects visible. The experiment therefore checks the implementation and the scale of the constant. It should not be read as an independent empirical proof of the identity.

\begin{table}[tbp]
\centering
\caption{Numerical check of the \(p=2\) identity
\(\widehat{\mathrm{T}}_{s,2}^2 \approx c_\alpha\,\mathrm{MMD}_{k_\alpha}^2\).}
\label{tab:exp5_identity}
\smallskip
\renewcommand{\arraystretch}{1.12}
\begin{tabular}{ccccccc}
\toprule\toprule
$d$ & $s$ & $\alpha=2s-d$ & theory \(c_\alpha\) & fitted \(c\) & rel.\ gap in \(c\) & median / 95\% rel.\ err. \\
\midrule
2  & 1.6 & 1.2 & 5.6579 & 5.6145 & 0.0077 & 0.187 / 0.739 \\
16 & 8.6 & 1.2 & 0.8831 & 0.9291 & 0.0521 & 0.342 / 1.081 \\
\bottomrule\bottomrule
\end{tabular}
\end{table}

\begin{figure}[tbp]
\centering
\subfloat[$d=16,s=8.6$]{
    \includegraphics[width=0.48\textwidth]{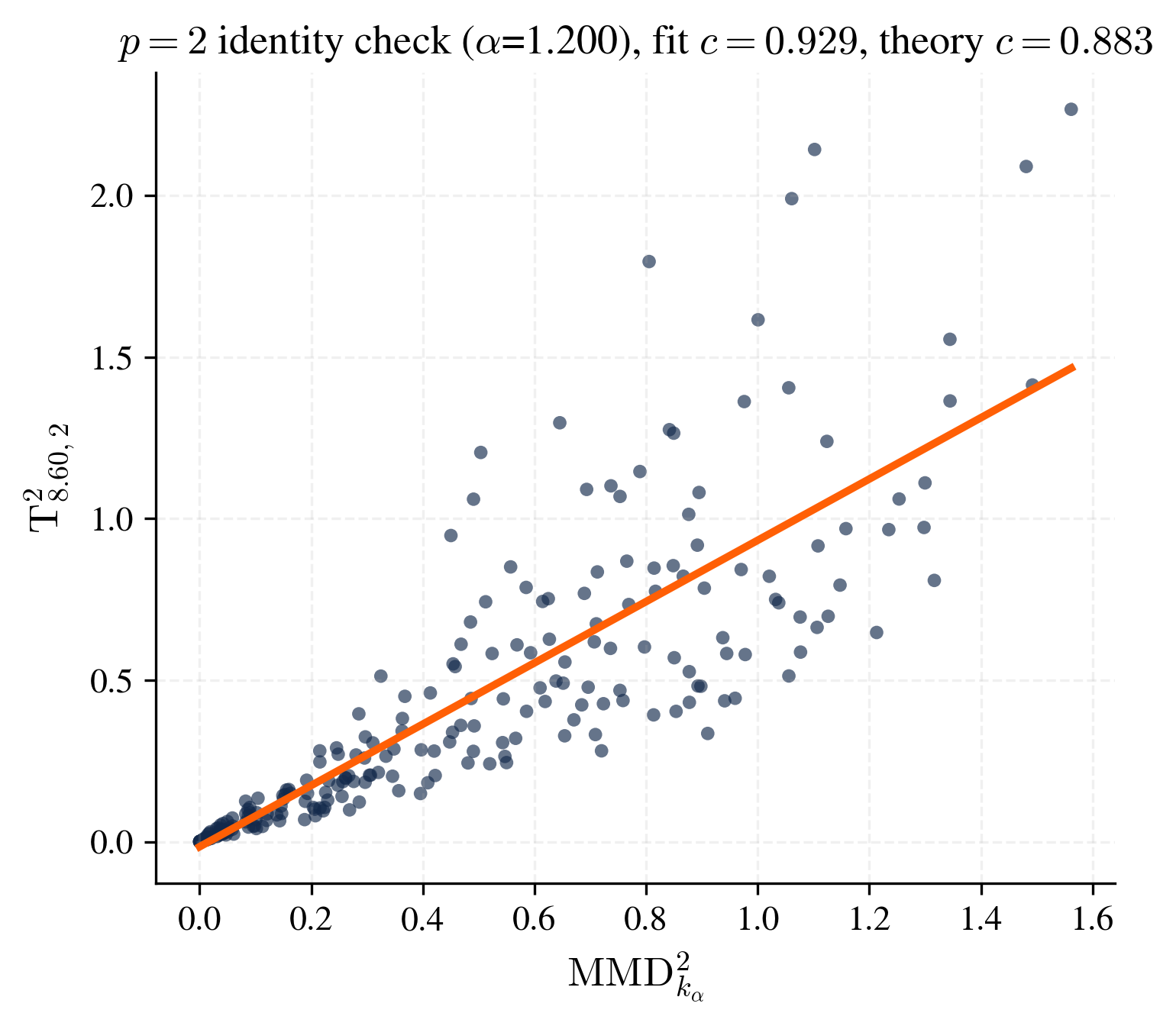}
}
\hfill
\subfloat[$d=2,s=1.6$]{
    \includegraphics[width=0.48\textwidth]{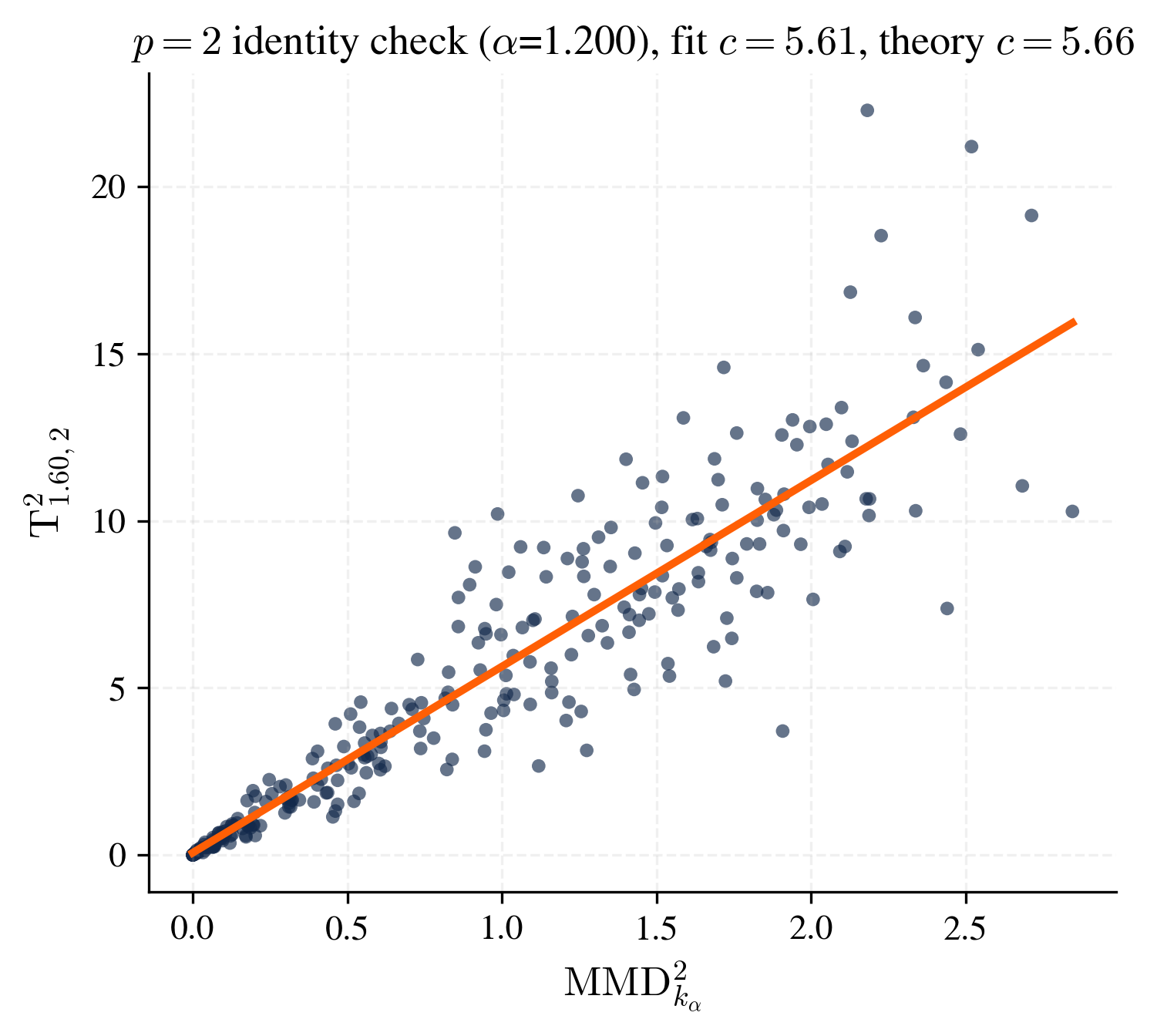}
}\\[0.5ex]
\caption{Scatter of \(\widehat{\mathrm{T}}_{s,2}^2\) versus \(\mathrm{MMD}_{k_\alpha}^2\) for \(\alpha=1.2\), together with the fitted linear relation through the origin. The fitted and theoretical constants are closer in \(d=2\) than in \(d=16\), and the point cloud is also substantially tighter in \(d=2\).}\label{fig:exp5}
\end{figure}

\subsection{Application to DOTmark Image Benchmark}\label{subsec:dotmark}
We evaluate the practical behavior of $\mathrm{T}_{s,p}$ on the DOTmark benchmark~\citep{schrieber2016dotmark}, which contains classes of gray-scale images at resolutions $32\times 32$, $64\times 64$, $128\times 128$, and $256\times 256$. Each image is converted into a probability measure on $[0,1]^2$ by normalizing pixel intensities to unit mass and rescaling pixel locations. This places all images on a common bounded domain, which is exactly the scale-controlled setting in which our theory predicts meaningful comparison between $\mathrm{T}_{s,p}$ and Wasserstein distances.\footnote{A remark on coordinates. For computational convenience, the FFT-based implementation evaluates $\widehat{\mathrm{T}}_{s,p}$ on \emph{pixel-index} coordinates (frequency box $[-\pi,\pi]^2$, with a low-frequency clamp at $2\pi/\max(H,W)$ for an $H\times W$ image), and the transport baselines likewise use pixel coordinates, so the two are internally consistent. By the exact dilation law of Lemma~\ref{lem:scaling}, passing from pixel coordinates to $[0,1]^2$ is a fixed monotone rescaling at each resolution ($\lambda=$ the resolution). Consequently all rank-based and calibration-based summaries reported below (Spearman correlation, ordinal agreement, ROC curves, and N-RMSE after affine calibration) are unaffected, while raw distance values and scatter-plot axes are on the pixel scale and are not directly comparable across resolutions. Two further processing details. For lattice-supported measures the characteristic function is $2\pi$-periodic in pixel coordinates, so the box $[-\pi,\pi]^2$ retains exactly one period and discards the (finite) contribution of higher periods to the population integral --- a truncation applied identically to every image at a fixed resolution, which therefore does not affect the comparisons. The CSV loader used throughout drops the first pixel row of each image, so the effective grids are $(R-1)\times R$ --- again identically for both metrics, preserving internal consistency.}

\ifincludeproofs
\begin{figure}[!htb]
\centering
\includegraphics[width=0.94\textwidth]{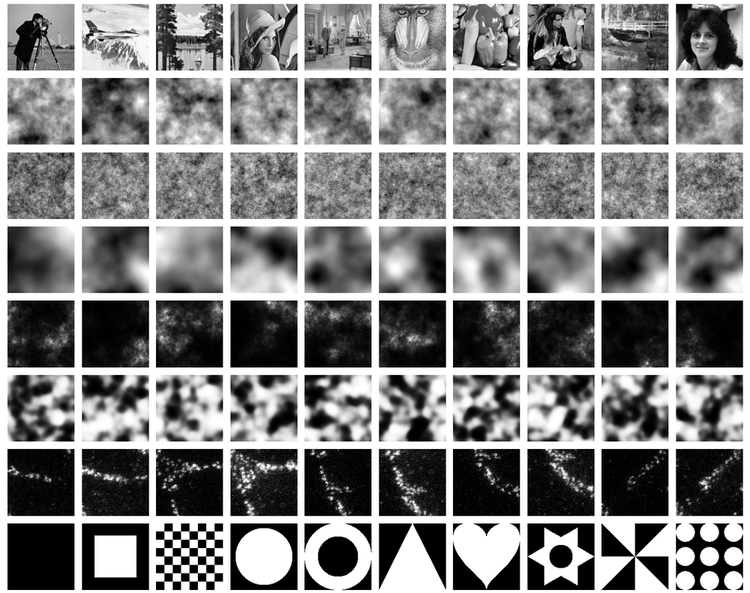}
\caption{DOTmark 1.0 data set at resolution \(128\times128\). Each row corresponds to
a class. The eight classes shown are those retained in the experiments.}
\label{fig:DOTmark_images}
\end{figure}

\begin{figure}[!htb]
\centering
\includegraphics[width=0.86\textwidth]{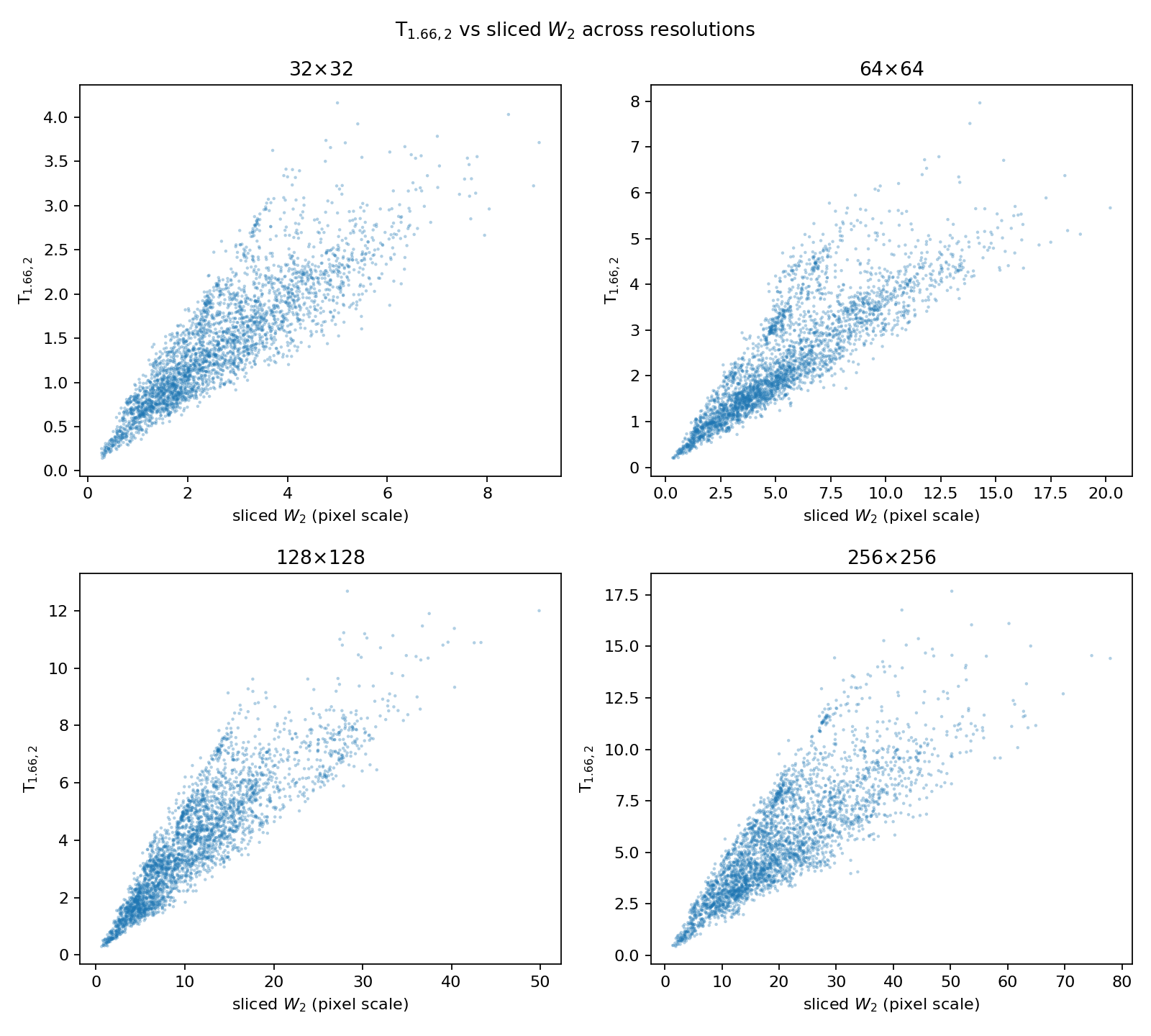}
\caption{DOTmark scatter comparison of \(\mathrm{T}_{s,2}\) (with \(s=1.66\), a value from
the tabulated grid of Table~\ref{tab:optimal_s_vertical_res}) against sliced \(W_2\), one
panel per resolution. Axes are on the pixel-coordinate scale of the footnote in
Section~\ref{subsec:dotmark} and are not comparable across resolutions.}
\label{fig:p2_scatter}
\end{figure}

\begin{figure}[!htb]
\centering
\includegraphics[width=0.86\textwidth]{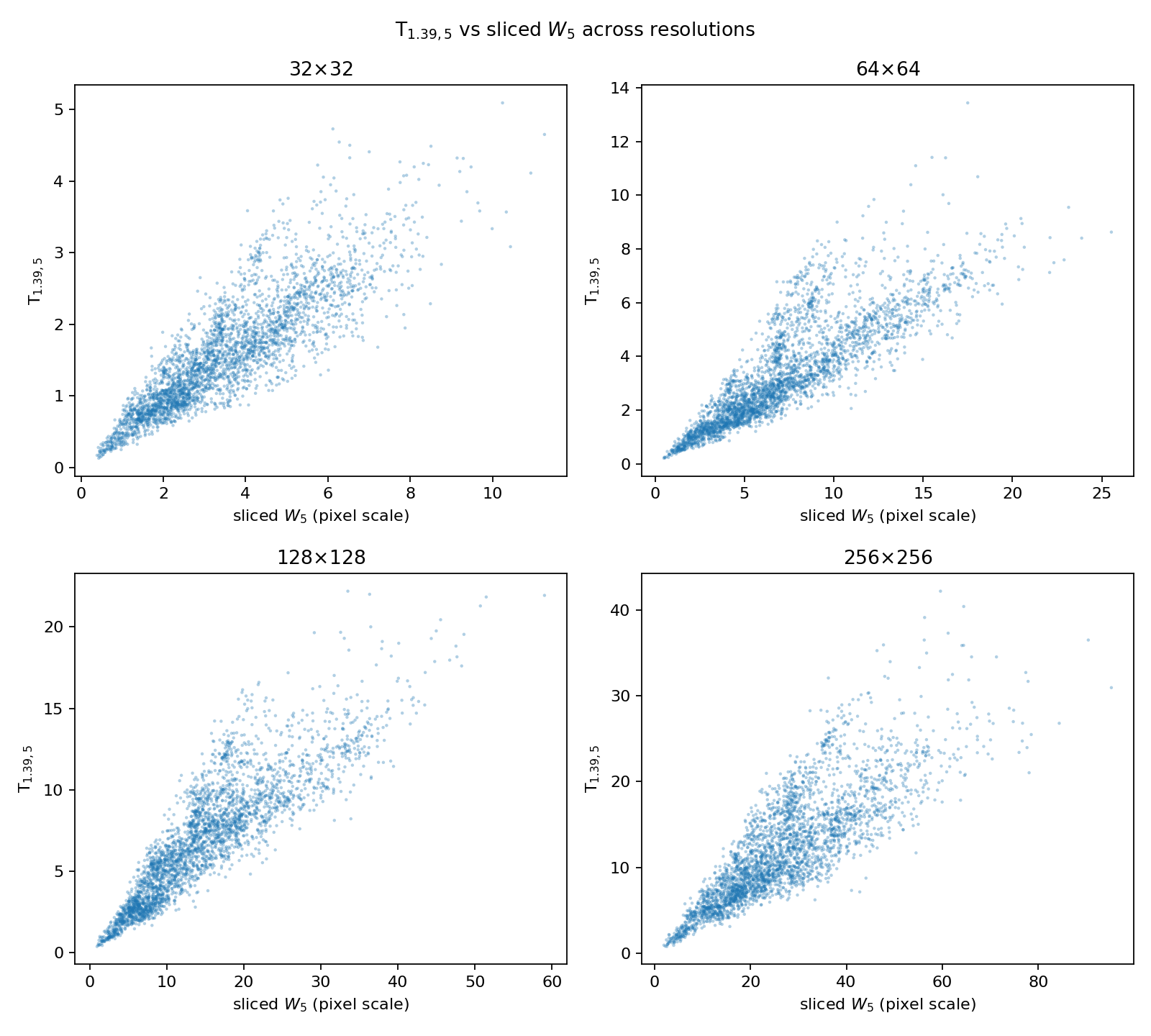}
\caption{As in Figure~\ref{fig:p2_scatter}, for \(\mathrm{T}_{s,5}\) with \(s=1.39\)
versus sliced \(W_5\).}
\label{fig:p5_scatter}
\end{figure}

\begin{figure}[!htb]
\centering
\includegraphics[width=0.86\textwidth]{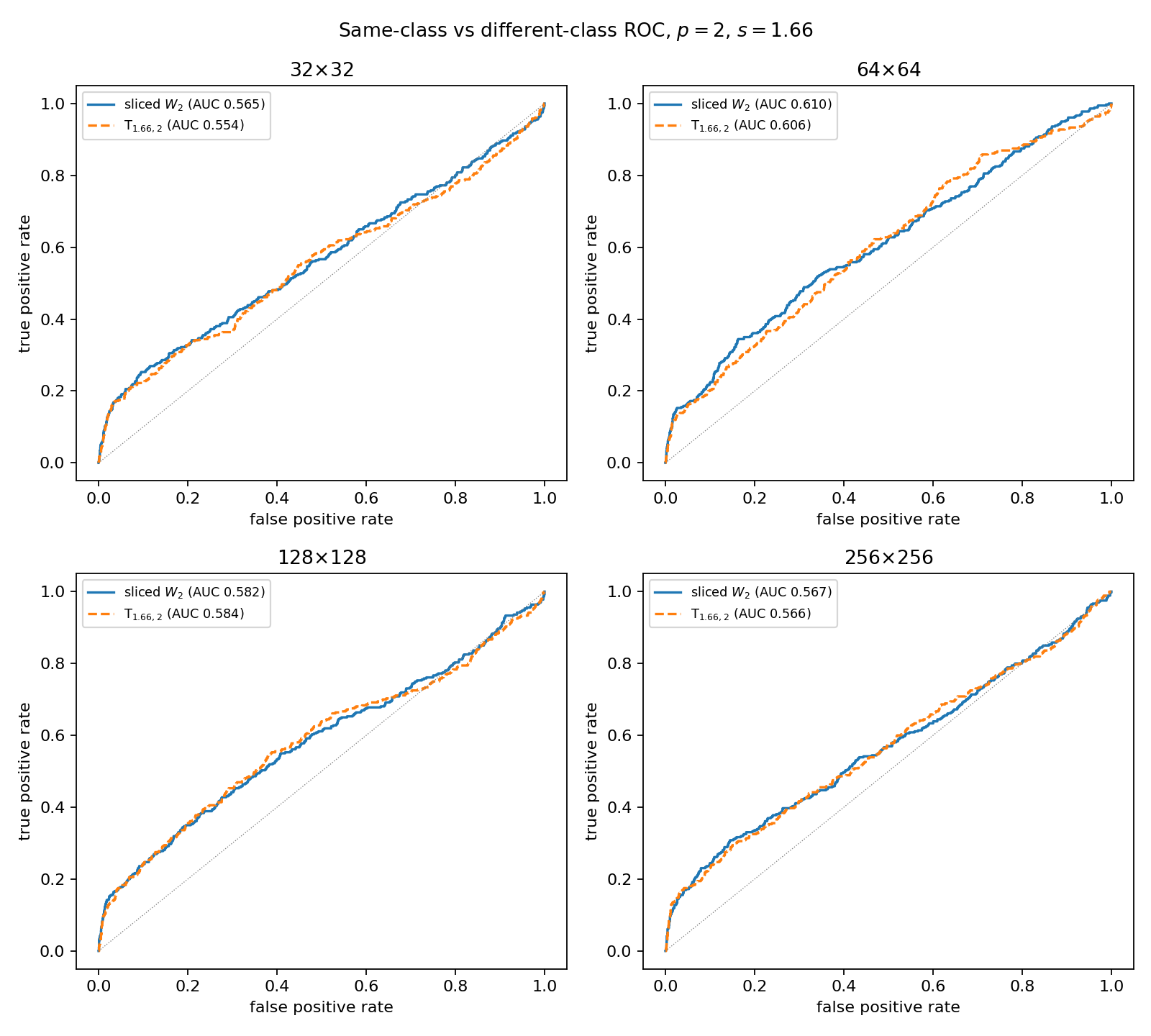}
\caption{DOTmark ROC curves for distinguishing same-class from different-class image
pairs by thresholding the distance (smaller distance indicates same class), for
\(\mathrm{T}_{s,2}\) with \(s=1.66\) and sliced \(W_2\), one panel per resolution, with
the area under each curve in the legends.}
\label{fig:p2_roc}
\end{figure}

\begin{figure}[!htb]
\centering
\includegraphics[width=0.86\textwidth]{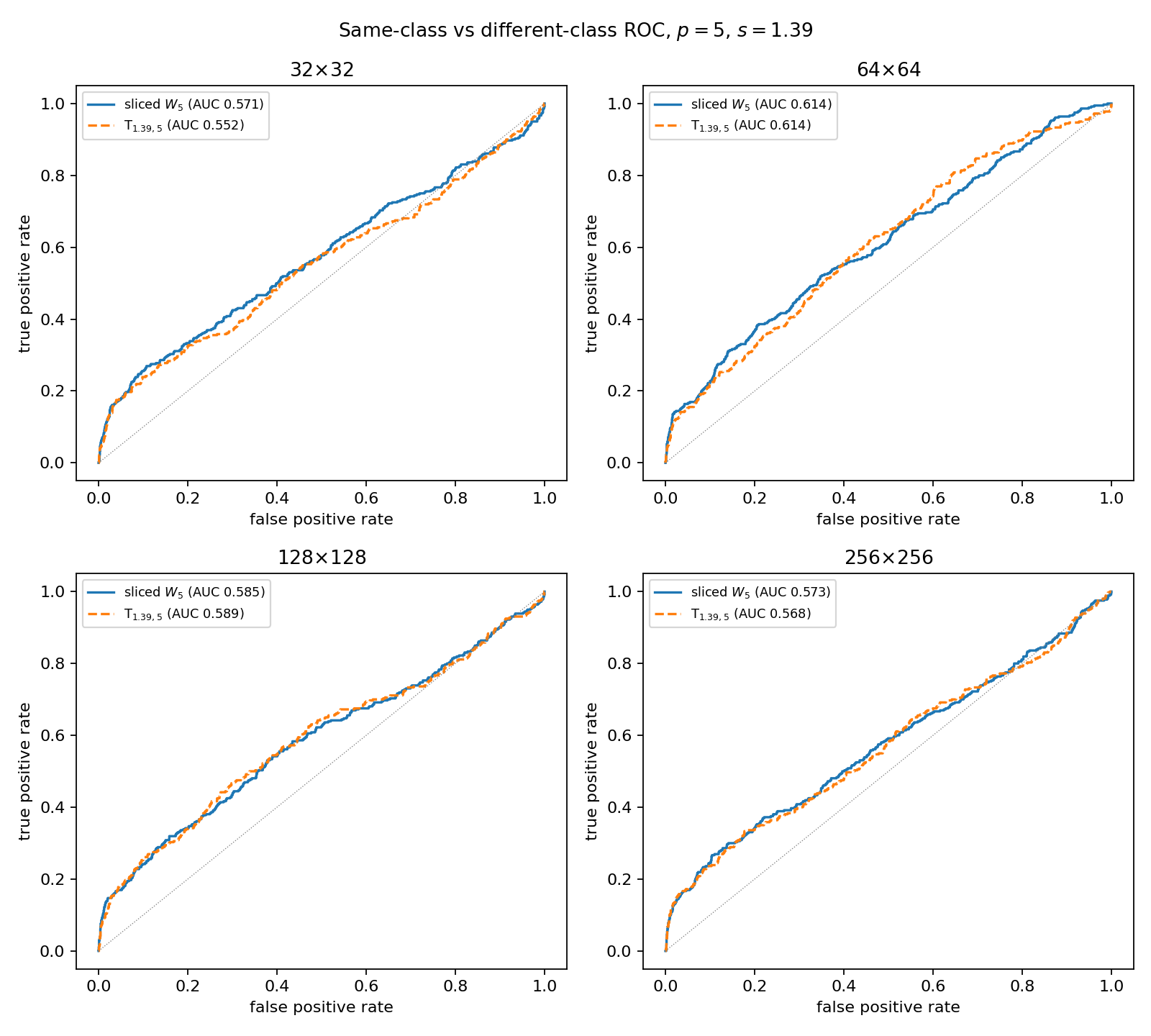}
\caption{As in Figure~\ref{fig:p2_roc}, for \(\mathrm{T}_{s,5}\) with \(s=1.39\) and
sliced \(W_5\).}
\label{fig:p5_roc}
\end{figure}

\fi

The purpose of this experiment is not to prove metric equivalence, but to ask whether $\mathrm{T}_{s,p}$ behaves similarly to the chosen sliced-\(W_p\) baseline on this image collection. We examine pairwise ordering, same-class versus different-class discrimination, and measured computation time. The DOTmark task uses eight classes with ten images per class --- ClassicImages, GRFmoderate, GRFrough, GRFsmooth, LogGRF, LogitGRF, MicroscopyImages, and Shapes, with CauchyDensity and WhiteNoise excluded --- so that all $\binom{80}{2}=3{,}160$ image pairs enter the pairwise comparisons. Sliced Wasserstein uses 64 random projections with random seed \(0\), and the supplementary Sinkhorn comparison uses entropic regularization \(0.01\).

\subsubsection{Pairwise Comparison with Sliced Wasserstein}
For $p\in\{2,5\}$ and several values of $s$ in the admissible range $d/p < s < 1+d/p$, we compute all pairwise $\mathrm{T}_{s,p}$ distances and compare them with sliced \(W_p\). We summarize agreement using three criteria. \emph{N-RMSE}. The $\mathrm{T}_{s,p}$ distances are affinely calibrated to the sliced-Wasserstein scale by ordinary least squares ($\mathrm{T}\approx a\,W+b$. Calibrated value $(\mathrm{T}-b)/a$), and we report the root-mean-square difference between the calibrated values and $W$ over all pairs, normalized by the mean of $W$. \emph{Spearman}. The rank correlation between the two distance vectors over all $3{,}160$ pairs. \emph{Ordinal (triplet) agreement}. Over all ordered triples $(a;b,c)$ of distinct images, the fraction for which the two metrics order the anchored comparison the same way, i.e.\ $\operatorname{sign}\big(D(a,b)-D(a,c)\big)$ agrees between $D=\mathrm{T}_{s,p}$ and $D={}^{\text{(sliced)}}W_p$. The first criterion measures agreement of calibrated values, the second global rank agreement, and the third local comparison agreement. Class labels enter only the ROC and two-sample analyses.

Table~\ref{tab:optimal_s_vertical_res} reports Spearman correlations from \(0.8165\) to \(0.9136\) and ordinal agreement from \(0.8071\) to \(0.8585\). Within each resolution the agreement metrics are lowest at \(s=0.41\) for \(p=5\) --- the grid value closest to the lower endpoint \(d/p=0.4\) of the admissible window --- consistent with the greater high-frequency and discretization sensitivity predicted near that endpoint. The differences are modest, however, and do not by themselves support a general tuning rule.

\ifincludeproofs
The scatter plots in Figures~\ref{fig:p2_scatter}--\ref{fig:p5_scatter}, computed at the
grid values $s=1.66$ ($p=2$) and $s=1.39$ ($p=5$), show an increasing relationship between
$\mathrm{T}_{s,p}$ and sliced \(W_p\), with greater dispersion among some large-distance pairs.
Figures~\ref{fig:p2_roc}--\ref{fig:p5_roc} compare same-class versus different-class
discrimination at the same $s$ values.
\else
The corresponding scatter and ROC plots, computed at the grid values $s=1.66$ ($p=2$) and
$s=1.39$ ($p=5$), are provided in the online appendix
(Figures~\ref{OA-fig:p2_scatter}--\ref{OA-fig:p5_roc}).
\fi
The ROC curves are similar across the tested resolutions, but separation is moderate for both metrics and the experiment contains only ten images per class.

On this bounded, normalized benchmark, $\mathrm{T}_{s,p}$ broadly preserves the rankings and anchored ordinal comparisons of the sliced-Wasserstein baseline while avoiding a transport solve. This is evidence about one implementation and one small image benchmark, not a claim of universal fidelity to \(W_p\).

\begin{table}[!tb]
\centering
\caption{DOTmark. Agreement of $\mathrm{T}_{s,p}$ with sliced $W_p$ across resolutions and
$s$ grids for $p\in\{2,5\}$. N-RMSE after affine calibration, Spearman rank correlation, and
ordinal (anchored triplet) agreement, as defined in the text, over all $3{,}160$ image pairs.}
\label{tab:optimal_s_vertical_res}
\smallskip
\footnotesize
\setlength{\tabcolsep}{3pt}
\begin{tabular}{@{}c@{\hspace{1em}}c@{}}
\begin{tabular}{@{}ccccc@{}}
\toprule\toprule
 & \textbf{$s$} & \textbf{N-RMSE} & \textbf{Spearman} & \textbf{Ordinal agr.} \\
\midrule
\multicolumn{5}{c}{32 $\times$ 32 Pixels} \\
\midrule
\multirow{4}{*}{\textbf{$p = 2$}}
  & 1.01 & 0.3176 & 0.8782 & 0.8403 \\
  & 1.34 & 0.3043 & 0.8886 & 0.8461 \\
  & 1.66 & 0.3000 & 0.8919 & 0.8476 \\
  & 1.99 & 0.2995 & 0.8918 & 0.8467 \\
\midrule
\multirow{4}{*}{\textbf{$p = 5$}}
  & 0.41 & 0.3057 & 0.8668 & 0.8229 \\
  & 0.74 & 0.2870 & 0.8861 & 0.8335 \\
  & 1.06 & 0.2838 & 0.8892 & 0.8358 \\
  & 1.39 & 0.2847 & 0.8879 & 0.8343 \\
\bottomrule\bottomrule
\end{tabular}
&
\begin{tabular}{@{}ccccc@{}}
\toprule\toprule
 & \textbf{$s$} & \textbf{N-RMSE} & \textbf{Spearman} & \textbf{Ordinal agr.} \\
\midrule
\multicolumn{5}{c}{128 $\times$ 128 Pixels} \\
\midrule
\multirow{4}{*}{\textbf{$p = 2$}}
  & 1.01 & 0.3359 & 0.8930 & 0.8469 \\
  & 1.34 & 0.3172 & 0.9055 & 0.8546 \\
  & 1.66 & 0.3099 & 0.9116 & 0.8580 \\
  & 1.99 & 0.3077 & 0.9136 & 0.8585 \\
\midrule
\multirow{4}{*}{\textbf{$p = 5$}}
  & 0.41 & 0.3325 & 0.8830 & 0.8291 \\
  & 0.74 & 0.3092 & 0.9014 & 0.8394 \\
  & 1.06 & 0.3026 & 0.9073 & 0.8429 \\
  & 1.39 & 0.3019 & 0.9086 & 0.8435 \\
\bottomrule\bottomrule
\end{tabular}
\\[1em]
\begin{tabular}{@{}ccccc@{}}
\toprule\toprule
 & \textbf{$s$} & \textbf{N-RMSE} & \textbf{Spearman} & \textbf{Ordinal agr.} \\
\midrule
\multicolumn{5}{c}{64 $\times$ 64 Pixels} \\
\midrule
\multirow{4}{*}{\textbf{$p = 2$}}
  & 1.01 & 0.3495 & 0.8797 & 0.8365 \\
  & 1.34 & 0.3403 & 0.8901 & 0.8437 \\
  & 1.66 & 0.3377 & 0.8938 & 0.8460 \\
  & 1.99 & 0.3376 & 0.8948 & 0.8461 \\
\midrule
\multirow{4}{*}{\textbf{$p = 5$}}
  & 0.41 & 0.3565 & 0.8681 & 0.8212 \\
  & 0.74 & 0.3378 & 0.8901 & 0.8341 \\
  & 1.06 & 0.3323 & 0.8955 & 0.8379 \\
  & 1.39 & 0.3309 & 0.8966 & 0.8383 \\
\bottomrule\bottomrule
\end{tabular}
&
\begin{tabular}{@{}ccccc@{}}
\toprule\toprule
 & \textbf{$s$} & \textbf{N-RMSE} & \textbf{Spearman} & \textbf{Ordinal agr.} \\
\midrule
\multicolumn{5}{c}{256 $\times$ 256 Pixels} \\
\midrule
\multirow{4}{*}{\textbf{$p = 2$}}
  & 1.01 & 0.3661 & 0.8413 & 0.8252 \\
  & 1.34 & 0.3551 & 0.8529 & 0.8320 \\
  & 1.66 & 0.3511 & 0.8573 & 0.8346 \\
  & 1.99 & 0.3498 & 0.8590 & 0.8351 \\
\midrule
\multirow{4}{*}{\textbf{$p = 5$}}
  & 0.41 & 0.3604 & 0.8165 & 0.8071 \\
  & 0.74 & 0.3421 & 0.8421 & 0.8189 \\
  & 1.06 & 0.3360 & 0.8491 & 0.8232 \\
  & 1.39 & 0.3337 & 0.8513 & 0.8236 \\
\bottomrule\bottomrule
\end{tabular}
\\
\end{tabular}
\end{table}
\subsubsection{Gaussian-Kernel MMD Two-Sample Testing}
Beyond pairwise distance comparisons, we also study whether different DOTmark classes can be distinguished at the \emph{distribution-of-images} level via a kernel two-sample test. For two classes $C_1$ and $C_2$, we treat each normalized image as one observation in the sample space of probability measures on $[0,1]^2$. Writing $\{\mu_i\}_{i=1}^{n_1}\subset C_1$ and $\{\nu_j\}_{j=1}^{n_2}\subset C_2$ for the corresponding image measures, we test
\[
H_0:\ \mathcal{L}(\mu)=\mathcal{L}(\nu)
\qquad \text{versus} \qquad
H_1:\ \mathcal{L}(\mu)\neq \mathcal{L}(\nu)
\]
where $\mathcal{L}(\mu)$ and $\mathcal{L}(\nu)$ denote the class-level distributions over images.

We use the unbiased quadratic-time estimator of the squared maximum mean discrepancy \citep{gretton2012kernel},
\[
\widehat{\mathrm{MMD}}_u^2
=
\frac{1}{n_1(n_1-1)}\sum_{i\neq i'} k_\sigma(\mu_i,\mu_{i'})
+
\frac{1}{n_2(n_2-1)}\sum_{j\neq j'} k_\sigma(\nu_j,\nu_{j'})
-
\frac{2}{n_1n_2}\sum_{i=1}^{n_1}\sum_{j=1}^{n_2} k_\sigma(\mu_i,\nu_j)
\]
with Gaussian kernel
\[
k_\sigma(\mu,\nu)
=
\exp\!\left(-\frac{D(\mu,\nu)^2}{2\sigma^2}\right)
\]
where $D$ is a distance between images, taken here to be either $\mathrm{T}_{s,p}$ or sliced \(W_p\). The bandwidth $\sigma$ is selected by the median heuristic applied to the pairwise distances $\{D(\mu,\nu)\}$ over the pooled sample. Statistical significance is assessed by a permutation test that repeatedly shuffles class labels across the pooled set $\{\mu_i\}\cup\{\nu_j\}$, following the standard methodology of \citet{gretton2012kernel}.

This experiment complements the pairwise analyses above. The earlier metrics evaluate agreement between $\mathrm{T}_{s,p}$ and sliced Wasserstein on individual image pairs, whereas the MMD test asks whether these distances induce comparable \emph{class-level} separation. This is a second-layer Gaussian kernel on the space of image-valued probability measures. It is distinct from the energy-kernel MMD on \(\mathbb R^d\) that is identified exactly with \(\mathrm{T}_{s,2}\) in Section~\ref{sec:mmd}.

We report results for both $p=2$ and $p=5$ to match the pairwise study
(Table~\ref{tab:mmd_dotmark_summary}). Across the 28 class pairs, the reject/non-reject
decisions agree with the sliced-Wasserstein version in \(21\) to \(25\) cases,
depending on \(p\) and \(s\). Every summary count is recomputable from the pair-level test
records, which come from the same run.
\ifincludeproofs
Tables~\ref{tab:mmd_dotmark_disagreements} and \ref{tab:mmd_dotmark_p5_disagreements} list the disagreements.
\else
The online appendix lists the pair-level disagreements
(Tables~\ref{OA-tab:mmd_dotmark_disagreements}
and~\ref{OA-tab:mmd_dotmark_p5_disagreements}) as well as the complete per-pair test record
(Tables~\ref{OA-tab:appendix_mmd_p2_res256} and~\ref{OA-tab:appendix_mmd_p5_res256}).
\fi
With only ten images per class and one benchmark, these percentages describe this run rather than the power, level, or general operating characteristics of either test.

\begin{table}[tbp]
\centering
\caption{Gaussian-kernel MMD two-sample testing on DOTmark at resolution $256\times256$.
There are $28$ class pairs, each with $n_1=n_2=10$ images and $2000$ permutations. Decisions are
at level $\alpha=0.05$. The counts are derived from the pair-level records in
Appendix~B of the online appendix.}
\label{tab:mmd_dotmark_summary}
\smallskip
\small
\setlength{\tabcolsep}{4pt}
\begin{tabular}{@{}llccc@{}}
\toprule\toprule
& \textbf{Method} & \textbf{Rejected pairs} & \textbf{Reject rate} & \textbf{Agreement} \\
\midrule
\multirow{5}{*}{$\mathbf{p=2}$}
& Sliced Wasserstein & 13/28 & 46.4\% & -- \\
& Toscani ($s=1.01$) & 18/28 & 64.3\% & 21/28 \ (75.0\%) \\
& Toscani ($s=1.34$) & 18/28 & 64.3\% & 21/28 \ (75.0\%) \\
& Toscani ($s=1.66$) & 17/28 & 60.7\% & 22/28 \ (78.6\%) \\
& Toscani ($s=1.99$) & 17/28 & 60.7\% & 22/28 \ (78.6\%) \\
\midrule
\multirow{5}{*}{$\mathbf{p=5}$}
& Sliced Wasserstein & 19/28 & 67.9\% & -- \\
& Toscani ($s=0.41$) & 16/28 & 57.1\% & 25/28 \ (89.3\%) \\
& Toscani ($s=0.74$) & 15/28 & 53.6\% & 24/28 \ (85.7\%) \\
& Toscani ($s=1.06$) & 15/28 & 53.6\% & 24/28 \ (85.7\%) \\
& Toscani ($s=1.39$) & 16/28 & 57.1\% & 23/28 \ (82.1\%) \\
\bottomrule\bottomrule
\end{tabular}
\end{table}

\ifincludeproofs
\begin{table}[!htb]
\centering
\caption{Class pairs on which Gaussian-kernel MMD tests based on Toscani distances and sliced
Wasserstein give different decisions at level \(\alpha=0.05\), for \(p=2\) at resolution
\(256\times256\). Pairs on which all five tests agree are omitted.}
\label{tab:mmd_dotmark_disagreements}
\smallskip
\small
\setlength{\tabcolsep}{3pt}
\begin{tabular}{@{}p{4.1cm}ccccc@{}}
\toprule\toprule
\textbf{Class pair} & \textbf{Sliced \(W_2\)} & \(\mathbf{s=1.01}\) & \(\mathbf{s=1.34}\) & \(\mathbf{s=1.66}\) & \(\mathbf{s=1.99}\) \\
\midrule
ClassicImages vs GRFsmooth & Reject & Non-reject & Non-reject & Non-reject & Non-reject \\
GRFmoderate vs Shapes & Non-reject & Reject & Reject & Reject & Reject \\
GRFsmooth vs MicroscopyImages & Non-reject & Reject & Reject & Reject & Reject \\
GRFsmooth vs Shapes & Non-reject & Reject & Reject & Non-reject & Non-reject \\
LogGRF vs MicroscopyImages & Non-reject & Reject & Reject & Reject & Reject \\
LogGRF vs Shapes & Non-reject & Reject & Reject & Reject & Reject \\
LogitGRF vs MicroscopyImages & Non-reject & Reject & Reject & Reject & Reject \\
\bottomrule\bottomrule
\end{tabular}
\end{table}

\begin{table}[!htb]
\centering
\caption{Class pairs on which Gaussian-kernel MMD tests based on Toscani distances and sliced
Wasserstein give different decisions at level \(\alpha=0.05\), for \(p=5\) at resolution
\(256\times256\). Pairs on which all five tests agree are omitted.}
\label{tab:mmd_dotmark_p5_disagreements}
\smallskip
\small
\setlength{\tabcolsep}{3pt}
\begin{tabular}{@{}p{4.1cm}ccccc@{}}
\toprule\toprule
\textbf{Class pair} & \textbf{Sliced \(W_5\)} & \(\mathbf{s=0.41}\) & \(\mathbf{s=0.74}\) & \(\mathbf{s=1.06}\) & \(\mathbf{s=1.39}\) \\
\midrule
ClassicImages vs GRFsmooth & Reject & Non-reject & Non-reject & Non-reject & Non-reject \\
ClassicImages vs LogitGRF & Reject & Non-reject & Non-reject & Non-reject & Non-reject \\
GRFsmooth vs Shapes & Reject & Reject & Non-reject & Non-reject & Non-reject \\
LogitGRF vs MicroscopyImages & Non-reject & Non-reject & Non-reject & Non-reject & Reject \\
MicroscopyImages vs Shapes & Reject & Non-reject & Non-reject & Non-reject & Non-reject \\
\bottomrule\bottomrule
\end{tabular}
\end{table}

\fi

\subsection{Computational Cost}\label{subsec:comp_cost}

Computational efficiency is crucial for large-scale and high-dimensional applications. 
The Toscani-Fourier distance $\mathrm{T}_{s,p}$ compares probability measures through weighted integrals of differences of characteristic functions, and it avoids the optimization step that is intrinsic to optimal-transport distances such as the Wasserstein metric $W_p$.
For an empirical measure $\mu_n=\frac1n\sum_{i=1}^n\delta_{x_i}$, the characteristic function is
\[
\Bchi_{\mu_n}(u)=\frac1n\sum_{i=1}^n e^{i u^\top x_i}
\]
Evaluating $\Bchi_{\mu_n}(u)$ at one frequency $u\in\R^d$ requires computing $u\cdot x_i$ for all $i$, hence
$\mathcal O(nd)$ arithmetic operations (plus $n$ complex exponentials). 
To compute $\mathrm{T}_{s,p}(\mu_n,\nu_m)$ in practice, one approximates the defining integral by quadrature or Monte Carlo sampling over frequencies $\{u_j\}_{j=1}^M$ (typically with truncation and importance sampling guided by the weight $\|u\|^{-ps}$). 
This yields an overall complexity $\mathcal O(Md(n+m))$, with memory that can be kept $\mathcal O(d(n+m))$ if frequencies are streamed in batches.

Crucially, the computation uses only dot products and exponentials and is embarrassingly parallel over $j$, making it well-suited for GPU acceleration and distributed evaluation.

In one dimension, $W_p$ between empirical measures can be computed exactly in $\mathcal O(n\log n)$ time via sorting and matching quantiles. 
In dimensions $d\ge2$, exact computation for discrete measures amounts to a linear program over couplings with $\Theta(n^2)$ variables, which is typically prohibitively expensive beyond moderate $n$ and requires $\Theta(n^2)$ memory (to store costs/couplings). \citep[see, e.g.,][]{peyre2019computational}. 
Entropy-regularized approximations (Sinkhorn-type methods) replace the LP by an iterative matrix-scaling procedure \citep{cuturi2013sinkhorn}. Each iteration costs $\mathcal O(n^2)$ in general (and requires access to an $n\times n$ cost/kernel matrix), while the number of iterations depends on the regularization level and the target accuracy. Although fast implementations exist and structure can sometimes be exploited (e.g.\ convolutional costs on grids), the quadratic scaling and the need to build or apply a dense pairwise cost operator remain important bottlenecks for large $n$. Sliced Wasserstein instead sorts projected samples and has a different accuracy-cost trade-off governed by the number of projections. The DOTmark implementation of \(\mathrm{T}_{s,p}\) uses a two-dimensional FFT on each image grid rather than the generic Monte Carlo estimator described above, so its measured cost is governed by grid FFTs and weighted array operations. Tables~\ref{tab:complexity}--\ref{tab:runtime} separate these algorithmic regimes and report within-run DOTmark timings. All timings were collected within a single run on one compute node of the Delta system at NCSA (allocated through ACCESS). Absolute values are hardware- and implementation-dependent, and only within-run comparisons are meaningful. The Sinkhorn baseline uses entropic regularization $0.01$ on pixel-coordinate costs with at most $2{,}000$ iterations and stopping threshold $10^{-9}$. Per-pair convergence status was not separately logged, so its timings should be read as indicative. Structured alternatives --- e.g.\ convolutional Sinkhorn iterations on grids with per-iteration cost $O(N\log N)$ \citep{peyre2019computational} --- would narrow the measured gap on gridded data.

\begin{table}[tbp]
\centering
\caption{Comparison of computational scaling and algorithmic structure. The frequency budget
$M$ controls accuracy in a distribution-dependent way (see the discussion below). The table
compares algorithmic structure, not error-matched cost. On grids, convolutional Sinkhorn variants
reduce the per-iteration cost to $O(N\log N)$ \citep{peyre2019computational}.}
\label{tab:complexity}
\smallskip
\small
\renewcommand{\arraystretch}{1.15}
\begin{tabular}{@{}p{3.0cm}p{2.5cm}p{4.0cm}p{2.5cm}@{}}
\toprule\toprule
\textbf{Metric} & \textbf{1D cost} & \textbf{Cost for $d\ge2$} & \textbf{Optimization} \\
\midrule
$W_p$ (exact, discrete) 
& $\mathcal O(n\log n)$ 
& LP on $\Theta(n^2)$ variables
& Yes 
\\

Sinkhorn / entropic OT 
& $\mathcal O(n^2)$ per iter 
& $\mathcal O(n^2)$ per iter
& Yes (iterative) 
 \\
Sliced \(W_p\) with \(L\) projections
& $\mathcal O(Ln\log n)$
& $\mathcal O(Ldn+Ln\log n)$
& No (projection sorting)
\\
$\mathrm{T}_{s,p}$ (sample-based approx.) 
& $\mathcal O(M(n+m))$ 
& $\mathcal O(M d (n+m))$ 
& No 
\\
\bottomrule\bottomrule
\end{tabular}
\end{table}

\begin{table}[tbp]
\centering
\begin{threeparttable}
\caption{Per-pair empirical mean runtime (seconds) of $\mathrm{T}_{s,p}$ versus Wasserstein-type baselines.}
\label{tab:runtime}
\smallskip
\renewcommand{\arraystretch}{1.1}
\begin{tabular}{@{}cccccc@{}}
\toprule\toprule
\textbf{$p$} & \textbf{Resolution} & \textbf{$s$} & \textbf{Sliced W} & \textbf{Sinkhorn} & \textbf{$\mathrm{T}_{s,p}$} \\
\midrule
\multirow{4}{*}{2}
& 32  & 1.99 & 0.0118 & 0.0593  & 0.0002 \\
& 64  & 1.99 & 0.0593 & 1.5615  & 0.0003 \\
& 128 & 1.99 & 0.2525 & 21.2620 & 0.0012 \\
& 256 & 1.99 & 1.5013
    & {\footnotesize$\spadesuit$}     & 0.0037
 \\
\midrule
\multirow{4}{*}{5}
& 32  & 1.39 & 0.0126 & 0.0466  & 0.0001 \\
& 64  & 1.39 & 0.0561 & 0.7048  & 0.0002 \\
& 128 & 1.39 & 0.2677 & 10.0410 & 0.0010 \\
& 256 & 1.39 & 1.5532
     & {\footnotesize$\spadesuit$}       & 0.0046\\
\bottomrule\bottomrule
\end{tabular}
\begin{tablenotes}[flushleft]
\item \footnotesize $\spadesuit$. Sinkhorn was not run at \(256\times256\). The grid has
\(65{,}536\) support points, above the implementation's \(16{,}384\)-point safety threshold.
\end{tablenotes}
\end{threeparttable}
\end{table}

Finally, decay of $\Bchi_\mu(u)$ and the weighting $\|u\|^{-ps}$ may permit useful approximations of $\mathrm{T}_{s,p}$ with moderate \(M\), but the required truncation and sample size depend on the distributions and parameters and must be diagnosed rather than assumed.
Standard tools from numerical harmonic analysis (e.g.\ importance sampling, quasi-Monte Carlo, sparse grids) may improve the accuracy-cost trade-off \citep{grafakos2008classical}.
Moreover, if one works with gridded densities (rather than point masses) and structured frequency sets, FFT-based methods can accelerate the evaluation of Fourier transforms, yielding additional speedups.

\subsection{Limitations and Reproducibility}\label{subsec:limitations}
The numerical evidence has four deliberate limits. First, the Fourier integrals are truncated or
discretized, and no finite-sample error bound for the resulting empirical discrepancy is proved
here. Second, the \(d\ge2\) transport baseline is sliced \(W_p\), whereas the comparison theorems
concern \(W_p\). Third, DOTmark supplies ten images per class, so the ROC, triplet, and permutation
summaries are descriptive for a small benchmark. Fourth, the reported times depend on the
implementation and hardware (a single Delta compute node, as noted above). The retained result
files do not contain sufficiently detailed metadata to support cross-machine timing claims, and
the Sinkhorn timings in particular should be read as indicative.

For reproducibility, the synthetic-experiment configurations are recorded in
Section~\ref{subsec:synth_Val}, and the DOTmark task configuration comprises the eight retained
classes, four resolutions, the \(s\)-grids of Table~\ref{tab:optimal_s_vertical_res}, 64
sliced-Wasserstein projections, regularization \(0.01\), random seed \(0\), and the ten-pair
runtime subset. DOTmark itself is not redistributed. It is available from the source cited above.
The Monte Carlo and FFT formulations are different numerical approximations of the same population
functional, and their raw values should not be interchanged without accounting for truncation,
discretization, and coordinate scaling.

\paragraph{Code availability.}
The complete implementation --- the Monte Carlo estimator of $\mathrm T_{s,p}$, the FFT
implementation for gridded measures, all synthetic-experiment drivers with their task
configurations, and the DOTmark pairwise and two-sample pipelines --- is provided as a
supplementary code archive accompanying this submission, together with an environment
specification and entry-point commands.

\section{Further Directions}\label{sec:further}

The most consequential gap is statistical, for $p\neq2$. At $p=2$, the identification with the
energy distance settles the basic questions. Proposition~\ref{prop:empirical-mean} gives the exact
mean of the squared empirical discrepancy, and the classical two-sample theory of energy
statistics and MMDs --- permutation tests, asymptotic null distributions, consistency --- applies
verbatim \citep{szekely2013energy,gretton2012kernel,sejdinovic2013equivalence}. For general $p$,
however, sharp finite-sample bounds, concentration \citep{ledoux2001concentration}, and limit
theorems for $\mathrm{T}_{s,p}(\hat\mu_n,\mu)$ are open, and are what would place the full family
on inferential footing for two-sample testing, goodness-of-fit assessment, and learning objectives
built on empirical characteristic functions. The natural
benchmark is the rate analysis for the empirical Wasserstein distance
\citep{FournierGuillin2015}, whose dimension-dependent rates quantify exactly the curse of
dimensionality that a Fourier-side discrepancy might partially avoid.

A second direction is high-dimensional geometry \citep{vershynin2018high}. The dependence on $d$ of
the comparison constants, of the effective regularity imposed by the spectral weight
$\|u\|^{-ps}$, and of the sample complexity is open, as is whether dimension-robust statements can
be recovered on structured families. Product models, low intrinsic dimension, anisotropic classes,
or scale-normalized families with controlled tails. Relatedly, the dual representation invites a
geometric program parallel to transport theory, asking which features of a measure the
harmonic-analytic test class emphasizes and how these differ from the Lipschitz geometry behind
$W_p$ --- a distinction that should be most visible in high dimension, where low-frequency
structure may dominate fine transport detail.

Computationally, scalable approximations based on structured or randomized frequency sampling,
quasi-Monte Carlo, or FFT implementations on gridded domains deserve systematic comparison with
sliced, projected, and entropic transport. Normalized or sliced variants of $\mathrm{T}_{s,p}$ are a
natural candidate for dimension-robust alternatives to projection-based OT. Finally, the
construction should extend beyond $\R^d$. Replacing Fourier modes by Laplace--Beltrami or graph
spectral bases would give spectral Toscani-type discrepancies on manifolds and graphs.

\section{Conclusion}

We studied the weighted integral Toscani--Fourier distance $\mathrm{T}_{s,p}$, a Fourier-domain discrepancy between probability measures defined through weighted $L^p$ differences of characteristic functions. We identified conditions under which $\mathrm{T}_{s,p}$ is finite and metrizes a natural class of probability measures, and we described the resulting metric structure through a weighted function-space embedding whose image is closed, so that the resulting space is complete. For all $1\le p<\infty$, we further showed that $\mathrm{T}_{s,p}$ admits a dual integral probability metric representation over homogeneous Fourier--Lebesgue test functions, with an explicit extremizer whenever $1<p<\infty$.

We then established a global comparison showing that Wasserstein closeness implies Toscani--Fourier closeness through a H\"older-type bound. Since dilation and mass splitting rule out any global reverse inequality, we instead identified scale-controlled model classes---including bounded-support and uniform moment/tail families---on which $\mathrm{T}_{s,p}$ and $W_p$ induce equivalent notions of convergence, and we proved explicit reverse moduli on bounded-support classes for every $1\le p<\infty$ (Theorem~\ref{thm:explicit-reverse-bounded}, Proposition~\ref{prop:reverse-bounded-HY}). In the quadratic case, we recorded the classical Sz\'ekely--Rizzo normalization identifying $\mathrm{T}_{s,2}$ with an energy-kernel MMD and a negative-Sobolev integral probability metric, imported complementary reverse bounds from the kernel literature, and derived an exact finite-sample identity for the mean of the empirical discrepancy (Proposition~\ref{prop:empirical-mean}).

Taken together, the results identify when $\mathrm{T}_{s,p}$ can serve as a mathematically tractable surrogate for transport on scale-controlled classes. The numerical section supplies implementation checks and a small image benchmark in which the Fourier method broadly agrees with sliced Wasserstein at lower measured computation time. Beyond the exact $p=2$ identity of Proposition~\ref{prop:empirical-mean}, it does not establish estimator-level statistical guarantees, nor universal computational superiority. High-dimensional constants, finite-sample rates, and controlled comparisons with sliced, projected, and entropic transport remain open and determine how broadly the method can be used in machine learning.
\acks{\textbf{Funding.} This work used the Delta computing resource at the University of Illinois
Urbana-Champaign through the Advanced Cyberinfrastructure Coordination Ecosystem: Services and
Support (ACCESS) program, supported by U.S. National Science Foundation grants 2138259, 2138286,
2138307, 2137603, and 2138296. \textbf{Competing interests.} The author declares no competing
interests.}
\fi

\ifincludeproofs
\appendix
\counterwithin{equation}{section}
\counterwithin{table}{section}
\counterwithin{figure}{section}
\section{Proofs}\label{app:proofs}
This appendix contains complete proofs of every numbered statement in the paper (the short proof
of Corollary~\ref{cor:upper-holder-max} appears inline in the main text), in the order in which
the results are used. No argument has been abbreviated, and no step has been left to the reader
beyond standard textbook facts, which are cited explicitly where invoked.

\subsection{Proof of Proposition~\ref{prop:integrability}}\label{proof:integrability}
\begin{proof}
(i) \textit{Non-negativity and Definiteness.}
For any $u\neq 0$
$$
\frac{|\Bchi_\mu(u)-\Bchi_\nu(u)|^p}{\|u\|^{ps}}\ge0
$$
so $\mathrm{T}_{s,p}(\mu,\nu)\ge 0$, and $\mathrm{T}_{s,p}(\mu,\mu)=0$ is immediate since the integrand vanishes identically.

Suppose conversely that $\mathrm{T}_{s,p}(\mu,\nu)=0$. Since the integrand is nonnegative and its integral vanishes, we conclude that
$$
|\Bchi_\mu(u)-\Bchi_\nu(u)|=0\qquad\text{for Lebesgue-almost every } u\in\R^d\setminus\{0\}
$$
This is an almost-everywhere statement only, and it must be upgraded to an everywhere statement before uniqueness theorems can be applied. To do so, recall that characteristic functions of probability measures are uniformly continuous on $\R^d$. For any probability measure $\eta$ and $u,v\in\R^d$
$$
|\Bchi_\eta(u)-\Bchi_\eta(v)|\le \E\big|e^{i\langle u,X\rangle}-e^{i\langle v,X\rangle}\big|\xrightarrow[\ \|u-v\|\to0\ ]{}0
$$
by dominated convergence (the integrand is bounded by $2$ and tends to $0$ pointwise). Hence $\Bchi_\mu-\Bchi_\nu$ is a continuous function on $\R^d$ vanishing on a set of full Lebesgue measure. A set of full Lebesgue measure is dense in $\R^d$, and a continuous function vanishing on a dense set vanishes identically. Therefore $\Bchi_\mu(u)=\Bchi_\nu(u)$ for \emph{every} $u\in\R^d$ (at $u=0$ both equal $1$ in any case). By the uniqueness theorem for characteristic functions (L\'evy's uniqueness theorem, see \citealp[Theorem~26.2]{billingsley2017probability}), $\mu=\nu$. Hence the distance is definite.

(ii) \textit{Symmetry.}
Since $|\Bchi_\mu(u)-\Bchi_\nu(u)|=|\Bchi_\nu(u)-\Bchi_\mu(u)|$ for every $u$, it follows immediately that
$$
\mathrm{T}_{s,p}(\mu,\nu)=\mathrm{T}_{s,p}(\nu,\mu)
$$

(iii) \textit{Triangle Inequality.}
For $1\le p<\infty$ and any $\mu,\nu,\lambda\in \mathcal{P}_p(\R^d)$, write the pointwise triangle inequality
\[
\begin{aligned}
|\Bchi_\mu(u)-\Bchi_\lambda(u)|
&\le|\Bchi_\mu(u)-\Bchi_\nu(u)|\\
&\quad+|\Bchi_\nu(u)-\Bchi_\lambda(u)|
\end{aligned}
\]
divide by $\|u\|^{s}$, and apply Minkowski's inequality in $L^p(\R^d\setminus\{0\})$.
\[
\begin{aligned}
&\left(\int_{\R^d\setminus\{0\}}
\frac{|\Bchi_\mu(u)-\Bchi_\lambda(u)|^p}{\|u\|^{ps}}\,du\right)^{1/p}\\
&\quad\le
\left(\int_{\R^d\setminus\{0\}}
\frac{|\Bchi_\mu(u)-\Bchi_\nu(u)|^p}{\|u\|^{ps}}\,du\right)^{1/p}\\
&\qquad+
\left(\int_{\R^d\setminus\{0\}}
\frac{|\Bchi_\nu(u)-\Bchi_\lambda(u)|^p}{\|u\|^{ps}}\,du\right)^{1/p}
\end{aligned}
\]

(iv) \textit{Finiteness.} We split the domain into two regions and use moment conditions. Note that only \emph{first} moments enter, and $\mathcal P_p(\R^d)\subseteq\mathcal P_1(\R^d)$ for $p\ge 1$ by Jensen's inequality.\\

- \textit{Near the Origin ($\|u\|<1$).}
Since $\mu,\nu\in\mathcal P_p(\R^d)$ with $p\ge 1$, they have finite first moments.
Using $|e^{it}-1|\le |t|$, for $X\sim\mu$ we have
\[
|\Bchi_\mu(u)-1|
=\big|\E(e^{i\langle u,X\rangle}-1)\big|
\le \E|e^{i\langle u,X\rangle}-1|
\le \E|\langle u,X\rangle|
\le \|u\| \E\|X\|
\]
Similarly, $|\Bchi_\nu(u)-1|\le \|u\| \E\|Y\|$ for $Y\sim\nu$, hence
\[
|\Bchi_\mu(u)-\Bchi_\nu(u)|
\le |\Bchi_\mu(u)-1|+|\Bchi_\nu(u)-1|
\le C\|u\|,
\qquad C:=\E\|X\|+\E\|Y\|<\infty
\]
Therefore on $\{\|u\|<1\}$ the integrand is bounded by $C^p\|u\|^{p-ps}$, and in polar coordinates, with $v_d:=|\mathbb S^{d-1}|$,
\[
\int_{\|u\|<1}\|u\|^{p-ps}\ du
=
v_d\int_0^1 r^{\,p-ps+d-1}\ dr
\]
which converges if and only if $p-ps+d>0$, that is
\begin{equation}\label{eq:origin_fin_cond}
    s<1+\frac{d}{p}
\end{equation}

- \textit{At Infinity ($\|u\|\ge1$).} the characteristic functions are uniformly bounded by $1$, so $|\Bchi_{\mu}(u)-\Bchi_{\nu}(u)|\le 2$. Thus for $\|u\|\ge 1$ we have $\frac{|\Bchi_{\mu}(u)-\Bchi_{\nu}(u)|^p}{\|u\|^{ps}} \le \frac{2^p}{\|u\|^{ps}}$, and in polar coordinates
$$
\int_{\|u\|\ge 1} \frac{du}{\|u\|^{ps}}
=
v_d\int_1^\infty r^{\,d-1-ps}\ dr
$$
converges if and only if $d-ps<0$, that is
\begin{equation}\label{eq:tail_fin_cond}
    s > \frac{d}{p}
\end{equation}

Under conditions \eqref{eq:origin_fin_cond} and \eqref{eq:tail_fin_cond} together, i.e.\ precisely when
$$\frac{d}{p}<s<1+\frac{d}{p}$$
the Toscani--Fourier distance $\text{T}_{s,p}(\mu,\nu)$ is finite for all $\mu,\nu\in\mathcal P_p(\R^d)$. (Proposition~\ref{prop:window-sharp} shows that neither endpoint restriction can be removed.)
\end{proof}

\subsection{Proof of Theorem~\ref{thm:seminorm}}\label{proof:norm_proof}
\begin{proof}
Throughout the proof, $\mu,\nu\in\mathcal M(\R^d)$ are arbitrary finite \emph{signed} measures, so we work with the general expression $\widehat\mu(u)-\widehat\mu(0)$ rather than with the probability-measure normalization $\Bchi_\mu(u)-1$ (the latter is the special case $\widehat\mu(0)=1$). The key structural observation is that the map
$$
\Phi:\mathcal M(\R^d)\longrightarrow \{\text{functions on }\R^d\setminus\{0\}\},
\qquad
(\Phi\mu)(u):=\frac{\widehat\mu(u)-\widehat\mu(0)}{\|u\|^{s}}
$$
is \emph{linear} in $\mu$. For $a,b\in\R$ and $\mu,\nu\in\mathcal M(\R^d)$, the Fourier transform satisfies $\widehat{a\mu+b\nu}=a\widehat\mu+b\widehat\nu$ pointwise (by linearity of the integral defining $\widehat{\,\cdot\,}$\,), hence
$$
\Phi(a\mu+b\nu)(u)
=\frac{a\big(\widehat\mu(u)-\widehat\mu(0)\big)+b\big(\widehat\nu(u)-\widehat\nu(0)\big)}{\|u\|^{s}}
=a(\Phi\mu)(u)+b(\Phi\nu)(u)
$$
By definition, $|\mu|_{\mathrm{T}_{s,p}}=\|\Phi\mu\|_{L^p(\R^d\setminus\{0\})}$ for $1\le p<\infty$ and $|\mu|_{\mathrm{T}_{s,\infty}}=\sup_{u\neq0}|(\Phi\mu)(u)|$. The seminorm properties now follow from the corresponding properties of the $L^p$ and supremum norms composed with a linear map.

\textbf{(i) Non-negativity.}
For every $u\neq0$
$$
\left|\frac{\widehat\mu(u)-\widehat\mu(0)}{\|u\|^{s}}\right|\ge0
$$
so $|\mu|_{\mathrm{T}_{s,p}}\ge 0$ for all $p\in[1,\infty]$ (the value $+\infty$ is allowed for the extended seminorm on all of $\mathcal M(\R^d)$. On the finite domain, all assertions below hold verbatim).
Definiteness is \emph{not} asserted. The kernel of the seminorm is computed exactly in Remark~\ref{rem:kernel}, where it is shown that $|\mu|_{\mathrm{T}_{s,p}}=0$ if and only if $\mu=\mu(\R^d)\,\delta_0$, using the uniqueness theorem for Fourier--Stieltjes transforms of finite signed measures.

\textbf{(ii) Absolute homogeneity.}
For any scalar $ a \in \mathbb{R}$, linearity of $\Phi$ gives $\Phi(a\mu)=a\,\Phi\mu$, hence for $1\le p<\infty$,
\[
\begin{aligned}
|a\mu|_{\mathrm{T}_{s,p}}
&=\left(\int_{\mathbb{R}^d\setminus\{0\}}
\left|\frac{a\big(\widehat\mu(u)-\widehat\mu(0)\big)}{\|u\|^s}\right|^p\,du\right)^{1/p}\\
&=|a|\left(\int_{\mathbb{R}^d\setminus\{0\}}
\left|\frac{\widehat\mu(u)-\widehat\mu(0)}{\|u\|^s}\right|^p\,du\right)^{1/p}\\
&=|a|\,|\mu|_{\mathrm{T}_{s,p}}
\end{aligned}
\]
and for $p=\infty$
$$
| a \mu |_{\mathrm{T}_{s,\infty}}
=\sup_{u\ne0}\left| \frac{a \big(\widehat\mu(u)-\widehat\mu(0)\big)}{\|u\|^{s}} \right|
= |a|\,\sup_{u\ne0}\left| \frac{\widehat\mu(u)-\widehat\mu(0)}{\|u\|^{s}} \right|
= |a|\, | \mu |_{\mathrm{T}_{s,\infty}}
$$

\textbf{(iii) Subadditivity.}
By linearity of $\Phi$, $\Phi(\mu+\nu)=\Phi\mu+\Phi\nu$ pointwise on $\R^d\setminus\{0\}$. For $1\le p<\infty$, Minkowski's inequality in $L^p(\R^d\setminus\{0\})$ yields
\begin{align*}
| \mu + \nu |_{\mathrm{T}_{s,p}} &= \left( \int_{\mathbb{R}^d\setminus\{0\}} \left| \frac{\big(\widehat\mu(u)-\widehat\mu(0)\big) + \big(\widehat\nu(u)-\widehat\nu(0)\big)}{\|u\|^{s}} \right|^p   du \right)^{\frac{1}{p}}\\
& \leq \left( \int_{\mathbb{R}^d\setminus\{0\}} \left| \frac{\widehat\mu(u)-\widehat\mu(0)}{\|u\|^{s}} \right|^p   du \right)^{\frac{1}{p}}
+ \left( \int_{\mathbb{R}^d\setminus\{0\}} \left| \frac{\widehat\nu(u)-\widehat\nu(0)}{\|u\|^{s}} \right|^p   du \right)^{\frac{1}{p}}
\end{align*}
For $p=\infty$, the pointwise triangle inequality followed by taking suprema gives
$$
\sup_{u\neq0}\left|\frac{\big(\widehat\mu(u)-\widehat\mu(0)\big)+\big(\widehat\nu(u)-\widehat\nu(0)\big)}{\|u\|^{s}}\right|
\le
\sup_{u\neq0}\left|\frac{\widehat\mu(u)-\widehat\mu(0)}{\|u\|^{s}}\right|
+
\sup_{u\neq0}\left|\frac{\widehat\nu(u)-\widehat\nu(0)}{\|u\|^{s}}\right|
$$
Hence, for all $p\in[1,\infty]$
$$
| \mu + \nu  |_{\mathrm{T}_{s,p}} \leq | \mu |_{\mathrm{T}_{s,p}} + | \nu |_{\mathrm{T}_{s,p}}
$$
with the convention that the right-hand side may be \(+\infty\). In particular, if \(\mu,\nu\in\mathcal D_{s,p}\), then \(\mu+\nu\in\mathcal D_{s,p}\). Absolute homogeneity gives closure under scalar multiplication. Thus \(\mathcal D_{s,p}\) is a linear subspace and the restriction of \(|\cdot|_{\mathrm{T}_{s,p}}\) to \(\mathcal D_{s,p}\) is a seminorm. This completes the proof.
\end{proof}

\subsection{Proof of Proposition~\ref{prop:window-sharp}}\label{proof:window-sharp}
\begin{proof}
Fix $x\neq0$. Since $\Bchi_{\delta_0}(u)\equiv1$ and $\Bchi_{\delta_x}(u)=e^{i\langle u,x\rangle}$,
\[
\mathrm{T}_{s,p}(\delta_0,\delta_x)^p
=
\int_{\R^d\setminus\{0\}}\frac{\big|1-e^{i\langle u,x\rangle}\big|^p}{\|u\|^{ps}}\ du
\]
Throughout we use the elementary identity
\begin{equation}\label{eq:one-minus-exp}
\big|1-e^{it}\big|^2=(1-\cos t)^2+\sin^2t=2(1-\cos t)=4\sin^2(t/2),
\qquad t\in\R
\end{equation}
so that $|1-e^{it}|=2|\sin(t/2)|$.

\smallskip
\noindent\textbf{Step 1. Sufficiency inside the window.}
If $\frac dp<s<1+\frac dp$, then $\delta_0,\delta_x\in\mathcal P_p(\R^d)$ and Proposition~\ref{prop:integrability} gives $\mathrm{T}_{s,p}(\delta_0,\delta_x)<\infty$.

\smallskip
\noindent\textbf{Step 2. Divergence at the origin when $s\ge 1+\frac dp$.}
On the region $\{u:\|u\|\le \pi/\|x\|\}$ we have, by Cauchy--Schwarz, $|\langle u,x\rangle|\le\|u\|\|x\|\le\pi$, hence $|\langle u,x\rangle|/2\in[0,\pi/2]$.
On $[0,\pi/2]$ the sine function is concave with $\sin0=0$ and $\sin(\pi/2)=1$, so $\sin t\ge \frac{2}{\pi}t$ there. Applying this with $t=|\langle u,x\rangle|/2$ in \eqref{eq:one-minus-exp} gives
\[
\big|1-e^{i\langle u,x\rangle}\big|
=2\big|\sin(\langle u,x\rangle/2)\big|
\ge \frac{2}{\pi}\,\big|\langle u,x\rangle\big|,
\qquad \|u\|\le \frac{\pi}{\|x\|}
\]
Therefore, writing $\theta=u/\|u\|$ and passing to polar coordinates $u=r\theta$,
\[
\int_{\|u\|\le \pi/\|x\|}\frac{|1-e^{i\langle u,x\rangle}|^p}{\|u\|^{ps}}\ du
\ \ge\
\Big(\frac{2}{\pi}\Big)^{p}\|x\|^{p}
\int_{\mathbb S^{d-1}}\big|\langle\theta,e\rangle\big|^p\ d\sigma(\theta)
\int_0^{\pi/\|x\|} r^{\,p-ps+d-1}\ dr
\]
where $e:=x/\|x\|$ and $\sigma$ is the surface measure on $\mathbb S^{d-1}$. The angular factor
\[
\kappa_{d,p}:=\int_{\mathbb S^{d-1}}|\langle\theta,e\rangle|^p\ d\sigma(\theta)
\]
is a strictly positive finite constant (the integrand is continuous, nonnegative, and not identically zero. By rotational invariance $\kappa_{d,p}$ does not depend on $e$). The radial factor $\int_0^{\pi/\|x\|}r^{\,p-ps+d-1}\,dr$ diverges precisely when $p-ps+d\le0$, i.e.\ when $s\ge1+\frac dp$. Hence $\mathrm{T}_{s,p}(\delta_0,\delta_x)=\infty$ whenever $s\ge1+\frac dp$.

\smallskip
\noindent\textbf{Step 3. Divergence at infinity when $s\le \frac dp$.}
We show that the oscillation $|1-e^{i\langle u,x\rangle}|$ does not decay on average over large annuli, so the tail behaves exactly like $\int^\infty r^{\,d-1-ps}\,dr$.
For $R>0$ let $B_R:=\{u:\|u\|\le R\}$ and recall the classical Bessel-function formula for the Fourier transform of a ball indicator \citep[Appendix~B]{grafakos2008classical}.
\[
\int_{B_R}e^{i\langle u,x\rangle}\ du
=(2\pi)^{d/2}\,\frac{R^{d/2}}{\|x\|^{d/2}}\,J_{d/2}\big(R\|x\|\big)
\]
Together with the standard decay estimate $|J_{d/2}(t)|\le C_d\,t^{-1/2}$ for $t\ge1$, this yields
\begin{equation}\label{eq:ball-osc}
\Big|\int_{B_R}e^{i\langle u,x\rangle}\ du\Big|
\le C_d'\,\|x\|^{-\frac{d+1}{2}}\,R^{\frac{d-1}{2}},
\qquad R\|x\|\ge1
\end{equation}
Consider the dyadic annuli $A_k:=\{u:2^k\le\|u\|\le 2^{k+1}\}$, whose volume is $|A_k|=c_d\,2^{kd}$ with $c_d:=(2^d-1)|B_1|>0$. Writing $y(u):=2\big(1-\cos\langle u,x\rangle\big)=|1-e^{i\langle u,x\rangle}|^2\in[0,4]$ and using \eqref{eq:ball-osc} for the two balls $B_{2^{k+1}}$ and $B_{2^{k}}$,
\[
\int_{A_k}y(u)\ du
=2|A_k|-2\,\mathrm{Re}\int_{A_k}e^{i\langle u,x\rangle}\ du
\ \ge\
2|A_k|-4C_d'\,\|x\|^{-\frac{d+1}{2}}\,2^{(k+1)\frac{d-1}{2}}
\]
Since $2^{k(d-1)/2}=o(2^{kd})$ as $k\to\infty$, there exists $k_*=k_*(d,x)$ such that for all $k\ge k_*$,
\begin{equation}\label{eq:annulus-mass}
\int_{A_k}y(u)\ du\ \ge\ \tfrac32\,|A_k|
\end{equation}
We now convert \eqref{eq:annulus-mass} into a lower bound for $\int_{A_k}y^{p/2}$, distinguishing two cases.
If $p\ge2$, the map $t\mapsto t^{p/2}$ is convex on $[0,\infty)$, so Jensen's inequality with respect to the normalized measure $|A_k|^{-1}du$ on $A_k$ gives
\[
\frac{1}{|A_k|}\int_{A_k}y(u)^{p/2}\ du
\ \ge\
\Big(\frac{1}{|A_k|}\int_{A_k}y(u)\ du\Big)^{p/2}
\ \ge\ \Big(\tfrac32\Big)^{p/2}\ \ge\ 1
\]
If $1\le p<2$, then since $0\le y\le4$ and $\frac p2-1\le0$ we have the pointwise bound $y^{p/2}=y\cdot y^{\frac p2-1}\ge 4^{\frac p2-1}\,y$, hence
\[
\int_{A_k}y(u)^{p/2}\ du\ \ge\ 4^{\frac p2-1}\int_{A_k}y(u)\ du\ \ge\ \tfrac32\cdot4^{\frac p2-1}\,|A_k|
\]
In both cases there is a constant $c_p>0$ with
\[
\int_{A_k}\big|1-e^{i\langle u,x\rangle}\big|^p\ du
=\int_{A_k}y(u)^{p/2}\ du\ \ge\ c_p\,|A_k|,
\qquad k\ge k_*
\]
On $A_k$ we have $\|u\|^{-ps}\ge 2^{-(k+1)ps}$, so
\[
\int_{A_k}\frac{|1-e^{i\langle u,x\rangle}|^p}{\|u\|^{ps}}\ du
\ \ge\ c_p\,c_d\,2^{-ps}\ 2^{k(d-ps)}
\]
Summing over $k\ge k_*$, the geometric-type series $\sum_k 2^{k(d-ps)}$ diverges precisely when $d-ps\ge0$, i.e.\ when $s\le\frac dp$. Hence $\mathrm{T}_{s,p}(\delta_0,\delta_x)=\infty$ whenever $s\le\frac dp$.

\smallskip
Combining Steps 1--3 proves the claimed equivalence, and the final assertion follows because $\delta_0,\delta_x$ possess finite moments of all orders, hence lie in $\mathcal P_p(\R^d)$ for every $p\in[1,\infty)$.
\end{proof}
\subsection{Proof of Proposition \ref{prop:embeddings}}\label{proof:embeddings}
\begin{proof}
The proof has two parts. First we establish the inclusion together with the quantitative estimate \eqref{eq:embedding_estimate}. Then we prove that the inclusion map is continuous, which does not follow from \eqref{eq:embedding_estimate} alone because of the additive constant on its right-hand side.

\medskip
\noindent\textbf{Part 1. Inclusion and the estimate \eqref{eq:embedding_estimate}.}
Fix $\mu\in\mathcal P^{\mathrm T}_{s,p}(\R^d)$ and split $\R^d\setminus\{0\}=\{\|u\|<1\}\cup\{\|u\|\ge1\}$.
For $\|u\|<1$ and $r\le s$,
\[
\frac{|\Bchi_\mu(u)-1|^q}{\|u\|^{qr}}
=
\Big(\frac{|\Bchi_\mu(u)-1|^p}{\|u\|^{ps}}\Big)^{q/p}
\|u\|^{q(s-r)}
 \le
\Big(\frac{|\Bchi_\mu(u)-1|^p}{\|u\|^{ps}}\Big)^{q/p}
\]
since $\|u\|^{q(s-r)}\le 1$ on $\{\|u\|<1\}$.
Hence, writing $A(u):=\frac{|\Bchi_\mu(u)-1|^p}{\|u\|^{ps}}$, we have
\[
\int_{\|u\|<1}\frac{|\Bchi_\mu(u)-1|^q}{\|u\|^{qr}} du
 \le
\int_{\|u\|<1} A(u)^{q/p} du
\]
Because $\{\|u\|<1\}$ has finite Lebesgue measure and $q\le p$, H\"older's inequality with exponents $p/q$ and $(1-q/p)^{-1}$ gives
\[
\int_{\|u\|<1} A(u)^{q/p} du
\le
|B_1|^{ 1-\frac{q}{p}}
\Big(\int_{\|u\|<1}A(u) du\Big)^{q/p}
\le
|B_1|^{ 1-\frac{q}{p}}
\|\mu\|_{\mathrm{T}_{s,p}}^{ q}
\]
Therefore
\begin{equation}\label{eq:local_bound}
\Big(\int_{\|u\|<1}\frac{|\Bchi_\mu(u)-1|^q}{\|u\|^{qr}} du\Big)^{1/q}
\le
|B_1|^{ \frac1q-\frac1p} \|\mu\|_{\mathrm{T}_{s,p}}
\end{equation}

\smallskip
On the tail part $\{\|u\|\ge1\}$, since $|\Bchi_\mu(u)-1|\le 2$ for all $u$, we have
\[
\int_{\|u\|\ge1}\frac{|\Bchi_\mu(u)-1|^q}{\|u\|^{qr}} du
\le
2^q\int_{\|u\|\ge1}\|u\|^{-qr} du
\]
The last integral is finite exactly under $qr > d$. Thus
\begin{equation}\label{eq:tail_bound_embed}
\Big(\int_{\|u\|\ge1}\frac{|\Bchi_\mu(u)-1|^q}{\|u\|^{qr}} du\Big)^{1/q}
\le
2\Big(\int_{\|u\|\ge1}\|u\|^{-qr} du\Big)^{1/q}
\end{equation}

\smallskip
Using $(a+b)^{1/q}\le a^{1/q}+b^{1/q}$ for $a,b\ge0$ and \eqref{eq:local_bound}--\eqref{eq:tail_bound_embed}
yields \eqref{eq:embedding_estimate}. In particular $\|\mu\|_{\mathrm{T}_{r,q}}<\infty$, so $\mu\in\mathcal P^{\mathrm T}_{r,q}(\R^d)$, proving the inclusion.

\medskip
\noindent\textbf{Part 2. Continuity of the inclusion.}
Estimate \eqref{eq:embedding_estimate} contains the additive constant $2\big(\int_{\|u\|\ge1}\|u\|^{-qr}du\big)^{1/q}$, which does not shrink with the distance. It therefore proves boundedness of the embedded set but not continuity. We now show directly that
\[
\mathrm{T}_{s,p}(\mu_n,\mu)\to 0
\quad\Longrightarrow\quad
\mathrm{T}_{r,q}(\mu_n,\mu)\to 0
\qquad (\mu_n,\mu\in\mathcal P^{\mathrm T}_{s,p}(\R^d))
\]
which is the assertion that the inclusion map is continuous (indeed uniformly continuous on the class, as the modulus below does not depend on $\mu_n,\mu$).
Set $f_n:=\Bchi_{\mu_n}-\Bchi_\mu$ and, for $R\ge 1$ to be chosen, split
\[
\mathrm{T}_{r,q}(\mu_n,\mu)
\le
\underbrace{\Big(\int_{\|u\|<1}\frac{|f_n|^q}{\|u\|^{qr}}\,du\Big)^{1/q}}_{=:I_n}
+\underbrace{\Big(\int_{1\le\|u\|\le R}\frac{|f_n|^q}{\|u\|^{qr}}\,du\Big)^{1/q}}_{=:II_n(R)}
+\underbrace{\Big(\int_{\|u\|>R}\frac{|f_n|^q}{\|u\|^{qr}}\,du\Big)^{1/q}}_{=:III_n(R)}
\]
using Minkowski's inequality for the decomposition of the domain.

\emph{The low-frequency piece $I_n$.} Exactly as in Part 1, with $A_n(u):=|f_n(u)|^p\|u\|^{-ps}$ in place of $A(u)$ (the argument used only $r\le s$, $q\le p$, and the finite measure of the unit ball),
\[
I_n\ \le\ |B_1|^{\frac1q-\frac1p}\,\Big(\int_{\|u\|<1}A_n(u)\,du\Big)^{1/p}
\ \le\ |B_1|^{\frac1q-\frac1p}\ \mathrm{T}_{s,p}(\mu_n,\mu)
\]

\emph{The intermediate piece $II_n(R)$.} On $\{1\le\|u\|\le R\}$ we write, as in Part 1,
\[
\frac{|f_n(u)|^q}{\|u\|^{qr}}
=
A_n(u)^{q/p}\,\|u\|^{q(s-r)}
\le
R^{\,q(s-r)}\,A_n(u)^{q/p}
\]
since $s\ge r$ and $\|u\|\le R$ there. H\"older's inequality on the finite-measure set $\{1\le\|u\|\le R\}$ (volume $|B_R|-|B_1|\le|B_1|R^d$) then gives
\[
II_n(R)
\le
R^{\,s-r}\,\big(|B_1|R^d\big)^{\frac1q-\frac1p}\,
\Big(\int_{1\le\|u\|\le R}A_n(u)\,du\Big)^{1/p}
\le
C_{d,p,q}\ R^{\,s-r+\frac dq-\frac dp}\ \mathrm{T}_{s,p}(\mu_n,\mu)
\]

\emph{The tail piece $III_n(R)$.} Using only $|f_n|\le2$,
\[
III_n(R)\ \le\ 2\Big(\int_{\|u\|>R}\|u\|^{-qr}\,du\Big)^{1/q}=:\tau(R),
\qquad
\tau(R)=2\Big(\frac{v_d}{qr-d}\Big)^{1/q}R^{\frac dq-r}\xrightarrow[R\to\infty]{}0
\]
where the explicit polar-coordinate evaluation $\int_{\|u\|>R}\|u\|^{-qr}du=\frac{v_d}{qr-d}R^{d-qr}$ is finite because $qr>d$.

\emph{Conclusion.} Let $\varepsilon>0$. First choose $R=R(\varepsilon)\ge1$ so large that $\tau(R)<\varepsilon/2$. With this $R$ fixed, the first two pieces are bounded by
\[
\Big(|B_1|^{\frac1q-\frac1p}+C_{d,p,q}R^{\,s-r+\frac dq-\frac dp}\Big)\,\mathrm{T}_{s,p}(\mu_n,\mu)
\]
which tends to $0$ as $n\to\infty$. Choose $N$ so that it is $<\varepsilon/2$ for all $n\ge N$. Then $\mathrm{T}_{r,q}(\mu_n,\mu)<\varepsilon$ for all $n\ge N$. Since $\varepsilon>0$ was arbitrary, $\mathrm{T}_{r,q}(\mu_n,\mu)\to0$, and the inclusion map is continuous.
\end{proof}
\subsection{Proof of Proposition~\ref{prop:asymptotics}}\label{proof:asymptotics}
\begin{proof}
Let $X\sim\mu$ and $Y\sim\nu$. Using $|e^{it}-1|\le |t|$,
\[
|\Bchi_\mu(u)-\Bchi_\nu(u)|
\le |\Bchi_\mu(u)-1|+|\Bchi_\nu(u)-1|
\le \|u\|(\E\|X\|+\E\|Y\|)
\]
Hence for $\|u\|<1$,
\[
f(u)=
\frac{|\Bchi_{\mu}(u)-\Bchi_{\nu}(u)|}{\|u\|^{s}}
\le (\E\|X\|+\E\|Y\|) \|u\|^{1-s}
\le \E\|X\|+\E\|Y\|<\infty
\]
Here the last inequality uses \(s\le1\), \(\|u\|<1\), and
\(\mu,\nu\in\mathcal P_1(\R^d)\).
For $\|u\|\ge1$, $|\Bchi_\mu(u)-\Bchi_\nu(u)|\le2$, so $f(u)\le 2$.
Thus $f\in L^\infty(\R^d)$.

\smallskip
Thus, on $\{\|u\|<1\}$ we have $f\in L^\infty$ and the set has finite Lebesgue measure, so
$\int_{\|u\|<1} |f(u)|^p du<\infty$.
On $\{\|u\|\ge1\}$,
\[
|f(u)|^p \le 2^p \|u\|^{-ps}
\]
and $\int_{\|u\|\ge1}\|u\|^{-ps}du<\infty$ iff $ps>d$. Hence $f\in L^p(\R^d)$ for all $p>d/s$.

\smallskip
Since $f\in L^\infty$ and $f\in L^{q_0}$ for every fixed $q_0>d/s$, we can compute the limit of $\|f\|_{L^p}$ as $p\to\infty$. Note first that on the \emph{infinite} measure space $(\R^d,du)$ the naive inequality $\|f\|_{L^p}\le\|f\|_\infty$ is false in general (constants are a counterexample), so the upper bound requires an interpolation argument against a fixed integrable power.

\emph{Upper bound via interpolation.} If $\|f\|_\infty=0$ then $f=0$ a.e.\ and both sides of the claimed limit vanish, so assume $\|f\|_\infty>0$. Fix $q_0>d/s$, so that $\|f\|_{L^{q_0}}<\infty$ by the previous paragraph. For any $p>q_0$, split the exponent as $p=(p-q_0)+q_0$ and bound
\[
\|f\|_{L^p}^{\,p}
=\int_{\R^d}|f(u)|^{\,p-q_0}\,|f(u)|^{\,q_0}\ du
\le \|f\|_\infty^{\,p-q_0}\int_{\R^d}|f(u)|^{\,q_0}\ du
= \|f\|_\infty^{\,p-q_0}\,\|f\|_{L^{q_0}}^{\,q_0}
\]
Taking $p$-th roots,
\[
\|f\|_{L^p}
\le
\|f\|_\infty^{\,1-\frac{q_0}{p}}\ \|f\|_{L^{q_0}}^{\,\frac{q_0}{p}}
\]
As $p\to\infty$, the right-hand side converges to $\|f\|_\infty\cdot 1=\|f\|_\infty$ (both exponents converge, and the bases are fixed positive numbers. If $\|f\|_{L^{q_0}}=0$ then $f=0$ a.e.\ and the claim is trivial). Hence
\[
\limsup_{p\to\infty}\|f\|_{L^p}\ \le\ \|f\|_\infty
\]

\emph{Lower bound via level sets.} For any $\varepsilon\in(0,\|f\|_\infty)$, the set
$A_\varepsilon:=\{u: |f(u)|>\|f\|_\infty-\varepsilon\}$ has positive Lebesgue measure by the definition of the essential supremum (which coincides with the supremum here since $f$ is continuous on $\R^d\setminus\{0\}$), and it has \emph{finite} measure. On $A_\varepsilon$ we have $|f|^{q_0}>(\|f\|_\infty-\varepsilon)^{q_0}$, so $|A_\varepsilon|\le(\|f\|_\infty-\varepsilon)^{-q_0}\|f\|_{L^{q_0}}^{q_0}<\infty$. Thus
\[
\|f\|_{L^p}\ge \Big(\int_{A_\varepsilon}|f|^p\Big)^{1/p}\ge(\|f\|_\infty-\varepsilon)\, |A_\varepsilon|^{1/p}\xrightarrow[p\to\infty]{}\|f\|_\infty-\varepsilon
\]
because $|A_\varepsilon|^{1/p}\to1$ for any fixed $0<|A_\varepsilon|<\infty$. Hence $\liminf_{p\to\infty}\|f\|_{L^p}\ge\|f\|_\infty-\varepsilon$, and letting $\varepsilon\downarrow0$,
\[
\liminf_{p\to\infty}\|f\|_{L^p}\ \ge\ \|f\|_\infty
\]

Combining the two bounds gives $\lim_{p\to\infty}\|f\|_{L^p}=\|f\|_\infty$.

Finally, $\mathrm{T}_{s,p}(\mu,\nu)=\|f\|_{L^p}$ for all $p$ large enough that $d/p<s$ (so that the norm is finite), and $\mathrm{T}_{s,\infty}(\mu,\nu)=\|f\|_{L^\infty}$ (the supremum equals the essential supremum by continuity of $f$ on $\R^d\setminus\{0\}$), proving the claim.
\end{proof}
\subsection{Proof of Proposition \ref{prop:Compactness}}\label{proof:Compactness}
\begin{proof}
We use the classical Kolmogorov--Riesz--Fr\'echet compactness criterion in $L^p(\R^d)$, $1\le p<\infty$ \citep[Theorem~4.26 and Corollary~4.27]{brezis2011functional}. A bounded family $\mathcal F\subset L^p(\R^d)$ is relatively compact if and only if it is $L^p$--translation equicontinuous, $\sup_{f\in\mathcal F}\|f(\cdot+h)-f\|_{L^p}\to0$ as $h\to0$, and has uniformly small tails, $\sup_{f\in\mathcal F}\int_{\|u\|>R}|f|^p\,du\to0$ as $R\to\infty$. We identify each measure $\mu\in\mathbf S$ with the function $f_\mu(u)=\frac{\Bchi_\mu(u)-1}{\|u\|^s}$, $u\in\R^d\setminus\{0\}$ (extended by any fixed value at $u=0$, a Lebesgue-null point), so that the Toscani norm satisfies $\|\mu\|_{\mathrm{T}_{s,p}}=\| f_\mu\|_{L^p(\R^d)}$, and verify the three conditions for the family $\{f_\mu:\mu\in\mathbf S\}$.
Let $M_1:=\sup_{\mu\in\mathbf S}\int\|x\| d\mu(x)<\infty$.

\smallskip
\noindent\textbf{(i) }
For all $\mu$ and all $u$, $|\Bchi_\mu(u)-1|\le 2$, hence for any $R\ge1$,
\[
\sup_{\mu\in\mathbf S}\int_{\|u\|>R}|f_\mu(u)|^p du
\le 2^p\int_{\|u\|>R}\|u\|^{-ps} du \xrightarrow[R\to\infty]{} 0
\]
since $ps>d$.

\smallskip
\noindent\textbf{(ii)}
Using $|e^{it}-1|\le |t|$,
\[
|\Bchi_\mu(u)-1|
=\Big|\int (e^{i\langle u,x\rangle}-1) d\mu(x)\Big|
\le \int |\langle u,x\rangle| d\mu(x)
\le \|u\| \int \|x\| d\mu(x)
\le M_1\|u\|
\]
Thus for $\|u\|<1$,
\[
|f_\mu(u)|^p \le M_1^p \|u\|^{p-ps}
\]
and since $s<1+\frac{d}{p}$ we have $p-ps>-d$, so
\[
\sup_{\mu\in\mathbf S}\int_{\|u\|<\delta}|f_\mu(u)|^p du
\le M_1^p\int_{\|u\|<\delta}\|u\|^{p-ps} du \xrightarrow[\delta\downarrow0]{} 0
\]

\smallskip
\noindent\textbf{(iii)}
Fix $\varepsilon>0$. Choose $\delta\in(0,1)$ so that
\[
\sup_{\mu\in\mathbf S}\int_{\|u\|<2\delta}|f_\mu(u)|^p du < (\varepsilon/4)^p
\]
Then for any $h$ with $\|h\|<\delta$,
\[
\int_{\|u\|<\delta}|f_\mu(u+h)-f_\mu(u)|^p du
\le 2^{p-1}\int_{\|u\|<\delta}|f_\mu(u+h)|^p du + 2^{p-1}\int_{\|u\|<\delta}|f_\mu(u)|^p du
\le (\varepsilon/2)^p
\]
uniformly in $\mu$, because $\{ \|u\|<\delta\}+h\subset\{\|u\|<2\delta\}$.

Now consider $\|u\|\ge \delta$ and $\|h\|<\delta/2$. Using again $|e^{it}-e^{is}|\le |t-s|$,
\[
|\Bchi_\mu(u+h)-\Bchi_\mu(u)|
\le \int |e^{i\langle u+h,x\rangle}-e^{i\langle u,x\rangle}| d\mu(x)
\le \int |\langle h,x\rangle| d\mu(x)
\le M_1\|h\|
\]
Also, since $\|h\|<\delta/2\le\|u\|/2$ we have $\|u+h\|\ge\|u\|-\|h\|\ge\|u\|/2$, so the mean
value theorem applied to $r\mapsto r^{-s}$ on the segment between $\|u\|$ and $\|u+h\|$ (both
$\ge\|u\|/2$) gives
\[
\Big|\|u+h\|^{-s}-\|u\|^{-s}\Big|
\le s\,\big(\|u\|/2\big)^{-s-1}\,\big|\,\|u+h\|-\|u\|\,\big|
\le s\,2^{s+1}\, \|h\|\, \|u\|^{-s-1}
\]
Therefore, using $|\Bchi_\mu(u)-1|\le 2$ and $\|u+h\|^{-s}\le 2^{s}\|u\|^{-s}$,
\[
|f_\mu(u+h)-f_\mu(u)|
\le \frac{M_1\|h\|}{\|u+h\|^{s}} + 2\,s\,2^{s+1}\, \|h\|\, \|u\|^{-s-1}
\le 2^{s}M_1\, \|h\|\, \|u\|^{-s} + s\,2^{s+2}\, \|h\|\, \|u\|^{-s-1}
\]
Raising to the $p$th power and integrating over $\{\|u\|\ge\delta\}$ gives
\[
\sup_{\mu\in\mathbf S}\int_{\|u\|\ge\delta}|f_\mu(u+h)-f_\mu(u)|^p du
\le C \|h\|^p
\]
since $\int_{\|u\|\ge\delta}\|u\|^{-ps} du<\infty$ (as $ps>d$) and
$\int_{\|u\|\ge\delta}\|u\|^{-p(s+1)} du<\infty$ as well.
Thus for $\|h\|$ small enough the latter is $<(\varepsilon/2)^p$.

Combining the $\|u\|<\delta$ and $\|u\|\ge\delta$ pieces yields
\[
\lim_{h\to0}\ \sup_{\mu\in\mathbf S}\|f_\mu(\cdot+h)-f_\mu\|_{L^p}=0
\]

\smallskip
\noindent\textbf{(iv)}
The family is bounded in $L^p$, as a consequence of the first-moment bound alone. The estimates
used in (i) and (ii) give $|\Bchi_\mu(u)-1|\le\min\{2,\,M_1\|u\|\}$ for every $\mu\in\mathbf S$,
hence
\[
\|f_\mu\|_{L^p(\R^d)}^p
\le M_1^p\int_{\|u\|<1}\|u\|^{\,p-ps}\,du+2^p\int_{\|u\|\ge1}\|u\|^{-ps}\,du
= \frac{M_1^p\,v_d}{d+p-ps}+\frac{2^p\,v_d}{ps-d}<\infty
\]
uniformly on $\mathbf S$. In particular $\mathbf S\subseteq\mathcal P^{\mathrm T}_{s,p}(\R^d)$
with uniformly bounded Toscani norms, proving the first two assertions of the statement. By (i)
the family has uniformly small tails, by (ii) it has uniformly small mass near the origin, and by
(iii) it is $L^p$--translation equicontinuous.
Hence, by the Kolmogorov--Riesz--Fréchet theorem, $\{f_\mu:\mu\in\mathbf S\}$ is relatively compact in $L^p(\R^d)$.

It remains only to note that the $L^p$ limits obtained this way are still represented by probability measures.
Let \((\mu_n)\subset\mathbf S\). The uniform first-moment bound implies tightness, so after passing to a
subsequence we may assume \(\mu_n\Rightarrow\mu\) for some probability measure \(\mu\). At the same time,
relative compactness in \(L^p\) gives a further subsequence such that
\[
f_{\mu_n}(u)=\frac{\Bchi_{\mu_n}(u)-1}{\|u\|^s}\longrightarrow f(u)
\quad\text{in }L^p
\]
The \(L^p\) convergence admits a further subsequence converging almost everywhere. Weak
convergence also gives \(\Bchi_{\mu_n}(u)\to\Bchi_\mu(u)\) for every \(u\), hence
\[
f(u)=\frac{\Bchi_\mu(u)-1}{\|u\|^s}\quad\text{for a.e. }u
\]
Thus \(J\mu=f\in L^p\), so \(\mu\in\mathcal P^{\mathrm T}_{s,p}\), and the chosen subsequence converges to
\(\mu\) in \(\mathrm T_{s,p}\). Therefore \(\mathbf S\) is relatively compact in the Toscani metric.
\end{proof}

\subsection{Proof of Theorem \ref{thm:completeness}}\label{proof:completeness}
\begin{proof}
Throughout, write $L^p:=L^p(\R^d\setminus\{0\})$ with respect to Lebesgue measure. Since $\{0\}$ is Lebesgue-null, $L^p(\R^d\setminus\{0\})$ and $L^p(\R^d)$ are canonically identified. Recall that $L^p$ is a Banach space for every $1\le p\le\infty$ (Riesz--Fischer).

\smallskip
\noindent\emph{Step 1. $J$ is well defined.}
For $\mu\in\mathcal P^{\mathrm T}_{s,p}(\R^d)$ and $1\le p<\infty$,
\[
\|J\mu\|_{L^p}
=\Big(\int_{\R^d\setminus\{0\}}\Big|\frac{\Bchi_\mu(u)-1}{\|u\|^s}\Big|^p\,du\Big)^{1/p}
=\mathrm{T}_{s,p}(\mu,\delta_0)<\infty
\]
because $\Bchi_{\delta_0}\equiv1$, so membership in the Toscani class is exactly the statement $J\mu\in L^p$. For $p=\infty$ define $(J\mu)(u):=(\Bchi_\mu(u)-1)/\|u\|^s$ as well. Then $J\mu$ is continuous on the open set $\R^d\setminus\{0\}$, and for continuous functions on an open set the supremum coincides with the essential supremum (any point where the function is close to its supremum has a neighborhood on which it remains so, and neighborhoods have positive Lebesgue measure). Hence
\[
\|J\mu\|_{L^\infty}=\sup_{u\neq0}\Big|\frac{\Bchi_\mu(u)-1}{\|u\|^s}\Big|=\mathrm{T}_{s,\infty}(\mu,\delta_0)<\infty
\]

\smallskip
\noindent\emph{Step 2. $J$ is an isometry.}
For $\mu,\nu\in\mathcal P^{\mathrm T}_{s,p}(\R^d)$ and $1\le p<\infty$, pointwise on $\R^d\setminus\{0\}$,
\[
(J\mu)(u)-(J\nu)(u)=\frac{\big(\Bchi_\mu(u)-1\big)-\big(\Bchi_\nu(u)-1\big)}{\|u\|^s}=\frac{\Bchi_\mu(u)-\Bchi_\nu(u)}{\|u\|^s}
\]
hence
\[
\|J\mu-J\nu\|_{L^p}
=\Big(\int_{\R^d\setminus\{0\}} \Big|\frac{\Bchi_\mu(u)-\Bchi_\nu(u)}{\|u\|^s}\Big|^p du\Big)^{1/p}
=\mathrm{T}_{s,p}(\mu,\nu)
\]
and the same computation with suprema (again using continuity of $J\mu-J\nu$ on $\R^d\setminus\{0\}$ to identify sup and ess-sup) gives $\|J\mu-J\nu\|_{L^\infty}=\mathrm{T}_{s,\infty}(\mu,\nu)$. In particular $J$ is injective on $\mathcal P^{\mathrm T}_{s,p}(\R^d)$. If $J\mu=J\nu$ in $L^p$ then $\mathrm{T}_{s,p}(\mu,\nu)=0$, and the definiteness argument of Proposition~\ref{prop:integrability}(i) (continuity of characteristic functions plus L\'evy uniqueness) yields $\mu=\nu$.

\smallskip
\noindent\emph{Step 3. The image is closed when $1\le p<\infty$ and $ps\ge d$.}
Let $\mu_n\in\mathcal P^{\mathrm T}_{s,p}(\R^d)$ and $f\in L^p$ with $\|J\mu_n-f\|_{L^p}\to0$.
By the Riesz representation theorem, the space of finite signed Borel measures on $\R^d$ is the
dual of $C_0(\R^d)$ \citep[Theorem~6.19]{rudin1987real}, and $C_0(\R^d)$ is separable. Hence
bounded sequences in the dual admit weak-* convergent subsequences
\citep[Corollary~3.30]{brezis2011functional}. Since $\|\mu_n\|_{C_0(\R^d)^*}=\mu_n(\R^d)=1$, the
sequence is bounded, and we obtain a subsequence
$(\mu_{n_k})$ and a finite measure $\nu$ with
$\int g\,d\mu_{n_k}\to\int g\,d\nu$ for every $g\in C_0(\R^d)$. Testing against nonnegative $g$
shows $\nu\ge0$, and testing against $0\le g\le1$ shows $\nu(\R^d)\le1$.

\emph{Identification of $f$.} Fix $\varphi\in C_c^\infty(\R^d\setminus\{0\})$ and set
$\psi(u):=\|u\|^{-s}\varphi(u)$. Since $\supp(\varphi)$ is a compact subset of
$\R^d\setminus\{0\}$, the factor $\|u\|^{-s}$ is smooth there, so
$\psi\in C_c^\infty(\R^d\setminus\{0\})\subset\mathcal S(\R^d)$. For any finite nonnegative
measure $\eta$, Fubini's theorem gives
\[
\int_{\R^d}\widehat\eta(u)\,\psi(u)\,du=\int_{\R^d}\widehat\psi(x)\,d\eta(x),
\qquad
\widehat\psi(x):=\int_{\R^d}e^{i\langle u,x\rangle}\psi(u)\,du\in\mathcal S(\R^d)\subset C_0(\R^d)
\]
Hence
\[
\begin{aligned}
\int (J\mu_{n_k})(u)\,\varphi(u)\,du
&=\int_{\R^d}\widehat\psi\,d\mu_{n_k}-\int_{\R^d}\psi(u)\,du\\
&\ \longrightarrow\
\int_{\R^d}\widehat\psi\,d\nu-\int_{\R^d}\psi(u)\,du
=\int_{\R^d}\frac{\Bchi_\nu(u)-1}{\|u\|^{s}}\,\varphi(u)\,du
\end{aligned}
\]
where $\Bchi_\nu(u):=\int e^{i\langle u,x\rangle}\,d\nu(x)$ and Fubini was used once more in the
last equality. On the other hand, $J\mu_{n_k}\to f$ in $L^p$ and $\varphi\in L^{p'}$, so
$\int(J\mu_{n_k})\varphi\to\int f\varphi$. Since $C_c^\infty(\R^d\setminus\{0\})$ separates
locally integrable functions on the open set $\R^d\setminus\{0\}$ (fundamental lemma of the
calculus of variations), we conclude
\[
f(u)=\frac{\Bchi_\nu(u)-1}{\|u\|^{s}}\qquad\text{for a.e.\ }u\in\R^d
\]

\emph{No loss of mass.} $\Bchi_\nu$ is continuous on $\R^d$ with
$\Bchi_\nu(0)=m:=\nu(\R^d)\in[0,1]$. If $m<1$, continuity provides $\eta>0$ with
$|\Bchi_\nu(u)-1|\ge\frac{1-m}{2}$ for $\|u\|\le\eta$, whence
\[
\int_{\|u\|\le\eta}|f(u)|^p\,du
\ \ge\ \Big(\frac{1-m}{2}\Big)^p\int_{\|u\|\le\eta}\|u\|^{-ps}\,du
= \Big(\frac{1-m}{2}\Big)^p\,v_d\int_0^{\eta}r^{\,d-1-ps}\,dr=\infty
\]
because $ps\ge d$. This contradicts $f\in L^p$. Hence $m=1$ and $\nu\in\mathcal P(\R^d)$. Then
$J\nu=f\in L^p$, so $\nu\in\mathcal P^{\mathrm T}_{s,p}(\R^d)$ and
$f\in J\big(\mathcal P^{\mathrm T}_{s,p}(\R^d)\big)$. The image is closed.

\smallskip
\noindent\emph{Step 4. The case $p=\infty$, any $s>0$.}
If $\|J\mu_n-f\|_{L^\infty}\to0$, pass to a weak-* subsequence and identify
$f=(\Bchi_\nu-1)\|u\|^{-s}$ a.e.\ exactly as in Step~3 (the pairing argument only uses
$\varphi\in C_c^\infty(\R^d\setminus\{0\})\subset L^1$). Both sides are continuous on
$\R^d\setminus\{0\}$, hence agree everywhere off the origin. If $m:=\nu(\R^d)<1$, then
$|\Bchi_\nu(u)-1|\to1-m>0$ as $u\to0$, while
$|\Bchi_\nu(u)-1|=\|u\|^{s}|f(u)|\le\|f\|_{L^\infty}\|u\|^{s}\to0$, a contradiction. Hence
$\nu\in\mathcal P(\R^d)$, $J\nu=f$, and the image is closed for every $s>0$.

\smallskip
\noindent\emph{Step 5. Completeness.}
If $(\mu_n)$ is $\mathrm T_{s,p}$-Cauchy, then $(J\mu_n)$ is Cauchy in the Banach space $L^p$,
hence converges to some $f\in L^p$. By Steps~3--4, $f=J\nu$ for some
$\nu\in\mathcal P^{\mathrm T}_{s,p}(\R^d)$, and
$\mathrm T_{s,p}(\mu_n,\nu)=\|J\mu_n-J\nu\|_{L^p}\to0$. Thus every Cauchy sequence converges in
$\big(\mathcal P^{\mathrm T}_{s,p}(\R^d),\mathrm{T}_{s,p}\big)$, which is therefore complete.
\end{proof}
\subsection{Proof of Proposition~\ref{prop:convergence}}\label{proof:convergence}
\begin{proof}
Write $f:=J\mu$, so that $\|J\mu_n-f\|_{L^p}=\mathrm T_{s,p}(\mu_n,\mu)\to0$.

\smallskip
\noindent\emph{Step 1. Every weak-* subsequential limit equals $\mu$.}
Set $g_n:=(\Bchi_{\mu_n}-\Bchi_\mu)\,\|\cdot\|^{-s}$, so that
$\|g_n\|_{L^p}=\mathrm T_{s,p}(\mu_n,\mu)\to0$. Note that only the \emph{difference} is assumed
integrable --- neither measure need lie in $\mathcal P^{\mathrm T}_{s,p}(\R^d)$. Let
$(\mu_{n_k})$ be an arbitrary subsequence. As in Step~3 of the proof of
Theorem~\ref{thm:completeness}, weak-* compactness provides a further subsequence (not relabeled)
converging in the weak-* sense against $C_0(\R^d)$ to a nonnegative measure $\nu$ with
$\nu(\R^d)\le1$. Fix $\varphi\in C_c^\infty(\R^d\setminus\{0\})$ and set
$\psi:=\|\cdot\|^{-s}\varphi\in\mathcal S(\R^d)$ as there. The pairing identity of that step,
applied to $\mu_{n_k}$ and to $\mu$, gives
\[
\begin{aligned}
\int_{\R^d} g_{n_k}(u)\,\varphi(u)\,du
&=\int_{\R^d}\widehat\psi\,d\mu_{n_k}-\int_{\R^d}\widehat\psi\,d\mu\\
&\ \longrightarrow\
\int_{\R^d}\widehat\psi\,d\nu-\int_{\R^d}\widehat\psi\,d\mu
=\int_{\R^d}\big(\Bchi_\nu(u)-\Bchi_\mu(u)\big)\,\|u\|^{-s}\varphi(u)\,du
\end{aligned}
\]
while H\"older's inequality gives $|\int g_{n_k}\varphi|\le\|g_{n_k}\|_{L^p}\|\varphi\|_{L^{p'}}\to0$.
Hence $\int(\Bchi_\nu-\Bchi_\mu)\|u\|^{-s}\varphi\,du=0$ for every
$\varphi\in C_c^\infty(\R^d\setminus\{0\})$, so $\Bchi_\nu=\Bchi_\mu$ a.e.\ on
$\R^d\setminus\{0\}$, and everywhere on $\R^d$ since both are continuous. In particular
$\nu(\R^d)=\Bchi_\nu(0)=\Bchi_\mu(0)=1$, so $\nu$ is a probability measure, and $\nu=\mu$ by the
uniqueness theorem for characteristic functions. (This argument uses no relation between $s$, $p$,
and $d$.)

\smallskip
\noindent\emph{Step 2. Weak convergence of the full sequence.}
By Step~1, every subsequence of $(\mu_n)$ has a further subsequence converging weak-* to the
\emph{probability} measure $\mu$. Weak-* convergence of probability measures to a probability
limit upgrades to weak convergence. Given $\varepsilon>0$, choose $g\in C_c(\R^d)$ with
$0\le g\le1$ and $\int g\,d\mu>1-\varepsilon$. Then $\int g\,d\mu_{n_k}>1-\varepsilon$
eventually, so the subsequence is uniformly tight, and for any bounded continuous $h$ and a
cutoff $g_R\in C_c$ with $0\le g_R\le1$ and $g_R\equiv1$ on $B(0,R)$,
\[
\Big|\int h\,d(\mu_{n_k}-\mu)\Big|
\le\Big|\int hg_R\,d(\mu_{n_k}-\mu)\Big|
+\|h\|_\infty\big(\mu_{n_k}(\{g_R<1\})+\mu(\{g_R<1\})\big)
\]
where the first term vanishes as $k\to\infty$ (since $hg_R$ is continuous and compactly
supported, hence in $C_0$) and the second is small uniformly in $k$ for $R$ large, by tightness.
Thus every subsequence has a further subsequence converging weakly to $\mu$. Since weak
convergence of probability measures is metrizable (e.g.\ by the bounded-Lipschitz metric), the
full sequence converges weakly to $\mu$, proving (i).

\smallskip
\noindent\emph{Step 3. Locally uniform convergence of characteristic functions.}
Weak convergence gives $\Bchi_{\mu_n}(u)\to\Bchi_\mu(u)$ pointwise (test against
$x\mapsto e^{i\langle u,x\rangle}$). For local uniformity, note that a weakly convergent sequence
of probability measures is uniformly tight, so given $\varepsilon>0$ there is $R$ with
$\sup_n\mu_n(\{\|x\|>R\})<\varepsilon$. The elementary bound
\[
|\Bchi_{\mu_n}(u)-\Bchi_{\mu_n}(v)|
\le\int_{\{\|x\|\le R\}}\big|e^{i\langle u,x\rangle}-e^{i\langle v,x\rangle}\big|\,d\mu_n
+2\mu_n(\{\|x\|>R\})
\le R\,\|u-v\|+2\varepsilon
\]
shows that $\{\Bchi_{\mu_n}\}$ is uniformly equicontinuous, and pointwise convergence of a
uniformly equicontinuous sequence is uniform on compact sets, proving (ii).
\end{proof}

\subsection{Proof of Theorem~\ref{thm:upper-holder}}\label{proof:upper-holder}
\begin{proof}
Fix $\mu,\nu\in\mathcal P_p(\mathbb R^d)$. If \(W_p(\mu,\nu)=0\), then \(\mu=\nu\) and the claim is immediate, so assume \(W_p(\mu,\nu)>0\). For any coupling $\pi\in\Pi(\mu,\nu)$ and every
\((x,y)\), the elementary bound
\[
|e^{i\langle u,x\rangle}-e^{i\langle u,y\rangle}|
\le|\langle u,x-y\rangle|
\le\|u\|\,\|x-y\|
\]
implies
\[
\begin{aligned}
|\Bchi_\mu(u)-\Bchi_\nu(u)|
&\le\int \|u\|\,\|x-y\|\,d\pi(x,y)\\
&=\|u\|\int \|x-y\|\,d\pi(x,y)
\end{aligned}
\]
Taking the infimum over $\pi$ yields
\begin{equation}\label{eq:chi-lip}
|\Bchi_\mu(u)-\Bchi_\nu(u)|\le \|u\| W_1(\mu,\nu)\le \|u\| W_p(\mu,\nu)
\end{equation}
where the last inequality uses monotonicity $W_1\le W_p$ for $p\ge 1$.
Also $|\Bchi_\mu(u)-\Bchi_\nu(u)|\le 2$.

Let $R>0$ (chosen below) and split frequencies.
\[
\mathrm{T}_{s,p}^p(\mu,\nu)
\le
\int_{\|u\|\le R}\frac{(\|u\|W_p(\mu,\nu))^p}{\|u\|^{ps}} du
 + 
\int_{\|u\|>R}\frac{2^p}{\|u\|^{ps}} du
\]
Using polar coordinates with $v_d:=|\mathbb S^{d-1}|$,
\[
W_p^p(\mu,\nu)\int_{\|u\|\le R}\|u\|^{p-ps} du
=W_p^p(\mu,\nu)\ 
v_d\int_0^R r^{d-1+p-ps} dr
=W_p^p(\mu,\nu)\ 
\frac{v_d}{d+p-ps} R^{d+p-ps}
\]
which is finite since $d+p-ps>0$ is equivalent to $s<1+\frac{d}{p}$.
Similarly,
\[
\int_{\|u\|>R}\|u\|^{-ps} du
=
v_d\int_R^\infty r^{d-1-ps} dr
=
\frac{v_d}{ps-d} R^{d-ps}
\]
finite since $ps-d>0$ is equivalent to $s>\frac{d}{p}$.
Therefore, with
\[
C_1(d,p,s):=\frac{v_d}{d+p-ps}\quad,
\qquad
C_2(d,p,s):=\frac{2^p v_d}{ps-d}
\]
we obtain, for every $R>0$,
\begin{equation}\label{eq:upper-split}
\mathrm{T}_{s,p}^p(\mu,\nu)
\le C_1(d,p,s)\ W_p^p(\mu,\nu)\ R^{d+p-ps} + C_2(d,p,s)\ R^{d-ps}
\end{equation}

The splitting radius $R$ in \eqref{eq:upper-split} is free, and the bound is valid for \emph{every} $R\in(0,\infty)$. In particular the choice below is admissible irrespective of the size of $W_p(\mu,\nu)$. Since $W:=W_p(\mu,\nu)\in(0,\infty)$ by assumption, we may take
\[
R:=\frac{1}{W}\in(0,\infty)
\]
Both terms of \eqref{eq:upper-split} then carry the same power of $W$. Indeed, for the low-frequency term,
\[
W^p\,R^{\,d+p-ps}
=W^{p}\,W^{-(d+p-ps)}
=W^{\,p-d-p+ps}
=W^{\,ps-d}
\]
and for the high-frequency term,
\[
R^{\,d-ps}=W^{-(d-ps)}=W^{\,ps-d}
\]
Substituting these two identities into \eqref{eq:upper-split} gives
\[
\mathrm{T}_{s,p}^p(\mu,\nu)
\le
\big(C_1(d,p,s)+C_2(d,p,s)\big)\ W_p(\mu,\nu)^{\,ps-d}
\]
Note that $ps-d>0$ by $s>\frac dp$, so the right-hand side is a genuine positive power of $W_p(\mu,\nu)$. Taking $p$-th roots and recalling $\frac{ps-d}{p}=s-\frac dp$,
\[
\mathrm{T}_{s,p}(\mu,\nu)
\le
\big(C_1(d,p,s)+C_2(d,p,s)\big)^{1/p}\ W_p(\mu,\nu)^{\,s-\frac dp}
=
C(d,p,s)\ W_p(\mu,\nu)^{\,s-\frac dp}
\]
which is exactly \eqref{eq:upper-holder} with the constant announced in the statement. This argument uses no case distinction on the size of $W_p(\mu,\nu)$. The two pointwise bounds \eqref{eq:chi-lip} and $|\Bchi_\mu-\Bchi_\nu|\le 2$ hold globally, and the optimizing radius $R=W_p^{-1}$ is admissible on all of $(0,\infty)$.

It remains to verify the sharpness assertion. Fix $x\neq 0$ and $\lambda>0$, and apply Lemma~\ref{lem:scaling} to the pair $(\delta_0,\delta_x)$, noting that $(D_\lambda)_\#\delta_0=\delta_0$ and $(D_\lambda)_\#\delta_x=\delta_{\lambda x}$. This gives
\[
\mathrm{T}_{s,p}(\delta_0,\delta_{\lambda x})=\lambda^{\,s-\frac dp}\ \mathrm{T}_{s,p}(\delta_0,\delta_{ x})\quad,
\qquad
W_p(\delta_0,\delta_{\lambda x})=\lambda\ W_p(\delta_0,\delta_{x})=\lambda\|x\|
\]
By Proposition~\ref{prop:window-sharp}, $t_0:=\mathrm{T}_{s,p}(\delta_0,\delta_x)\in(0,\infty)$ inside the window. Consequently the ratio
\[
\frac{\mathrm{T}_{s,p}(\delta_0,\delta_{\lambda x})}{W_p(\delta_0,\delta_{\lambda x})^{\,s-\frac dp}}
=
\frac{\lambda^{\,s-d/p}\,t_0}{\lambda^{\,s-d/p}\,\|x\|^{\,s-d/p}}
=
\frac{t_0}{\|x\|^{\,s-d/p}}
\]
is independent of $\lambda$, so \eqref{eq:upper-holder} is saturated up to a fixed constant along the whole dilation family. If \eqref{eq:upper-holder} held with some exponent $\theta>s-\frac dp$ and some finite constant $C'$, then evaluating along the same family would force $\lambda^{\,s-d/p}t_0\le C'\lambda^{\theta}\|x\|^{\theta}$ for all $\lambda>0$, i.e.\ $\lambda^{\,s-d/p-\theta}\le C'\|x\|^{\theta}/t_0$, which fails as $\lambda\downarrow0$ because $s-\frac dp-\theta<0$. Hence no larger exponent is admissible.
\end{proof}

\subsection{Proof of Lemma~\ref{lem:scaling}}\label{proof:scaling}
\begin{proof}
Fix $\lambda>0$ and define $D_\lambda(x)=\lambda x$ and $\mu^\lambda=(D_\lambda)_\#\mu$.

Let $\pi\in\Pi(\mu,\nu)$ be any coupling. Consider the pushforward coupling
\[
\pi^\lambda := (D_\lambda\times D_\lambda)_\#\pi \in \Pi(\mu^\lambda,\nu^\lambda)
\]
since the first marginal of $\pi^\lambda$ is $(D_\lambda)_\#\mu=\mu^\lambda$ and similarly for the second.
Then
\[
\int_{\R^d\times\R^d}\|x-y\|^p d\pi^\lambda(x,y)
=
\int_{\R^d\times\R^d}\|\lambda x-\lambda y\|^p d\pi(x,y)
=
\lambda^p\int_{\R^d\times\R^d}\|x-y\|^p d\pi(x,y)
\]
Taking the infimum over $\pi\in\Pi(\mu,\nu)$ yields
\[
W_p(\mu^\lambda,\nu^\lambda)^p
\le \lambda^p W_p^p(\mu,\nu)
\]
Applying the same argument with $\lambda^{-1}$ to $(\mu^\lambda,\nu^\lambda)$ gives the reverse inequality,
hence equality.
\[
W_p(\mu^\lambda,\nu^\lambda)=\lambda W_p(\mu,\nu)
\]

For, $\mathrm{T}_{s,p}$, we compute the characteristic function of the pushforward. For any $u\in\R^d$,
\begin{align*}
\Bchi_{\mu^\lambda}(u)
&=\int_{\R^d} e^{i\langle u,x\rangle} d\mu^\lambda(x)
=\int_{\R^d} e^{i\langle u,D_\lambda(x)\rangle} d\mu(x)
=\int_{\R^d} e^{i\langle u,\lambda x\rangle} d\mu(x)
=\Bchi_\mu(\lambda u)
\end{align*}
Therefore,
\begin{align*}
\mathrm{T}_{s,p}(\mu^\lambda,\nu^\lambda)^p
&=\int_{\R^d\setminus\{0\}}\frac{|\Bchi_{\mu^\lambda}(u)-\Bchi_{\nu^\lambda}(u)|^p}{\|u\|^{ps}} du =\int_{\R^d\setminus\{0\}}\frac{|\Bchi_{\mu}(\lambda u)-\Bchi_{\nu}(\lambda u)|^p}{\|u\|^{ps}} du
\end{align*}
Make the change of variables $v=\lambda u$, so $u=v/\lambda$ and $du=\lambda^{-d}dv$. Then
\[
\|u\|^{ps}=\Big\|\frac{v}{\lambda}\Big\|^{ps}=\lambda^{-ps}\|v\|^{ps}
\]
and hence
\begin{align*}
\mathrm{T}_{s,p}(\mu^\lambda,\nu^\lambda)^p
&=\int_{\R^d\setminus\{0\}}
\frac{|\Bchi_{\mu}(v)-\Bchi_{\nu}(v)|^p}{\lambda^{-ps}\|v\|^{ps}} \lambda^{-d} dv \\
&=\lambda^{ps-d}\int_{\R^d\setminus\{0\}}
\frac{|\Bchi_{\mu}(v)-\Bchi_{\nu}(v)|^p}{\|v\|^{ps}} dv \\
&=\lambda^{ps-d} \mathrm{T}_{s,p}^p(\mu,\nu)
\end{align*}
Taking $p$-th roots gives
\[
\mathrm{T}_{s,p}(\mu^\lambda,\nu^\lambda)=\lambda^{s-\frac{d}{p}} \mathrm{T}_{s,p}(\mu,\nu)
\]

We turn to (i). Assume that for some $\rho>0$ and $C>0$,
\[
W_p(\mu,\nu)\le C\ \mathrm{T}_{s,p}^\rho(\mu,\nu)\qquad\forall \mu,\nu\in\mathcal P_p(\R^d)
\]
Apply it to $(\mu^\lambda,\nu^\lambda)$ and use the scaling relations.
\[
\lambda W_p(\mu,\nu)
\le C\ \big(\lambda^{s-\frac{d}{p}}\ \mathrm{T}_{s,p}(\mu,\nu)\big)^\rho
= C \lambda^{\rho\left(s-\frac{d}{p}\right)} \mathrm{T}_{s,p}^\rho(\mu,\nu)
\]
Equivalently,
\[
W_p(\mu,\nu)\le C \lambda^{\rho\left(s-\frac{d}{p}\right)-1} \mathrm{T}_{s,p}^\rho(\mu,\nu)
\]
Fix once and for all a pair $(\mu,\nu)$ with $W_p(\mu,\nu)>0$ and $\mathrm{T}_{s,p}(\mu,\nu)<\infty$. For instance $\mu=\delta_0$ and $\nu=\delta_{e_1}$, which lies in $\mathcal P_p(\R^d)$ for every $p$ and has $\mathrm{T}_{s,p}(\delta_0,\delta_{e_1})\in(0,\infty)$ by Proposition~\ref{prop:window-sharp}. For this pair the displayed inequality must hold for every $\lambda>0$ with the \emph{same} constant $C$. If $\rho(s-\frac dp)-1<0$, letting $\lambda\uparrow\infty$ drives the right-hand side to $0$ and forces $W_p(\mu,\nu)=0$, a contradiction. If $\rho(s-\frac dp)-1>0$, letting $\lambda\downarrow0$ gives the same contradiction. Hence the exponent of $\lambda$ must vanish, i.e.
\[
\rho\Big(s-\frac{d}{p}\Big)-1=0
\qquad\Longrightarrow\qquad
\rho=\Big(s-\frac{d}{p}\Big)^{-1}
\]
Finally, the standing window gives $0<s-\frac dp<1$. Positivity from $s>\frac dp$ and the upper bound from $s<1+\frac dp$. Hence $\rho=(s-\frac dp)^{-1}>1$, so no reverse bound with exponent $\rho\in(0,1]$ can hold globally. This proves (i).

It remains to prove (ii), for which we combine dilation with mass splitting. For
$\varepsilon\in(0,1]$ and $\lambda>0$ set
\[
\mu_{\varepsilon,\lambda}:=(1-\varepsilon)\,\delta_0+\varepsilon\,\delta_{\lambda e_1}
\in\mathcal P_p(\R^d)
\]
Since $\mu_{\varepsilon,\lambda}-\delta_0=\varepsilon\,(\delta_{\lambda e_1}-\delta_0)$, absolute
homogeneity of the zero-mass norm (Theorem~\ref{thm:seminorm}) and the dilation identity just
proved give
\[
\mathrm T_{s,p}(\mu_{\varepsilon,\lambda},\delta_0)
=\varepsilon\,\big\|\delta_{\lambda e_1}-\delta_0\big\|_{\mathrm T_{s,p}}
=\varepsilon\,\lambda^{\,s-\frac dp}\,t_0,
\qquad t_0:=\mathrm T_{s,p}(\delta_0,\delta_{e_1})\in(0,\infty)
\]
by Proposition~\ref{prop:window-sharp}. Since $\delta_0$ is a Dirac measure, the only element of
$\Pi(\mu_{\varepsilon,\lambda},\delta_0)$ is the product coupling
$\mu_{\varepsilon,\lambda}\otimes\delta_0$, whose cost is exactly $\varepsilon\,\lambda^p$. Hence
\[
W_p(\mu_{\varepsilon,\lambda},\delta_0)=\varepsilon^{1/p}\,\lambda
\]
Given $t>0$, choose $\varepsilon_\lambda:=t\,\lambda^{-(s-d/p)}/t_0$, which lies in $(0,1]$ for
all $\lambda$ large enough. Then $\mathrm T_{s,p}(\mu_{\varepsilon_\lambda,\lambda},\delta_0)=t$
for every such $\lambda$, while
\[
W_p(\mu_{\varepsilon_\lambda,\lambda},\delta_0)
=\Big(\frac{t}{t_0}\Big)^{1/p}\lambda^{\,1-\frac{s-d/p}{p}}
\ \xrightarrow[\lambda\to\infty]{}\ \infty
\]
because $s-\frac dp<1\le p$ makes the exponent $1-\frac{s-d/p}{p}$ strictly positive. Hence
$\sup\{W_p:\mathrm T_{s,p}\le t\}=\infty$ for every $t>0$, and no everywhere-finite function
$\Phi$ can satisfy $W_p\le\Phi(\mathrm T_{s,p})$ on all of $\mathcal P_p(\R^d)$.
\end{proof}

\subsection{Proof of Lemma~\ref{lem:weak-moment-wp}}\label{proof:weak-moment-wp}
\begin{proof}
Let
\[
M:=\sup_{n}\int_{\R^d}\|x\|^\gamma d\mu_n(x)<\infty,
\qquad \gamma>p
\]

First, \(\mu\) has finite \(\gamma\)-moment. Indeed, for each \(R>0\) the function
\(\phi_R(x):=\|x\|^\gamma\wedge R\) is bounded and continuous, so \(\mu_n\overset{w}{\rightharpoonup} \mu\) implies
\[
\int \phi_R d\mu_n  \longrightarrow  \int \phi_R d\mu
\]
Moreover \(\int \phi_R d\mu_n\le \int \|x\|^\gamma d\mu_n\le M\) for all \(n\), hence
\(\int \phi_R d\mu\le M\) for all \(R\). Letting \(R\uparrow\infty\) and using monotone convergence gives
\[
\int \|x\|^\gamma d\mu \le M <\infty
\]

Next, the \(\gamma\)-moment bound gives uniform control of \(p\)-tails. For any \(R>0\) and any \(n\),
on \(\{\|x\|>R\}\) we have \(\|x\|^{\gamma-p}\ge R^{\gamma-p}\), hence
\[
\|x\|^p \mathbf 1_{\{\|x\|>R\}}
=\frac{\|x\|^\gamma}{\|x\|^{\gamma-p}}\mathbf 1_{\{\|x\|>R\}}
\le R^{p-\gamma} \|x\|^\gamma \mathbf 1_{\{\|x\|>R\}}
\]
Therefore
\begin{equation}\label{eq:tail_bound}
\sup_n \int_{\{\|x\|>R\}} \|x\|^p d\mu_n(x)\le R^{p-\gamma} M,
\qquad
\int_{\{\|x\|>R\}} \|x\|^p d\mu(x)\le R^{p-\gamma} M
\end{equation}
Since \(\gamma>p\), the right-hand side tends to \(0\) as \(R\to\infty\), so \(\{\|x\|^p\}\) is uniformly
integrable under \(\{\mu_n\}\).

Now show convergence of \(p\)-moments. Define the bounded continuous truncation
\[
\psi_R(x):=\|x\|^p\wedge R^p
\]
Then \(\mu_n\overset{w}{\rightharpoonup} \mu\) implies
\[
\int \psi_R d\mu_n  \longrightarrow  \int \psi_R d\mu \qquad \text{for each fixed }R>0
\]
Also,
\[
0\le \int \|x\|^p d\mu_n - \int \psi_R d\mu_n
= \int (\|x\|^p-R^p)_+ d\mu_n
\le \int_{\{\|x\|>R\}} \|x\|^p d\mu_n
\]
and similarly for \(\mu\). Using \eqref{eq:tail_bound}, for any \(\varepsilon>0\) choose \(R\) large enough
so that \(R^{p-\gamma}M<\varepsilon/3\). Then for this \(R\),
\[
\sup_n\Big|\int \|x\|^p d\mu_n-\int \psi_R d\mu_n\Big|\le \varepsilon/3
\qquad
\Big|\int \|x\|^p d\mu-\int \psi_R d\mu\Big|\le \varepsilon/3
\]
Finally choose \(N\) such that for all \(n\ge N\),
\[
\Big|\int \psi_R d\mu_n-\int \psi_R d\mu\Big|<\varepsilon/3
\]
Combining these three bounds yields, for \(n\ge N\),
\[
\Big|\int \|x\|^p d\mu_n-\int \|x\|^p d\mu\Big|
\le \varepsilon
\]
Hence \(\int \|x\|^p d\mu_n\to \int \|x\|^p d\mu\).

A standard characterization of \(W_p\)-convergence \citep[Theorem~6.9]{villani2008optimal} states that
\[
W_p(\mu_n,\mu)\to 0
\quad\Longleftrightarrow\quad
\mu_n\overset{w}{\rightharpoonup} \mu \ \text{ and }\ \int \|x\|^p d\mu_n \to \int \|x\|^p d\mu
\]
Since both conditions hold, \(W_p(\mu_n,\mu)\to 0\).
\end{proof}
\subsection{Proof of Theorem~\ref{thm:Wp-Tsp-equivalence}}\label{proof:Wp-Tsp-equivalence}
\begin{proof}
The direction \(W_p(\mu_n,\mu)\to 0 \Rightarrow \mathrm{T}_{s,p}(\mu_n,\mu)\to 0\) follows from
Theorem~\ref{thm:upper-holder}.

Conversely, assume \(\mathrm{T}_{s,p}(\mu_n,\mu)\to 0\) with \(\mu_n,\mu\in\mathcal P_{p,K}(\R^d)\).
Proposition~\ref{prop:convergence} directly yields \(\mu_n\overset{w}{\rightharpoonup}\mu\), with
no use of the common support. (A quantitative alternative. Lemma~\ref{lem:chi-lip} gives the
uniform Lipschitz bound \(\mathrm{Lip}(\Bchi_{\mu_n}-\Bchi_\mu)\le 2K\), so
Lemma~\ref{lem:Tp-to-pointwise} applies with \(L=2K\) and produces pointwise convergence of
characteristic functions with explicit rates, and L\'evy's continuity theorem again gives weak
convergence.)

It remains to upgrade weak convergence to \(W_p\)-convergence using the bounded support.
Set \(B:=B(0,K)\). Because \(\text{supp}(\mu_n)\cup\text{supp}(\mu)\subset B\), every coupling of \((\mu_n,\mu)\) is
supported in \(B\times B\), hence \(W_1(\mu_n,\mu)\) (and \(W_p(\mu_n,\mu)\)) can be computed using only
\((B,\|\cdot\|)\).

By Kantorovich--Rubinstein duality on the compact metric space \(B\),
\[
W_1(\mu_n,\mu)=\sup_{\mathrm{Lip}(\varphi)\le 1}\int_B \varphi d(\mu_n-\mu)
\]
Subtracting the constant \(\varphi(0)\) does not change the integral, so the supremum equals
\[
\sup_{\varphi\in\mathcal F}\int_B \varphi d(\mu_n-\mu),
\qquad
\mathcal F:=\{\varphi\in C(B): \mathrm{Lip}(\varphi)\le 1,\ \varphi(0)=0\}
\]
For \(\varphi\in\mathcal F\) and \(x\in B\), \(|\varphi(x)|\le \|x-0\|\le K\), hence \(\mathcal F\) is
uniformly bounded and equicontinuous. By Arzel\`a--Ascoli, \(\mathcal F\) is compact in
\((C(B),\|\cdot\|_\infty)\). Fix \(\varepsilon>0\) and take a finite \(\varepsilon\)-net
\(\{\varphi_1,\dots,\varphi_m\}\subset\mathcal F\) in \(\|\cdot\|_\infty\).
For any \(\varphi\in\mathcal F\), choose \(j\) with \(\|\varphi-\varphi_j\|_\infty\le\varepsilon\). Then
\[
\int_B \varphi d(\mu_n-\mu)
=
\int_B \varphi_j d(\mu_n-\mu)+\int_B (\varphi-\varphi_j) d(\mu_n-\mu)
\le
\int_B \varphi_j d(\mu_n-\mu)+2\varepsilon
\]
since \(|\int_B (\varphi-\varphi_j) d(\mu_n-\mu)|\le 2\|\varphi-\varphi_j\|_\infty\).
Taking the supremum over \(\varphi\in\mathcal F\) gives
\[
W_1(\mu_n,\mu)
\le
\max_{1\le j\le m}\Big|\int_B \varphi_j d(\mu_n-\mu)\Big| + 2\varepsilon
\]
Each \(\varphi_j\) is bounded continuous on \(B\), so \(\mu_n\overset{w}{\rightharpoonup}  \mu\) implies
\(\int_B \varphi_j d\mu_n\to\int_B \varphi_j d\mu\) for every \(j\). Hence the maximum term tends to \(0\),
and therefore \(\limsup_{n\to\infty} W_1(\mu_n,\mu)\le 2\varepsilon\).
Since \(\varepsilon>0\) is arbitrary, \(W_1(\mu_n,\mu)\to 0\).

Finally, on a common set of diameter \(\mathrm{diam}(B)=2K\), the elementary interpolation bound
\(\|x-y\|^p\le (2K)^{p-1}\|x-y\|\) under any coupling yields
\[
W_p(\mu_n,\mu)\le (2K)^{1-\frac1p} W_1^{1/p}(\mu_n,\mu)
\]
Thus \(W_1(\mu_n,\mu)\to 0\) implies \(W_p(\mu_n,\mu)\to 0\), completing the proof.
\end{proof}
\subsection{Proof of Theorem~\ref{thm:Wp-Tsp-equivalence-moment}}\label{proof:Wp-Tsp-equivalence-moment}
\begin{proof}
Let \(\{\mu_n\}\subset\mathcal M_{\gamma,S}\) and \(\mu\in\mathcal M_{\gamma,S}\).

\smallskip
\noindent\emph{\((\Rightarrow)\)} Assume \(W_p(\mu_n,\mu)\to 0\).  
By Theorem~\ref{thm:upper-holder}, \(\mathrm{T}_{s,p}(\mu_n,\mu)\to 0\).

\smallskip
\noindent\emph{\((\Leftarrow)\)} Assume \(\mathrm{T}_{s,p}(\mu_n,\mu)\to 0\).
By Proposition~\ref{prop:convergence},
\[
\mu_n \overset{w}{\rightharpoonup} \mu
\]

Moreover, because \(\mu_n\in\mathcal M_{\gamma,S}\), we have the uniform \(\gamma\)-moment bound
\[
\sup_n \int \|x\|^\gamma d\mu_n(x)\le S<\infty
\]
and by assumption \(\gamma>p\). Thus the hypotheses of Lemma~\ref{lem:weak-moment-wp} are satisfied, and it follows that
\[
W_p(\mu_n,\mu)\to 0
\]

\smallskip
Combining the two implications proves the equivalence.
\end{proof}

\subsection{Proof of Lemma~\ref{lem:chi-lip}}\label{proof:chi-lip}
\begin{proof}
Fix \(u,v\in\mathbb R^d\). Since the map \(t\mapsto e^{it}\) is \(1\)-Lipschitz on \(\mathbb R\) (because
\(|(e^{it})'|=|ie^{it}|=1\)), we have for every \(x\in\mathbb R^d\),
\[
\big|e^{i\langle u,x\rangle}-e^{i\langle v,x\rangle}\big|
\le \big|\langle u-v,x\rangle\big|
\le \|u-v\| \|x\|
\]
Therefore,
\begin{align*}
|\Bchi_\mu(u)-\Bchi_\mu(v)|
&=\Big|\int_{\mathbb R^d}\Big(e^{i\langle u,x\rangle}-e^{i\langle v,x\rangle}\Big) d\mu(x)\Big| \\
&\le \int_{\mathbb R^d}\big|e^{i\langle u,x\rangle}-e^{i\langle v,x\rangle}\big| d\mu(x) \\
&\le \|u-v\|\int_{\mathbb R^d}\|x\| d\mu(x)
\end{align*}
If \(\mu\in\mathcal P_{p,K}(\mathbb R^d)\), then \(\mathrm{supp}(\mu)\subset B(0,K)\), hence \(\|x\|\le K\)
\(\mu\)-a.s., and thus \(\int \|x\| d\mu(x)\le K\). Consequently,
\[
|\Bchi_\mu(u)-\Bchi_\mu(v)|\le K\|u-v\|,\qquad \forall u,v\in\mathbb R^d
\]
so \(\Bchi_\mu\) is globally Lipschitz with \(\mathrm{Lip}(\Bchi_\mu)\le K\).

Finally, for \(\mu,\nu\in\mathcal P_{p,K}(\mathbb R^d)\) and \(f:=\Bchi_\mu-\Bchi_\nu\),
\[
|f(u)-f(v)|
\le |\Bchi_\mu(u)-\Bchi_\mu(v)|+|\Bchi_\nu(u)-\Bchi_\nu(v)|
\le 2K\|u-v\|
\]
hence \(\mathrm{Lip}(f)\le 2K\).
\end{proof}

\subsection{Proof of Lemma~\ref{lem:chi-lip-moment}}\label{proof:chi-lip-moment}
\begin{proof}
Fix \(u,v\in\mathbb R^d\). For every \(x\in\mathbb R^d\), using that \(t\mapsto e^{it}\) is \(1\)-Lipschitz on
\(\mathbb R\),
\[
\big|e^{i\langle u,x\rangle}-e^{i\langle v,x\rangle}\big|
\le \big|\langle u-v,x\rangle\big|
\le \|u-v\| \|x\|
\]
Hence, by Jensen/triangle inequality,
\begin{align*}
|\Bchi_\mu(u)-\Bchi_\mu(v)|
&=\Big|\int_{\mathbb R^d}\big(e^{i\langle u,x\rangle}-e^{i\langle v,x\rangle}\big) d\mu(x)\Big| \\
&\le \int_{\mathbb R^d}\big|e^{i\langle u,x\rangle}-e^{i\langle v,x\rangle}\big| d\mu(x) \\
&\le \|u-v\|\int_{\mathbb R^d}\|x\| d\mu(x)
\end{align*}
Since \(\mu\in\mathcal M_{\gamma,S}\) and \(\gamma\ge 1\), H\"older's inequality (equivalently, monotonicity of
\(L^p\)-norms) yields
\[
\int_{\mathbb R^d}\|x\| d\mu(x)
\le \Big(\int_{\mathbb R^d}\|x\|^\gamma d\mu(x)\Big)^{1/\gamma}
\le S^{1/\gamma}
\]
Combining the bounds gives
\[
|\Bchi_\mu(u)-\Bchi_\mu(v)|\le \|u-v\| S^{1/\gamma},\qquad \forall u,v\in\mathbb R^d
\]
so \(\mathrm{Lip}(\Bchi_\mu)\le S^{1/\gamma}\). Finally, for \(\mu,\nu\in\mathcal M_{\gamma,S}\) and
\(f:=\Bchi_\mu-\Bchi_\nu\),
\[
|f(u)-f(v)|
\le |\Bchi_\mu(u)-\Bchi_\mu(v)|+|\Bchi_\nu(u)-\Bchi_\nu(v)|
\le 2S^{1/\gamma}\|u-v\|
\]
hence \(\mathrm{Lip}(f)\le 2S^{1/\gamma}\).
\end{proof}

\subsection{Proof of Lemma~\ref{lem:Tp-to-pointwise}}\label{proof:Tp-to-pointwise}
\begin{proof}
Fix \(u\neq 0\) and \(r\in(0,\|u\|/2)\), and set \(f:=\sigma=\Bchi_\mu-\Bchi_\nu\). Write \(B:=B(u,r)\) and denote its volume by \(|B|=|B(u,r)|\).
Let
\[
m_B := \frac{1}{|B|}\int_{B} f(v) dv
\]
be the average of \(f\) over \(B\). Then
\[
|f(u)| \le |f(u)-m_B| + |m_B|
\]

For the average term, H\"older gives
\[
|m_B|
\le
\frac{1}{|B|}\int_B |f(v)| dv
\le
\frac{1}{|B|^{1/p}}\Big(\int_B |f(v)|^p dv\Big)^{1/p}
=
|B|^{-1/p} \|f\|_{L^p(B)}
\]
Since \(|B(u,r)|=|B(0,1)| r^d\),
\[
|B|^{-1/p}=|B(0,1)|^{-1/p}r^{-d/p}
\]
Setting \(c_{d,p}:=|B(0,1)|^{-1/p}\) yields
\[
|m_B|\le c_{d,p} r^{-d/p} \|f\|_{L^p(B)}
\]

For the deviation term, use the Lipschitz property of \(f\).
\[
|f(u)-m_B|
\le
\frac{1}{|B|}\int_B |f(u)-f(v)| dv
\le
\frac{1}{|B|}\int_B \mathrm{Lip}(f) \|u-v\| dv
\le
\mathrm{Lip}(f) r
\]
By hypothesis, \(f=\sigma\) is globally \(L\)--Lipschitz, so \(\mathrm{Lip}(f)\le L\). Combining the two bounds gives the first inequality.
\[
|f(u)|
\le
c_{d,p} r^{-d/p} \|f\|_{L^p(B(u,r))} + L r
\]

To obtain the second inequality, note that \(r<\|u\|/2\) implies \(\|v\|\ge \|u\|-r\ge \|u\|/2>0\) for all
\(v\in B(u,r)\), and also \(\|v\|\le \|u\|+r\). In particular, on \(B(u,r)\),
\[
\|v\|^{ps}\le (\|u\|+r)^{ps}
\]
Therefore,
\[
\begin{aligned}
\int_{B(u,r)} |f(v)|^p\,dv
&=\int_{B(u,r)} \|v\|^{ps}
\frac{|f(v)|^p}{\|v\|^{ps}}\,dv\\
&\le(\|u\|+r)^{ps}
\int_{B(u,r)}\frac{|f(v)|^p}{\|v\|^{ps}}\,dv\\
&\le(\|u\|+r)^{ps}\mathrm{T}_{s,p}(\mu,\nu)^p
\end{aligned}
\]
hence
\[
\|f\|_{L^p(B(u,r))}
\le
(\|u\|+r)^{s} \mathrm{T}_{s,p}(\mu,\nu)
\]
Substituting into the first inequality yields
\[
|f(u)|
\le
c_{d,p} r^{-d/p}(\|u\|+r)^{s} \mathrm{T}_{s,p}(\mu,\nu) + L r
\]
as claimed.

For the pointwise convergence, fix \(u\in\R^d\) with \(u\neq 0\) and suppose
\(\mathrm{T}_{s,p}(\mu_n,\mu)\to 0\), with \(\sigma_n:=\Bchi_{\mu_n}-\Bchi_\mu\) uniformly \(L\)--Lipschitz. Let \(T_n:=\mathrm{T}_{s,p}(\mu_n,\mu)\). If \(T_n=0\) for some \(n\), then \(\mu_n=\mu\) by the definiteness established in Proposition~\ref{prop:integrability}(i) and \(f_n\equiv0\), so such indices may be discarded. We may therefore assume \(T_n>0\) for all \(n\) and choose
\[
r_n := T_n^{ p/(d+p)}\in(0,\infty)
\]
Then \(r_n\downarrow 0\) because \(T_n\to0\) and the exponent \(p/(d+p)\) is positive. For all sufficiently large \(n\), we have \(r_n<\|u\|/2\), and the bound above gives
\[
|f_n(u)|
\le
c_{d,p} r_n^{-d/p}(\|u\|+r_n)^{s} T_n + L r_n\quad ,
\qquad f_n:=\Bchi_{\mu_n}-\Bchi_\mu
\]
Now \(r_n^{-d/p}T_n = T_n^{1-d/(d+p)} = T_n^{p/(d+p)}\to 0\), and also \(r_n\to 0\).
Since \((\|u\|+r_n)^s\to \|u\|^s\), the right-hand side tends to \(0\), hence
\(|\Bchi_{\mu_n}(u)-\Bchi_\mu(u)|\to 0\).
The case \(u=0\) is trivial since \(\Bchi_\eta(0)=1\) for any probability measure \(\eta\).
\end{proof}

\subsection{Proof of Corollary~\ref{cor:pointwise-bound}}\label{proof:pointwise-bound}
\begin{proof}
Fix $u\neq 0$ and write $f:=\Bchi_\mu-\Bchi_\nu$. By the hypotheses of Lemma~\ref{lem:Tp-to-pointwise}, $f$ is globally $L$--Lipschitz and the bound of Lemma~\ref{lem:Tp-to-pointwise} applies. For every $r\in(0,\|u\|/2)$,
\begin{equation}\label{eq:tp-pointwise-r}
|f(u)|
\le c_{d,p}\, r^{-d/p}\,(\|u\|+r)^s\, \mathrm{T}_{s,p}(\mu,\nu)+Lr
\end{equation}
Choose a sequence $r_n\uparrow \|u\|/2$ with $r_n\in(0,\|u\|/2)$. For instance $r_n:=\frac{\|u\|}{2}\bigl(1-\frac1n\bigr)$. Then \eqref{eq:tp-pointwise-r} yields for each $n$,
\[
|f(u)|
\le c_{d,p}\, r_n^{-d/p}\,(\|u\|+r_n)^s\, \mathrm{T}_{s,p}(\mu,\nu)+Lr_n
\]
The map $r\mapsto r^{-d/p}(\|u\|+r)^s$ is continuous on $(0,\|u\|/2]$, hence letting $n\to\infty$ gives
\[
r_n^{-d/p}(\|u\|+r_n)^s
\longrightarrow
\Bigl(\tfrac{\|u\|}{2}\Bigr)^{-d/p}\Bigl(\tfrac{3\|u\|}{2}\Bigr)^s
=
2^{d/p}\Bigl(\tfrac32\Bigr)^s\|u\|^{\,s-d/p}\quad,
\qquad
Lr_n\longrightarrow \tfrac{L}{2}\|u\|
\]
Since the inequality holds for every $n$, passing to the limit on the right-hand side yields
\[
|f(u)|
\le
c_{d,p}\,2^{d/p}\Bigl(\tfrac32\Bigr)^s\,\|u\|^{\,s-d/p}\, \mathrm{T}_{s,p}(\mu,\nu)
+\frac{L}{2}\|u\|
\]
Recalling that $|f(u)|=|\Bchi_\mu(u)-\Bchi_\nu(u)|$, this is equation \eqref{eq:pointwise-bound} with
$C_{d,p,s}:=c_{d,p}\,2^{d/p}(3/2)^s$.
\end{proof}

\subsection{Proof of Theorem~\ref{thm:dual-ipm}}\label{proof:dual-ipm}
\begin{proof}
Let $1\le p<\infty$ and $q$ be its H\"older conjugate. Fix $\sigma\in\mathcal M_0(\R^d)$ such that
\[
f(u):=\frac{\widehat{\sigma}(-u)}{\|u\|^{s}}\in L^p(\R^d)
\]

\smallskip
\noindent\textbf{Step 1. The Fourier pairing.}
For $\phi\in\mathcal S(\R^d)$, Fourier inversion with our convention (stated before \eqref{eq:fourier-pairing}) gives
\[
\phi(x)=(2\pi)^{-d}\int_{\R^d} e^{-i\langle u,x\rangle}\,\widehat{\phi}(u)\,du
\]
hence, using $|\widehat{\sigma}(u)|\le |\sigma|(\R^d)<\infty$ and $\widehat{\phi}\in\mathcal S\subset L^1$, Fubini's theorem yields
\begin{align*}
\int_{\R^d}\phi\,d\sigma
&=(2\pi)^{-d}\int_{\R^d}\int_{\R^d} e^{-i\langle u,x\rangle}\widehat{\phi}(u)\,du\,d\sigma(x)\\
&=(2\pi)^{-d}\int_{\R^d}\widehat{\phi}(u)\left(\int_{\R^d} e^{-i\langle u,x\rangle}\,d\sigma(x)\right)\,du\\
&=(2\pi)^{-d}\int_{\R^d}\widehat{\phi}(u)\,\widehat{\sigma}(-u)\,du
\end{align*}
since $\int e^{-i\langle u,x\rangle}d\sigma(x)=\widehat\sigma(-u)$ by the definition of $\widehat\sigma$. Writing $g_\phi(u):=\|u\|^{s}\widehat{\phi}(u)$, this becomes
\begin{equation}\label{eq:pairing-fg}
\int_{\R^d}\phi\,d\sigma=(2\pi)^{-d}\int_{\R^d} f(u)\,g_\phi(u)\,du
\end{equation}

\smallskip
\noindent\textbf{Step 2. Upper bound (all $1\le p<\infty$).}
By H\"older's inequality applied to \eqref{eq:pairing-fg},
\[
\Big|\int_{\R^d}\phi\,d\sigma\Big|
\le (2\pi)^{-d}\|f\|_{L^p}\,\|g_\phi\|_{L^q}
=(2\pi)^{-d}\Big\|\frac{\widehat{\sigma}}{\|u\|^{s}}\Big\|_{L^p}\,\|\phi\|_{\dot{\mathcal F}L^q_s}
\]
(the norm of $f$ equals the norm of $\widehat\sigma(u)/\|u\|^s$ by the change of variables $u\mapsto-u$, which preserves both the Lebesgue measure and the weight $\|u\|^s$).
This shows that $\phi\mapsto\int\phi\,d\sigma$ is bounded on $\mathcal S$ with respect to the $\dot{\mathcal F}L^q_s$-seminorm. Since $\dot{\mathcal F}L^q_s$ is the completion of $\mathcal S$ under this seminorm, the functional extends uniquely to a continuous linear functional on $\dot{\mathcal F}L^q_s$, with operator norm
\begin{equation}\label{eq:upper-bound-opnorm}
\sup_{\substack{\phi\in\dot{\mathcal F}L^q_s\\ \|\phi\|_{\dot{\mathcal F}L^q_s}\le 1}}
\mathrm{Re}\int_{\R^d}\phi\,d\sigma
\ \le\
\sup_{\substack{\phi\in\dot{\mathcal F}L^q_s\\ \|\phi\|_{\dot{\mathcal F}L^q_s}\le 1}}
\Big|\int_{\R^d}\phi\,d\sigma\Big|
\le (2\pi)^{-d}\|f\|_{L^p}
=(2\pi)^{-d}\Big\|\frac{\widehat{\sigma}}{\|u\|^{s}}\Big\|_{L^p}
\end{equation}

\smallskip
\noindent\textbf{Step 3. Lower bound for $1<p<\infty$.}
If $f=0$ the identity is trivial, so assume $f\neq0$. Proposition~\ref{prop:dual-optimizer} (whose proof is independent of the present theorem. It uses only the pairing \eqref{eq:pairing-fg} for Schwartz functions, established in Step 1) produces $\phi_n\in\mathcal S(\R^d)$ with $\|\phi_n\|_{\dot{\mathcal F}L^q_s}\le1$ and
\[
\mathrm{Re}\int_{\R^d}\phi_n\,d\sigma\longrightarrow(2\pi)^{-d}\|f\|_{L^p}
\]
Hence the supremum on the left of \eqref{eq:upper-bound-opnorm} is at least $(2\pi)^{-d}\|f\|_{L^p}$, and combining with Step 2 proves \eqref{eq:dual-ipm} for $1<p<\infty$. Moreover, Proposition~\ref{prop:dual-optimizer} shows the supremum is attained in $\dot{\mathcal F}L^q_s$ in this range.

\smallskip
\noindent\textbf{Step 4. Lower bound for $p=1$ ($q=\infty$).}
Here the extremizing ``sign pattern'' $g_*:=\overline{f}/|f|\,\mathbf 1_{\{f\neq0\}}$ lies in the unit ball of $L^\infty$ and satisfies $\int f g_*=\int|f|=\|f\|_{L^1}$, but it is in general discontinuous and cannot be reached in the $L^\infty$ norm by functions of the form $\|u\|^s\widehat\phi$ with $\phi\in\mathcal S$ (whose closure in $L^\infty$ consists of continuous functions vanishing at infinity and at the origin). We therefore approximate $g_*$ \emph{pointwise almost everywhere with a uniform bound}, which suffices by dominated convergence.

Fix a mollifier $\rho\in C_c^\infty(B(0,1))$ with $\rho\ge0$ and $\int\rho=1$. Choose radii $\varepsilon_n\downarrow0$ and $R_n\uparrow\infty$, and radial cutoffs $\eta_n\in C_c^\infty(\R^d)$ with
\[
\begin{aligned}
0\le\eta_n\le1,\qquad&
\eta_n\equiv0 &&\text{on }B(0,\varepsilon_n),\\
&\eta_n\equiv1 &&\text{on }B(0,R_n)\setminus B(0,2\varepsilon_n),\\
&&&\supp(\eta_n)\subseteq B(0,2R_n)
\end{aligned}
\]
Set $g_n:=\eta_n\,g_*$ (bounded by $1$, measurable, compactly supported, vanishing on $B(0,\varepsilon_n)$) and define the smoothed truncations
\[
h_n:=g_n*\rho_{\delta_n},
\qquad \rho_{\delta}(x)=\delta^{-d}\rho(x/\delta)
\]
with $0<\delta_n<\varepsilon_n/2$ and $\delta_n\downarrow0$. These functions have four properties.
(i) $h_n\in C_c^\infty(\R^d)$, being the convolution of a compactly supported bounded function with a $C_c^\infty$ kernel.
(ii) $\|h_n\|_{L^\infty}\le\|g_n\|_{L^\infty}\|\rho\|_{L^1}\le1$, using $\rho\ge0$ and $\int\rho=1$.
(iii) $h_n\equiv0$ on $B(0,\varepsilon_n/2)$. If $\|x\|<\varepsilon_n/2$ and $\|y\|\le\delta_n<\varepsilon_n/2$ then $\|x-y\|<\varepsilon_n$, where $g_n$ vanishes.
(iv) $h_n\to g_*$ almost everywhere. For (iv), let $u_0\neq0$ be a Lebesgue point of $g_*$ (almost every point is one, since $g_*\in L^\infty\subset L^1_{\mathrm{loc}}$). For $n$ large enough that $2\varepsilon_n<\|u_0\|/2$ and $R_n>\|u_0\|+1$, we have $g_n\equiv g_*$ on the fixed ball $B(u_0,r_0)$ with $r_0:=\min\{\|u_0\|/4,\,1\}$. Hence for $\delta_n<r_0/2$,
\[
h_n(u_0)=\int g_n(u_0-y)\rho_{\delta_n}(y)\,dy=\int g_*(u_0-y)\rho_{\delta_n}(y)\,dy
\]
and
\[
\begin{aligned}
|h_n(u_0)-g_*(u_0)|
&\le\int\big|g_*(u_0-y)-g_*(u_0)\big|
\,\rho_{\delta_n}(y)\,dy\\
&\le\|\rho\|_\infty\,\delta_n^{-d}
\int_{B(0,\delta_n)}
\big|g_*(u_0-y)-g_*(u_0)\big|\,dy\\
&\longrightarrow0
\end{aligned}
\]
by the Lebesgue differentiation theorem. This proves (iv).

Now define $\phi_n\in\mathcal S(\R^d)$ through its Fourier transform, $\widehat{\phi_n}(u):=\|u\|^{-s}h_n(u)$. This is $C_c^\infty$ because $h_n$ is $C_c^\infty$ and vanishes in a neighborhood of the origin where $\|u\|^{-s}$ would be singular, so $\phi_n$ (its inverse transform) is Schwartz. By construction $\|\phi_n\|_{\dot{\mathcal F}L^\infty_s}=\|h_n\|_{L^\infty}\le1$, and by the pairing \eqref{eq:pairing-fg},
\[
\int_{\R^d}\phi_n\,d\sigma
=(2\pi)^{-d}\int_{\R^d}f(u)\,h_n(u)\,du
\]
Since $|f\,(h_n-g_*)|\le2|f|\in L^1$ and $h_n\to g_*$ a.e., dominated convergence gives
\[
\int f h_n\longrightarrow\int f g_*=\int|f|=\|f\|_{L^1}
\]
Taking real parts (the limit $\|f\|_{L^1}$ is real),
\[
\mathrm{Re}\int_{\R^d}\phi_n\,d\sigma\longrightarrow(2\pi)^{-d}\|f\|_{L^1}
\]
so the supremum in \eqref{eq:dual-ipm} is at least $(2\pi)^{-d}\|f\|_{L^1}$. With Step 2 this proves the identity at $p=1$. The supremum need not be attained at $p=1$. Any attaining test function would correspond to a
unit-ball element of the $L^\infty$-closure of $\{\|u\|^s\widehat\phi:\phi\in\mathcal S\}$ ---
whose members are continuous on $\R^d$ and vanish at the origin (since
$\|u\|^{s}|\widehat\phi(u)|\le\|u\|^{s}\|\phi\|_{L^1}\to0$ as $u\to0$, and this persists under
uniform limits) --- coinciding a.e.\ with the unimodular pattern $g_*$ on $\{f\neq0\}$. Whenever
$\{f\neq0\}$ has full measure in some neighborhood of the origin, this is impossible. Continuity
and vanishing at the origin force small modulus near $0$, while $|g_*|=1$ a.e.\ there. This is
the generic situation. For instance, it holds whenever $\mu\neq\nu$ have compact support, since
$\Bchi_\mu-\Bchi_\nu$ is then real-analytic and not identically zero, so its zero set is
Lebesgue-null.

\smallskip
\noindent\textbf{Step 5. Conclusion.}
When $\sigma=\mu-\nu$, the right-hand side of \eqref{eq:upper-bound-opnorm} equals $(2\pi)^{-d} \mathrm{T}_{s,p}(\mu,\nu)$ by the definition of $ \mathrm{T}_{s,p}$, and the equivalent reformulation follows by multiplying both sides by $(2\pi)^d$. The equality of the $\mathrm{Re}$-supremum with the modulus supremum was noted in Remark~\ref{rem:dual-real}.
\end{proof}

\subsection{Proof of Proposition~\ref{prop:dual-optimizer}}\label{proof:dual-optimizer}
\begin{proof}
Let $\sigma\in\mathcal M_0(\R^d)$ and define
\[
f(u):=\frac{\widehat{\sigma}(-u)}{\|u\|^{s}}\in L^p(\R^d),\qquad f\not\equiv 0
\]
with $1<p<\infty$ and conjugate $q\in(1,\infty)$. The $L^p$--$L^q$ duality yields
\begin{equation}\label{eq:lp-dual}
\|f\|_{L^p}=\sup_{\|g\|_{L^q}\le 1}\Re\int_{\R^d} f(u)\,g(u)\,du
\end{equation}
and the supremum is unchanged if $\Re$ is replaced by the modulus, by multiplying $g$ with a unimodular constant.

\smallskip
\noindent\emph{Step 1. Extremizer in $L^q$.}
Set
\[
g^\ast(u):=\frac{|f(u)|^{p-2}\overline{f(u)}}{\|f\|_{L^p}^{\,p-1}}
\]
(interpreted as $0$ where $f=0$. Note the complex conjugate, which is essential because $f$ is complex-valued in general). Then $g^\ast\in L^q$ and, pointwise,
\[
|g^\ast(u)|^q=\frac{\big(|f(u)|^{p-2}|f(u)|\big)^q}{\|f\|_{L^p}^{(p-1)q}}
=\frac{|f(u)|^{(p-1)q}}{\|f\|_{L^p}^{(p-1)q}}
=\frac{|f(u)|^{p}}{\|f\|_{L^p}^{p}}
\]
since $(p-1)q=p$. Integrating gives $\|g^\ast\|_{L^q}^q=\|f\|_{L^p}^{-p}\int |f|^p=1$, hence $\|g^\ast\|_{L^q}=1$.
Moreover, using $f\,\overline f=|f|^2$,
\[
\int_{\R^d} f(u)\,g^\ast(u)\,du
=\frac{1}{\|f\|_{L^p}^{p-1}}\int_{\R^d} |f(u)|^{p-2}\,f(u)\overline{f(u)}\,du
=\frac{1}{\|f\|_{L^p}^{p-1}}\int_{\R^d} |f(u)|^{p}\,du
=\|f\|_{L^p}
\]
which is real. Thus equality holds in H\"older's inequality, and $g^\ast$ attains the supremum in \eqref{eq:lp-dual}.

\smallskip
\noindent\emph{Step 2. The completion is all of $L^q$.}
The map $\Phi:\phi\mapsto\|u\|^s\widehat\phi$ is a linear isometry from $(\mathcal S(\R^d),\|\cdot\|_{\dot{\mathcal F}L^q_s})$ into $L^q(\R^d)$, by the very definition of the norm. We claim its image
\[
\Phi(\mathcal S)=\big\{\|u\|^s\widehat\phi:\ \phi\in\mathcal S(\R^d)\big\}
\]
is \emph{dense} in $L^q(\R^d)$ for $1<q<\infty$. Indeed, $C_c^\infty(\R^d\setminus\{0\})$ is dense in $L^q(\R^d)$. Since $\{0\}$ is Lebesgue-null, $L^q(\R^d\setminus\{0\})=L^q(\R^d)$, and compactly supported smooth functions on the open set $\R^d\setminus\{0\}$ are dense in its $L^q$ space for $q<\infty$ (mollify and truncate). And every $h\in C_c^\infty(\R^d\setminus\{0\})$ lies in $\Phi(\mathcal S)$. Setting $\widehat\phi(u):=\|u\|^{-s}h(u)$ defines a $C_c^\infty(\R^d)$ function (the factor $\|u\|^{-s}$ is smooth on the support of $h$, which avoids the origin), so its inverse Fourier transform $\phi$ is Schwartz and $\Phi(\phi)=h$. Consequently the extension of $\Phi$ to the abstract completion $\dot{\mathcal F}L^q_s$ is a surjective linear isometry
\[
\Phi:\dot{\mathcal F}L^q_s\ \xrightarrow{\ \cong\ }\ L^q(\R^d)
\]
(an isometry with dense image onto a complete space is onto its closure, which is everything here).

\smallskip
\noindent\emph{Step 3. Attainment in $\dot{\mathcal F}L_s^q$.}
By Step 2 there exists $\phi^\ast\in\dot{\mathcal F}L_s^q$ with
\[
\begin{aligned}
\|u\|^{s}\widehat{\phi^\ast}(u)&=g^\ast(u)
&&\text{in }L^q(\R^d),\\
\|\phi^\ast\|_{\dot{\mathcal F}L_s^q}
&=\|g^\ast\|_{L^q}=1
\end{aligned}
\]
(the equation is shorthand for $\Phi(\phi^\ast)=g^\ast$).
The pairing identity from \eqref{eq:pairing-fg} holds for $\phi\in\mathcal S$ and extends to all of
\(\dot{\mathcal F}L^q_s\) by continuity in \(\|\cdot\|_{\dot{\mathcal F}L^q_s}\). For the left
side this is the extension constructed in Step 2 of the proof of Theorem~\ref{thm:dual-ipm}.
For the right side it follows from H\"older's inequality. Evaluating at $\phi^\ast$,
\[
\int_{\R^d}\phi^\ast\,d\sigma=(2\pi)^{-d}\int_{\R^d} f(u)\,g^\ast(u)\,du
=(2\pi)^{-d}\|f\|_{L^p}
\]
which is real, so the supremum in \eqref{eq:dual-ipm} is attained at $\phi^\ast$.

\smallskip
\noindent\emph{Step 4. Near-attainment by Schwartz functions.}
Let $\varepsilon_n\downarrow 0$ and $R_n\uparrow\infty$, and choose radial cutoffs $\eta_n\in C_c^\infty(\R^d)$ satisfying
\[
\begin{aligned}
0\le\eta_n\le1,\qquad&
\eta_n\equiv0 &&\text{on }B(0,\varepsilon_n),\\
&\eta_n\equiv1 &&\text{on }B(0,R_n)\setminus B(0,2\varepsilon_n),\\
&&&\supp(\eta_n)\subseteq B(0,2R_n)
\end{aligned}
\]
Then $\eta_n g^\ast\to g^\ast$ in $L^q$. Pointwise $\eta_ng^\ast\to g^\ast$ a.e., with domination $|\eta_ng^\ast-g^\ast|^q\le|g^\ast|^q\in L^1$, so dominated convergence applies.
Each $\eta_ng^\ast$ is supported in the compact annulus $\overline{B(0,2R_n)}\setminus B(0,\varepsilon_n)$. Mollifying at a sufficiently small scale $\delta_n<\varepsilon_n/2$ produces $h_n:=(\eta_ng^\ast)*\rho_{\delta_n}\in C_c^\infty(\R^d\setminus\{0\})$ with
\[
\|h_n-\eta_n g^\ast\|_{L^q}\le 2^{-n}
\]
(possible since mollification converges in $L^q$ for $q<\infty$). Then $h_n\to g^\ast$ in $L^q$, and replacing $h_n$ by $h_n/\max\{1,\|h_n\|_{L^q}\}$ ensures $\|h_n\|_{L^q}\le 1$ while preserving $h_n\to g^\ast$ in $L^q$ (because $\|h_n\|_{L^q}\to\|g^\ast\|_{L^q}=1$).

Define $\phi_n$ by prescribing its Fourier transform
\[
\widehat{\phi_n}(u):=\|u\|^{-s}h_n(u),\qquad u\in\R^d
\]
Because $h_n\in C_c^\infty$ and vanishes in a neighborhood of $0$, the factor $\|u\|^{-s}$ is smooth on $\supp(h_n)$. Hence $\widehat{\phi_n}\in C_c^\infty(\R^d)$ and therefore $\phi_n\in\mathcal S(\R^d)$.
Moreover,
\[
\|\phi_n\|_{\dot{\mathcal F}L_s^q}=\big\|\|u\|^{s}\widehat{\phi_n}(u)\big\|_{L^q}=\|h_n\|_{L^q}\le 1
\]
Finally, using the Fourier pairing \eqref{eq:pairing-fg} (valid for $\phi_n\in\mathcal S$) and the definition of $f$,
\[
\begin{aligned}
\int_{\R^d}\phi_n\,d\sigma
&=(2\pi)^{-d}\int_{\R^d}
\widehat{\phi_n}(u)\widehat{\sigma}(-u)\,du\\
&=(2\pi)^{-d}\int_{\R^d}
\|u\|^{-s}h_n(u)\widehat{\sigma}(-u)\,du\\
&=(2\pi)^{-d}\int_{\R^d} f(u)h_n(u)\,du
\end{aligned}
\]
Since $h_n\to g^\ast$ in $L^q$ and $f\in L^p$, H\"older's inequality implies
\[
\Big|\int f h_n-\int f g^\ast\Big|\le\|f\|_{L^p}\|h_n-g^\ast\|_{L^q}\longrightarrow0
\]
hence $\int f h_n\to \int f g^\ast=\|f\|_{L^p}$ and, taking real parts,
\[
\mathrm{Re}\int_{\R^d}\phi_n\,d\sigma \longrightarrow (2\pi)^{-d}\|f\|_{L^p}
\]
When $\sigma=\mu-\nu$, $\|f\|_{L^p}=\big\|\widehat{\sigma}/\|u\|^s\big\|_{L^p}= \mathrm{T}_{s,p}(\mu,\nu)$ by definition (using again the change of variables $u\mapsto-u$), which gives the stated convergence.
\end{proof}

\subsection{Proof of Lemma~\ref{lem:mollified-FL}}\label{proof:mollified-FL}
\begin{proof}
Fix $\rho\in C_c^\infty(\R^d)$ with $\rho\ge 0$, $\int_{\R^d}\rho=1$, and define $\rho_\varepsilon(x)=\varepsilon^{-d}\rho(x/\varepsilon)$.
Let $\psi:\R^d\to\R$ be $1$--Lipschitz with $\psi(0)=0$ and $\supp(\psi)\subseteq B(0,2K)$, and set $\psi_\varepsilon:=\psi*\rho_\varepsilon$.\\

\smallskip
Since $\rho_\varepsilon\in C_c^\infty$ and $\psi\in L^1\cap L^\infty$ with compact support, $\psi_\varepsilon\in C^\infty$ and all derivatives satisfy
$\partial^\beta\psi_\varepsilon=\psi*(\partial^\beta\rho_\varepsilon)$.
Moreover $\supp(\psi_\varepsilon)\subseteq \supp(\psi)+\supp(\rho_\varepsilon)\subseteq B(0,2K+R\varepsilon)$, where $R$ is such that $\supp(\rho)\subseteq B(0,R)$.
Thus $\psi_\varepsilon$ is smooth with compact support, hence $\psi_\varepsilon\in\mathcal S(\R^d)$.\\

\smallskip
\noindent\textbf{Step 1. Two bounds on $\widehat\psi$.}
(Throughout this proof, $c_d$ and $C_d$ denote dimensional constants whose value may change from
line to line.)
Since $\psi$ is $1$--Lipschitz, Rademacher's theorem implies $\nabla\psi$ exists a.e.\ and $\|\nabla\psi\|_{L^\infty}\le 1$.
Because $\supp(\psi)\subseteq B(0,2K)$, it follows that
\[
\|\nabla\psi\|_{L^1}\le |B(0,2K)|\,\|\nabla\psi\|_{L^\infty}\le c_d\,K^d
\]
For $u\neq 0$, integration by parts (justified since $\psi$ has compact support and $\nabla\psi\in L^1$) yields, with the Fourier convention $\widehat\psi(u)=\int_{\R^d}e^{i\langle u,x\rangle}\psi(x)\,dx$ fixed in Section~\ref{sec:dual-ipm}, for any index $j$ with $u_j\neq 0$,
\[
\widehat{\psi}(u)=\int_{\R^d}e^{i\langle u,x\rangle}\psi(x)\,dx
=-\frac{1}{iu_j}\int_{\R^d}e^{i\langle u,x\rangle}\partial_j\psi(x)\,dx
\]
(only the modulus of $\widehat\psi$ is used below, so the sign is immaterial. We record it to keep the convention consistent throughout the paper)
Choosing $j$ with $|u_j|=\max_{1\le k\le d}|u_k|\ge \|u\|/\sqrt d$, we obtain
\begin{equation}\label{eq:psi-hat-decay}
|\widehat{\psi}(u)|
\le \frac{\|\partial_j\psi\|_{L^1}}{|u_j|}
\le \frac{\sqrt d\ \|\nabla\psi\|_{L^1}}{\|u\|}
\le c_d\,K^d\,\|u\|^{-1}
\end{equation}
In addition, $\psi$ is bounded by $|\psi(x)|=|\psi(x)-\psi(0)|\le \|x\|\le 2K$ on its support, so
\[
\|\psi\|_{L^1}\le (2K)\,|B(0,2K)|\le c_d\,K^{d+1},
\qquad\Rightarrow\qquad
|\widehat\psi(u)|\le \|\psi\|_{L^1}\le c_d\,K^{d+1}
\]
Combining this with \eqref{eq:psi-hat-decay} yields the uniform estimate
\begin{equation}\label{eq:psi-hat-min}
|\widehat\psi(u)|\le C_d\,K^d\,\min\{K,\|u\|^{-1}\},\qquad u\in\R^d
\end{equation}
Since $\rho\in C_c^\infty$, $\widehat\rho$ is Schwartz. Thus for any $m\ge 0$ there exists $C_{\rho,m}$ such that
\begin{equation}\label{eq:rho-decay}
|\widehat\rho(\xi)|\le C_{\rho,m}\,(1+\|\xi\|)^{-m},\qquad \xi\in\R^d
\end{equation}
and also $|\widehat\rho(\xi)|\le\|\rho\|_{L^1}=1$ for all $\xi$.
Fix once and for all an integer
\begin{equation}\label{eq:m-choice}
m > s-1+d
\ \ \ \Big(\text{so that in particular } mq>(s-1)q+d \text{ for every } q\ge 1\Big)
\end{equation}

\smallskip
\noindent\textbf{Step 2. The case $1<p<\infty$ (so $q<\infty$).}
Using $\widehat{\psi_\varepsilon}(u)=\widehat{\psi}(u)\widehat{\rho_\varepsilon}(u)=\widehat{\psi}(u)\widehat{\rho}(\varepsilon u)$, one has
\[
\|\psi_\varepsilon\|_{\dot{\mathcal F}L_s^q}^q
=
\int_{\R^d}\|u\|^{sq}\,|\widehat{\psi}(u)|^q\,|\widehat{\rho}(\varepsilon u)|^q\,du
\]
Let $A:=\varepsilon^{-1}$. Decompose $\R^d=\{ \|u\|\le A\}\cup\{\|u\|>A\}$ and write
\[
\|\psi_\varepsilon\|_{\dot{\mathcal F}L_s^q}^q \le I_{\mathrm{low}}+I_{\mathrm{high}}
\]
where
\[
I_{\mathrm{low}}:=\int_{\|u\|\le A}\|u\|^{sq}\,|\widehat{\psi}(u)|^q\,|\widehat{\rho}(\varepsilon u)|^q\,du,
\qquad
I_{\mathrm{high}}:=\int_{\|u\|> A}\|u\|^{sq}\,|\widehat{\psi}(u)|^q\,|\widehat{\rho}(\varepsilon u)|^q\,du
\]
\emph{Crucially, we use the decay bound \eqref{eq:psi-hat-decay} on \emph{both} regions}. The crude $L^1$ bound $|\widehat\psi|\le c_dK^{d+1}$ is never needed away from the origin, and avoiding it is what produces matching rates in the two regions without any normalization of $K$.

\smallskip
\noindent\emph{Low frequencies.}
Since $|\widehat\rho(\varepsilon u)|\le 1$, using \eqref{eq:psi-hat-decay} on $\{\|u\|\le A\}$,
\[
I_{\mathrm{low}}
\le \int_{\|u\|\le A}\|u\|^{sq}\,(c_dK^d)^q\,\|u\|^{-q}\,du
= C_{d,q}\,K^{dq}\int_{\|u\|\le A}\|u\|^{(s-1)q}\,du
\]
In polar coordinates, $\int_{\|u\|\le A}\|u\|^{(s-1)q}\,du
= v_d\int_0^{A} r^{(s-1)q+d-1}\,dr$, which is finite because
\[
(s-1)q+d>0
\quad\Longleftrightarrow\quad
s>1-\frac{d}{q}
\]
and the latter holds throughout the window. Indeed $s>\frac dp$ and, since $\frac1p=1-\frac1q$,
\[
\frac{d}{p}=d\Big(1-\frac1q\Big)=d-\frac{d}{q}\ \ge\ 1-\frac{d}{q}
\]
because $d\ge1$. Hence $s>\frac dp\ge 1-\frac dq$, so $(s-1)q+d>0$. (Note that the \emph{identity} $1-\frac dq=\frac dp$ holds only when $d=1$. What is true and needed for all $d\ge1$ is the displayed \emph{inequality}.) Evaluating the radial integral,
\begin{equation}\label{eq:Ilow}
I_{\mathrm{low}}\le C_{d,p,s}\,K^{dq}\,A^{(s-1)q+d}
= C_{d,p,s}\,K^{dq}\,\varepsilon^{-(s-1)q-d}
\end{equation}

\smallskip
\noindent\emph{High frequencies.}
On $\{\|u\|>A\}$ we combine the decay bound \eqref{eq:psi-hat-decay} with the Schwartz decay \eqref{eq:rho-decay}.
\[
\begin{aligned}
I_{\mathrm{high}}
&\le (c_dK^d)^qC_{\rho,m}^q
\int_{\|u\|>A}\|u\|^{(s-1)q}
\big(1+\varepsilon\|u\|\big)^{-mq}\,du\\
&=C_{d,\rho,m}K^{dq}
\int_{\|u\|>1/\varepsilon}\|u\|^{(s-1)q}
\big(1+\varepsilon\|u\|\big)^{-mq}\,du
\end{aligned}
\]
Substituting $w:=\varepsilon u$ (so $du=\varepsilon^{-d}dw$ and $\|u\|=\|w\|/\varepsilon$),
\[
I_{\mathrm{high}}
\le C_{d,\rho,m}\,K^{dq}\,\varepsilon^{-(s-1)q-d}\int_{\|w\|>1}\|w\|^{(s-1)q}\,(1+\|w\|)^{-mq}\,dw
\]
The last integral is a finite constant depending only on $(d,s,q,m)$. In polar coordinates it is
$v_d\int_1^\infty r^{(s-1)q-mq+d-1}\,dr<\infty$, because $mq>(s-1)q+d$ by \eqref{eq:m-choice}. Hence
\begin{equation}\label{eq:Ihigh}
I_{\mathrm{high}}
\le C_{d,p,s,\rho}\,K^{dq}\,\varepsilon^{-(s-1)q-d}
\end{equation}
which matches \eqref{eq:Ilow} \emph{exactly}, in both the power of $K$ and the power of $\varepsilon$. No absorption of one region into the other, and no auxiliary assumption on the size of $K$, is required.

\smallskip
\noindent\emph{Conclusion for $q<\infty$.}
Adding \eqref{eq:Ilow} and \eqref{eq:Ihigh},
\[
\|\psi_\varepsilon\|_{\dot{\mathcal F}L_s^q}^q
\le C_{d,p,s,\rho}\,K^{dq}\,\varepsilon^{-(s-1)q-d}
\]
Taking $q$-th roots gives
\[
\|\psi_\varepsilon\|_{\dot{\mathcal F}L_s^q}
\le C_{d,p,s,\rho}\,K^{d}\,\varepsilon^{-(s-1)-d/q}
= C_{d,p,s,\rho}\,K^{d}\,\varepsilon^{\,1-s-\frac{d}{q}}
\]
which is \eqref{eq:mollified-FL-bound}.

\smallskip
\noindent\textbf{Step 3. The case $p=1$ (so $q=\infty$).}
Here the window reads $d<s<1+d$, so in particular $s>1$, and the claimed rate is $\varepsilon^{1-s}$ (the term $d/q$ vanishes). We must bound
\[
\|\psi_\varepsilon\|_{\dot{\mathcal F}L_s^\infty}
=\sup_{u\neq0}\ \|u\|^{s}\,|\widehat\psi(u)|\,|\widehat\rho(\varepsilon u)|
\]
For $0<\|u\|\le 1/\varepsilon$ we use \eqref{eq:psi-hat-decay} and $|\widehat\rho|\le 1$ to obtain
\[
\|u\|^{s}\,|\widehat\psi(u)|\,|\widehat\rho(\varepsilon u)|
\le c_d K^d\,\|u\|^{s-1}
\le c_d K^d\,\varepsilon^{1-s}
\]
since $s-1>0$ makes $\|u\|^{s-1}$ increasing in $\|u\|$, with maximum at $\|u\|=1/\varepsilon$. (For $\|u\|$ near $0$ one may alternatively use \eqref{eq:psi-hat-min}, which gives $\|u\|^s\,c_dK^{d+1}\to0$. Either way the supremum over this region obeys the same bound.)
For $\|u\|>1/\varepsilon$ we combine \eqref{eq:psi-hat-decay}, \eqref{eq:rho-decay}, and $(1+\varepsilon\|u\|)^{-m}\le(\varepsilon\|u\|)^{-m}$ to obtain
\[
\|u\|^{s}\,|\widehat\psi(u)|\,|\widehat\rho(\varepsilon u)|
\le c_dK^d\,C_{\rho,m}\ \varepsilon^{-m}\,\|u\|^{\,s-1-m}
\le c_dK^d\,C_{\rho,m}\ \varepsilon^{-m}\,\varepsilon^{\,m-s+1}
= C_{d,s,\rho}\,K^d\,\varepsilon^{1-s}
\]
since $m>s-1$ makes $\|u\|^{s-1-m}$ decreasing in $\|u\|$, with maximum at $\|u\|=1/\varepsilon$. Taking the supremum over both regions,
\[
\|\psi_\varepsilon\|_{\dot{\mathcal F}L_s^\infty}
\le C_{d,s,\rho}\,K^{d}\,\varepsilon^{\,1-s}
\]
which is \eqref{eq:mollified-FL-bound} with $q=\infty$. This completes the proof in all cases $1\le p<\infty$.
\end{proof}

\subsection{Proof of Theorem~\ref{thm:explicit-reverse-bounded}}\label{proof:explicit-reverse-bounded}
\begin{proof}
Fix $K>0$ and $\mu,\nu\in\mathcal P_{p,K}(\R^d)$, so that $\supp(\mu)\cup\supp(\nu)\subseteq B(0,K)$.
Let $\sigma:=\mu-\nu\in\mathcal M_0(\R^d)$. By Kantorovich--Rubinstein duality,
\[
W_1(\mu,\nu)=\sup\Big\{\int_{\R^d}\psi\,d\sigma:\ \mathrm{Lip}(\psi)\le 1\Big\}
\]
Fix $\psi$ with $\mathrm{Lip}(\psi)\le 1$.
Replacing $\psi$ by $\psi-\psi(0)$ does not change $\int \psi\,d\sigma$ because $\sigma(\R^d)=0$, hence one may assume $\psi(0)=0$, which implies $|\psi(x)|\le \|x\|$ for all $x$.
Since $\sigma$ is supported in $B(0,K)$, only the values of $\psi$ on $B(0,K)$ affect $\int\psi\,d\sigma$.
Define the McShane extension of $\psi|_{B(0,K)}$ to $\R^d$ by
\[
\widetilde\psi(x):=\sup_{y\in B(0,K)}\big(\psi(y)-\|x-y\|\big)
\]
which is $1$--Lipschitz on $\R^d$ and satisfies $\widetilde\psi=\psi$ on $B(0,K)$.
Set $c_K(x):=(2K-\|x\|)_+$, which is $1$--Lipschitz and supported in $B(0,2K)$, and define the clamp
\[
\psi_K(x):=\max\big\{-c_K(x),\,\min\{\widetilde\psi(x),c_K(x)\}\big\}
\]
As pointwise minima and maxima preserve Lipschitz constants, $\psi_K$ is $1$--Lipschitz.
Moreover, $c_K\equiv 0$ on $\R^d\setminus B(0,2K)$, hence $\psi_K\equiv 0$ there, and $\psi_K(0)=0$.
On $B(0,K)$ one has $c_K(x)\ge K$ and $|\psi(x)|\le \|x\|\le K$, so $\psi_K(x)=\widetilde\psi(x)=\psi(x)$ for all $x\in B(0,K)$, and therefore
\[
\int_{\R^d}\psi\,d\sigma=\int_{\R^d}\psi_K\,d\sigma
\]
Consequently,
\[
W_1(\mu,\nu)=\sup\Big\{\int_{\R^d}\psi\,d\sigma:\ \mathrm{Lip}(\psi)\le 1,\ \psi(0)=0,\ \supp(\psi)\subseteq B(0,2K)\Big\}
\]
Fix such a $\psi$ for the remainder of the proof.

Let $\rho\in C_c^\infty(\R^d)$ satisfy $\rho\ge 0$ and $\int\rho=1$ (a nonnegative mollifier, whose nonnegativity is used in the error estimate below), and set $\rho_\varepsilon(x)=\varepsilon^{-d}\rho(x/\varepsilon)$.
Define $\psi_\varepsilon:=\psi*\rho_\varepsilon$.
Then
\begin{equation}\label{eq:decompose-moll}
\int_{\R^d}\psi\,d\sigma=\int_{\R^d}(\psi-\psi_\varepsilon)\,d\sigma+\int_{\R^d}\psi_\varepsilon\,d\sigma
\end{equation}
For the mollification error, using $\mathrm{Lip}(\psi)\le 1$, $\rho_\varepsilon\ge0$, and $\int\rho_\varepsilon=1$,
\[
\psi(x)-\psi_\varepsilon(x)
=\int_{\R^d}\big(\psi(x)-\psi(x-y)\big)
\rho_\varepsilon(y)\,dy
\]
Therefore,
\[
\begin{aligned}
|\psi(x)-\psi_\varepsilon(x)|
&\le \int_{\R^d}\big|\psi(x)-\psi(x-y)\big|
\rho_\varepsilon(y)\,dy\\
&\le \int_{\R^d}\|y\|\rho_\varepsilon(y)\,dy\\
&=\varepsilon \int_{\R^d}\|z\|\rho(z)\,dz
\end{aligned}
\]
Denoting $m_1(\rho):=\int_{\R^d}\|z\|\rho(z)\,dz<\infty$, it follows that
\begin{equation}\label{eq:moll-error}
\Big|\int_{\R^d}(\psi-\psi_\varepsilon)\,d\sigma\Big|
\le \int |\psi-\psi_\varepsilon|\,d\mu+\int |\psi-\psi_\varepsilon|\,d\nu
\le 2\,\|\psi-\psi_\varepsilon\|_\infty
\le 2m_1(\rho)\,\varepsilon
\end{equation}

For the second term in \eqref{eq:decompose-moll}, Theorem~\ref{thm:dual-ipm} applied to $\sigma=\mu-\nu$ gives the operator norm bound
\begin{equation}\label{eq:dual-bound-apply}
\Big|\int_{\R^d}\psi_\varepsilon\,d\sigma\Big|
\le (2\pi)^{-d}\, \mathrm{T}_{s,p}(\mu,\nu)\,\|\psi_\varepsilon\|_{\dot{\mathcal F}L_s^q}
\end{equation}
Since $\psi$ is $1$--Lipschitz, $\psi(0)=0$, and $\supp(\psi)\subseteq B(0,2K)$, Lemma~\ref{lem:mollified-FL} yields
\begin{equation}\label{eq:mollified-FL-apply}
\|\psi_\varepsilon\|_{\dot{\mathcal F}L_s^q}
\le C_{d,p,s,\rho}\,K^d\,\varepsilon^{\,1-s-\frac{d}{q}}
\end{equation}
Combining \eqref{eq:decompose-moll}--\eqref{eq:mollified-FL-apply} gives, for every $\varepsilon>0$,
\begin{equation}\label{eq:W1-eps-bound}
\int_{\R^d}\psi\,d\sigma
\le 2m_1(\rho)\,\varepsilon
+(2\pi)^{-d}C_{d,p,s,\rho}\,K^d\,\varepsilon^{\,1-s-\frac{d}{q}}\  \mathrm{T}_{s,p}(\mu,\nu)
\end{equation}
Taking the supremum over admissible $\psi$ yields the same bound for $W_1(\mu,\nu)$.

Set $\beta:=s+\frac{d}{q}$ and introduce the abbreviations
\[
\begin{aligned}
A&:=2m_1(\rho),&
B&:=(2\pi)^{-d}C_{d,p,s,\rho},\\
T&:=\mathrm{T}_{s,p}(\mu,\nu),&
Q&:=BK^dT
\end{aligned}
\]
We first verify $\beta>1$. Since $s>\frac dp$ and $\frac dp=d-\frac dq$ (from $\frac1p+\frac1q=1$),
\[
\beta=s+\frac dq>\frac dp+\frac dq=d\ge 1
\]
(For $p=1$, $q=\infty$, this reads $\beta=s>d\ge1$, again by the window.) Then \eqref{eq:W1-eps-bound} reads
\[
\begin{aligned}
W_1(\mu,\nu)
&\le A\varepsilon+Q\varepsilon^{\,1-\beta}\\
&=A\varepsilon+Q\varepsilon^{-(\beta-1)}
=:h(\varepsilon)
\end{aligned}
\]
for every $\varepsilon>0$. If $T=0$ the claim is immediate (let $\varepsilon\downarrow0$). Otherwise we minimize $h$ over $\varepsilon>0$. Since
\[
h'(\varepsilon)=A-(\beta-1)Q\varepsilon^{-\beta}
\]
is negative for small $\varepsilon$ and positive for large $\varepsilon$, $h$ attains its global minimum at the unique critical point
\[
\varepsilon_*=\Big(\frac{(\beta-1)Q}{A}\Big)^{1/\beta}
\]
Substituting $\varepsilon_*$ back and simplifying, we compute the two terms separately.
\[
\begin{aligned}
A\varepsilon_*
&=A\cdot A^{-1/\beta}\big((\beta-1)Q\big)^{1/\beta}\\
&=A^{\frac{\beta-1}{\beta}}
  (\beta-1)^{\frac1\beta}Q^{\frac1\beta}
\end{aligned}
\]
and
\[
\begin{aligned}
Q\varepsilon_*^{-(\beta-1)}
&=Q\Big(\frac{A}{(\beta-1)Q}\Big)^{\frac{\beta-1}{\beta}}\\
&=A^{\frac{\beta-1}{\beta}}
  (\beta-1)^{-\frac{\beta-1}{\beta}}
  Q^{1-\frac{\beta-1}{\beta}}\\
&=A^{\frac{\beta-1}{\beta}}
  (\beta-1)^{-\frac{\beta-1}{\beta}}Q^{\frac1\beta}
\end{aligned}
\]
Adding the two terms and factoring out the common quantity
$A^{\frac{\beta-1}{\beta}}Q^{\frac1\beta}
(\beta-1)^{-\frac{\beta-1}{\beta}}$ gives
\[
\begin{aligned}
h(\varepsilon_*)
&=A^{\frac{\beta-1}{\beta}}Q^{\frac1\beta}
  (\beta-1)^{-\frac{\beta-1}{\beta}}\\
&\quad{}\times
  \Big[(\beta-1)^{\frac1\beta+\frac{\beta-1}{\beta}}+1\Big]\\
&=A^{\frac{\beta-1}{\beta}}Q^{\frac1\beta}
  (\beta-1)^{-\frac{\beta-1}{\beta}}\big[(\beta-1)+1\big]
\end{aligned}
\]
since $\frac1\beta+\frac{\beta-1}{\beta}=1$. Hence
\[
\begin{aligned}
W_1(\mu,\nu)
&\le h(\varepsilon_*)\\
&=\beta(\beta-1)^{-(\beta-1)/\beta}
  A^{(\beta-1)/\beta}B^{1/\beta}\\
&\quad{}\times
  \Big(K^d\mathrm{T}_{s,p}(\mu,\nu)\Big)^{1/\beta}
\end{aligned}
\]
Absorbing the fixed mollifier constants into a single constant $C=C(d,p,s)<\infty$ (with $\rho$ fixed once and for all, e.g.\ a fixed nonnegative bump supported in $B(0,1)$) gives
\[
W_1(\mu,\nu)\le C\,K^{d/\beta}\, \mathrm{T}_{s,p}(\mu,\nu)^{1/\beta},
\qquad \beta=s+\frac{d}{q}
\]
which is \eqref{eq:Wp-by-Tsp-bounded}.

For $r\ge 1$, note that $\supp(\mu),\supp(\nu)\subseteq B(0,K)$ implies $\|x-y\|\le 2K$ for $(x,y)$ in the support of any coupling of $(\mu,\nu)$.
Hence $\|x-y\|^r\le (2K)^{r-1}\|x-y\|$, and for any coupling $\pi$,
\[
\int_{\R^d\times\R^d}\|x-y\|^r\,d\pi(x,y)
\le (2K)^{r-1}\int_{\R^d\times\R^d}\|x-y\|\,d\pi(x,y)
\]
Taking the infimum over $\pi$ yields $W_r(\mu,\nu)^r\le (2K)^{r-1}W_1(\mu,\nu)$, that is,
\[
W_r(\mu,\nu)\le (2K)^{1-\frac1r}\,W_1(\mu,\nu)^{1/r}
\]
Substituting \eqref{eq:Wp-by-Tsp-bounded} into this inequality gives \eqref{eq:Wrr-by-Tsp-bounded}.
\end{proof}

\subsection{Proof of Proposition~\ref{prop:reverse-bounded-HY}}\label{proof:reverse-bounded-HY}
\begin{proof}
We retain all notation from the proof of Theorem~\ref{thm:explicit-reverse-bounded}. By
Kantorovich--Rubinstein duality and the McShane/clamp construction there, it suffices to bound
$\int\psi\,d\sigma$ for every $1$-Lipschitz $\psi$ with $\psi(0)=0$ and
$\supp(\psi)\subseteq B(0,2K)$, where $\sigma=\mu-\nu$, via the decomposition
\eqref{eq:decompose-moll} and the mollification error \eqref{eq:moll-error}. The only new
ingredient is a sharper bound on $\|\psi_\varepsilon\|_{\dot{\mathcal F}L_s^q}$.

\smallskip
\noindent\emph{Step 1. A Hausdorff--Young bound.}
As in Step~1 of the proof of Lemma~\ref{lem:mollified-FL}, $\nabla\psi$ exists a.e.\ with
$\|\nabla\psi\|_{L^\infty}\le1$ (Rademacher), and integration by parts with our Fourier
convention gives $\widehat{\partial_j\psi}(u)=-iu_j\,\widehat\psi(u)$, hence
$|u_j|\,|\widehat\psi(u)|=|\widehat{\partial_j\psi}(u)|$ for every $u$ and every $j$. Since
$\|u\|\le\sqrt d\,\max_j|u_j|$ and $\widehat{\psi_\varepsilon}(u)=\widehat\psi(u)\widehat\rho(\varepsilon u)$,
\[
\|u\|^{s}\,\big|\widehat{\psi_\varepsilon}(u)\big|
\ \le\ \sqrt d\ \|u\|^{s-1}\,\big|\widehat\rho(\varepsilon u)\big|\ \sum_{j=1}^d\big|\widehat{\partial_j\psi}(u)\big|
\]
Because $s\ge1$ and $\widehat\rho$ is Schwartz, the substitution $w=\varepsilon u$ gives
\[
\sup_{u\in\R^d}\ \|u\|^{s-1}\,\big|\widehat\rho(\varepsilon u)\big|
=\varepsilon^{\,1-s}\ \sup_{w\in\R^d}\ \|w\|^{s-1}\,\big|\widehat\rho(w)\big|
=:\varepsilon^{\,1-s}\,C_{\rho,s}<\infty
\]
Taking $L^q$ norms,
\[
\|\psi_\varepsilon\|_{\dot{\mathcal F}L_s^q}
\ \le\ \sqrt d\ \varepsilon^{\,1-s}\,C_{\rho,s}\ \sum_{j=1}^d\big\|\widehat{\partial_j\psi}\big\|_{L^q}
\]
Each $\partial_j\psi$ lies in $L^{p}(\R^d)$ with
$\|\partial_j\psi\|_{L^p}\le\|\nabla\psi\|_{L^\infty}\,|B(0,2K)|^{1/p}\le c_d\,K^{d/p}$, and
since $1\le p\le2$ with conjugate $q\ge2$, the Hausdorff--Young inequality (for $p=1$, $q=\infty$,
the trivial bound $\|\widehat g\|_{L^\infty}\le\|g\|_{L^1}$) gives
$\|\widehat{\partial_j\psi}\|_{L^q}\le C_{d,p}\,\|\partial_j\psi\|_{L^p}$. Altogether
\begin{equation}\label{eq:HY-bound}
\|\psi_\varepsilon\|_{\dot{\mathcal F}L_s^q}
\ \le\ C_{d,p,s,\rho}\ K^{d/p}\ \varepsilon^{\,1-s}
\end{equation}

\smallskip
\noindent\emph{Step 2. Optimization.}
Substituting \eqref{eq:HY-bound} for \eqref{eq:mollified-FL-apply} in the chain
\eqref{eq:decompose-moll}--\eqref{eq:W1-eps-bound} yields, for every $\varepsilon>0$,
\[
W_1(\mu,\nu)\ \le\ 2m_1(\rho)\,\varepsilon
+(2\pi)^{-d}\,C_{d,p,s,\rho}\,K^{d/p}\,\varepsilon^{\,1-s}\ \mathrm T_{s,p}(\mu,\nu)
\]
If $s>1$, the right-hand side has the form $A\varepsilon+Q\varepsilon^{-(s-1)}$ with
$A:=2m_1(\rho)$ and $Q:=(2\pi)^{-d}C_{d,p,s,\rho}\,K^{d/p}\,\mathrm T_{s,p}(\mu,\nu)$. Minimizing
over $\varepsilon>0$ exactly as in the proof of Theorem~\ref{thm:explicit-reverse-bounded}, with
$\beta$ there replaced by $s>1$, gives
\[
W_1(\mu,\nu)\ \le\ s\,(s-1)^{-(s-1)/s}\,A^{(s-1)/s}\,Q^{1/s}
\ \le\ C(d,p,s)\ K^{\frac d{ps}}\ \mathrm T_{s,p}(\mu,\nu)^{1/s}
\]
after absorbing the fixed mollifier constants (with $\rho$ fixed once and for all). If $s=1$, the
second term does not depend on $\varepsilon$, and letting $\varepsilon\downarrow0$ gives
\[
W_1(\mu,\nu)\ \le\ (2\pi)^{-d}\,C_{d,p,1,\rho}\ K^{d/p}\ \mathrm T_{1,p}(\mu,\nu)
\]
Finally, the hypothesis $s\ge1$ is automatic when $d\ge2$, because the window forces
$s>\frac dp\ge\frac d2\ge1$.
\end{proof}
\subsection{Proof of Lemma~\ref{lem:energy-constant}}\label{proof:energy-constant}
\begin{proof}
The Fourier identity \eqref{eq:riesz-alpha} below, together with the resulting representation of
the energy distance, is classical \citep[Lemma~1]{szekely2007measuring}. We include the full
computation so that the constant is verified independently in our normalization.
Fix $d\ge 1$ and $H\in(0,1)$, and write $\alpha:=2H\in(0,2)$ and $s:=(d+\alpha)/2$.
Let $\mu,\nu\in\mathcal P_{\alpha}(\R^d)$, and let $X,X'\sim\mu$ i.i.d. and $Y,Y'\sim\nu$ i.i.d., all independent.
Since $\mu,\nu\in\mathcal P_{\alpha}$, one has $\E\|X\|^{\alpha}<\infty$ and $\E\|Y\|^{\alpha}<\infty$, hence also
\[
\E\|X-Y\|^{\alpha}\le \E\big(\|X\|+\|Y\|\big)^{\alpha}
\le c_\alpha\bigl(\E\|X\|^{\alpha}+\E\|Y\|^{\alpha}\bigr)<\infty,
\qquad c_\alpha:=\max\{1,2^{\alpha-1}\}
\]
(using subadditivity of $t\mapsto t^{\alpha}$ for $\alpha\le1$ and convexity for $\alpha\ge1$),
and similarly $\E\|X-X'\|^{\alpha}<\infty$ and $\E\|Y-Y'\|^{\alpha}<\infty$. For $k_H(x,y)=\|x\|^\alpha+\|y\|^\alpha-\|x-y\|^\alpha$,
\[
\mathrm{MMD}_{k_H}(\mu,\nu)^2
=\E k_H(X,X')+\E k_H(Y,Y')-2\E k_H(X,Y)
\]
Expanding each term and cancelling $\E\|X\|^\alpha$ and $\E\|Y\|^\alpha$ gives
\begin{equation}\label{eq:mmd-energy-alpha}
\mathrm{MMD}_{k_H}(\mu,\nu)^2
=
2\,\E\|X-Y\|^{\alpha}-\E\|X-X'\|^{\alpha}-\E\|Y-Y'\|^{\alpha}
\end{equation}
For $z\in\R^d$ define
\[
I_\alpha(z):=\int_{\R^d}\frac{1-\cos\langle u,z\rangle}{\|u\|^{d+\alpha}}\,du
\]
The integrand is nonnegative and integrable over $\R^d$ for $\alpha\in(0,2)$, hence $I_\alpha(z)\in(0,\infty)$ for $z\neq 0$ and $I_\alpha(0)=0$.
Rotational invariance and the change of variables $u=\|z\|^{-1}Ru'$ (with $R$ orthogonal and $Rz=\|z\|e_1$) yield the scaling identity
\begin{equation}\label{eq:I-scaling}
I_\alpha(z)=\|z\|^\alpha\, I_\alpha(e_1),\qquad z\in\R^d
\end{equation}
It remains to compute $I_\alpha(e_1)$. Writing $u=(t,v)\in\R\times\R^{d-1}$ gives
\[
I_\alpha(e_1)=\int_{\R}(1-\cos t)\left(\int_{\R^{d-1}}\frac{dv}{(t^2+\|v\|^2)^{(d+\alpha)/2}}\right)dt
\]
A polar-coordinate computation in $\R^{d-1}$ (Beta integral) yields, for $t\neq 0$,
\begin{equation}\label{eq:inner-v}
\int_{\R^{d-1}}\frac{dv}{(t^2+\|v\|^2)^{(d+\alpha)/2}}
=\pi^{\frac{d-1}{2}}\frac{\Gamma\!\left(\frac{1+\alpha}{2}\right)}{\Gamma\!\left(\frac{d+\alpha}{2}\right)}\,|t|^{-1-\alpha}
\end{equation}
Consequently,
\begin{equation}\label{eq:I-e1-reduced}
I_\alpha(e_1)
=
2\pi^{\frac{d-1}{2}}\frac{\Gamma\!\left(\frac{1+\alpha}{2}\right)}{\Gamma\!\left(\frac{d+\alpha}{2}\right)}
\int_0^\infty (1-\cos t)\,t^{-1-\alpha}\,dt
\end{equation}
To evaluate the remaining one-dimensional integral, use the Laplace representation
$t^{-1-\alpha}=\Gamma(1+\alpha)^{-1}\int_0^\infty r^{\alpha}e^{-rt}\,dr$ for $\alpha\in(0,2)$, which gives (by Tonelli)
\[
\int_0^\infty (1-\cos t)\,t^{-1-\alpha}\,dt
=
\frac{1}{\Gamma(1+\alpha)}\int_0^\infty r^\alpha
\left(\int_0^\infty (1-\cos t)e^{-rt}\,dt\right)dr
\]
The inner integral equals
\[
\int_0^\infty e^{-rt}\,dt-\int_0^\infty e^{-rt}\cos t\,dt
=\frac1r-\frac{r}{r^2+1}
=\frac{1}{r(r^2+1)}
\]
hence
\[
\int_0^\infty (1-\cos t)\,t^{-1-\alpha}\,dt
=
\frac{1}{\Gamma(1+\alpha)}\int_0^\infty \frac{r^{\alpha-1}}{1+r^2}\,dr
=
\frac{1}{\Gamma(1+\alpha)}\frac{\pi}{2\sin(\pi\alpha/2)}
\]
using the classical Beta integral $\int_0^\infty \frac{r^{a-1}}{1+r^2}\,dr=\frac{\pi}{2\sin(\pi a/2)}$ for $a\in(0,2)$.
Substituting into \eqref{eq:I-e1-reduced} yields
\[
I_\alpha(e_1)
=
\pi^{\frac{d-1}{2}}\frac{\Gamma\!\left(\frac{1+\alpha}{2}\right)}{\Gamma\!\left(\frac{d+\alpha}{2}\right)}
\cdot \frac{\pi}{\Gamma(1+\alpha)\sin(\pi\alpha/2)}
\]
Now set $\alpha=2H$. Using the Gamma duplication formula
\[
\Gamma(1+2H)=H\,2^{2H}\pi^{-1/2}\Gamma(H)\Gamma\!\left(H+\tfrac12\right)
\]
together with the reflection formula $\Gamma(H)\Gamma(1-H)=\pi/\sin(\pi H)$, and noting that
$\Gamma\!\left(\frac{1+\alpha}{2}\right)=\Gamma\!\left(H+\frac12\right)$ and $\Gamma\!\left(\frac{d+\alpha}{2}\right)=\Gamma\!\left(\frac d2+H\right)$,
one obtains
\[
I_\alpha(e_1)=\frac{\pi^{d/2}\Gamma(1-H)}{H\,2^{2H}\Gamma(\frac d2+H)}=:c(d,H)
\]
Combining this with \eqref{eq:I-scaling} gives the Fourier identity
\begin{equation}\label{eq:riesz-alpha}
\|z\|^{\alpha}=\frac{1}{c(d,H)}\int_{\R^d}\frac{1-\cos\langle u,z\rangle}{\|u\|^{d+\alpha}}\,du,
\qquad z\in\R^d
\end{equation}

Applying \eqref{eq:riesz-alpha} to each term in \eqref{eq:mmd-energy-alpha}, since the integrand in \eqref{eq:riesz-alpha} is nonnegative, Tonelli's theorem yields, for instance,
\[
\E\|X-Y\|^\alpha
=
\frac{1}{c(d,H)}\int_{\R^d}\frac{1-\E\cos\langle u,X-Y\rangle}{\|u\|^{d+\alpha}}\,du
\]
and similarly for $X-X'$ and $Y-Y'$. Substituting these representations into \eqref{eq:mmd-energy-alpha} cancels the constant terms and gives
\[
\mathrm{MMD}_{k_H}(\mu,\nu)^2
=
\frac{1}{c(d,H)}\int_{\R^d}\frac{
\E\cos\langle u,X-X'\rangle+\E\cos\langle u,Y-Y'\rangle-2\,\E\cos\langle u,X-Y\rangle
}{\|u\|^{d+\alpha}}\,du
\]
Using independence and characteristic functions,
\[
\E e^{i\langle u,X-X'\rangle}=|\Bchi_\mu(u)|^2,\qquad
\E e^{i\langle u,Y-Y'\rangle}=|\Bchi_\nu(u)|^2,\qquad
\E e^{i\langle u,X-Y\rangle}=\Bchi_\mu(u)\overline{\Bchi_\nu(u)}
\]
hence
\[
\begin{aligned}
&\E\cos\langle u,X-X'\rangle
 +\E\cos\langle u,Y-Y'\rangle
 -2\,\E\cos\langle u,X-Y\rangle\\
&\quad=
|\Bchi_\mu(u)|^2+|\Bchi_\nu(u)|^2
-2\Re\!\bigl(\Bchi_\mu(u)\overline{\Bchi_\nu(u)}\bigr)\\
&\quad=|\Bchi_\mu(u)-\Bchi_\nu(u)|^2
\end{aligned}
\]
Therefore,
\[
\mathrm{MMD}_{k_H}(\mu,\nu)^2
=
\frac{1}{c(d,H)}
\int_{\R^d}\frac{|\Bchi_\mu(u)-\Bchi_\nu(u)|^2}{\|u\|^{d+\alpha}}\,du
\]
which is \eqref{eq:energy-const}. Finally, since $2s=d+\alpha$, the definition of $\mathrm T_{s,2}$ gives
\[
\mathrm T_{s,2}(\mu,\nu)^2
=
\int_{\R^d}\frac{|\Bchi_\mu(u)-\Bchi_\nu(u)|^2}{\|u\|^{2s}}\,du
=
\int_{\R^d}\frac{|\Bchi_\mu(u)-\Bchi_\nu(u)|^2}{\|u\|^{d+\alpha}}\,du
=
c(d,H)\,\mathrm{MMD}_{k_H}(\mu,\nu)^2
\]
which is \eqref{eq:Ts2-is-MMD}.
\end{proof}

\subsection{Proof of Corollary~\ref{cor:hilbert-dual}}\label{proof:hilbert-dual}
\begin{proof}
Assume $p=q=2$. By definition,
\[
\|\phi\|_{\dot{\mathcal F}L_s^2}=\big\|\|u\|^{s}\widehat{\phi}(u)\big\|_{L^2(\R^d)}
\]
On the other hand, the homogeneous Sobolev seminorm satisfies the Plancherel identity
\[
\|\phi\|_{\dot H^{s}}^2
:=\|\ (-\Delta)^{s/2}\phi\ \|_{L^2}^2
=(2\pi)^{-d}\int_{\R^d}\|u\|^{2s}\,|\widehat{\phi}(u)|^2\,du
\]
Therefore
\[
\|\phi\|_{\dot H^{s}}=(2\pi)^{-d/2}\,\|\phi\|_{\dot{\mathcal F}L_s^2}
\]
In particular, $\dot{\mathcal F}L_s^2$ coincides with the homogeneous Sobolev space modulo constants as a vector space. Constants have $\widehat{\phi}$ supported at $u=0$ and hence vanish under the homogeneous seminorm, so the natural identification is
\[
\dot{\mathcal F}L_s^2=\dot{H}^s(\R^d)/\R
\]
with seminorms related by
\[
\|\phi\|_{\dot{\mathcal F}L_s^2}=(2\pi)^{d/2}\|\phi\|_{\dot H^s}
\]
A word on the \emph{realization} of $\dot H^s$ in the present range is in order, since for $s>d/2$ the homogeneous Sobolev seminorm does not separate points of $\mathcal S'$ without a quotient. Throughout our window $\frac d2<s<\frac d2+1$, the seminorm $\|\cdot\|_{\dot H^s}$ vanishes precisely on polynomials of degree $0$ (constants), and the completion of $\mathcal S(\R^d)$ under it is standardly realized as a space of continuous functions modulo constants. Indeed, for $\frac d2<s<\frac d2+1$ one has the Morrey-type embedding of $\dot H^s$ into the homogeneous H\"older class $\dot C^{\,s-d/2}$, so equivalence classes admit continuous representatives determined up to an additive constant. See \citet[Chapter~6]{grafakos2008classical} for the relevant Littlewood--Paley characterizations. This is exactly the quotient-by-constants convention adopted in the construction of $\dot{\mathcal F}L^q_s$ in Section~\ref{sec:dual-ipm}, so the identification above is consistent with both constructions.

Applying Theorem~\ref{thm:dual-ipm} with $p=q=2$ and $\sigma=\mu-\nu$ gives
\[
\sup_{\substack{\phi\in\dot{\mathcal F}L_s^2\\ \|\phi\|_{\dot{\mathcal F}L_s^2}\le 1}}
\mathrm{Re}\int_{\R^d}\phi\,d(\mu-\nu)
=(2\pi)^{-d}\,\mathrm T_{s,2}(\mu,\nu)
\]
By Remark~\ref{rem:dual-real}, the supremum may equivalently be taken over \emph{real-valued} test functions with $\int\phi\,d(\mu-\nu)$ in place of its real part. Moreover, at $p=2$ the extremizer of Proposition~\ref{prop:dual-optimizer} is itself real-valued, since its Fourier transform $\widehat{\phi^\ast}(u)=\|u\|^{-2s}\,\widehat\sigma(u)/\|f\|_{L^2}$ satisfies the Hermitian symmetry $\widehat{\phi^\ast}(-u)=\overline{\widehat{\phi^\ast}(u)}$ (because $\sigma$ is a real measure, so $\widehat\sigma(-u)=\overline{\widehat\sigma(u)}$).
Using $\|\phi\|_{\dot{\mathcal F}L_s^2}=(2\pi)^{d/2}\|\phi\|_{\dot H^s}$, the constraint
$\|\phi\|_{\dot{\mathcal F}L_s^2}\le 1$ is equivalent to $\|\phi\|_{\dot H^s}\le (2\pi)^{-d/2}$, hence by rescaling $\phi\mapsto(2\pi)^{-d/2}\phi$,
\[
\begin{aligned}
\sup_{\|\phi\|_{\dot H^{s}}\le 1}
\int_{\R^d}\phi\,d(\mu-\nu)
&=(2\pi)^{d/2}
\sup_{\|\phi\|_{\dot{\mathcal F}L_s^2}\le 1}
\int_{\R^d}\phi\,d(\mu-\nu)\\
&=(2\pi)^{d/2}(2\pi)^{-d}
\mathrm T_{s,2}(\mu,\nu)\\
&=(2\pi)^{-d/2}\mathrm T_{s,2}(\mu,\nu)
\end{aligned}
\]
which is the stated identity. The final sentence follows because $\mathrm T_{s,2}$ equals (up to the explicit constant from Lemma~\ref{lem:energy-constant} when $2s=d+2H$) the translation-invariant energy-kernel MMD, so the same quantity is both a negative-Sobolev IPM and a translation-invariant MMD in the Hilbert case.
\end{proof}

\subsection{Proof of Theorem~\ref{thm:reverse-holder}}\label{proof:reverse-holder}
\begin{proof}
Set \(H:=s-\frac d2\in(0,1)\) and \(\alpha:=2H=2s-d\in(0,2)\), so that \(2s=d+\alpha\). Fix \(\mu,\nu\in\mathcal M_{\gamma,S}\) with \(\gamma>1\).

\smallskip
\noindent\textbf{Step 1. Kernel dictionary.}
Our kernel from Lemma~\ref{lem:energy-constant} is
\[
k_H(x,y)=\|x\|^{\alpha}+\|y\|^{\alpha}-\|x-y\|^{\alpha}
=\|x\|^{2H}+\|y\|^{2H}-\|x-y\|^{2H}
\]
which is \emph{exactly} the Energy Kernel of order $\alpha_{\mathrm{MD}}=H$ in the sense of \citet[Section~3.3]{modeste2024translation} (recall from Section~\ref{sec:mmd} that their order parameter, which they denote $\alpha$, is half our kernel exponent). We write $d_H:=d_{k_H}$ for the associated MMD. Their standing assumption for this kernel family is that the order lie in $(0,1)$, i.e.\ $H\in(0,1)$, which is precisely our hypothesis $s\in(\frac d2,\frac d2+1)$. No further restriction on $\alpha=2H\in(0,2)$ is imposed.

\smallskip
\noindent\textbf{Step 2. Identification of $d_H$ with $\mathrm T_{s,2}$ on $\mathcal M_{\gamma,S}$.}
By \citet[Section~3.3]{modeste2024translation}, the MMD $d_H$ is defined on the class of (signed) measures with finite $H$-moment. Since $H<1<\gamma$, H\"older's inequality gives, for every $\mu\in\mathcal M_{\gamma,S}$,
\[
\int\|x\|^{H}\,d\mu(x)\le\Big(\int\|x\|^{\gamma}\,d\mu(x)\Big)^{H/\gamma}\le S^{H/\gamma}<\infty
\]
so $d_H(\mu,\nu)$ is well defined for all \(\mu,\nu\in\mathcal M_{\gamma,S}\).
Moreover, $k_H$ admits the variogram/spectral representation with spectral density $\frac{1}{c(d,H)}\|u\|^{-d-2H}$. By the deterministic Fourier identity established in the proof of Lemma~\ref{lem:energy-constant} (equation~\eqref{eq:riesz-alpha}, whose derivation is moment-free),
\[
\|z\|^{2H}=\frac{1}{c(d,H)}\int_{\R^d}\frac{1-\cos\langle u,z\rangle}{\|u\|^{d+2H}}\,du,
\qquad z\in\R^d
\]
and expanding $k_H(x,y)$ using this identity at $z\in\{x,\,y,\,x-y\}$ together with the evenness of the spectral density gives
\[
k_H(x,y)
=\frac{1}{c(d,H)}\int_{\R^d}\big(1-e^{i\langle u,x\rangle}\big)\overline{\big(1-e^{i\langle u,y\rangle}\big)}\,\frac{du}{\|u\|^{d+2H}}
\]
(the imaginary parts cancel by symmetry $u\mapsto-u$, and the real part reproduces $\|x\|^{2H}+\|y\|^{2H}-\|x-y\|^{2H}$ term by term).
Consequently, for probability measures $\mu,\nu$ in the domain of the kernel embedding, the associated MMD has the spectral form
\begin{equation}\label{eq:energy-fourier}
d_H^2(\mu,\nu)=\mathrm{MMD}_{k_H}^2(\mu,\nu)
=
\frac{1}{c(d,H)}\int_{\mathbb R^d}\frac{|\Bchi_\mu(u)-\Bchi_\nu(u)|^2}{\|u\|^{d+2H}}\,du
\end{equation}
this is the specialization to $k_H$ of the explicit Fourier formula for translation-invariant MMDs in \citet[Proposition~7]{modeste2024translation} (for probability pairs the zero-mass correction term there vanishes), and, under the additional moment assumption $\mu,\nu\in\mathcal P_\alpha(\R^d)$, it is verified independently and with the explicit constant in our Lemma~\ref{lem:energy-constant}.
Since \(2s=d+\alpha=d+2H\), the definition of \(\mathrm{T}_{s,2}\) gives
\[
\mathrm{T}_{s,2}^2(\mu,\nu)
=
\int_{\mathbb R^d\setminus\{0\}}
\frac{|\Bchi_\mu(u)-\Bchi_\nu(u)|^2}{\|u\|^{2s}} du
=
\int_{\mathbb R^d\setminus\{0\}}
\frac{|\Bchi_\mu(u)-\Bchi_\nu(u)|^2}{\|u\|^{d+2H}} du
\]
Comparing with \eqref{eq:energy-fourier} yields the exact identification
\begin{equation}\label{eq:equiv}
\mathrm{MMD}_{k_H}(\mu,\nu)
=
c(d,H)^{-1/2}\, \mathrm{T}_{s,2}(\mu,\nu),
\qquad \mu,\nu\in\mathcal M_{\gamma,S}
\end{equation}
with the explicit constant $c(d,H)$ of Lemma~\ref{lem:energy-constant}. In particular there is a single normalization constant in the paper, and the reciprocal relation between the two ways of writing the identity (constant on the MMD side versus on the Toscani side) is fixed once and for all by \eqref{eq:energy-const}.

\smallskip
\noindent\textbf{Step 3. Importing the moment-class reverse inequality.}
\citet[Proposition~17]{modeste2024translation} states. For order $H\in(0,1)$ and any $\gamma>1$, $S>0$, there is an explicit constant
\(C_0=C_0(d,H,\gamma,S)<\infty\) \citep[Equation~37]{modeste2024translation} such that, with
\[
\rho=\frac{\gamma-1}{\gamma(d+H+1)-H}
\]
one has, for all \(\mu,\nu\) in their class \(\mathcal T_{\gamma,S}(\R^d)=\{\mu\in\mathcal P(\R^d):\int\|x\|^\gamma\,d\mu(x)< S\}\),
\begin{equation}\label{eq:MD}
W_1(\mu,\nu)\le C_0\, d_H^{\,\rho}(\mu,\nu)=C_0\,\mathrm{MMD}_{k_H}^{\,\rho}(\mu,\nu)
\end{equation}
Their class is stated with a strict inequality, whereas our $\mathcal M_{\gamma,S}$ uses ``$\le S$''. This is immaterial, since $\mathcal M_{\gamma,S}\subset\mathcal T_{\gamma,S'}(\R^d)$ for every $S'>S$, so \eqref{eq:MD} applies verbatim on $\mathcal M_{\gamma,S}$ with the constant $C_0(d,H,\gamma,S')$, and one may simply take $S'=2S$. We absorb this once and for all into the constant and continue to write $C_0=C_0(d,H,\gamma,S)$.
Substituting \eqref{eq:equiv} into \eqref{eq:MD} gives
\[
W_1(\mu,\nu)
\le
C_0 \big(c(d,H)^{-1/2}\,\mathrm{T}_{s,2}(\mu,\nu)\big)^{\rho}
=
\underbrace{C_0\, c(d,H)^{-\rho/2}}_{=:C(d,H,\gamma,S)}
\ \mathrm{T}_{s,2}^\rho(\mu,\nu)
\]
which is the claimed inequality, with $\rho$ expressed in terms of $\alpha$ as $\rho=\frac{\gamma-1}{\gamma(d+\frac\alpha2+1)-\frac\alpha2}$.

\smallskip
\noindent\textbf{Step 4. The exponent lies in $(0,1)$ and both sides are finite.}
Since \(\gamma>1\), the numerator $\gamma-1$ is positive, so \(\rho>0\). For $\rho<1$, note
\[
\gamma(d+H+1)-H-(\gamma-1)
=\gamma(d+H)+(1-H)
>0
\]
because $\gamma>0$, $d+H>0$, and $H<1$. Hence the denominator strictly exceeds the numerator and \(\rho\in(0,1)\).

Finally, \(\mu\in\mathcal M_{\gamma,S}\) implies \(\int\|x\| d\mu(x)\le S^{1/\gamma}\) by H\"older (as $\gamma>1$), hence
\(W_1(\mu,\nu)\le\int\|x\|d\mu+\int\|x\|d\nu\le 2S^{1/\gamma}<\infty\). Likewise $\mathrm T_{s,2}(\mu,\nu)<\infty$ on $\mathcal M_{\gamma,S}$ throughout the window. By Lemma~\ref{lem:chi-lip-moment}, $|\Bchi_\mu(u)-\Bchi_\nu(u)|\le\min\{2,\,2S^{1/\gamma}\|u\|\}$, and
\[
\int_{\R^d\setminus\{0\}}\frac{\min\{4,\,4S^{2/\gamma}\|u\|^2\}}{\|u\|^{2s}}\,du<\infty
\]
since near the origin the integrand is $O(\|u\|^{2-2s})$ with $2-2s>-d$ (as $s<\frac d2+1$), and at infinity it is $O(\|u\|^{-2s})$ with $2s>d$. This completes the proof.
\end{proof}

\subsection{Proof of Corollary~\ref{cor:bounded-support}}\label{proof:bounded-support}
\begin{proof}
Fix \(\alpha\in(0,2)\), \(H=\alpha/2\in(0,1)\), and \(s=(d+\alpha)/2\). Assume \(\text{supp}(\mu)\cup\text{supp}(\nu)\subset B(0,K)\). In the notation of \citet{modeste2024translation}, $\mu,\nu\in\mathcal T_{\infty,K}(\R^d)=\{\eta\in\mathcal P(\R^d):\eta(\R^d\setminus B(0,K))=0\}$.

\citet[Proposition~16]{modeste2024translation} states. For order $H\in(0,1)$ and all $\mu,\nu\in\mathcal T_{\infty,K}(\R^d)$,
\begin{equation}\label{eq:MD-bounded}
W_1(\mu,\nu)\le C_{\mathrm{MD}}\,
d_H(\mu,\nu)^{ \widetilde{\rho}}
=C_{\mathrm{MD}}\,
\mathrm{MMD}_{k_H}^{ \widetilde{\rho}}(\mu,\nu),
\qquad
\widetilde{\rho}=\frac{1}{d+1+H}
\end{equation}
where $k_H$ is their Energy Kernel of order $H$ --- i.e.\ our kernel of exponent $\alpha=2H$ from Lemma~\ref{lem:energy-constant}, by the dictionary of Section~\ref{sec:mmd} --- and the constant is explicit, $C_{\mathrm{MD}}=2^{d/(d+H+1)}D\,K^{(d+1)/(d+H+1)}$ with $D=D(d,H)$ given in their Equation~(35) \citep[Equation~36]{modeste2024translation}.

Bounded support gives all moments, so \(\mu,\nu\in\mathcal P_\alpha(\R^d)\) and Lemma~\ref{lem:energy-constant} applies directly. Since \(2s=d+\alpha\),
\[
\begin{aligned}
\mathrm{T}_{s,2}^2(\mu,\nu)
&=\int_{\R^d\setminus\{0\}}
\frac{|\Bchi_\mu(u)-\Bchi_\nu(u)|^2}{\|u\|^{2s}}\,du\\
&=\int_{\R^d\setminus\{0\}}
\frac{|\Bchi_\mu(u)-\Bchi_\nu(u)|^2}
     {\|u\|^{d+\alpha}}\,du\\
&=c(d,H)\mathrm{MMD}_{k_H}^2(\mu,\nu)
\end{aligned}
\]
that is,
\begin{equation}\label{eq:MMD-T}
\mathrm{MMD}_{k_H}(\mu,\nu) = c(d,H)^{-1/2}\ \mathrm{T}_{s,2}(\mu,\nu)
\end{equation}
with the explicit constant $c(d,H)$ of Lemma~\ref{lem:energy-constant} (the same single normalization constant used throughout the paper).

Substituting \eqref{eq:MMD-T} into \eqref{eq:MD-bounded} yields
\[
\begin{aligned}
W_1(\mu,\nu)
&\le C_{\mathrm{MD}}
\big(c(d,H)^{-1/2}\mathrm{T}_{s,2}(\mu,\nu)\big)^{\widetilde{\rho}}\\
&=\underbrace{C_{\mathrm{MD}}c(d,H)^{-\widetilde{\rho}/2}}
_{=:\widetilde{C}(d,H,K)}
\mathrm{T}_{s,2}^{\widetilde{\rho}}(\mu,\nu)
\end{aligned}
\]
This proves the first claim, with $\widetilde\rho=\frac{1}{d+1+H}=\frac{1}{d+1+\frac\alpha2}$.

Assume now that \(\text{supp}(\mu)\cup\text{supp}(\nu)\) is contained in a set of diameter \(\mathrm{diam}\).
Then for any coupling \(\pi\in\Pi(\mu,\nu)\), one has \(\|x-y\|\le \mathrm{diam}\) for \(\pi\)-a.e. \((x,y)\),
and hence
\[
\int \|x-y\|^r d\pi(x,y)
=
\int \|x-y\|^{r-1}\|x-y\| d\pi(x,y)
\le
\mathrm{diam}^{r-1}\int \|x-y\| d\pi(x,y)
\]
Taking the infimum over \(\pi\in\Pi(\mu,\nu)\) gives
\[
W_r^r(\mu,\nu) \le \mathrm{diam}^{r-1}W_1(\mu,\nu),
\qquad\text{hence}\qquad
W_r(\mu,\nu)\le \mathrm{diam}^{1-\frac1r} W_1^{1/r}(\mu,\nu)
\]

Using the \(W_1\) bound from above,
\[
\begin{aligned}
W_r(\mu,\nu)
&\le \mathrm{diam}^{1-\frac1r}
\big(\widetilde{C}\mathrm{T}_{s,2}^{\widetilde{\rho}}(\mu,\nu)\big)^{1/r}\\
&=\underbrace{\mathrm{diam}^{1-\frac1r}\widetilde{C}^{1/r}}
_{=:\dbtilde{C}(d,H,K,\mathrm{diam},r)}
\mathrm{T}_{s,2}(\mu,\nu)^{\dbtilde{\rho}}
\end{aligned}
\]
where
\[
\dbtilde{\rho}=\frac{\widetilde{\rho}}{r}=\frac{1}{r\big(d+1+\frac\alpha2\big)}
\]
which completes the proof.
\end{proof}

\subsection{Proof of Corollary~\ref{cor:top_equiv}}\label{proof:top_equiv}
\begin{proof}
\textit{(1) Moment class \(\mathcal M_{\gamma,S}\).}\\
Fix \(d\ge 1\), \(s\in(\frac d2,\frac d2+1)\), and set \(\alpha:=2s-d\in(0,2)\) and \(H:=\alpha/2\in(0,1)\). Let
\(\{\mu_n\}\subset\mathcal M_{\gamma,S}\) and
\(\mu\in\mathcal M_{\gamma,S}\).

\emph{(\(\Rightarrow\))} Assume \(W_1(\mu_n,\mu)\to 0\). Then \(\mu_n\Rightarrow\mu\), hence
\(\Bchi_{\mu_n}(u)\to\Bchi_\mu(u)\) for every \(u\). Since \(\mu_n,\mu\in\mathcal M_{\gamma,S}\), H\"older's inequality gives
\[
\int\|x\|\,d\mu_n(x)+\int\|x\|\,d\mu(x)\le 2S^{1/\gamma}
\]
and therefore, by Lemma~\ref{lem:chi-lip-moment} applied at \(v=0\) together with the trivial bound \(|\Bchi_{\mu_n}-\Bchi_\mu|\le2\),
\[
\frac{|\Bchi_{\mu_n}(u)-\Bchi_\mu(u)|^2}{\|u\|^{2s}}
\le
\frac{\min\{4,\ 4S^{2/\gamma}\|u\|^2\}}{\|u\|^{2s}}
\]
The right-hand side is integrable on \(\R^d\setminus\{0\}\). At infinity the numerator is \(4\) and \(2s>d\). At the origin the numerator is \(4S^{2/\gamma}\|u\|^2\), giving exponent \(2-2s>-d\), which holds throughout the full window \(s<\frac d2+1\). Dominated
convergence (pointwise convergence of the integrand to \(0\), with the fixed integrable dominating function above) gives \(\mathrm{T}_{s,2}(\mu_n,\mu)\to 0\).

\emph{(\(\Leftarrow\))} Assume \(\mathrm{T}_{s,2}(\mu_n,\mu)\to 0\). Since \(\mu_n,\mu\in\mathcal M_{\gamma,S}\),
Theorem~\ref{thm:reverse-holder} applies throughout the full window. There exist \(C'=C'(d,H,\gamma,S)<\infty\) and
\(\rho=\frac{\gamma-1}{\gamma(d+H+1)-H}\in(0,1)\) such that
\[
W_1(\mu_n,\mu)\ \le\ C' \mathrm{T}_{s,2}(\mu_n,\mu)^{\rho}\qquad\forall n
\]
Therefore \(W_1(\mu_n,\mu)\to 0\). This proves the equivalence on \(\mathcal M_{\gamma,S}\).

\textit{(2) Bounded-diameter class and \(W_p\)\ , for \(p\ge 1\).}\\
Fix \(p\ge 1\) and assume \(\text{supp}(\mu_n)\cup\text{supp}(\mu)\subseteq A\) for some set \(A\) with
\(\mathrm{diam}(A)\le \mathrm{diam}\) (uniform in \(n\)).
Choose \(x_0\in A\) and set \(K_A:=\|x_0\|+\mathrm{diam}\), so \(A\subseteq B(0,K_A)\).

\emph{(\(\Rightarrow\))} Assume \(W_p(\mu_n,\mu)\to 0\). Then \(W_1(\mu_n,\mu)\to0\) and hence
\(\mu_n\Rightarrow\mu\). By Lemma~\ref{lem:chi-lip} (first moments bounded by \(K_A\) on the common support ball),
\[
\frac{|\Bchi_{\mu_n}(u)-\Bchi_\mu(u)|^2}{\|u\|^{2s}}
\le
\frac{\min\{4,\ 4K_A^2\|u\|^2\}}{\|u\|^{2s}}
\]
and the same dominated-convergence argument as above (the domination is integrable on the full window) yields
\(\mathrm{T}_{s,2}(\mu_n,\mu)\to 0\).

\emph{(\(\Leftarrow\))} Assume \(\mathrm{T}_{s,2}(\mu_n,\mu)\to 0\). Since \(A\) is bounded, the measures are
supported in the common ball \(B(0,K_A)\). By Corollary~\ref{cor:bounded-support}, valid for all \(\alpha\in(0,2)\), there exist \(\widetilde C<\infty\) and
\(\widetilde\rho=\frac{1}{d+1+\frac\alpha2}\in(0,1)\) such that
\[
W_1(\mu_n,\mu)\ \le\ \widetilde C \mathrm{T}_{s,2}(\mu_n,\mu)^{\widetilde\rho}\qquad\forall n
\]
hence \(W_1(\mu_n,\mu)\to 0\). On a common set of diameter \(\mathrm{diam}\), the interpolation bound
(from Corollary~\ref{cor:bounded-support}) gives for every \(n\),
\[
W_p(\mu_n,\mu)\ \le\ \mathrm{diam}^{1-\frac1p}  W_1(\mu_n,\mu)^{1/p}
\]
Therefore \(W_p(\mu_n,\mu)\to 0\). This establishes the equivalence of convergences on bounded-diameter classes.

\end{proof}
\subsection{Proof of Proposition~\ref{prop:empirical-mean}}\label{proof:empirical-mean}
\begin{proof}
Fix $u\in\R^d$ and write $Z_j(u):=e^{i\langle u,X_j\rangle}-\Bchi_\mu(u)$, so that
\[
\Bchi_{\hat\mu_n}(u)-\Bchi_\mu(u)=\frac1n\sum_{j=1}^nZ_j(u)
\]
with $Z_1(u),\dots,Z_n(u)$ i.i.d., bounded, and centered. Independence gives $\E[Z_j(u)\overline{Z_k(u)}]=0$ for $j\neq k$,
hence
\[
\E\big|\Bchi_{\hat\mu_n}(u)-\Bchi_\mu(u)\big|^2
=\frac1n\,\E|Z_1(u)|^2
=\frac1n\Big(\E\big|e^{i\langle u,X\rangle}\big|^2-|\Bchi_\mu(u)|^2\Big)
=\frac{1-|\Bchi_\mu(u)|^2}{n}
\]
The map $(\omega,u)\mapsto\big|\Bchi_{\hat\mu_n(\omega)}(u)-\Bchi_\mu(u)\big|^2\|u\|^{-2s}$ is
jointly measurable and nonnegative, so Tonelli's theorem yields
\[
\E\,\mathrm T_{s,2}^2(\hat\mu_n,\mu)
=\int_{\R^d\setminus\{0\}}\E\big|\Bchi_{\hat\mu_n}(u)-\Bchi_\mu(u)\big|^2\,\frac{du}{\|u\|^{2s}}
=\frac1n\int_{\R^d}\frac{1-|\Bchi_\mu(u)|^2}{\|u\|^{2s}}\,du
\]
which is the first identity. For the second, note that
$|\Bchi_\mu(u)|^2=\E\,e^{i\langle u,X-X'\rangle}=\E\cos\langle u,X-X'\rangle$, the last step
because $X-X'$ is symmetric. Hence, using Tonelli once more (the integrand
$1-\cos\langle u,x\rangle\ge0$) and the Fourier identity \eqref{eq:riesz-alpha} from the proof of
Lemma~\ref{lem:energy-constant} with $\alpha=2H$ and $2s=d+2H$,
\[
\int_{\R^d}\frac{1-|\Bchi_\mu(u)|^2}{\|u\|^{2s}}\,du
=\E\int_{\R^d}\frac{1-\cos\langle u,X-X'\rangle}{\|u\|^{\,d+2H}}\,du
=c(d,H)\,\E\|X-X'\|^{2H}
\]
For the finiteness claim. If $\mu\in\mathcal P_{2H}(\R^d)$, then
$\E\|X-X'\|^{2H}\le\max\{1,2^{2H-1}\}\big(\E\|X\|^{2H}+\E\|X'\|^{2H}\big)<\infty$. Conversely, if
$\E\|X-X'\|^{2H}<\infty$, then by Fubini there is $x_0\in\R^d$ with
$\E\|X-x_0\|^{2H}<\infty$, and
$\|X\|^{2H}\le\max\{1,2^{2H-1}\}\big(\|X-x_0\|^{2H}+\|x_0\|^{2H}\big)$ gives
$\mu\in\mathcal P_{2H}(\R^d)$. Finally, Markov's inequality yields, for every $\lambda>0$,
\[
\P\Big(\mathrm T_{s,2}(\hat\mu_n,\mu)>\lambda\,n^{-1/2}\Big)
\le\frac{n\,\E\,\mathrm T_{s,2}^2(\hat\mu_n,\mu)}{\lambda^{2}}
=\frac{c(d,H)\,\E\|X-X'\|^{2H}}{\lambda^{2}}
\]
which is the $O_P(n^{-1/2})$ claim with explicit constant.
\end{proof}
\section{Additional Details on Two-Sample Testing}\label{app:two_sample_detailed}
Tables~\ref{tab:appendix_mmd_p2_res256} and~\ref{tab:appendix_mmd_p5_res256} record every
Gaussian-kernel MMD two-sample test underlying the summary in the main text. One block per
class pair, with the bandwidth from the median heuristic, the unbiased statistic, the
permutation $p$-value, and the decision at level $\alpha=0.05$, for the sliced-Wasserstein
baseline and for $\mathrm{T}_{s,p}$ at every tabulated $s$.

\begingroup
\footnotesize
\setlength{\tabcolsep}{2.5pt}
\begin{longtable}{
>{\raggedright\arraybackslash}p{4.6cm}
>{\raggedright\arraybackslash}p{1.6cm}
>{\centering\arraybackslash}p{0.7cm}
>{\centering\arraybackslash}p{1.8cm}
>{\centering\arraybackslash}p{1.8cm}
>{\centering\arraybackslash}p{1.4cm}
>{\centering\arraybackslash}p{1.1cm}
}
\caption{Detailed Gaussian-kernel MMD two-sample testing results on DOTmark at resolution $256\times256$ with $p=2$, for the sliced-Wasserstein baseline and for $\mathrm{T}_{s,2}$ at every tabulated $s$. All tests use $n_1=n_2=10$ images and $2000$ permutations. For Sliced W, the $s$ column is left blank.}%
\label{tab:appendix_mmd_p2_res256} \\
\toprule\toprule
\textbf{Class pair} & \textbf{Metric} & $\mathbf{s}$ & \textbf{Bandwidth} & $\mathbf{\widehat{\mathrm{MMD}}_u^2}$ & $\mathbf{p}$\textbf{-value} & \textbf{Reject} \\
\midrule
\endfirsthead

\toprule\toprule
\textbf{Class pair} & \textbf{Metric} & $\mathbf{s}$ & \textbf{Bandwidth} & $\mathbf{\widehat{\mathrm{MMD}}_u^2}$ & $\mathbf{p}$\textbf{-value} & \textbf{Reject} \\
\midrule
\endhead

\midrule
\multicolumn{7}{r}{\emph{Continued on next page}} \\
\endfoot

\bottomrule\bottomrule
\endlastfoot
ClassicImages vs GRFmoderate & Sliced W &  & 11.259 & -0.001 & 0.442 & 0 \\
ClassicImages vs GRFmoderate & Toscani & 1.01 & 0.310 & -0.011 & 0.641 & 0 \\
ClassicImages vs GRFmoderate & Toscani & 1.34 & 0.985 & -0.013 & 0.625 & 0 \\
ClassicImages vs GRFmoderate & Toscani & 1.66 & 3.035 & -0.014 & 0.633 & 0 \\
ClassicImages vs GRFmoderate & Toscani & 1.99 & 9.883 & -0.015 & 0.628 & 0 \\

ClassicImages vs GRFrough & Sliced W &  & 6.828 & 0.243 & 4.998e-04 & 1 \\
ClassicImages vs GRFrough & Toscani & 1.01 & 0.202 & 0.163 & 4.998e-04 & 1 \\
ClassicImages vs GRFrough & Toscani & 1.34 & 0.620 & 0.173 & 4.998e-04 & 1 \\
ClassicImages vs GRFrough & Toscani & 1.66 & 1.869 & 0.182 & 4.998e-04 & 1 \\
ClassicImages vs GRFrough & Toscani & 1.99 & 6.001 & 0.185 & 4.998e-04 & 1 \\

ClassicImages vs GRFsmooth & Sliced W &  & 15.071 & 0.104 & 0.026 & 1 \\
ClassicImages vs GRFsmooth & Toscani & 1.01 & 0.386 & 0.024 & 0.208 & 0 \\
ClassicImages vs GRFsmooth & Toscani & 1.34 & 1.243 & 0.025 & 0.193 & 0 \\
ClassicImages vs GRFsmooth & Toscani & 1.66 & 3.932 & 0.024 & 0.218 & 0 \\
ClassicImages vs GRFsmooth & Toscani & 1.99 & 12.972 & 0.023 & 0.235 & 0 \\

ClassicImages vs LogGRF & Sliced W &  & 19.774 & 0.164 & 0.003 & 1 \\
ClassicImages vs LogGRF & Toscani & 1.01 & 0.522 & 0.114 & 0.002 & 1 \\
ClassicImages vs LogGRF & Toscani & 1.34 & 1.620 & 0.122 & 0.005 & 1 \\
ClassicImages vs LogGRF & Toscani & 1.66 & 5.046 & 0.125 & 0.002 & 1 \\
ClassicImages vs LogGRF & Toscani & 1.99 & 16.699 & 0.124 & 0.005 & 1 \\

ClassicImages vs LogitGRF & Sliced W &  & 15.463 & 0.058 & 0.079 & 0 \\
ClassicImages vs LogitGRF & Toscani & 1.01 & 0.448 & 0.023 & 0.163 & 0 \\
ClassicImages vs LogitGRF & Toscani & 1.34 & 1.410 & 0.018 & 0.214 & 0 \\
ClassicImages vs LogitGRF & Toscani & 1.66 & 4.324 & 0.016 & 0.251 & 0 \\
ClassicImages vs LogitGRF & Toscani & 1.99 & 13.885 & 0.014 & 0.275 & 0 \\

ClassicImages vs MicroscopyImages & Sliced W &  & 19.921 & 0.351 & 4.998e-04 & 1 \\
ClassicImages vs MicroscopyImages & Toscani & 1.01 & 0.614 & 0.262 & 4.998e-04 & 1 \\
ClassicImages vs MicroscopyImages & Toscani & 1.34 & 1.876 & 0.289 & 4.998e-04 & 1 \\
ClassicImages vs MicroscopyImages & Toscani & 1.66 & 5.782 & 0.307 & 0.001 & 1 \\
ClassicImages vs MicroscopyImages & Toscani & 1.99 & 18.941 & 0.319 & 4.998e-04 & 1 \\

ClassicImages vs Shapes & Sliced W &  & 14.819 & 0.150 & 0.014 & 1 \\
ClassicImages vs Shapes & Toscani & 1.01 & 0.555 & 0.153 & 0.005 & 1 \\
ClassicImages vs Shapes & Toscani & 1.34 & 1.678 & 0.160 & 0.023 & 1 \\
ClassicImages vs Shapes & Toscani & 1.66 & 4.971 & 0.166 & 0.011 & 1 \\
ClassicImages vs Shapes & Toscani & 1.99 & 15.744 & 0.168 & 0.019 & 1 \\

GRFmoderate vs GRFrough & Sliced W &  & 7.874 & 0.154 & 0.002 & 1 \\
GRFmoderate vs GRFrough & Toscani & 1.01 & 0.249 & 0.099 & 4.998e-04 & 1 \\
GRFmoderate vs GRFrough & Toscani & 1.34 & 0.777 & 0.106 & 4.998e-04 & 1 \\
GRFmoderate vs GRFrough & Toscani & 1.66 & 2.392 & 0.111 & 9.995e-04 & 1 \\
GRFmoderate vs GRFrough & Toscani & 1.99 & 7.671 & 0.119 & 9.995e-04 & 1 \\

GRFmoderate vs GRFsmooth & Sliced W &  & 16.266 & 0.019 & 0.259 & 0 \\
GRFmoderate vs GRFsmooth & Toscani & 1.01 & 0.386 & -0.002 & 0.443 & 0 \\
GRFmoderate vs GRFsmooth & Toscani & 1.34 & 1.247 & -0.004 & 0.454 & 0 \\
GRFmoderate vs GRFsmooth & Toscani & 1.66 & 3.909 & -0.005 & 0.469 & 0 \\
GRFmoderate vs GRFsmooth & Toscani & 1.99 & 12.971 & -0.007 & 0.462 & 0 \\

GRFmoderate vs LogGRF & Sliced W &  & 20.780 & 0.134 & 0.018 & 1 \\
GRFmoderate vs LogGRF & Toscani & 1.01 & 0.552 & 0.107 & 0.004 & 1 \\
GRFmoderate vs LogGRF & Toscani & 1.34 & 1.740 & 0.110 & 0.005 & 1 \\
GRFmoderate vs LogGRF & Toscani & 1.66 & 5.420 & 0.112 & 0.003 & 1 \\
GRFmoderate vs LogGRF & Toscani & 1.99 & 17.809 & 0.113 & 0.005 & 1 \\

GRFmoderate vs LogitGRF & Sliced W &  & 16.192 & -0.012 & 0.542 & 0 \\
GRFmoderate vs LogitGRF & Toscani & 1.01 & 0.443 & 0.008 & 0.349 & 0 \\
GRFmoderate vs LogitGRF & Toscani & 1.34 & 1.404 & 0.002 & 0.404 & 0 \\
GRFmoderate vs LogitGRF & Toscani & 1.66 & 4.375 & -0.002 & 0.442 & 0 \\
GRFmoderate vs LogitGRF & Toscani & 1.99 & 14.020 & -0.003 & 0.485 & 0 \\

GRFmoderate vs MicroscopyImages & Sliced W &  & 19.581 & 0.248 & 0.002 & 1 \\
GRFmoderate vs MicroscopyImages & Toscani & 1.01 & 0.605 & 0.228 & 0.001 & 1 \\
GRFmoderate vs MicroscopyImages & Toscani & 1.34 & 1.866 & 0.239 & 9.995e-04 & 1 \\
GRFmoderate vs MicroscopyImages & Toscani & 1.66 & 5.649 & 0.257 & 9.995e-04 & 1 \\
GRFmoderate vs MicroscopyImages & Toscani & 1.99 & 18.670 & 0.259 & 0.001 & 1 \\

GRFmoderate vs Shapes & Sliced W &  & 16.312 & 0.080 & 0.052 & 0 \\
GRFmoderate vs Shapes & Toscani & 1.01 & 0.559 & 0.125 & 0.018 & 1 \\
GRFmoderate vs Shapes & Toscani & 1.34 & 1.679 & 0.130 & 0.025 & 1 \\
GRFmoderate vs Shapes & Toscani & 1.66 & 4.942 & 0.134 & 0.023 & 1 \\
GRFmoderate vs Shapes & Toscani & 1.99 & 15.402 & 0.136 & 0.027 & 1 \\

GRFrough vs GRFsmooth & Sliced W &  & 13.236 & 0.248 & 4.998e-04 & 1 \\
GRFrough vs GRFsmooth & Toscani & 1.01 & 0.320 & 0.163 & 4.998e-04 & 1 \\
GRFrough vs GRFsmooth & Toscani & 1.34 & 1.047 & 0.166 & 4.998e-04 & 1 \\
GRFrough vs GRFsmooth & Toscani & 1.66 & 3.350 & 0.165 & 4.998e-04 & 1 \\
GRFrough vs GRFsmooth & Toscani & 1.99 & 11.169 & 0.164 & 4.998e-04 & 1 \\

GRFrough vs LogGRF & Sliced W &  & 19.669 & 0.285 & 4.998e-04 & 1 \\
GRFrough vs LogGRF & Toscani & 1.01 & 0.487 & 0.227 & 4.998e-04 & 1 \\
GRFrough vs LogGRF & Toscani & 1.34 & 1.537 & 0.234 & 4.998e-04 & 1 \\
GRFrough vs LogGRF & Toscani & 1.66 & 4.865 & 0.231 & 4.998e-04 & 1 \\
GRFrough vs LogGRF & Toscani & 1.99 & 16.174 & 0.229 & 4.998e-04 & 1 \\

GRFrough vs LogitGRF & Sliced W &  & 13.476 & 0.191 & 4.998e-04 & 1 \\
GRFrough vs LogitGRF & Toscani & 1.01 & 0.411 & 0.125 & 4.998e-04 & 1 \\
GRFrough vs LogitGRF & Toscani & 1.34 & 1.294 & 0.125 & 4.998e-04 & 1 \\
GRFrough vs LogitGRF & Toscani & 1.66 & 3.983 & 0.129 & 4.998e-04 & 1 \\
GRFrough vs LogitGRF & Toscani & 1.99 & 13.035 & 0.128 & 4.998e-04 & 1 \\

GRFrough vs MicroscopyImages & Sliced W &  & 18.671 & 0.416 & 4.998e-04 & 1 \\
GRFrough vs MicroscopyImages & Toscani & 1.01 & 0.591 & 0.306 & 4.998e-04 & 1 \\
GRFrough vs MicroscopyImages & Toscani & 1.34 & 1.843 & 0.322 & 4.998e-04 & 1 \\
GRFrough vs MicroscopyImages & Toscani & 1.66 & 5.519 & 0.360 & 4.998e-04 & 1 \\
GRFrough vs MicroscopyImages & Toscani & 1.99 & 17.867 & 0.380 & 4.998e-04 & 1 \\

GRFrough vs Shapes & Sliced W &  & 10.605 & 0.290 & 9.995e-04 & 1 \\
GRFrough vs Shapes & Toscani & 1.01 & 0.548 & 0.204 & 9.995e-04 & 1 \\
GRFrough vs Shapes & Toscani & 1.34 & 1.628 & 0.224 & 4.998e-04 & 1 \\
GRFrough vs Shapes & Toscani & 1.66 & 4.740 & 0.242 & 0.001 & 1 \\
GRFrough vs Shapes & Toscani & 1.99 & 14.398 & 0.259 & 0.001 & 1 \\

GRFsmooth vs LogGRF & Sliced W &  & 25.381 & 0.014 & 0.302 & 0 \\
GRFsmooth vs LogGRF & Toscani & 1.01 & 0.611 & 0.037 & 0.121 & 0 \\
GRFsmooth vs LogGRF & Toscani & 1.34 & 1.945 & 0.032 & 0.147 & 0 \\
GRFsmooth vs LogGRF & Toscani & 1.66 & 6.137 & 0.029 & 0.192 & 0 \\
GRFsmooth vs LogGRF & Toscani & 1.99 & 20.546 & 0.026 & 0.193 & 0 \\

GRFsmooth vs LogitGRF & Sliced W &  & 20.385 & -0.039 & 0.796 & 0 \\
GRFsmooth vs LogitGRF & Toscani & 1.01 & 0.518 & -0.035 & 0.911 & 0 \\
GRFsmooth vs LogitGRF & Toscani & 1.34 & 1.657 & -0.044 & 0.959 & 0 \\
GRFsmooth vs LogitGRF & Toscani & 1.66 & 5.186 & -0.051 & 0.980 & 0 \\
GRFsmooth vs LogitGRF & Toscani & 1.99 & 17.074 & -0.055 & 0.992 & 0 \\

GRFsmooth vs MicroscopyImages & Sliced W &  & 23.646 & 0.091 & 0.075 & 0 \\
GRFsmooth vs MicroscopyImages & Toscani & 1.01 & 0.644 & 0.168 & 0.006 & 1 \\
GRFsmooth vs MicroscopyImages & Toscani & 1.34 & 2.021 & 0.170 & 0.009 & 1 \\
GRFsmooth vs MicroscopyImages & Toscani & 1.66 & 6.410 & 0.169 & 0.012 & 1 \\
GRFsmooth vs MicroscopyImages & Toscani & 1.99 & 20.971 & 0.174 & 0.006 & 1 \\

GRFsmooth vs Shapes & Sliced W &  & 21.784 & 0.040 & 0.157 & 0 \\
GRFsmooth vs Shapes & Toscani & 1.01 & 0.633 & 0.088 & 0.034 & 1 \\
GRFsmooth vs Shapes & Toscani & 1.34 & 1.991 & 0.085 & 0.050 & 1 \\
GRFsmooth vs Shapes & Toscani & 1.66 & 6.161 & 0.083 & 0.051 & 0 \\
GRFsmooth vs Shapes & Toscani & 1.99 & 20.560 & 0.081 & 0.055 & 0 \\

LogGRF vs LogitGRF & Sliced W &  & 24.903 & 0.015 & 0.298 & 0 \\
LogGRF vs LogitGRF & Toscani & 1.01 & 0.659 & 0.015 & 0.235 & 0 \\
LogGRF vs LogitGRF & Toscani & 1.34 & 2.082 & 0.017 & 0.237 & 0 \\
LogGRF vs LogitGRF & Toscani & 1.66 & 6.533 & 0.018 & 0.239 & 0 \\
LogGRF vs LogitGRF & Toscani & 1.99 & 21.412 & 0.020 & 0.233 & 0 \\

LogGRF vs MicroscopyImages & Sliced W &  & 29.586 & 0.093 & 0.064 & 0 \\
LogGRF vs MicroscopyImages & Toscani & 1.01 & 0.801 & 0.121 & 0.016 & 1 \\
LogGRF vs MicroscopyImages & Toscani & 1.34 & 2.546 & 0.134 & 0.012 & 1 \\
LogGRF vs MicroscopyImages & Toscani & 1.66 & 8.021 & 0.143 & 0.014 & 1 \\
LogGRF vs MicroscopyImages & Toscani & 1.99 & 26.631 & 0.150 & 0.015 & 1 \\

LogGRF vs Shapes & Sliced W &  & 27.884 & 0.071 & 0.054 & 0 \\
LogGRF vs Shapes & Toscani & 1.01 & 0.795 & 0.074 & 0.042 & 1 \\
LogGRF vs Shapes & Toscani & 1.34 & 2.502 & 0.077 & 0.042 & 1 \\
LogGRF vs Shapes & Toscani & 1.66 & 7.893 & 0.080 & 0.050 & 1 \\
LogGRF vs Shapes & Toscani & 1.99 & 25.963 & 0.082 & 0.045 & 1 \\

LogitGRF vs MicroscopyImages & Sliced W &  & 23.161 & 0.094 & 0.062 & 0 \\
LogitGRF vs MicroscopyImages & Toscani & 1.01 & 0.691 & 0.108 & 0.023 & 1 \\
LogitGRF vs MicroscopyImages & Toscani & 1.34 & 2.147 & 0.114 & 0.026 & 1 \\
LogitGRF vs MicroscopyImages & Toscani & 1.66 & 6.646 & 0.120 & 0.038 & 1 \\
LogitGRF vs MicroscopyImages & Toscani & 1.99 & 21.878 & 0.125 & 0.035 & 1 \\

LogitGRF vs Shapes & Sliced W &  & 21.062 & 0.005 & 0.368 & 0 \\
LogitGRF vs Shapes & Toscani & 1.01 & 0.638 & 0.028 & 0.170 & 0 \\
LogitGRF vs Shapes & Toscani & 1.34 & 1.975 & 0.028 & 0.176 & 0 \\
LogitGRF vs Shapes & Toscani & 1.66 & 6.164 & 0.028 & 0.199 & 0 \\
LogitGRF vs Shapes & Toscani & 1.99 & 20.325 & 0.029 & 0.197 & 0 \\

MicroscopyImages vs Shapes & Sliced W &  & 24.437 & 0.089 & 0.064 & 0 \\
MicroscopyImages vs Shapes & Toscani & 1.01 & 0.769 & 0.045 & 0.111 & 0 \\
MicroscopyImages vs Shapes & Toscani & 1.34 & 2.396 & 0.039 & 0.138 & 0 \\
MicroscopyImages vs Shapes & Toscani & 1.66 & 7.438 & 0.034 & 0.184 & 0 \\
MicroscopyImages vs Shapes & Toscani & 1.99 & 24.518 & 0.031 & 0.198 & 0 \\
\end{longtable}

\footnotesize
\setlength{\tabcolsep}{2.5pt}
\begin{longtable}{
>{\raggedright\arraybackslash}p{4.6cm}
>{\raggedright\arraybackslash}p{1.6cm}
>{\centering\arraybackslash}p{0.7cm}
>{\centering\arraybackslash}p{1.8cm}
>{\centering\arraybackslash}p{1.8cm}
>{\centering\arraybackslash}p{1.4cm}
>{\centering\arraybackslash}p{1.1cm}
}
\caption{Detailed Gaussian-kernel MMD two-sample testing results on DOTmark at resolution $256\times256$ with $p=5$, for the sliced-Wasserstein baseline and for $\mathrm{T}_{s,5}$ at every tabulated $s$. All tests use $n_1=n_2=10$ images and $2000$ permutations. For Sliced W, the $s$ column is left blank.}%
\label{tab:appendix_mmd_p5_res256} \\
\toprule\toprule
\textbf{Class pair} & \textbf{Metric} & $\mathbf{s}$ & \textbf{Bandwidth} & $\mathbf{\widehat{\mathrm{MMD}}_u^2}$ & $\mathbf{p}$\textbf{-value} & \textbf{Reject} \\
\midrule
\endfirsthead

\toprule\toprule
\textbf{Class pair} & \textbf{Metric} & $\mathbf{s}$ & \textbf{Bandwidth} & $\mathbf{\widehat{\mathrm{MMD}}_u^2}$ & $\mathbf{p}$\textbf{-value} & \textbf{Reject} \\
\midrule
\endhead

\midrule
\multicolumn{7}{r}{\emph{Continued on next page}} \\
\endfoot

\bottomrule\bottomrule
\endlastfoot
ClassicImages vs GRFmoderate & Sliced W &  & 14.727 & -0.006 & 0.501 & 0 \\
ClassicImages vs GRFmoderate & Toscani & 0.41 & 0.188 & -0.027 & 0.798 & 0 \\
ClassicImages vs GRFmoderate & Toscani & 0.74 & 0.615 & -0.027 & 0.788 & 0 \\
ClassicImages vs GRFmoderate & Toscani & 1.06 & 1.940 & -0.027 & 0.776 & 0 \\
ClassicImages vs GRFmoderate & Toscani & 1.39 & 6.448 & -0.028 & 0.791 & 0 \\

ClassicImages vs GRFrough & Sliced W &  & 9.430 & 0.240 & 4.998e-04 & 1 \\
ClassicImages vs GRFrough & Toscani & 0.41 & 0.117 & 0.198 & 4.998e-04 & 1 \\
ClassicImages vs GRFrough & Toscani & 0.74 & 0.388 & 0.192 & 4.998e-04 & 1 \\
ClassicImages vs GRFrough & Toscani & 1.06 & 1.212 & 0.192 & 4.998e-04 & 1 \\
ClassicImages vs GRFrough & Toscani & 1.39 & 3.939 & 0.191 & 9.995e-04 & 1 \\

ClassicImages vs GRFsmooth & Sliced W &  & 19.961 & 0.102 & 0.016 & 1 \\
ClassicImages vs GRFsmooth & Toscani & 0.41 & 0.239 & 0.025 & 0.211 & 0 \\
ClassicImages vs GRFsmooth & Toscani & 0.74 & 0.788 & 0.027 & 0.237 & 0 \\
ClassicImages vs GRFsmooth & Toscani & 1.06 & 2.546 & 0.025 & 0.250 & 0 \\
ClassicImages vs GRFsmooth & Toscani & 1.39 & 8.608 & 0.023 & 0.251 & 0 \\

ClassicImages vs LogGRF & Sliced W &  & 25.655 & 0.158 & 0.003 & 1 \\
ClassicImages vs LogGRF & Toscani & 0.41 & 0.317 & 0.095 & 0.006 & 1 \\
ClassicImages vs LogGRF & Toscani & 0.74 & 1.043 & 0.098 & 0.010 & 1 \\
ClassicImages vs LogGRF & Toscani & 1.06 & 3.394 & 0.097 & 0.017 & 1 \\
ClassicImages vs LogGRF & Toscani & 1.39 & 11.465 & 0.096 & 0.019 & 1 \\

ClassicImages vs LogitGRF & Sliced W &  & 20.596 & 0.073 & 0.036 & 1 \\
ClassicImages vs LogitGRF & Toscani & 0.41 & 0.273 & 0.028 & 0.149 & 0 \\
ClassicImages vs LogitGRF & Toscani & 0.74 & 0.878 & 0.026 & 0.180 & 0 \\
ClassicImages vs LogitGRF & Toscani & 1.06 & 2.791 & 0.023 & 0.204 & 0 \\
ClassicImages vs LogitGRF & Toscani & 1.39 & 9.291 & 0.020 & 0.237 & 0 \\

ClassicImages vs MicroscopyImages & Sliced W &  & 27.474 & 0.362 & 4.998e-04 & 1 \\
ClassicImages vs MicroscopyImages & Toscani & 0.41 & 0.371 & 0.237 & 0.002 & 1 \\
ClassicImages vs MicroscopyImages & Toscani & 0.74 & 1.165 & 0.288 & 9.995e-04 & 1 \\
ClassicImages vs MicroscopyImages & Toscani & 1.06 & 3.729 & 0.314 & 0.002 & 1 \\
ClassicImages vs MicroscopyImages & Toscani & 1.39 & 12.541 & 0.332 & 9.995e-04 & 1 \\

ClassicImages vs Shapes & Sliced W &  & 19.459 & 0.193 & 0.010 & 1 \\
ClassicImages vs Shapes & Toscani & 0.41 & 0.361 & 0.198 & 4.998e-04 & 1 \\
ClassicImages vs Shapes & Toscani & 0.74 & 1.156 & 0.166 & 0.005 & 1 \\
ClassicImages vs Shapes & Toscani & 1.06 & 3.570 & 0.154 & 0.013 & 1 \\
ClassicImages vs Shapes & Toscani & 1.39 & 10.981 & 0.155 & 0.014 & 1 \\

GRFmoderate vs GRFrough & Sliced W &  & 11.219 & 0.130 & 0.001 & 1 \\
GRFmoderate vs GRFrough & Toscani & 0.41 & 0.145 & 0.107 & 4.998e-04 & 1 \\
GRFmoderate vs GRFrough & Toscani & 0.74 & 0.466 & 0.117 & 4.998e-04 & 1 \\
GRFmoderate vs GRFrough & Toscani & 1.06 & 1.463 & 0.126 & 9.995e-04 & 1 \\
GRFmoderate vs GRFrough & Toscani & 1.39 & 4.851 & 0.131 & 9.995e-04 & 1 \\

GRFmoderate vs GRFsmooth & Sliced W &  & 21.122 & 0.006 & 0.360 & 0 \\
GRFmoderate vs GRFsmooth & Toscani & 0.41 & 0.241 & 0.005 & 0.373 & 0 \\
GRFmoderate vs GRFsmooth & Toscani & 0.74 & 0.793 & 0.004 & 0.376 & 0 \\
GRFmoderate vs GRFsmooth & Toscani & 1.06 & 2.576 & 0.002 & 0.381 & 0 \\
GRFmoderate vs GRFsmooth & Toscani & 1.39 & 8.683 & -0.001 & 0.395 & 0 \\

GRFmoderate vs LogGRF & Sliced W &  & 27.299 & 0.119 & 0.017 & 1 \\
GRFmoderate vs LogGRF & Toscani & 0.41 & 0.333 & 0.100 & 0.007 & 1 \\
GRFmoderate vs LogGRF & Toscani & 0.74 & 1.097 & 0.100 & 0.009 & 1 \\
GRFmoderate vs LogGRF & Toscani & 1.06 & 3.528 & 0.101 & 0.015 & 1 \\
GRFmoderate vs LogGRF & Toscani & 1.39 & 11.871 & 0.100 & 0.016 & 1 \\

GRFmoderate vs LogitGRF & Sliced W &  & 21.333 & -0.009 & 0.546 & 0 \\
GRFmoderate vs LogitGRF & Toscani & 0.41 & 0.278 & 0.011 & 0.326 & 0 \\
GRFmoderate vs LogitGRF & Toscani & 0.74 & 0.897 & 0.011 & 0.323 & 0 \\
GRFmoderate vs LogitGRF & Toscani & 1.06 & 2.808 & 0.012 & 0.316 & 0 \\
GRFmoderate vs LogitGRF & Toscani & 1.39 & 9.159 & 0.012 & 0.326 & 0 \\

GRFmoderate vs MicroscopyImages & Sliced W &  & 26.922 & 0.248 & 9.995e-04 & 1 \\
GRFmoderate vs MicroscopyImages & Toscani & 0.41 & 0.366 & 0.188 & 0.005 & 1 \\
GRFmoderate vs MicroscopyImages & Toscani & 0.74 & 1.155 & 0.222 & 0.002 & 1 \\
GRFmoderate vs MicroscopyImages & Toscani & 1.06 & 3.687 & 0.243 & 0.002 & 1 \\
GRFmoderate vs MicroscopyImages & Toscani & 1.39 & 12.418 & 0.257 & 0.008 & 1 \\

GRFmoderate vs Shapes & Sliced W &  & 21.559 & 0.115 & 0.025 & 1 \\
GRFmoderate vs Shapes & Toscani & 0.41 & 0.348 & 0.181 & 9.995e-04 & 1 \\
GRFmoderate vs Shapes & Toscani & 0.74 & 1.111 & 0.144 & 0.016 & 1 \\
GRFmoderate vs Shapes & Toscani & 1.06 & 3.570 & 0.126 & 0.023 & 1 \\
GRFmoderate vs Shapes & Toscani & 1.39 & 10.947 & 0.127 & 0.031 & 1 \\

GRFrough vs GRFsmooth & Sliced W &  & 17.596 & 0.229 & 4.998e-04 & 1 \\
GRFrough vs GRFsmooth & Toscani & 0.41 & 0.198 & 0.194 & 4.998e-04 & 1 \\
GRFrough vs GRFsmooth & Toscani & 0.74 & 0.661 & 0.194 & 4.998e-04 & 1 \\
GRFrough vs GRFsmooth & Toscani & 1.06 & 2.133 & 0.193 & 4.998e-04 & 1 \\
GRFrough vs GRFsmooth & Toscani & 1.39 & 7.221 & 0.189 & 4.998e-04 & 1 \\

GRFrough vs LogGRF & Sliced W &  & 24.905 & 0.284 & 4.998e-04 & 1 \\
GRFrough vs LogGRF & Toscani & 0.41 & 0.314 & 0.188 & 4.998e-04 & 1 \\
GRFrough vs LogGRF & Toscani & 0.74 & 1.039 & 0.192 & 4.998e-04 & 1 \\
GRFrough vs LogGRF & Toscani & 1.06 & 3.366 & 0.192 & 4.998e-04 & 1 \\
GRFrough vs LogGRF & Toscani & 1.39 & 11.389 & 0.191 & 4.998e-04 & 1 \\

GRFrough vs LogitGRF & Sliced W &  & 18.776 & 0.202 & 4.998e-04 & 1 \\
GRFrough vs LogitGRF & Toscani & 0.41 & 0.253 & 0.158 & 4.998e-04 & 1 \\
GRFrough vs LogitGRF & Toscani & 0.74 & 0.822 & 0.167 & 4.998e-04 & 1 \\
GRFrough vs LogitGRF & Toscani & 1.06 & 2.600 & 0.174 & 4.998e-04 & 1 \\
GRFrough vs LogitGRF & Toscani & 1.39 & 8.765 & 0.169 & 4.998e-04 & 1 \\

GRFrough vs MicroscopyImages & Sliced W &  & 26.864 & 0.379 & 4.998e-04 & 1 \\
GRFrough vs MicroscopyImages & Toscani & 0.41 & 0.368 & 0.279 & 4.998e-04 & 1 \\
GRFrough vs MicroscopyImages & Toscani & 0.74 & 1.177 & 0.320 & 4.998e-04 & 1 \\
GRFrough vs MicroscopyImages & Toscani & 1.06 & 3.477 & 0.403 & 4.998e-04 & 1 \\
GRFrough vs MicroscopyImages & Toscani & 1.39 & 11.502 & 0.435 & 4.998e-04 & 1 \\

GRFrough vs Shapes & Sliced W &  & 16.200 & 0.307 & 9.995e-04 & 1 \\
GRFrough vs Shapes & Toscani & 0.41 & 0.377 & 0.255 & 4.998e-04 & 1 \\
GRFrough vs Shapes & Toscani & 0.74 & 1.156 & 0.239 & 4.998e-04 & 1 \\
GRFrough vs Shapes & Toscani & 1.06 & 3.386 & 0.243 & 4.998e-04 & 1 \\
GRFrough vs Shapes & Toscani & 1.39 & 10.329 & 0.253 & 0.002 & 1 \\

GRFsmooth vs LogGRF & Sliced W &  & 32.261 & 0.022 & 0.245 & 0 \\
GRFsmooth vs LogGRF & Toscani & 0.41 & 0.375 & 0.027 & 0.188 & 0 \\
GRFsmooth vs LogGRF & Toscani & 0.74 & 1.251 & 0.026 & 0.217 & 0 \\
GRFsmooth vs LogGRF & Toscani & 1.06 & 4.086 & 0.026 & 0.215 & 0 \\
GRFsmooth vs LogGRF & Toscani & 1.39 & 13.544 & 0.026 & 0.223 & 0 \\

GRFsmooth vs LogitGRF & Sliced W &  & 26.533 & -0.032 & 0.755 & 0 \\
GRFsmooth vs LogitGRF & Toscani & 0.41 & 0.320 & -0.052 & 0.967 & 0 \\
GRFsmooth vs LogitGRF & Toscani & 0.74 & 1.048 & -0.057 & 0.975 & 0 \\
GRFsmooth vs LogitGRF & Toscani & 1.06 & 3.381 & -0.059 & 0.976 & 0 \\
GRFsmooth vs LogitGRF & Toscani & 1.39 & 11.157 & -0.061 & 0.979 & 0 \\

GRFsmooth vs MicroscopyImages & Sliced W &  & 30.007 & 0.107 & 0.038 & 1 \\
GRFsmooth vs MicroscopyImages & Toscani & 0.41 & 0.395 & 0.128 & 0.016 & 1 \\
GRFsmooth vs MicroscopyImages & Toscani & 0.74 & 1.309 & 0.145 & 0.024 & 1 \\
GRFsmooth vs MicroscopyImages & Toscani & 1.06 & 4.248 & 0.158 & 0.019 & 1 \\
GRFsmooth vs MicroscopyImages & Toscani & 1.39 & 14.376 & 0.168 & 0.016 & 1 \\

GRFsmooth vs Shapes & Sliced W &  & 28.831 & 0.061 & 0.048 & 1 \\
GRFsmooth vs Shapes & Toscani & 0.41 & 0.428 & 0.100 & 0.020 & 1 \\
GRFsmooth vs Shapes & Toscani & 0.74 & 1.367 & 0.081 & 0.052 & 0 \\
GRFsmooth vs Shapes & Toscani & 1.06 & 4.247 & 0.077 & 0.070 & 0 \\
GRFsmooth vs Shapes & Toscani & 1.39 & 13.888 & 0.075 & 0.079 & 0 \\

LogGRF vs LogitGRF & Sliced W &  & 32.481 & 0.008 & 0.328 & 0 \\
LogGRF vs LogitGRF & Toscani & 0.41 & 0.406 & 0.020 & 0.218 & 0 \\
LogGRF vs LogitGRF & Toscani & 0.74 & 1.322 & 0.024 & 0.201 & 0 \\
LogGRF vs LogitGRF & Toscani & 1.06 & 4.190 & 0.028 & 0.210 & 0 \\
LogGRF vs LogitGRF & Toscani & 1.39 & 13.867 & 0.030 & 0.192 & 0 \\

LogGRF vs MicroscopyImages & Sliced W &  & 38.216 & 0.110 & 0.045 & 1 \\
LogGRF vs MicroscopyImages & Toscani & 0.41 & 0.503 & 0.124 & 0.016 & 1 \\
LogGRF vs MicroscopyImages & Toscani & 0.74 & 1.665 & 0.146 & 0.020 & 1 \\
LogGRF vs MicroscopyImages & Toscani & 1.06 & 5.401 & 0.161 & 0.016 & 1 \\
LogGRF vs MicroscopyImages & Toscani & 1.39 & 18.338 & 0.170 & 0.017 & 1 \\

LogGRF vs Shapes & Sliced W &  & 36.482 & 0.058 & 0.073 & 0 \\
LogGRF vs Shapes & Toscani & 0.41 & 0.469 & 0.068 & 0.058 & 0 \\
LogGRF vs Shapes & Toscani & 0.74 & 1.549 & 0.068 & 0.071 & 0 \\
LogGRF vs Shapes & Toscani & 1.06 & 4.923 & 0.070 & 0.062 & 0 \\
LogGRF vs Shapes & Toscani & 1.39 & 16.647 & 0.072 & 0.065 & 0 \\

LogitGRF vs MicroscopyImages & Sliced W &  & 29.538 & 0.096 & 0.050 & 0 \\
LogitGRF vs MicroscopyImages & Toscani & 0.41 & 0.419 & 0.072 & 0.072 & 0 \\
LogitGRF vs MicroscopyImages & Toscani & 0.74 & 1.358 & 0.092 & 0.055 & 0 \\
LogitGRF vs MicroscopyImages & Toscani & 1.06 & 4.319 & 0.110 & 0.052 & 0 \\
LogitGRF vs MicroscopyImages & Toscani & 1.39 & 14.415 & 0.126 & 0.040 & 1 \\

LogitGRF vs Shapes & Sliced W &  & 28.595 & 0.035 & 0.144 & 0 \\
LogitGRF vs Shapes & Toscani & 0.41 & 0.398 & 0.043 & 0.100 & 0 \\
LogitGRF vs Shapes & Toscani & 0.74 & 1.280 & 0.032 & 0.167 & 0 \\
LogitGRF vs Shapes & Toscani & 1.06 & 4.078 & 0.033 & 0.185 & 0 \\
LogitGRF vs Shapes & Toscani & 1.39 & 13.586 & 0.035 & 0.175 & 0 \\

MicroscopyImages vs Shapes & Sliced W &  & 32.000 & 0.113 & 0.023 & 1 \\
MicroscopyImages vs Shapes & Toscani & 0.41 & 0.459 & 0.016 & 0.271 & 0 \\
MicroscopyImages vs Shapes & Toscani & 0.74 & 1.487 & 0.020 & 0.254 & 0 \\
MicroscopyImages vs Shapes & Toscani & 1.06 & 4.771 & 0.028 & 0.213 & 0 \\
MicroscopyImages vs Shapes & Toscani & 1.39 & 15.917 & 0.037 & 0.187 & 0 \\
\end{longtable}
\endgroup
\ifdefined\ONLINEAPPENDIX
\section{Additional Numerical Material}
The following figures and tables are separated from the main JMLR manuscript to keep the journal
PDF concise while preserving the complete numerical record. Figure~\ref{fig:DOTmark_images}
shows the DOTmark montage. Figures~\ref{fig:p2_scatter}--\ref{fig:p5_roc} the scatter and ROC
panels. Figure~\ref{fig:exp3_translation} the translation-family panels. Tables~\ref{tab:mmd_dotmark_disagreements}--\ref{tab:mmd_dotmark_p5_disagreements} the
two-sample disagreements.
\subsection{Translation-Family Diagnostics}

\FloatBarrier
\subsection{DOTmark Images, Scatter Plots, and ROC Curves}
\begin{figure}[!htb]
\centering
\includegraphics[width=0.94\textwidth]{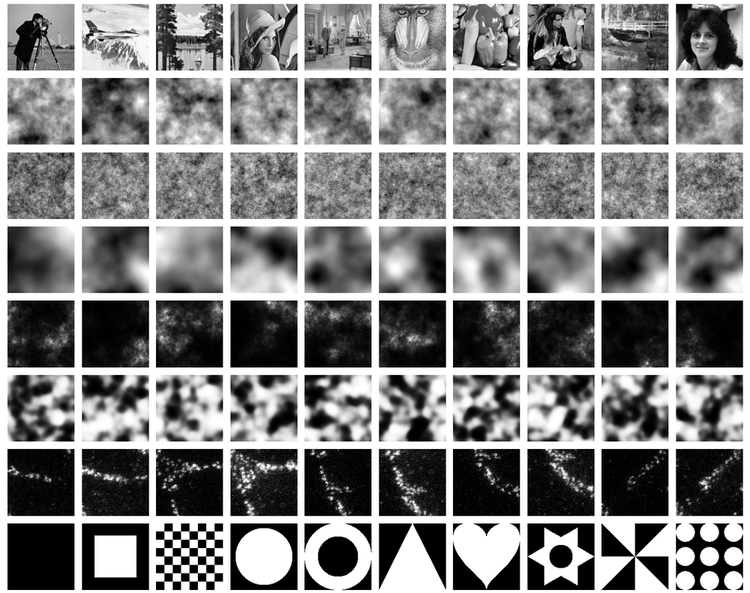}
\caption{DOTmark 1.0 data set at resolution \(128\times128\). Each row corresponds to
a class. The eight classes shown are those retained in the experiments.}
\label{fig:DOTmark_images}
\end{figure}

\begin{figure}[!htb]
\centering
\includegraphics[width=0.86\textwidth]{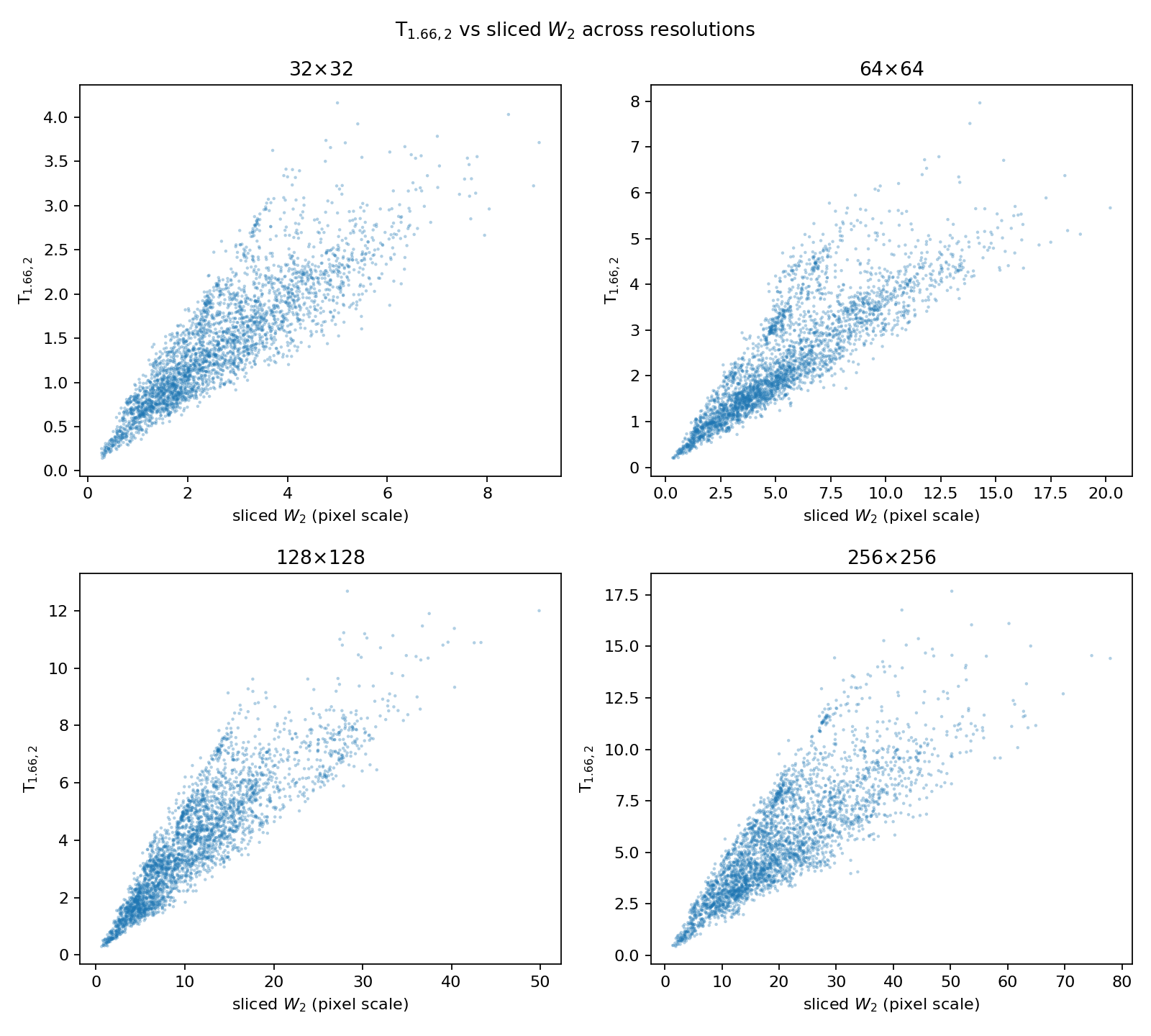}
\caption{DOTmark scatter comparison of \(\mathrm{T}_{s,2}\) (with \(s=1.66\), a value from
the tabulated grid of Table~\ref{tab:optimal_s_vertical_res}) against sliced \(W_2\), one
panel per resolution. Axes are on the pixel-coordinate scale of the footnote in
Section~\ref{subsec:dotmark} and are not comparable across resolutions.}
\label{fig:p2_scatter}
\end{figure}

\begin{figure}[!htb]
\centering
\includegraphics[width=0.86\textwidth]{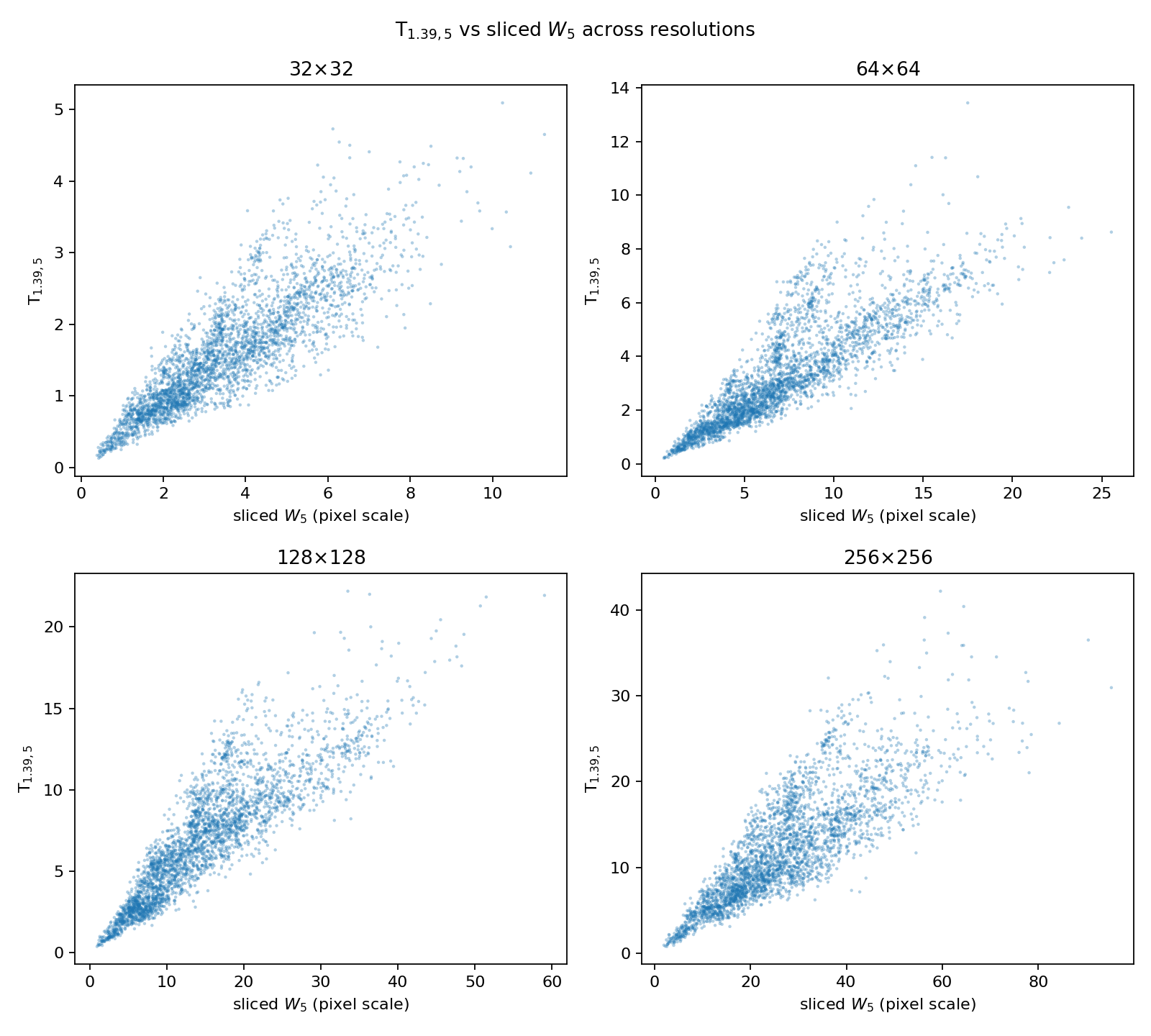}
\caption{As in Figure~\ref{fig:p2_scatter}, for \(\mathrm{T}_{s,5}\) with \(s=1.39\)
versus sliced \(W_5\).}
\label{fig:p5_scatter}
\end{figure}

\begin{figure}[!htb]
\centering
\includegraphics[width=0.86\textwidth]{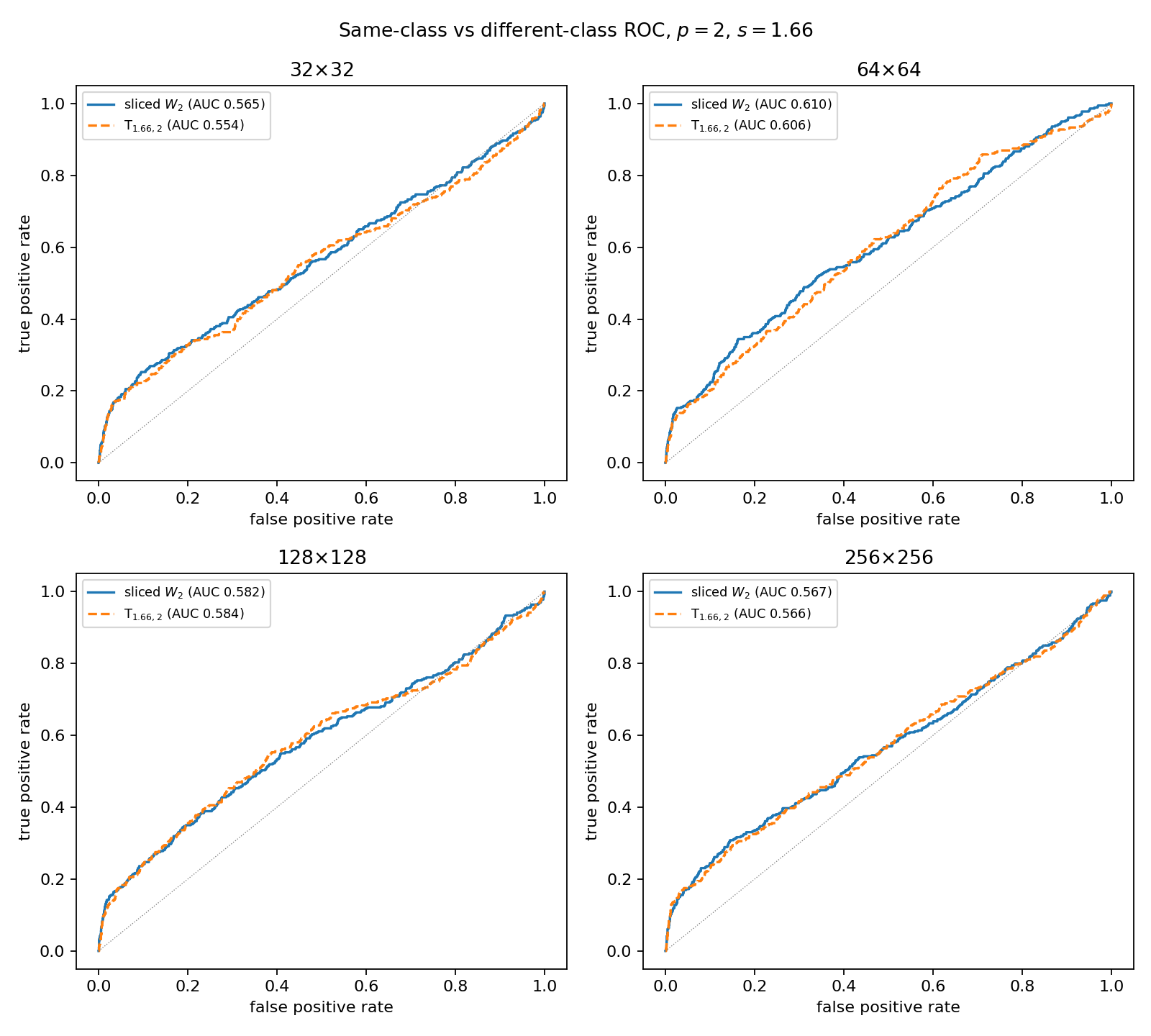}
\caption{DOTmark ROC curves for distinguishing same-class from different-class image
pairs by thresholding the distance (smaller distance indicates same class), for
\(\mathrm{T}_{s,2}\) with \(s=1.66\) and sliced \(W_2\), one panel per resolution, with
the area under each curve in the legends.}
\label{fig:p2_roc}
\end{figure}

\begin{figure}[!htb]
\centering
\includegraphics[width=0.86\textwidth]{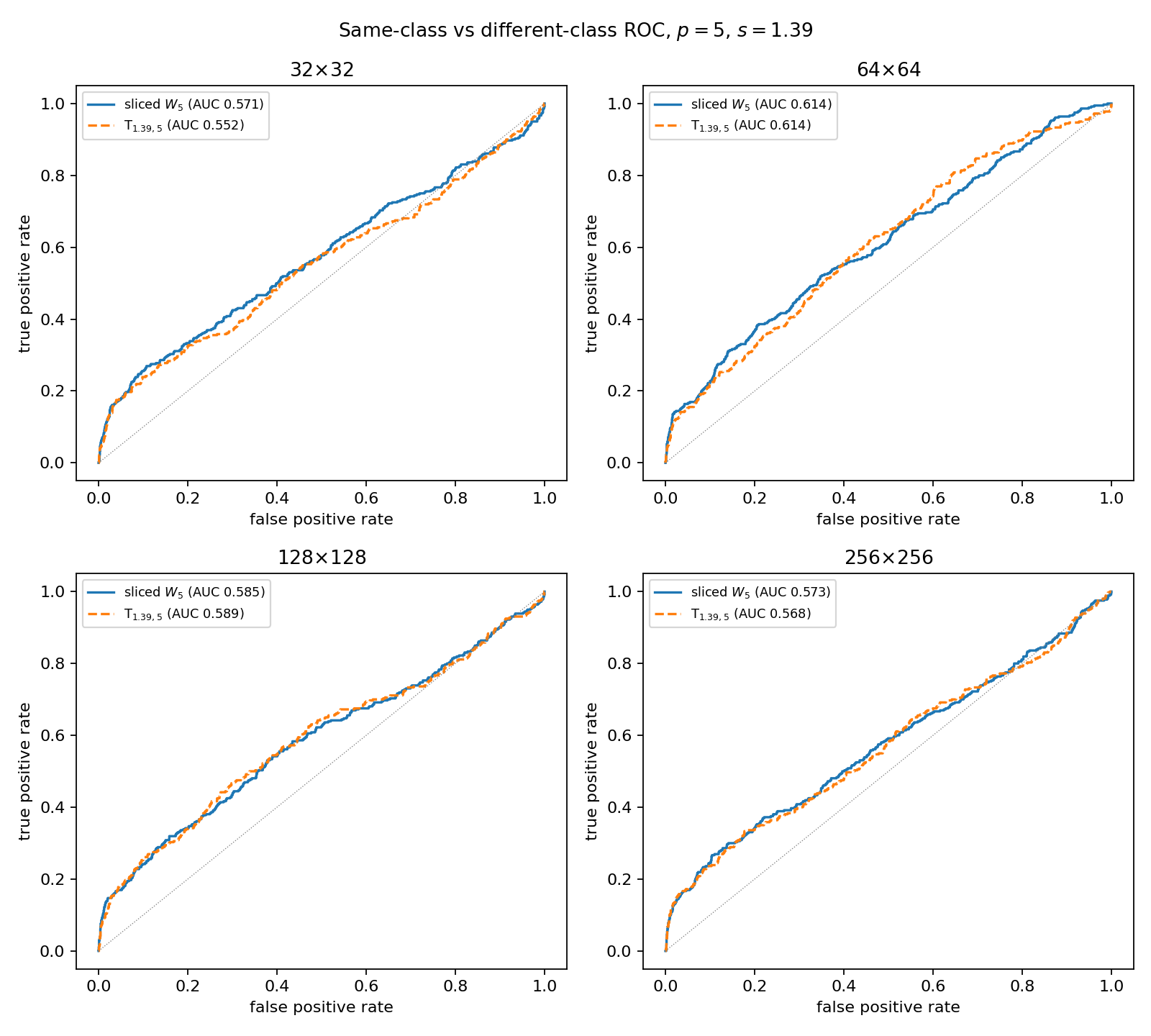}
\caption{As in Figure~\ref{fig:p2_roc}, for \(\mathrm{T}_{s,5}\) with \(s=1.39\) and
sliced \(W_5\).}
\label{fig:p5_roc}
\end{figure}

\FloatBarrier
\subsection{Pair-Level MMD Decisions}

\FloatBarrier
\fi
\fi
\bibliography{bib}
\end{document}